\documentclass[11pt]{article}
\def\bc{{\mathbf{c}}}
\def\bg{{\mathbf{g}}}

\def\bR{{\mathbb{R}}}

\def\calC{{\mathcal C}}

\def\calI{{\mathcal I}}

\def\calO{{\mathcal O}}
\def\calU{{\mathcal U}}

\def\calX{{\mathcal X}}

\def\hY{{\hat Y}}

\def\1{{\bf 1}}

\usepackage[margin=1in]{geometry}

\usepackage[T1]{fontenc}
\usepackage[utf8]{inputenc}
\usepackage{microtype}

\usepackage{amsmath,amsthm,amssymb,mathtools}
\usepackage{bm}
\usepackage{bbm}
\usepackage{upgreek}
\usepackage{optidef}
\usepackage{thm-restate}

\usepackage{float}
\usepackage{subcaption}
\usepackage{booktabs}
\usepackage{multirow}
\usepackage{makecell}
\usepackage[export]{adjustbox}

\usepackage{enumitem}
\setlist[enumerate]{leftmargin=2em}
\setlist[itemize]{leftmargin=1em}

\usepackage{listings}

\usepackage[dvipsnames]{xcolor}

\usepackage{comment}
\usepackage{footmisc}
\usepackage{nicefrac}
\usepackage{tcolorbox}
\usepackage{pifont}
\usepackage{fontawesome}
\usepackage{tocloft}

\usepackage{titlesec}
\titleformat{\paragraph}[runin]
  {\normalfont\normalsize\bfseries}
  {}
  {0pt}
  {}
\titlespacing*{\paragraph}
  {0pt}
  {0.5ex}
  {1em}

\usepackage[colorlinks]{hyperref}
\definecolor{bleudefrance}{rgb}{0.19, 0.55, 0.91}
\definecolor{pastelblue}{rgb}{0.48, 0.58, 0.81}
\definecolor{oxfordblue}{rgb}{0.0, 0.13, 0.28}
\definecolor{lavender}{rgb}{0.75, 0.58, 0.89}

\hypersetup{
  colorlinks = true,
  urlcolor   = bleudefrance,
  linkcolor  = magenta,
  citecolor  = pastelblue
}

\usepackage{algorithm}
\usepackage{algorithmic}

\usepackage[sort&compress]{natbib}
\usepackage[capitalize,sort&compress]{cleveref}

\newtheorem{lemma}{Lemma}[section]

\newcommand{\greencheck}{{\color{green}\checkmark}}
\newcommand{\redx}{{\color{red}\ding{55}}}

\title{{\bf Parameter-Free and Group Conditional\\ Online Conformal Prediction}}
\author{
    {\small Beepul~Bharti$^{1}$\footnote{Corresponding author.}}\\
    {\small bbharti1@jhu.edu}
    \and
    {\small Ambar Pal$^{2}$\footnote{This work is not related to AP's position at Amazon.}}\\
    {\small ambarpal@amazon.com}
    \and
    {\small Jacopo Teneggi$^{1}$}\\
    {\small jtenegg1@jhu.edu}
    \and
    {\small Jeremias~Sulam$^{1}$}\\
    {\small jsulam1@jhu.edu}
}
\date{
\small{$^1$Johns Hopkins University}\\
\small{$^2$Amazon Responsible AI}\\
\vspace{1em}
\today
\vspace{0em}
}
\begin{document}
\maketitle
\begin{abstract}
Quantifying the uncertainty in predictions made by machine learning models is challenging in online settings where the characteristics of the data and labels change over time. Online conformal prediction (OCP) methods tackle this problem by producing intervals for each observation and guaranteeing that a proportion of them contain their labels, a property known as marginal coverage. OCP methods, however, may systematically miscover observations belonging to groups defined by attributes like age or race, raising fairness concerns. Group-conditional OCP (G-OCP) methods address this by producing intervals that provide coverage for multiple groups. Existing approaches rely on learning rates, which control how they react to changes in the data stream, but can be hard to tune in online settings. Leveraging advancements in learning-rate-free methods, we develop a learning-rate-free G-OCP algorithm with strong group-coverage guarantees that hold under weaker assumptions than those required in prior work. Empirical results show that our method provides group coverage, produces interval lengths competitive with those of standard methods, and reacts more effectively to changes in the data than fixed-learning-rate methods.
\end{abstract}

\tableofcontents

\section{Introduction}
\label{sec:intro}
Machine learning models are increasingly deployed in high-stakes domains, such as supporting patient care in hospitals \citep{hong2022state}, informing loan decisions in banking \citep{bhatore2020machine}, and enabling real-time perception in autonomous vehicles \citep{grigorescu2020survey}. Since these models impact consequential decisions, it is critical to quantify the uncertainty in their predictions.

Conformal prediction (CP) is a versatile framework for uncertainty quantification. Given an input $X$, its unknown label $Y$, a predictor for $Y$, and miscoverage level $\alpha \in (0,1)$, CP methods construct prediction sets that contain the unknown label with probability $1-\alpha$ \citep{saunders1999transduction, papadopoulos2002inductive, vovk2005algorithmic, shafer2008tutorial,lei2014distribution, angelopoulos2023conformal}. This guarantee, however, only holds when the data is exchangeable, an assumption rarely satisfied in practice. In medical triage, for example, the stream of incoming patients may change due to seasonal events, such as the flu \citep{anderer20252024}, and large transportation accidents which can cause surges in patient admission \citep{rozenberg2017outcomes}. 

Online conformal prediction (OCP) methods \citep{gibbs2021adaptive, zaffran2022adaptive, bastani2022practical, angelopoulos2023conformalpid, gibbs2024conformal, bhatnagar2023improved, podkopaev2024adaptive, areces2025online, liu2026online} address this limitation by generating a sequence of prediction intervals $(\calI_t)_{t \geq 1}$ for the labels of an arbitrary data stream $((X_t, Y_t))_{t \geq 1} = (X_1, Y_1), (X_2, Y_2), \dots$. Most OCP methods provide \emph{marginal} coverage: as the time horizon $T$ grows, approximately $(1-\alpha)$-fraction of the intervals will contain their labels, i.e. $\sum^T_{t=1} {\bf 1} \left\{Y_t \in \calI_t\right\} \approx T(1-\alpha)$. This guarantee, however, may harm certain sub-populations in the data. For example, if the samples $X_t$ of a minority group appear with frequency \(\alpha\), an OCP method can
systematically fail to cover that group while maintaining marginal
coverage. This may lead to harmful downstream decisions for this group.

Accordingly, several works have developed OCP algorithms with group-conditional coverage guarantees
\citep{bastani2022practical, ramalingam2025relationship}.
That is, for a set of $k$ groups in the data, these methods generate a single sequence of intervals $(\calI_t)_{t \geq 1}$ that provides $(1-\alpha)$-coverage within each group,
i.e., for all groups $j \in \{1, \dots, k\}$, $\sum_{t=1}^T {\bf 1}\left\{Y_t \in \calI_t
~\text{and}~ X_t \in \text{group }j\right\}
\approx T_j(1-\alpha)$, where $T_j$ is the number of samples in group $j$. However, they are not \emph{parameter-free}: their performance depends on a learning rate which can be problematic for OCP where, ideally, little to no assumptions are made about the data stream. Since the data may be adversarial, relying on a validation set to tune the learning rate can be futile: if one identifies a ``good'' learning rate, a shift in the incoming data could make the choice inadequate \citep{orabona2019modern}. While there has been a rise in parameter-free OCP methods, existing ones only provide marginal coverage \citep{zhang2024discounted,podkopaev2024adaptive,liu2026online}. Thus, there is no parameter-free OCP method that provides group-conditional coverage.

\paragraph{Contributions.} In this work, we address this gap. Specifically, our contributions are:
\begin{enumerate}
\item We introduce the first parameter-free algorithm for group-conditional online conformal prediction, that we dub {\bf P}ortfolios for {\bf O}nline {\bf G}roup C{\bf o}nformal (POGO). 

\item We establish rigorous finite-time group-conditional coverage guarantees for POGO with the best-known rates with respect to time $T$ (up to logarithmic factors), that hold under weaker assumptions than those required by prior work.

\item We perform extensive experiments on synthetic and real-world datasets, showing that POGO is competitive with existing approaches: POGO can effectively react to shifts in the data stream that other methods cannot and produce small prediction intervals that provide group coverage.
\end{enumerate}


\subsection{Related Work}
\label{sec:related_works}
Closest to our work are two lines of research that draw on classical and parameter-free online learning: marginal OCP and group-conditional coverage in traditional stochastic settings.

\paragraph{Online conformal prediction.} \citet{gibbs2021adaptive} introduced Adaptive Conformal Inference (ACI) for marginal OCP. ACI parameterizes the prediction sets' width via a single parameter and updates this parameter with Online Subgradient Descent (OSD) on the quantile loss. To address ACI's sensitivity to its step size, \citet{bastani2022practical} introduced MultiValid Predictor, which provides long-term threshold-calibrated coverage but cannot rapidly react to changes in the data. Subsequent approaches--Aggregated ACI \citep{zaffran2022adaptive}, Strongly Adaptive OCP \citep{bhatnagar2023improved}, and Dynamically-Tuned ACI \citep{gibbs2021adaptive}--try to avoid manual step-size tuning by using online learning tools to automatically adjust ACI’s updates. A few works explore related problems. \citet{angelopoulos2024online} analyzed OSD with decaying step sizes and developed a method for both marginal OCP and CP with stochastic data. \citet{angelopoulos2025gradient} then introduced a general theory, showing when OSD works for OCP and similar problems. \citet{areces2025online} developed OCP algorithms for a class of stochastic processes with temporal dependence. Others works used control theory, recasting OCP as a feedback-control problem \citep{angelopoulos2023conformalpid, yang2024bellman}. These methods, however, require learning rates and sometimes other hyperparameters. To address this, \citet{zhang2024discounted} and \citet{podkopaev2024adaptive} used parameter-free online learning methods \citep{orabona2016coin, jun2017online, jun2017improved, orabona2021parameter} for OCP. Recently, \citet{liu2026online} developed Universal Portfolio OCP (UP-OCP), a parameter-free algorithm that 
provides marginal coverage. The concepts introduced in \citet{liu2026online} provide the foundation for our method. We elaborate in Sec. \ref{sec:marginal_coverage}.

\paragraph{Group-conditional coverage.} The goal of providing group-conditional coverage 
was introduced in the i.i.d. data setting. \citet{vovk2003mondrian} and \citet{romano2020malice} proposed algorithms with coverage guarantees for disjoint groups. Shortly after came a conservative method with guarantees for intersecting groups  \citep{foygel2021limits}. Numerous works followed. \citet{jung2021moment, jung2023batch} developed methods that produce smaller prediction sets which provide coverage conditional on group membership and the threshold used to produce the set. \citet{gibbs2025conformal} provided further improvements, giving tighter finite-sample coverage, allowing for discrete outcomes, and richer conditional guarantees. These works laid the foundation for group-conditional OCP. The method of \citet{gupta2022online} generates prediction sets in an online setting with a group calibration-like guarantees. This was improved by \citet{bastani2022practical} and then \citet{ramalingam2025relationship}, who developed group-conditional ACI (GCACI). GCACI parameterizes sets using group-specific parameters and updates them via online quantile regression with follow-the-regularized-leader; this procedure requires a learning rate. Concurrent work \citep{angelopoulos2025gradient} derived similar results, extending the theory to decaying learning rates. Recently, \citet{lyu2026normalization} studied how normalizing GCACI's update rule can yield improvements. Many of these approaches represent prediction sets as linear functions of group membership, an idea that we also adopt.

\section{Problem Setting}
\label{sec:problem_setting}
We observe a stream of data $((X_t, Y_t))_{t \geq 1}$ and have a predictor $f_t$ that depends on the data up to time $t-1$. At each time $t$, $X_t \in \bR^d$ is a $d$-dimensional input, $Y_t \in \bR$ is a univariate outcome, and denote $\hY_t = f_t(X_t)$ to be the prediction of $f_t$ on $X_t$. We make no assumptions on the data stream, or $f_t$ which may change arbitrarily over time. For a sequence of radii $(\tau_t)_{t \geq 1}$, $\tau_t \in \bR$, we define prediction intervals as $\calI_t(\tau_t) = [\hY_t - \tau_t, \hY_t + \tau_t]$ with the convention that $\calI_t$ is \emph{empty} if $\tau_t < 0$.

\paragraph{Marginal OCP.} The goal of OCP is to generate intervals $(\calI_t)_{t\geq 1}$ that control the empirical coverage rate at level $1-\alpha \in (0,1)$, i.e. as the time horizon $T$ grows, we require $\frac{1}{T} \sum_{t=1}^T \1\{Y_t \in \calI_t\} \approx 1-\alpha$. At each time $t$, the radius $\tau_t$ of the interval $\calI_t(\tau_t)$ must be selected using only the data up to $t-1$. Denote $S_t = \lvert Y_t - \hY_t \rvert$ to be \emph{non-conformity scores}. Noting that $\mathbf{1} \{Y_t \in \calI_t\} = \mathbf{1} \{S_t \leq \tau_t\}$, the goal of marginal OCP is equivalent to learning a sequence of radii $(\tau_t)_{t \geq 1}$ that satisfies
\begin{equation}
    \tag{OCP$_{\text{M}}$}
    \label{eq:marg_cover}
    \lim_{T\rightarrow \infty} \left|\frac{1}{T}\sum_{t=1}^T\mathbf{1}\left\{S_t \leq \tau_t \right\} - (1-\alpha) \right| = 0.
\end{equation}
Unfortunately, radii $(\tau_t)_{t\geq1}$ can provide \eqref{eq:marg_cover} by systematically miscovering the $Y_t$ of certain groups. This raises fairness and safety concerns, motivating group-conditional OCP (G-OCP).

\paragraph{Group-conditional OCP.} Consider a collection of $k$ potentially intersecting groups represented by a set of $k$ functions $\calC = \{c_j: \calX \to \{0,1\} \mid j \in [k]\},$ $[k] \coloneqq \{1, \dots, k\}$. Here $c_j(X) = 1$ indicates that $X$ belongs to group $j$ and $T_j = \sum^T_{t=1} c_j(X_t)$ is the count of samples belonging to group $j$ up to time $T$. As an example, groups may be defined by biological sex, age, or race, and an individual $X$ may belong to, say, both the \textsc{male} and \textsc{under 50} groups.  G-OCP extends the objective of marginal OCP, requiring that the learned sequence of radii $(\tau_t)_{t \geq 1}$ satisfies\footnote{Implicit is the requirement 
that $T_j>0$ and $T_j \rightarrow \infty$.}
\begin{align}
    \label{eq:cov_guarantee}
    \tag{OCP$_{\text{G}}$}
    \lim_{T\rightarrow \infty} \left|\frac{1}{T_j}\sum^T_{t=1}\mathbf{1}\{S_t \leq \tau_t \} c_j(X_t) - (1-\alpha) \right| = 0   &\quad\text{for all $j \in [k]$}.
\end{align}
In this work, we develop a parameter-free algorithm for G-OCP. \footnote{Our algorithm and all presented analysis is for the more general setting where $c_j(X_t) \in [0,1]$.} We begin with an overview of a parameter-free marginal OCP method \citep{liu2026online} whose key ideas we adapt.

\section{Marginal Coverage via Portfolio Optimization}
\label{sec:marginal_coverage}
Similar to CP with i.i.d data \citep{angelopoulos2023conformal,shafer2008tutorial}, online marginal coverage \eqref{eq:marg_cover} can be obtained by learning the $(1-\alpha)$-quantile of the scores $(S_t)_{t \geq 1}$. Since OCP considers a stream of data, this is typically done with online quantile regression. At each time $t$, one considers the quantile loss \citep{koenker2001quantile,gneiting2007strictly}
\begin{equation}
    l_t(\tau_t, S_t) = \max\{(1-\alpha)(S_t-\tau_t), \alpha(\tau_t - S_t)\},
\end{equation}
and uses classic online learning methods  \citep{cesa2006prediction,orabona2019modern} to update $\tau_t$. Interestingly, the subgradients of $l_t$ (with respect to $\tau_t$) quantify the miscoverage error \citep{gibbs2021adaptive,angelopoulos2025gradient}. Observe that
\begin{equation}
    g_t(\tau_t) = \1\{S_t \leq \tau_t\} - (1-\alpha)
\end{equation}
is a subgradient of $l_t$ at $\tau_t$ and that the miscoverage error at time $T$ can be expressed as
\begin{equation}\label{eq:MisCov}
    \textsf{MisCov}_T =  \left|\frac{1}{T}\sum^T_{t=1}\mathbf{1}\{S_t \leq \tau_{t}\} - (1-\alpha)\right| = \frac{1}{T}\left| \sum^T_{t=1} g_{t}(\tau_t) \right|.
\end{equation}
Thus, radii $\tau_t$ for which the cumulative subgradient $|\sum^T_{t=1} g_{t}(\tau_t)|$ grows slowly will necessarily control the empirical miscoverage error.

\subsection{A Parameter-Free Method for OCP via Universal Portfolio}
\label{sec:up-ocp}
\citet{liu2026online} proposed UP-OCP, a portfolio-optimization-based method for OCP that uses a sequential betting game to control the cumulative subgradient $|\sum^T_{t=1} g_{t}(\tau_t)|$. The method begins with wealth $W_0 = 1$. At each time $t$, using the data up to $t-1$, UP-OCP chooses a \emph{signed betting fraction} $\beta_t \in \left[\nicefrac{-1}{1-\alpha}, \nicefrac{1}{\alpha}\right]$ of the current wealth to bet on the negative subgradient $-g_t(\tau_t) \in \{-\alpha, 1-\alpha\}$. This sequential game induces a wealth process $(W_t)_{t\geq1}$ given by
\begin{equation}
    W_{t} = W_{t-1}(1 - \beta_{t} g_{t}). 
\end{equation}
UP-OCP employs Universal Portfolio (UP) \citep{cover1991universal,cover1996universal} to learn \((\beta_t)_{t\ge1}\) that promote rapid growth of \(\ln W_t\), and sets the radii as \(\tau_t=\beta_t W_{t-1}\). \citet{liu2026online} show that for this choice of $\tau_t$, $|\sum_{t=1}^T g_t(\tau_t)|$ is bounded by a slowly growing function in $T$, yielding $\textsf{MisCov}_T = \mathcal{O}\left((\ln T)/T + \sqrt{(\alpha \ln T)/T}\right)$. Notably, UP-OCP provides this guarantee under a mild polynomial growth condition on the non-conformity scores: a condition weaker than what is assumed in prior works.\footnote{\citet{liu2026online} assume $0 \leq S_t \leq Dt^q$ for $D > 0, ~q\geq 0$. Previous works assume that $S_t \in [0,D]$.} Moreover, it does not use a learning rate, making it \emph{parameter-free}.

While UP-OCP only provides marginal (not group-conditional) guarantees, it demonstrates that \emph{solving a portfolio optimization problem can produce a parameter-free OCP algorithm}. We show how a suitable, carefully-constructed, generalization of this insight leads to a G-OCP algorithm.
\vspace{-0.75em}

\section{Online Group Coverage via Betting \& Portfolio Optimization}
\vspace{-0.75em}
\label{sec:g_ocp_methods}
We now present POGO, a parameter-free algorithm for G-OCP. Like UP-OCP, POGO utilizes betting and portfolio optimization, extending it to G-OCP.
 
In standard OCP, there is a \emph{single} coverage objective \eqref{eq:marg_cover} and the radii $\tau_t$ are thus scalars. In G-OCP, the radii must satisfy $k$ coverage guarantees, one for each group \eqref{eq:cov_guarantee}. This naturally suggests parameterizing $\tau_t$ with $k$ parameters and learning them sequentially. Accordingly, for grouping functions $c_1, \dots, c_{k}$, define $\bc = (c_1, \dots, c_{k})^T$ to be the group membership vector. POGO parametrizes the radii $\tau_t(X_t)$ as a linear function of the group membership vector\footnote{If \(c_j(X_t)=0\) for all \(j\), then \(\tau_t=0\) which will typically miscover. This is prevented by augmenting the group vector to \(\bc_t'=(1,\bc_t)^\top\), which ensures that we provide at least marginal coverage.}
\begin{equation}
    \label{eq:linear_model}
    \tau_t(X_t) = \langle \bm{\theta}_t, \bc(X_t)\rangle, ~~ \bm{\theta}_t \in \bR^k,
\end{equation}
and learns $(\bm{\theta}_t)_{t\geq1}$ that provide group-conditional coverage. This parameterization is simple to implement and empirically very effective. Consequently, it has been used in several works \citep{jung2023batch, ramalingam2025relationship, areces2025online}.

More importantly though, it is amenable to an analysis of miscoverage. That is, the linear model provides a link between subgradients and miscoverage. Denote $\bc_t = \bc(X_t)$ and $l_t(\bm{\theta_t}, S_t) \coloneqq \max\{(1-\alpha)(S_t-\langle \bm{\theta_t}, \bc_t \rangle), \alpha(\langle \bm{\theta_t}, \bc_t \rangle - S_t)\}$ to be the quantile loss at $\bm{\theta}_t$ (i.e., the loss incurred by the radius $\tau_t(X_t)$). A subgradient vector $\bg_t \in \bR^k$ of $l_t$ at $\bm{\theta}_t$ is 
\begin{equation}
    \bg_t = [\mathbf{1}\{S_t \leq \langle \bm{\theta}_t, \bc_t\rangle\} - (1-\alpha)] ~\bc_t.
\end{equation}
Analogous to marginal OCP, these subgradients quantify the group miscoverage errors. Namely, the $j^{\text{th}}$ entry of the vector $\left|\sum^T_{t=1}\bg_t\right|$ quantifies the miscoverage error of group $j$: 
\begin{equation}
    \textsf{MisCov}_T(j) \coloneqq \left|\frac{1}{T_{j}}{\sum^T_{t=1}\mathbf{1}\{S_t \leq \langle \bm{\theta}_t, \bc_t\rangle\}} c_{t,j} - (1-\alpha)\right| = \frac{1}{T_{j}}\left|\sum^T_{t=1}g_{t,j}\right|.
\end{equation}
Thus, learning $\bm{\theta}_t$ such that each coordinate of the cumulative subgradient $|\sum^T_{t=1}\bg_t|$ grows slowly will yield group coverage. The remaining question is whether this can be achieved with a parameter-free algorithm. Recall from Sec. \ref{sec:up-ocp}, UP-OCP learns scalar radii through a wealth process and betting. POGO generalizes UP-OCP by maintaining $k$ wealth processes and maximizing each via UP. In doing so, it learns $k$-dimensional parameters $(\bm \theta_t)_{t\geq1}$ that provide group coverage guarantees.
\begin{figure}[t]
\begin{algorithm}[H]
    \caption{\label{algo:coord_betting} Portfolios for Online Group Conformal (POGO)}
    \raggedright\textbf{Input:} Target miscoverage level $\alpha$.\\
    \raggedright\textbf{Initialize:} Wealths $W_{0,j} \gets 1/k$ for all $j \in [k]$ and Jeffrey prior $\mu(\bar  \lambda) = 1/(\bar \lambda(1- \bar \lambda))^{1/2}$
    \vspace{-1em}
    \begin{algorithmic}[1]
        \FOR{$t = 1, \dots$}
            \FOR{$j = 1, \dots, k$}
                \STATE Define $w_j \colon \mathbb{R} \to \mathbb{R}$ as $w_j(\bar \lambda) = \prod^{t-1}_{i=1} \left(\bar \lambda \left(1 - \frac{g_{i,j}}{\alpha}\right) + (1-\bar \lambda)\left(1 + \frac{g_{i,j}}{1-\alpha}\right)\right)$
                \STATE \label{line:update-lambda} Update portfolio $\lambda_{t,j} \gets \frac{\int^1_{0}\bar \lambda w_j(\bar \lambda)d\mu(\bar \lambda)}{\int^1_{0}w_j(\bar \lambda)d\mu(\bar \lambda)}$
                \STATE Update $j$-th betting fraction $\beta_{t,j} = \frac{\lambda_{t,j} - \alpha}{\alpha(1-\alpha)}$ and $j$-th entry of $\bm \theta_t$, $\theta_{t,j} \gets \beta_{t,j}W_{t-1,j}$
            \ENDFOR
            \STATE Observe $(X_t, \bc(X_t))$, compute radius $\tau_t \gets \langle \bm{\theta}_t, \bc(X_t) \rangle$
            \STATE Observe $Y_t$, compute score $S_t \gets |Y_t - \hat{Y}_t|$ and subgradient $\bg_t \gets [\mathbf{1}\{S_t \leq \tau_t\} - (1-\alpha)] \bc_t$
            \FOR{$j = 1, \dots, k$} 
            \STATE Update individual wealth processes: $W_{t,j} \gets W_{t-1,j}\left(1 - \left(\frac{\lambda_{t,j} - \alpha}{\alpha(1-\alpha)}\right) g_{t,j}\right)$
            \ENDFOR
        \ENDFOR
    \end{algorithmic}
\end{algorithm}
\end{figure}

\subsection{Group Conditional Online Conformal via POGO}
\label{sec:cw_method}
POGO learns parameters $(\bm \theta_t)_{t\geq1}$ that control each coordinate of $|\sum_{t=1}^T \bg_t|$ by playing $k$ sequential betting games in parallel. 
For each group $j \in [k]$, POGO maintains a wealth process $(W_{t,j})_{t\geq1}$, initialized at $W_{0,j} = 1/k$. At time $t$, in the $j$-th game, POGO chooses a signed betting fraction $\beta_{t,j} \in \left[\nicefrac{-1}{1-\alpha}, \nicefrac{1}{\alpha}\right]$ of the current wealth $W_{t-1,j}$ in that game, and places this bet on $-g_{t,j}$. The fraction $\beta_{t,j}$ depends only on the history of subgradients until that point, $g_{1,j}, \dots, g_{t-1,j}$. Accordingly, the $j$-th wealth process evolves as
\begin{equation}
    W_{t,j} = W_{t-1,j}(1 - \beta_{t,j}g_{t,j}).
\end{equation}
POGO is presented in \cref{algo:coord_betting}. Specifically, POGO uses UP to learn the betting fractions $(\beta_{t,j})_{t\geq1}$ for each game, ensuring the logarithm of the corresponding wealth process, $\ln W_{t,j}$, grows quickly. The parameter vector $\bm\theta_t$ is defined coordinate-wise by $\theta_{t,j} = \beta_{t,j}W_{t-1,j}$. This construction ensures that each coordinate of the cumulative subgradient $\left|\sum_{t=1}^T \bg_t\right|$ grows slowly, thereby controlling $\textsf{MisCov}_T(j)$ for all $j \in [k]$ and yielding the following finite-time coverage guarantee.
\begin{restatable}[POGO coverage guarantee]{theorem}{coordcov}
\label{theorem:CW_group_cov}
Let $\alpha \in (0,1)$, and let $(S_t, \bc_t)^T_{t=1}$ be the sequence of non-conformity scores and group membership vectors observed by \cref{algo:coord_betting}, respectively. Let $D > 0$ and $q \geq 0$, and assume $S_t \leq Dt^q,~ \forall t \geq 1$ and $T_j > 0$. For every integer $T \geq 1$, define
\begin{equation}
    U_{T}(k) = \ln\left(1 + (1-\alpha)\frac{D(T+1)^{q+1}}{q+1}\right) + \frac{1}{2}\ln(\pi(T+1)) + \ln(k).
\end{equation}
Then, POGO guarantees $\emph{\textsf{MisCov}}_T(j) \leq \frac{1}{T_j}\left[U_{T}(k) + \sqrt{2T_j\alpha(1-\alpha)U_{T}(k)}\right]$.
\end{restatable}

We now compare POGO's guarantee with those of three standard G-OCP algorithms, GCACI \citep{ramalingam2025relationship}, GCACI with decaying learning rates (D-GCACI) \citep{angelopoulos2025gradient}, and Normalized-GCACI (N-GCACI) \citep{lyu2026normalization}. At their core, all three methods use online subgradient descent, updating $\bm{\theta_t}$ with the rule $\bm{\theta_t} = \bm{\theta_{t-1}} - \eta \bg_{t-1}$ ($\bm{\theta_1}= {\bf 0}$). GCACI and N-GCACI use a fixed $\eta$ while D-GCACI uses $\eta_t = \eta_0 t^{-a}$ for small $a \in (0,1/2)$ and $\eta_0 \in \bR$.

\cref{table:coverage_rate} presents the coverage guarantees for various methods. POGO's guarantee holds under a weaker growth condition $S_t \leq Dt^q$, for $D>0$ and $q\geq 0$, whereas others rely on uniformly bounded scores, $S_t\in[0,D]$. \emph{Thus, POGO's guarantee applies to a broader class of score sequences}. Nonetheless, in the interest of comparing methods, assume uniformly bounded scores, $S_t\in[0,D]$. Moreover, suppose $\alpha$ is small and each group occurs with non-vanishing frequency, i.e., $T_j \asymp \rho_j T$ for some constant $\rho_j>0$. Then, all bounds vanish as $T\to\infty$, yielding asymptotic coverage. They differ, however, in their dependence on $k$ and $\alpha$, as well as in their large-$T$ behavior.

\paragraph{Varying $T$, fixed $\alpha$ and $k$.} POGO guarantees $\textsf{MisCov}_T(j)=\mathcal{O}(\sqrt{\ln(T)/T})$,
whereas, for fixed $\eta$ and $\eta_0$, GCACI and D-GCACI guarantee
rates of $\mathcal{O}(1/\sqrt{\eta T})$ and $\mathcal{O}(1/\sqrt{\eta_0 T^{1-a}})$, respectively. Thus, POGO's miscoverage bound decreases faster than
D-GCACI's but slower than GCACI's. However, GCACI's guarantee depends on a fixed learning rate, which can be problematic. Relative to GCACI, POGO incurs an additional factor of
$\sqrt{\ln T}$ in exchange for being parameter-free. This factor is
\emph{very} mild. 
D-GCACI's bound has the slowest 
rate, as expected: using decaying learning rates trades convergence speed for stability and smaller prediction-set radii \citep{angelopoulos2024online}. N-GCACI's guarantee is better than GCACI's \footnote{We suppress the dependence on $\alpha$ and $\eta$ for simplicity.} for $T=\Omega(4^k/k)$, yielding $\textsf{MisCov}_T(j) = \calO((1+\eta)/(\eta T))$. However, this
improvement occurs only at impractical time horizons (when $k=25$, the crossover occurs only after approximately $T \approx 4.5\times 10^{13}$ samples). Thus, it is also important to consider the smaller finite $T$ regime. In the analysis that follows, we focus on the finite-horizon setting $T\lesssim 4^k/k$. Since, in this regime, the N-GCACI and D-GCACI miscoverage bounds decrease at a slower rate than GCACI's, we only compare POGO to GCACI.

\paragraph{Fixed $\alpha$ and varying $k$.} POGO yields $\textsf{MisCov}_T(j)=\calO(\sqrt{(1/T)\ln(kT)})$, while GCACI with fixed $\eta$ gives $\calO(\sqrt{(1/T)(k+(1/\eta))})$. POGO's bound depends on $k$ via $\ln k$ whereas GCACI's scales with $\sqrt{k}$ (and also depends on $\eta$). When $k$ is very small (i.e. $k=1$) GCACI may have a better dependence in $T$. In contrast, when $k$ is moderate or large, POGO provides better guarantees.

\paragraph{Fixed $k$ and $\alpha \to 0$.} When $k$ is fixed and $\alpha\to 0$, POGO's guarantees $\textsf{MisCov}_T(c_j)= \calO((\ln T)/T)$. By contrast, for fixed $\eta$, the GCACI bound is $\calO(1/\sqrt{\eta T})$. Thus, as $\alpha \to 0$, POGO achieves exact asymptotic coverage at a faster rate than GCACI. Moreover, POGO's rate achieves the fastest known rate for time-averaged constraint violation (up to log factor) \citep{yu2020low}.

Lastly, let us consider the rare scenario where a group $j$ does appear with a vanishing frequency in time $T$. Then, the miscoverage bound of GCACI will go to $0$ as $T \rightarrow \infty$ as long as $T_j = \omega(\sqrt{T})$. On the other hand, the bound of POGO will approach $0$ as $T \rightarrow \infty$ as long as $T_j = \omega(\ln T)$. Thus, in this scenario, POGO can provide coverage for groups that are rarer than GCACI can.

In conclusion, POGO incurs a mild $\sqrt{\ln T}$ factor in its finite-horizon coverage guarantee for two advantages. First, its guarantee remains valid under polynomially growing non-conformity scores, rather
than requiring uniformly bounded scores. Second, POGO is parameter-free. Although GCACI's miscoverage bound may decrease faster for small $T$ for suitably chosen $\eta$, selecting such $\eta$ is difficult as the data stream may change arbitrarily. We illustrate why picking a good $\eta$ is difficult and the issues that can arise empirically in \cref{sec:experiments}.

\begin{table}[t!]
\centering
\caption{\label{table:coverage_rate}Coverage guarantees of different G-OCP methods for $\alpha < 1/2$ (omiting constants).}
\resizebox{\columnwidth}{!}{
\renewcommand{\arraystretch}{1.5}
\begin{tabular}{lccc}
\toprule
Method                                  & Parameter-Free  & Score Assumption & Coverage Rate\\
\midrule
GCACI \citep{ramalingam2025relationship} & \redx           & $S_t \in [0,D]$ & $T_j^{-1}\sqrt{T(k (1-\alpha)^2+2(1-\alpha)/\eta)}$
\\
D-GCACI \citep{angelopoulos2025gradient} & \redx      &  $S_t \in [0,D]$   & $T^{-1}_j \sqrt{T^{2a}[(1-\alpha)^2k \sum^T_{t=1}t^{-2a} + (2(1-\alpha)\sqrt{k}/\eta_0)\sum^T_{t=1}t^{-a}]}$\\
N-GCACI \citep{lyu2026normalization} & \redx          &  $S_t \in [0,D]$  & $T_j^{-1}\min\left\{\sqrt{Tk((1-\alpha)^2+2(1-\alpha)/\eta)},2^{k+1}(1+\eta)/\eta\right\}$\\
POGO                                    & \greencheck  &  $S_t \leq Dt^q$  & $T_j^{-1}(\ln kT  + \sqrt{T_j\alpha \ln (kT)})$\\
\bottomrule
\end{tabular}%
\renewcommand{\arraystretch}{1}
}
\end{table}

\subsection{Implementation Details}
At each time $t \in \{1, \ldots, T\}$, POGO evaluates $k$ integrals, each requiring $\mathcal{O}(T^2)$ operations by time $T$ \citep{cover1991universal}, leading to a total time complexity of $\mathcal{O}(k T^2)$. Such numerical integration is an approximation, whereas the coverage guarantee applies to the exact update. Fortunately, for binary groups $c_j(X_t) \in \{0,1\}$, the  update in Line~\ref{line:update-lambda} of Alg.~\ref{algo:coord_betting} has the closed form \citep[App.~G]{liu2026online}
\begin{equation}
    \lambda_{t,j} = \frac{1}{\sum_{i=1}^{t-1} c_j(X_i) + 1}
        \left(\sum_{i=1}^{t-1}
        \mathbf{1}\{g_{i,j} = \alpha - 1\}
        + \frac{1}{2}\right).
\end{equation}
Therefore, our experiments consider only the binary-group setting. For soft grouping functions $c_j(X_t) \in [0,1]$, we provide details on approximating the integrals in Appendix~\ref{supp:soft_group_implementation}.

\subsection{Alternative Naive Approaches}
\label{sec:naive_method}
Before presenting our experiments, we address a natural question: \emph{why} use POGO rather than running UP-OCP separately for each group? Concretely, for each group $j$, one can parameterize a radius as $\tau_{t,j} = \theta_{t,j}\,c_j(X_t),$ for $\theta_{t,j}\in\mathbb R$. Since $c_j$ is fixed, learning $(\tau_{t,j})_{t\geq 1}$ reduces to learning $(\theta_{t,j})_{t\geq 1}$, which can be done with UP-OCP. Specifically, define the quantile loss $l_t(\theta_{t,j}, S_t) = \max\{(1-\alpha)(S_t-\theta_{t,j}\,c_j(X_t)), \alpha(\theta_{t,j}\,c_j(X_t) - S_t)\}$. 
Then a subgradient of $l_t$ at $\theta_{t,j}$ is $g_{t,j} = [{\bf 1}\{S_t \leq \theta_{t,j}c_j(X_t)\} - (1-\alpha)]\cdot c_j(X_t)$ and $\frac{1}{T_j}|\sum_{t=1}^T g_{t,j}|$ is precisely the miscoverage error of group $j$. With UP-OCP, one can learn a sequence  $(\theta_{t,j})_{t\geq1}$ which yields a fast decreasing bound (in $T$) on $\frac{1}{T_j}|\sum_{t=1}^T g_{t,j}|$ and hence a coverage guarantee for group $j$. 

However, running UP-OCP separately for each group will produce $k$ sequences of radii $\{(\tau_{t,j})_{t=1}^T\}_{j=1}^k$. This is impractical: given a new sample $X_t$ belonging to multiple groups, which radii should we use? Returning the maximum radius will likely over cover \citep{foygel2021limits}, and returning all complicates downstream decision-making. We therefore seek algorithms that return a \emph{single} sequence of radii $(\tau_t)_{t\ge1}$ that provide coverage for all groups $j\in[k]$. Moreover, if the algorithm is parameter-free, we overcome common issues with tuning learning rates in online settings. POGO satisfies these requirements, making it well-suited for G-OCP.

\section{Experiments}
\label{sec:experiments}
We support our theoretical results with experiments on synthetic and real-world datasets. Code for reproducing all of the experiments is available at \href{https://github.com/beepulbharti/pogo}{\texttt{github.com/beepulbharti/pogo}}.

\paragraph{Baselines.} We compare POGO with GCACI, D-GCACI, and N-GCACI. We exclude MVP \citep{bastani2022practical} since GCACI outperforms MVP theoretically and empirically \citep{ramalingam2025relationship}. For the real-data experiments, we also compare with UP-OCP. We do not compare with UP-OCP in the synthetic experiments because they are designed so that marginal methods perform poorly (e.g., UP-OCP achieves group coverage rates of $\leq 70\%$).

\paragraph{Metrics.} Evaluating methods only by whether they provide coverage is insufficient because arbitrarily large radii can achieve coverage. For a target level \(1-\alpha\), a G-OCP method should (i) attain \(1-\alpha\) coverage for all groups, (ii) produce small radii, and (iii) quickly react to changes in the data. To evaluate coverage, we report the worst observed group coverage: the coverage of the group with the largest absolute deviation from the target. To summarize the tradeoffs between coverage, radii length, and reactivity, we provide Pareto-frontier plots of observed coverage versus radius length and reactivity \citep{liu2026online}. We quantify reactivity by computing, for each group, the longest consecutive run of miscoverage events $\{S_t>\tau_t\}$, and reporting the maximum (lower is better).

\subsection{Synthetic experiments} 
\label{sec:syn_experiments}
Groups are modeled to have different frequencies. Denote $\calU \subset [k]$ to be the set of under-represented groups. At each time $t$, we generate a group  vector $\bc(X_t) = (c_1(X_t), \dots, c_k(X_t)) \in \{0,1\}^k$, where $c_j(X_t) \sim \mathrm{Bernoulli}(p_j)$ independently across $j \in [k]$ with
$p_j = 0.05$ if $j\in \calU$ and $0.25$ if $j\in [k]\setminus \calU$. Thus, groups in $\calU$ appear infrequently, making it challenging to ensure their coverage.

Following prior work \citep{angelopoulos2023conformalpid, liu2026online}, we directly generate non-conformity scores. The scores are modeled as
\begin{align}
\label{eq:conformity_dgp}
S_t
&= \max\{0, S^{\mathrm{base}}_t
+ \left\langle \mathbf{b}(t), \bc(X_t)\right\rangle\}. 
\end{align}
Here, $S^{\mathrm{base}}_t \sim \mathrm{Beta}(1,20)$ and 
$\mathbf b(t)= (b_1(t), \dots, b_k(t))$ introduces time-varying, group-specific shifts. Specifically, $\mathbf b(t) $ is a first-order auto-regressive (AR(1)) process with bounded uniform noise. The specific parameters of $\mathbf b(t)$ are provided in Appendix~\ref{supp:syn_experiment_details}. The key detail is that $\mathbf b(t)$ only alters the scores of the under-represented groups: for $j\in\calU$, $b_j(t)$ is an AR(1) process whereas for $j\notin\calU$, $b_j(t) \sim \mathrm{Unif}(-\nicefrac{1}{10},\nicefrac{1}{10})$ for all $t$. We set $|\calU|=\left\lfloor \tfrac{1}{20}k\right\rfloor$ and run all methods for target coverage levels $1-\alpha \in \{0.90, 0.91, \dots, 0.99\}$ and $T=50{,}000$. The following results are for $k=50$ groups and presented as averages over 50 independent runs (other $k$ are in Appendix~\ref{supp:additional_synthetic_results}).

\paragraph{Bounded-$[0,1]$ scores.} We model time-varying group shifts with no drift. To ensure a fair comparison to the GCACI-based methods, we clip the scores to $[0, 1]$. \cref{fig:bounded_row_no_cp} (left) shows that, with one exception, all methods, including POGO, provide group coverage within \(3\%\) of the target. Moreover, POGO produces the smallest average radii at all observed
coverage levels (\cref{fig:bounded_row_no_cp}, middle) and remains competitive in terms of reactivity: its longest sequence of miscovers across all groups is within $2$ of that of the best GCACI-based method (\cref{fig:bounded_row_no_cp}, right).

\paragraph{Bounded-$[0,1]$ scores with a shift.} In the previous setting, in addition to POGO, GCACI $(\eta = 0.01)$, N-GCACI $(\eta = 0.25)$, and D-GCACI $(\eta_t = 0.01t^{-0.1})$ achieved favorable tradeoffs between coverage, radius length, and the longest miscoverage sequence. We now model an abrupt shift after a long stable period to illustrate the risks of committing to learning rates. The scores follow the same model until \(t= \lfloor T/3 \rfloor\), after which the scores of one group \(j^* \in \calU\) increase. Specifically,
\begin{equation}
S_t = \max\{0, S_t^{\mathrm{base}}
+ \left\langle \mathbf{b}(t), \mathbf{c}(X_t)\right\rangle
+ 0.6\cdot\mathbf{1}\{t \ge \lfloor T/3 \rfloor\}c_{j^*}(X_t)\}
\end{equation}
Now, D-GCACI $(\eta_t = 0.01t^{-0.1})$, does not provide adequate group coverage, whereas the other methods do (\cref{fig:bounded_row_cp} (left)). POGO again provides the best coverage–radius tradeoff, outperforming GCACI \((\eta = 0.01)\) and N-GCACI $(\eta = 0.25)$. Although these two methods still provide coverage and attain small radii (\cref{fig:bounded_row_cp} (middle)), they respond poorly to the shift, especially GCACI $(\eta = 0.01)$. They exhibit longer runs of miscovers, particularly at lower coverage levels (\cref{fig:bounded_row_cp} (right)). POGO maintains short runs and competes with the best-performing methods. These results illustrate a key benefit of POGO: a learning rate $\eta$ selected in a no-shift regime can yield poor reactivity after a shift, whereas POGO remains robust across coverage, radius, and miscoverage-sequence metrics.

\begin{figure}[t!]
    \begin{minipage}[t]{0.84\textwidth}
        \begin{subfigure}[t]{\linewidth}
            \begin{subfigure}[t]{0.31\linewidth}
                \includegraphics[width=\linewidth]{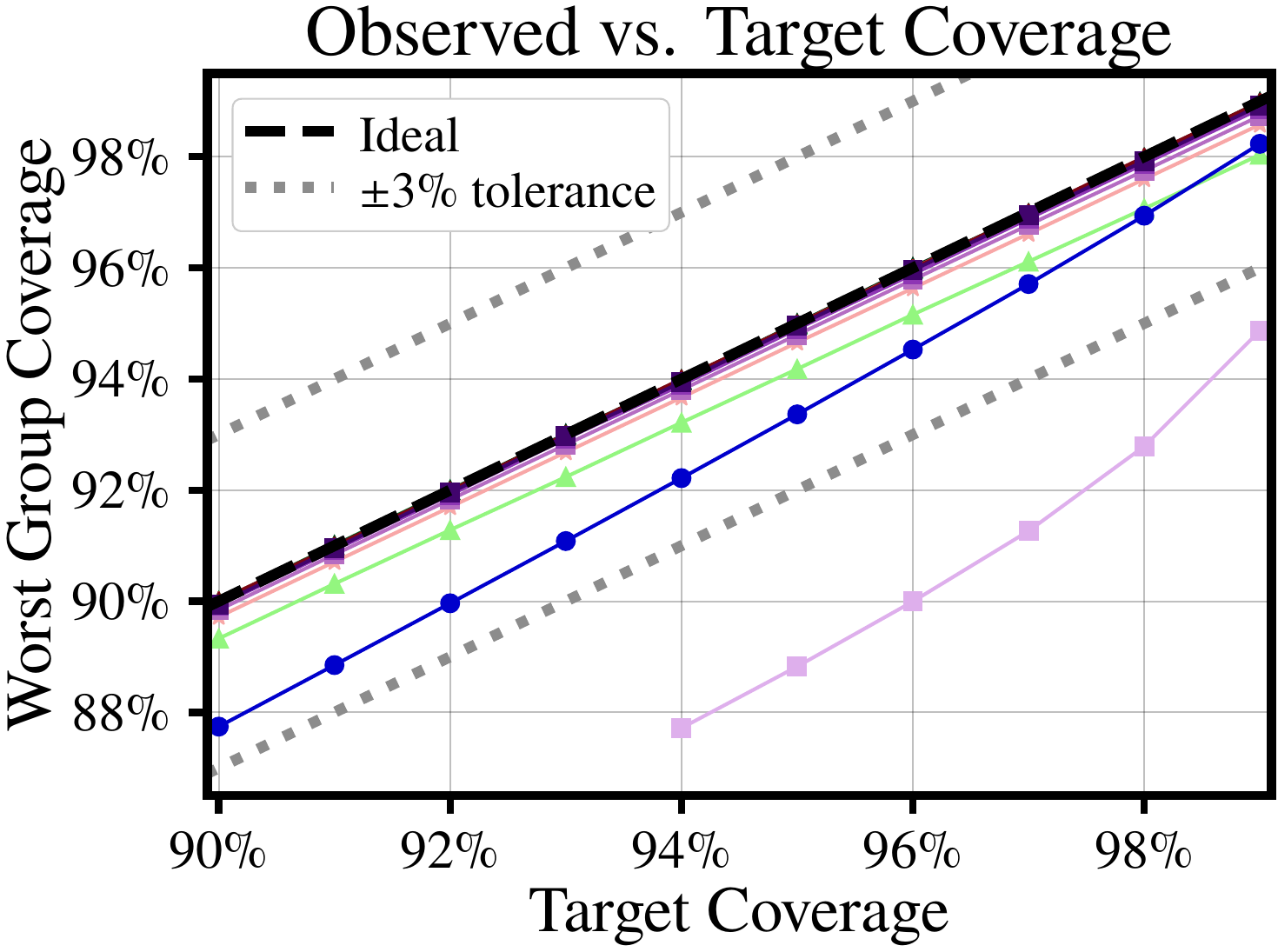}
            \end{subfigure}
            \hspace{0.2em}
            \begin{subfigure}[t]{0.31\linewidth}
                \includegraphics[width=\linewidth]{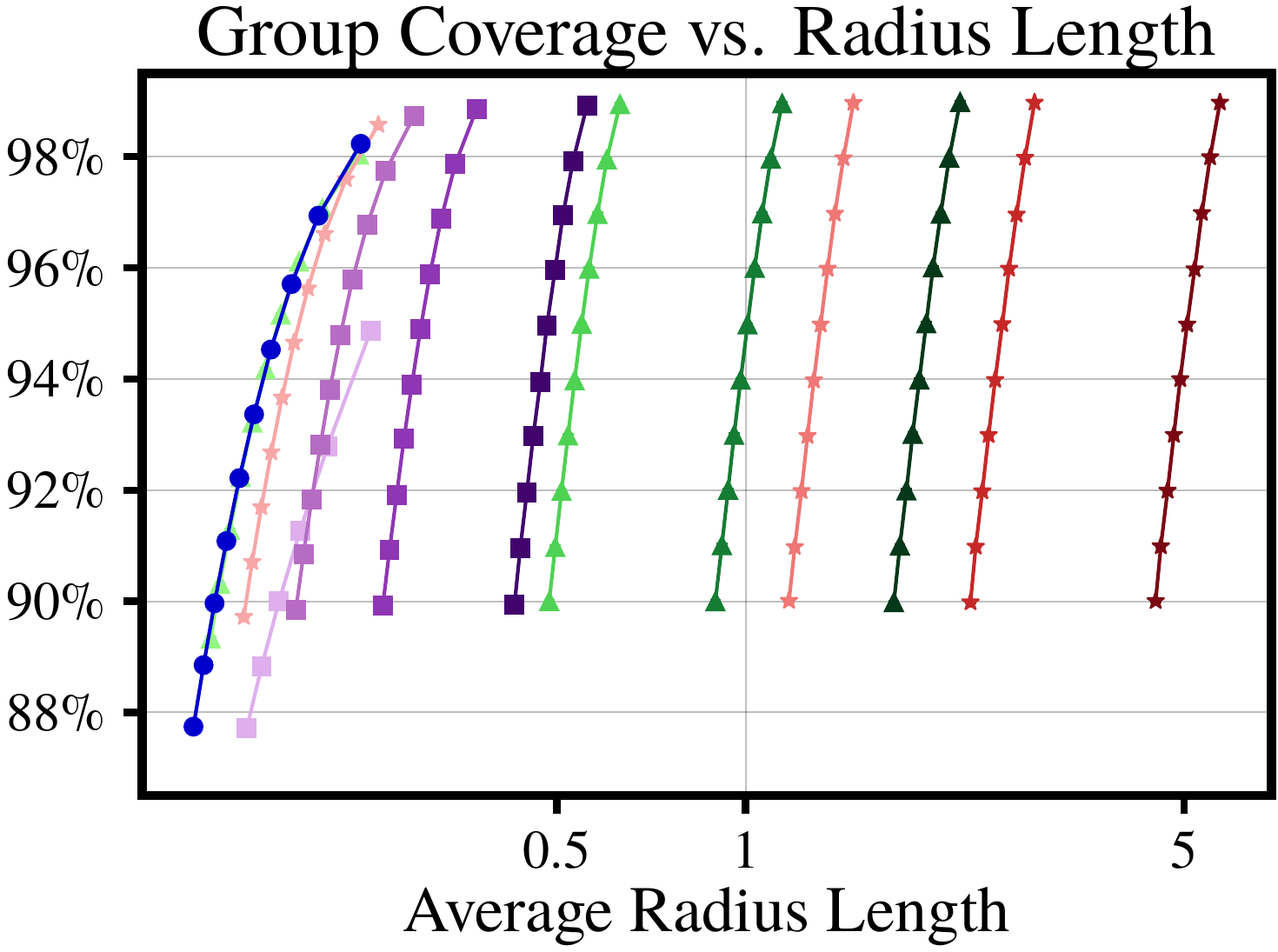}
            \end{subfigure}
            \hspace{0.2em}
            \begin{subfigure}[t]{0.31\linewidth}
                \includegraphics[width=\linewidth]{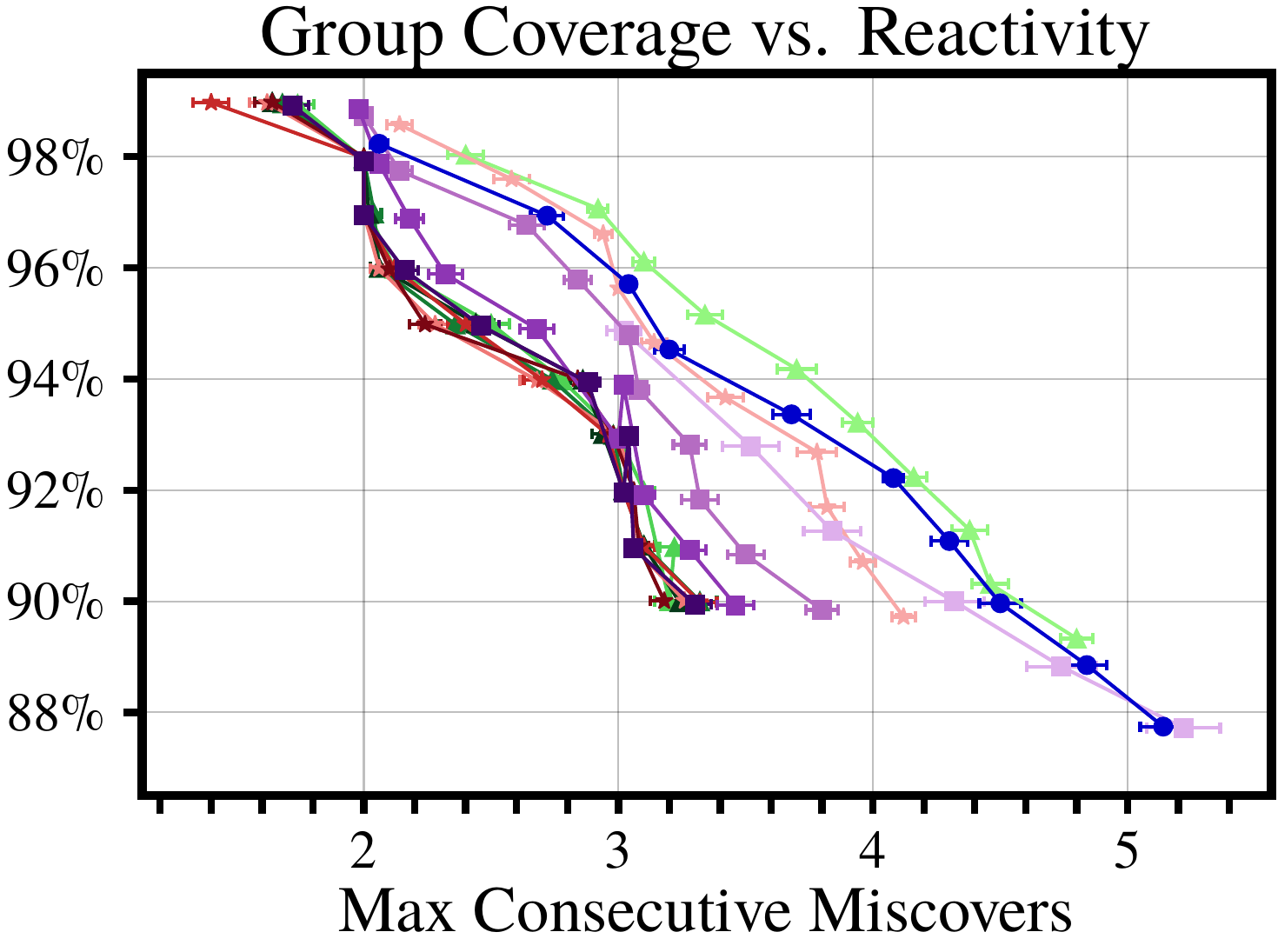}
            \end{subfigure}
            \caption{Results for the synthetic setting with bounded,
            gradually varying scores.}
            \label{fig:bounded_row_no_cp}
        \end{subfigure}
        \vspace{0.5em}
        
        \begin{subfigure}[t]{\linewidth}
            \begin{subfigure}[t]{0.31\linewidth}
                \includegraphics[width=\linewidth]{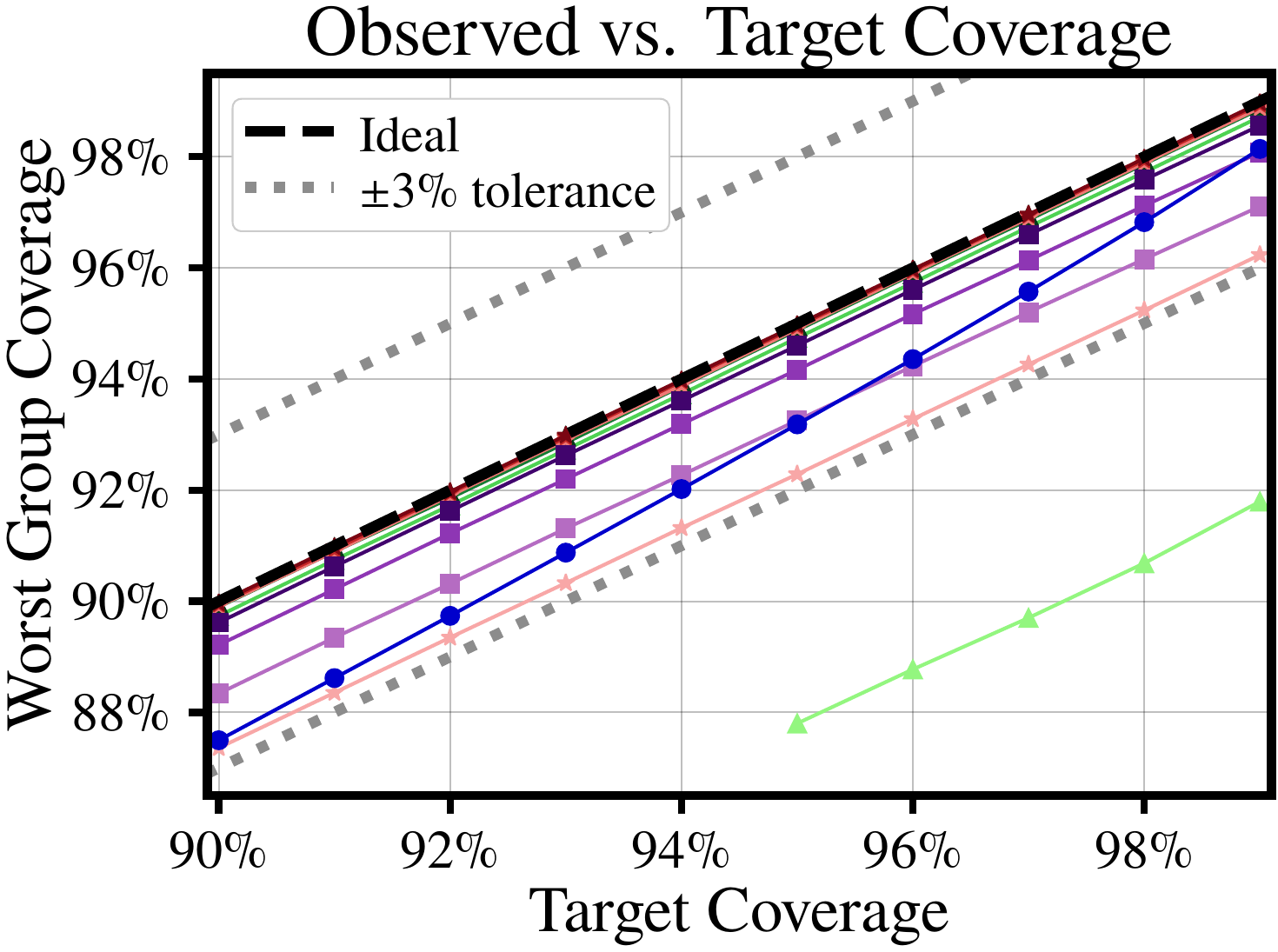}
            \end{subfigure}
            \hspace{0.2em}
            \begin{subfigure}[t]{0.31\linewidth}
                \includegraphics[width=\linewidth]{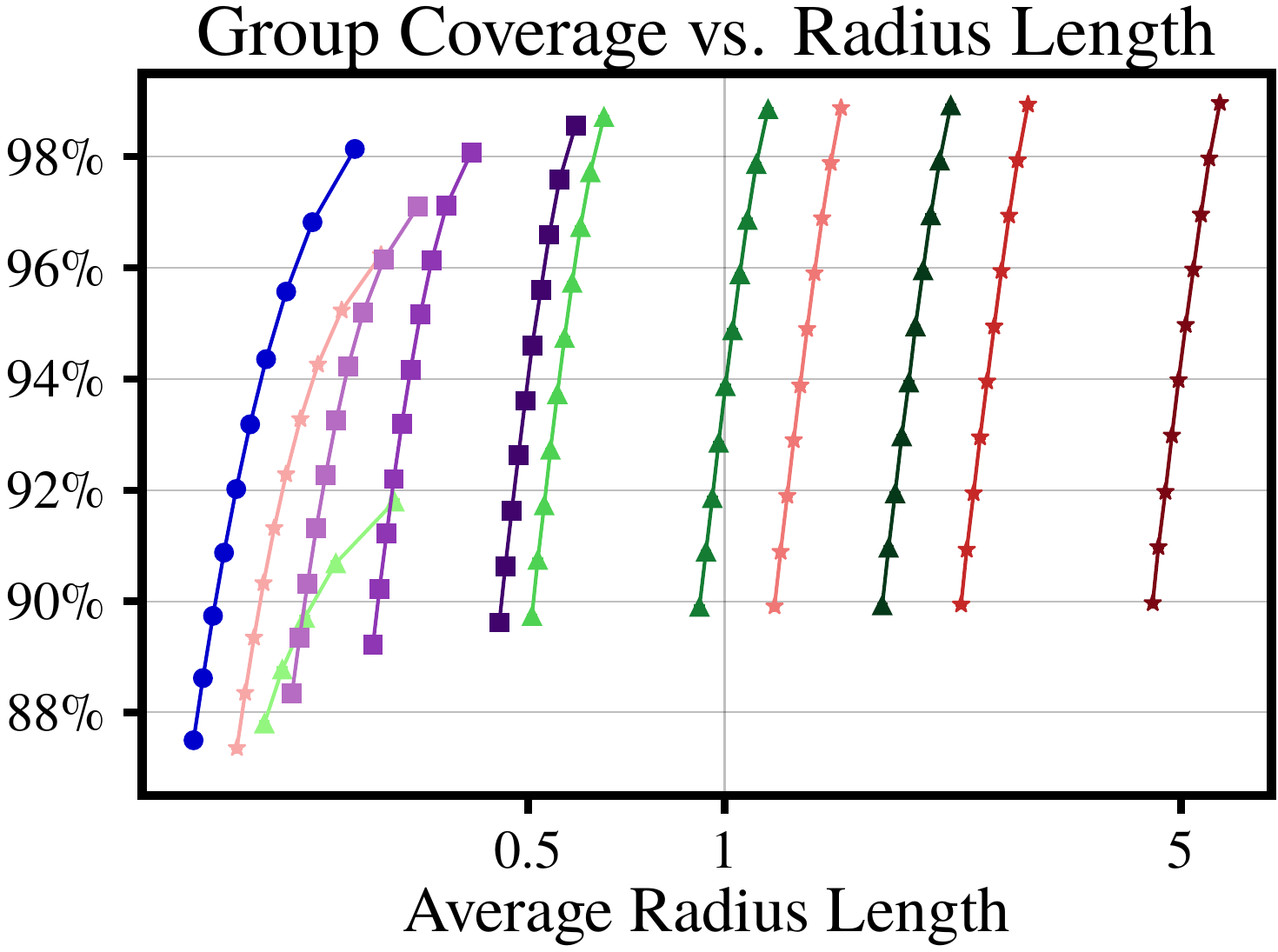}
            \end{subfigure}
            \hspace{0.2em}
            \begin{subfigure}[t]{0.31\linewidth}
                \includegraphics[width=\linewidth]{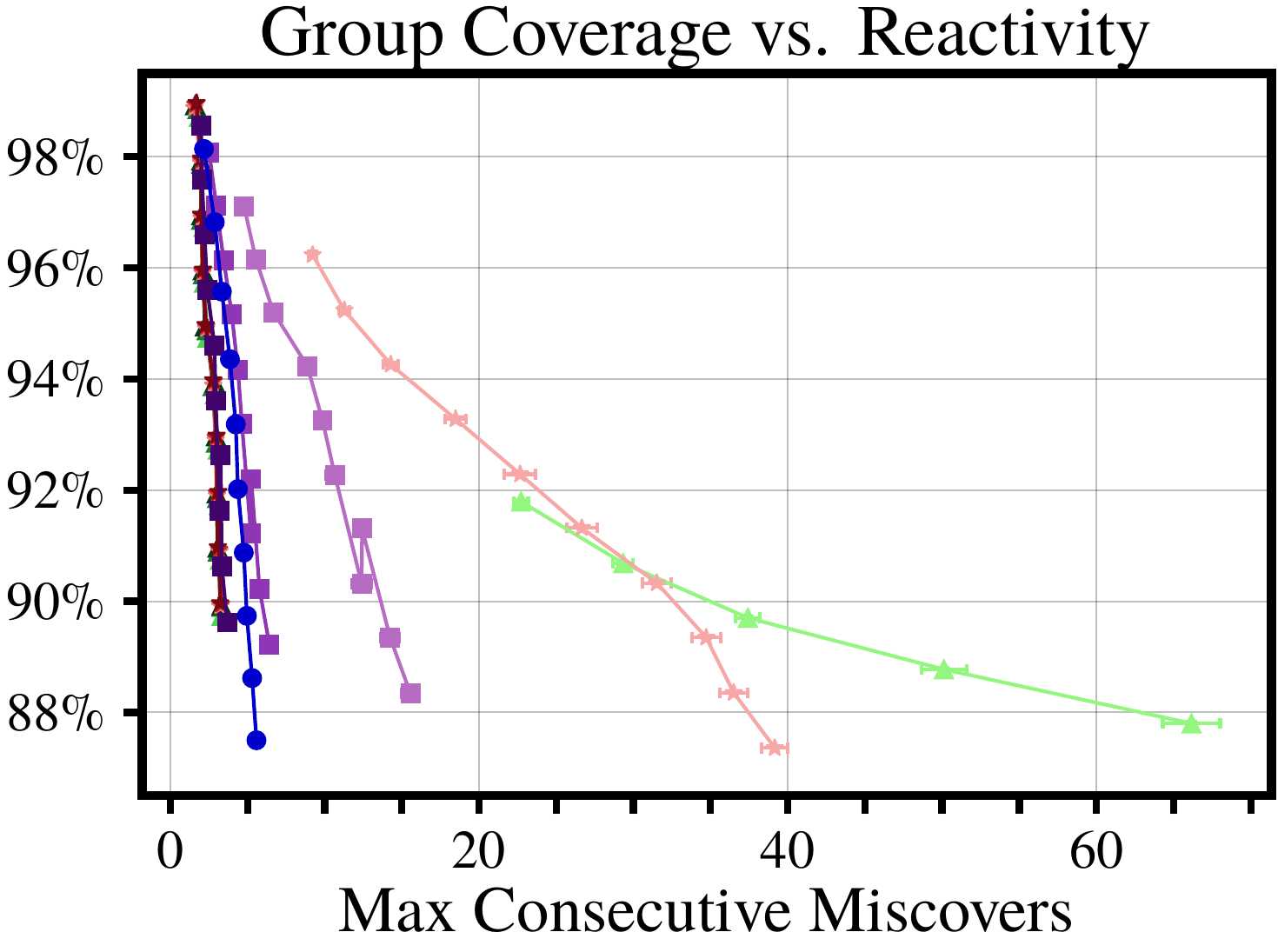}
            \end{subfigure}
            \caption{Results for the synthetic setting with bounded scores
            and a changepoint.}
            \vspace{0.5em}
            \label{fig:bounded_row_cp}
        \end{subfigure}
    \end{minipage}%
    \hspace{-1em}
    \begin{minipage}[c]{0.13\textwidth}
    \raisebox{-3.75cm}{
        \includegraphics[
            width=1.3\linewidth,
        ]{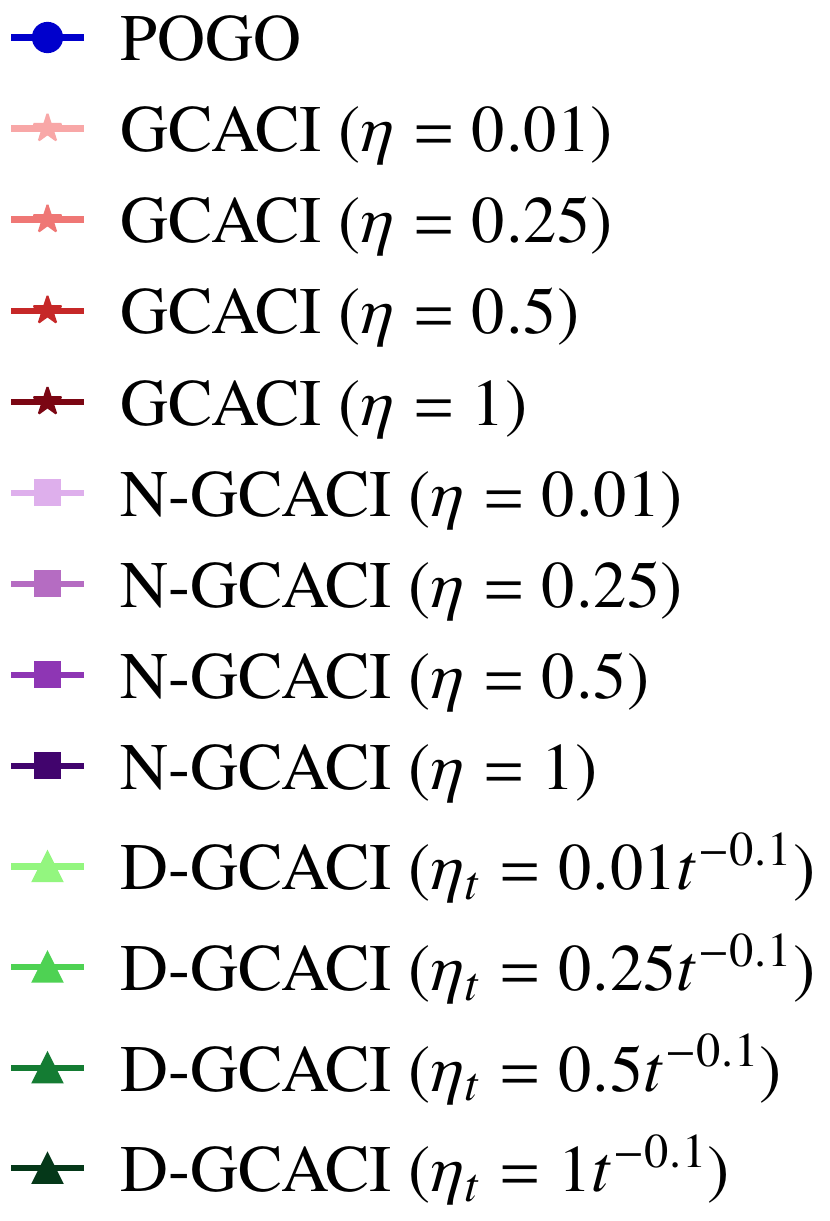}
    }
    \end{minipage}
    
\vspace{-1em}
\begin{minipage}[t]{0.84\textwidth}
    \begin{subfigure}[t]{\linewidth}
    \begin{subfigure}[t]{0.31\linewidth}
        \includegraphics[
            width=\linewidth
        ]{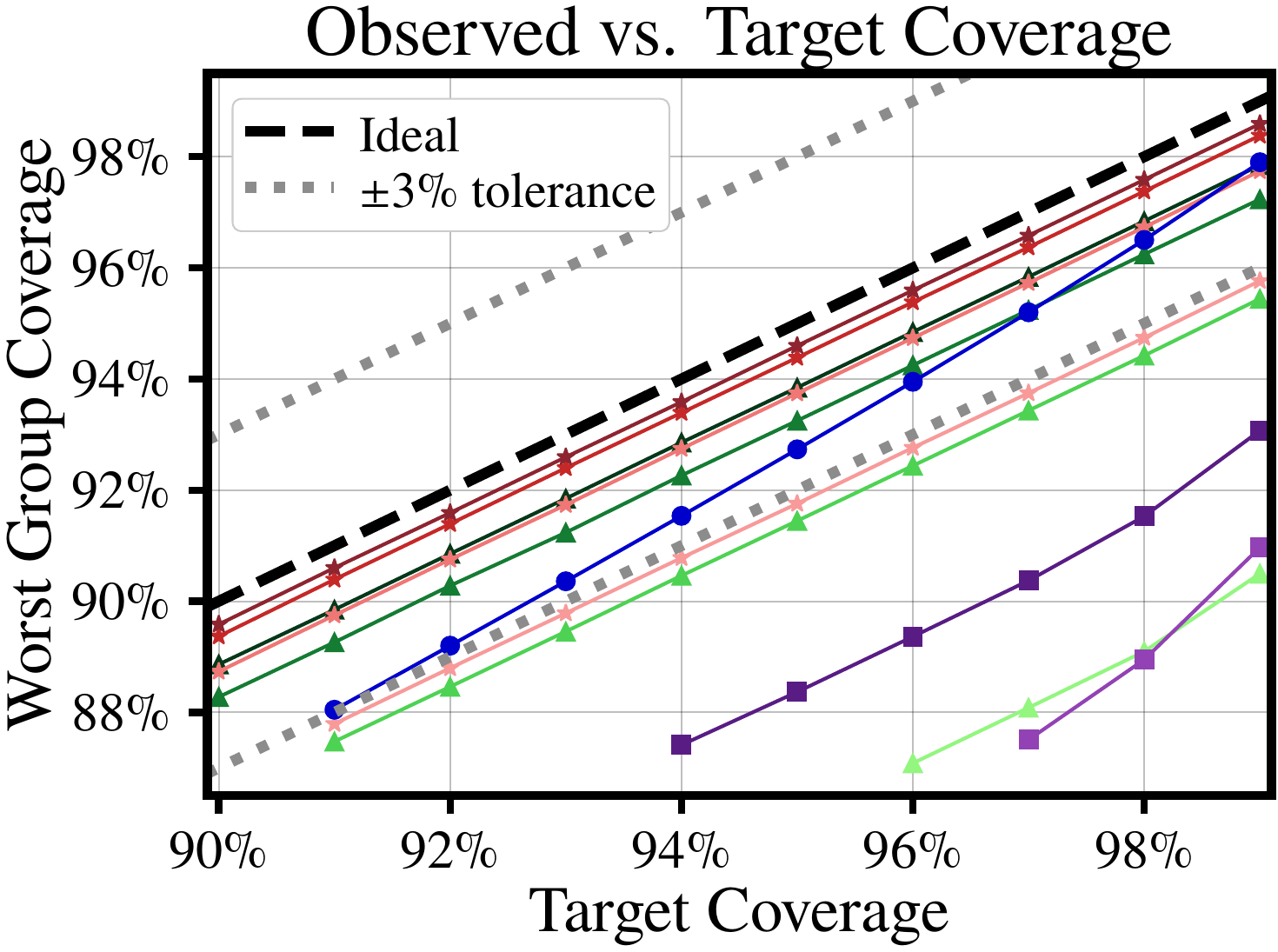}
    \end{subfigure}
    \hspace{0.27em}
    \begin{subfigure}[t]{0.31\linewidth}
        \includegraphics[
            width=\linewidth
        ]{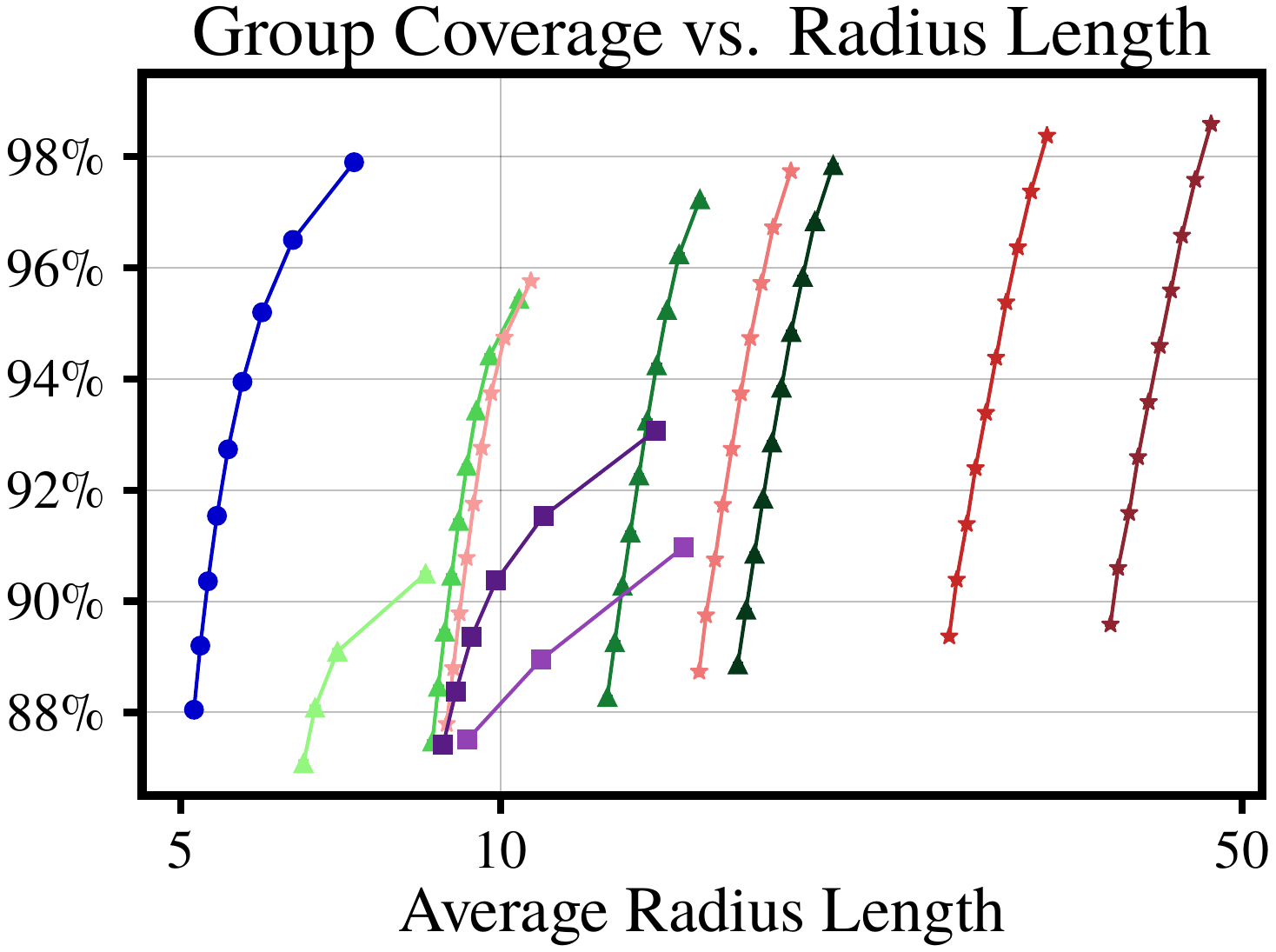}
    \end{subfigure}
    \hspace{0.21em}
    \begin{subfigure}[t]{0.31\linewidth}
        \includegraphics[
            width=\linewidth
        ]{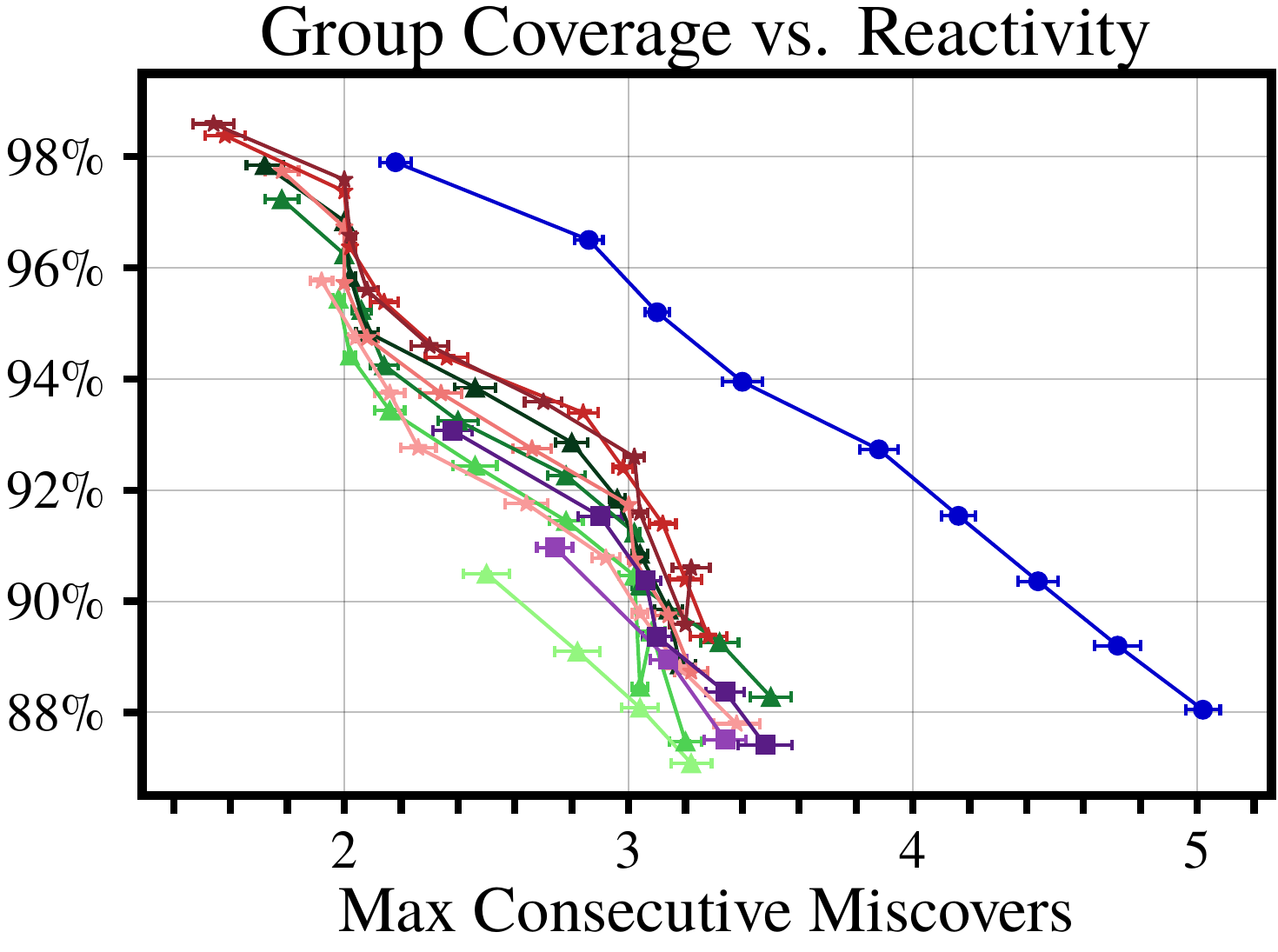}
    \end{subfigure}
    \caption{Results for the synthetic setting with quadratic growing scores ($q=2, A=25$).}
    \label{fig:quad_moderate}
\end{subfigure}
\end{minipage}%
\hspace{-1em}
    \begin{minipage}[c]{0.13\textwidth}
    \raisebox{2.7cm}{
        \includegraphics[
            width=1.3\linewidth,
        ]{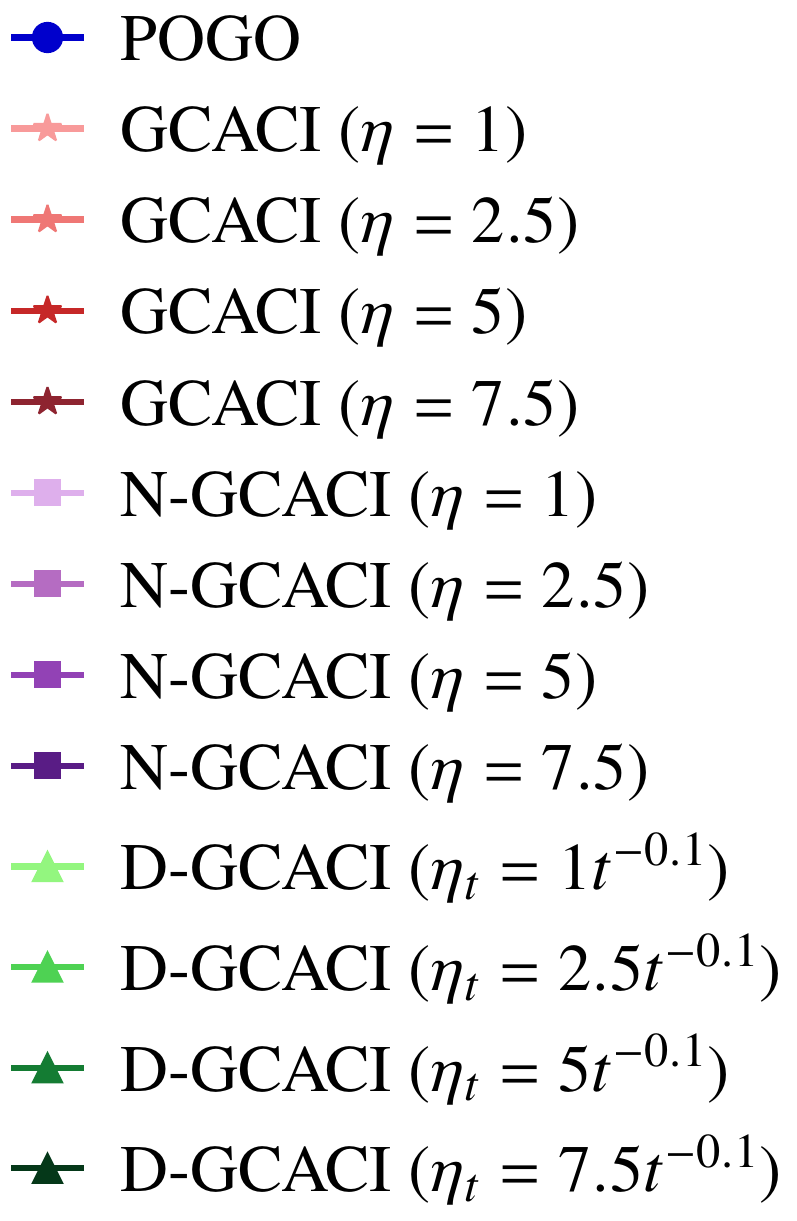}
    }
    \end{minipage}
    \vspace{-7em}
\caption{Performance across target coverage levels \(1-\alpha \in \{0.90, \ldots, 0.99\}\). {\bf Left:} worst-group coverage vs. target coverage. {\bf Middle/right:} worst-group coverage vs. radius length / max consecutive miscovers across groups (curves closer to the top-left indicate better performance; data point shown if observed coverage $\geq 0.87$). Error bars show \(x\)-axis standard error of the mean (SEM).}
\label{fig:synthetic_results}
\end{figure}

\paragraph{Polynomial growing scores.} We now select a group $j^* \in \calU$ and generate scores according to
\begin{equation}
\label{eq:quad_growth}
S_t = \max\{0, S_t^{\mathrm{base}}
+ \left\langle \mathbf{b}(t), \mathbf{c}(X_t)\right\rangle
+ A(1+\nu_t)\left(\nicefrac{t}{T}\right)^q c_{j^*}(X_t)\},
\end{equation}
where $\nu_t \sim \mathrm{Unif}(-\nicefrac{1}{2}, \nicefrac{1}{2})$, $q=2$, and $A=25$. The scores of group $j^*$ grow quadratically, up to the multiplicative noise $1+\nu_t$, with $A$ controlling the magnitude of the shift. To ensure a fair comparison, we consider $\eta, \eta_0 \in \{1, 2.5, 5, 7.5\}$ for GCACI, N-GCACI, and D-GCACI, so they can accommodate the growing scores. Results for additional $A$ and $q = 3$ are deferred to Appendix~\ref{supp:additional_synthetic_results}.

Several of the GCACI-based methods with smaller $\eta, \eta_0$ fail to provide group coverage within $3\%$ of the target (\cref{fig:quad_moderate} (left)). Among the methods that do, D-GCACI ($\eta_t = 5t^{-0.1}$) and GCACI ($\eta = 2.5$) achieve the most favorable coverage-radius length tradeoff. POGO, however, outperforms them: it provides group coverage while producing radii $50\%$ shorter than the best GCACI-based methods (\cref{fig:quad_moderate} (middle)). Although the results suggest that a better $\eta$ for GCACI or $\eta_0$ for D-GCACI may lie in $[1,2.5]$, they also indicate such variants would produce larger radii than POGO.

\subsection{Real World Data}
\label{sec:realworld}

\begin{figure}[t!]
    \begin{minipage}[t]{0.84\textwidth}
        \begin{subfigure}[t]{\linewidth}
            \begin{subfigure}[t]{0.31\linewidth}
                \includegraphics[width=\linewidth]{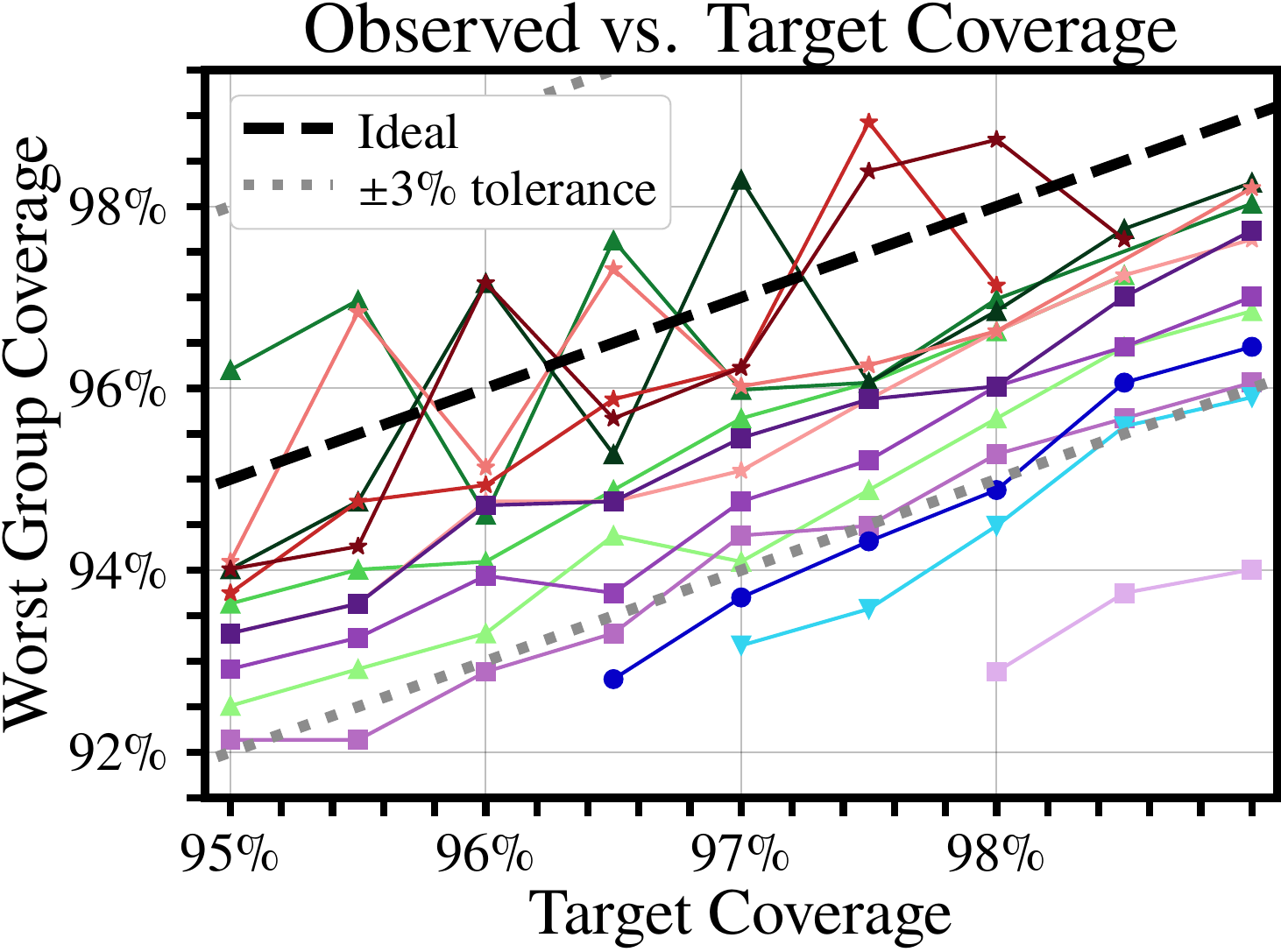}
                
            \end{subfigure}
            \hspace{0.2em}
            \begin{subfigure}[t]{0.31\linewidth}
                \includegraphics[width=\linewidth]{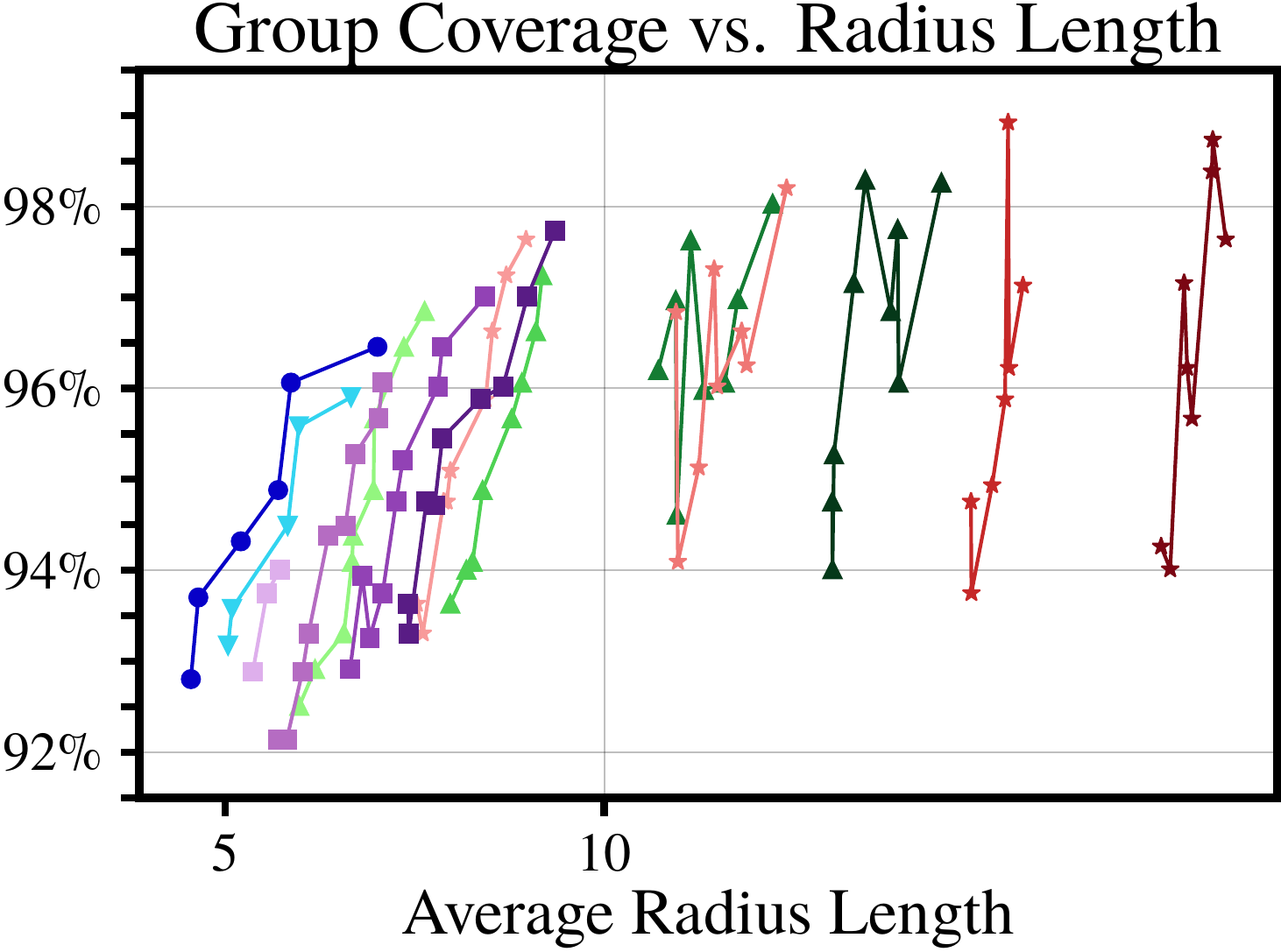}
            \end{subfigure}
            \hspace{0.2em}
            \begin{subfigure}[t]{0.31\linewidth}
                \includegraphics[width=\linewidth]{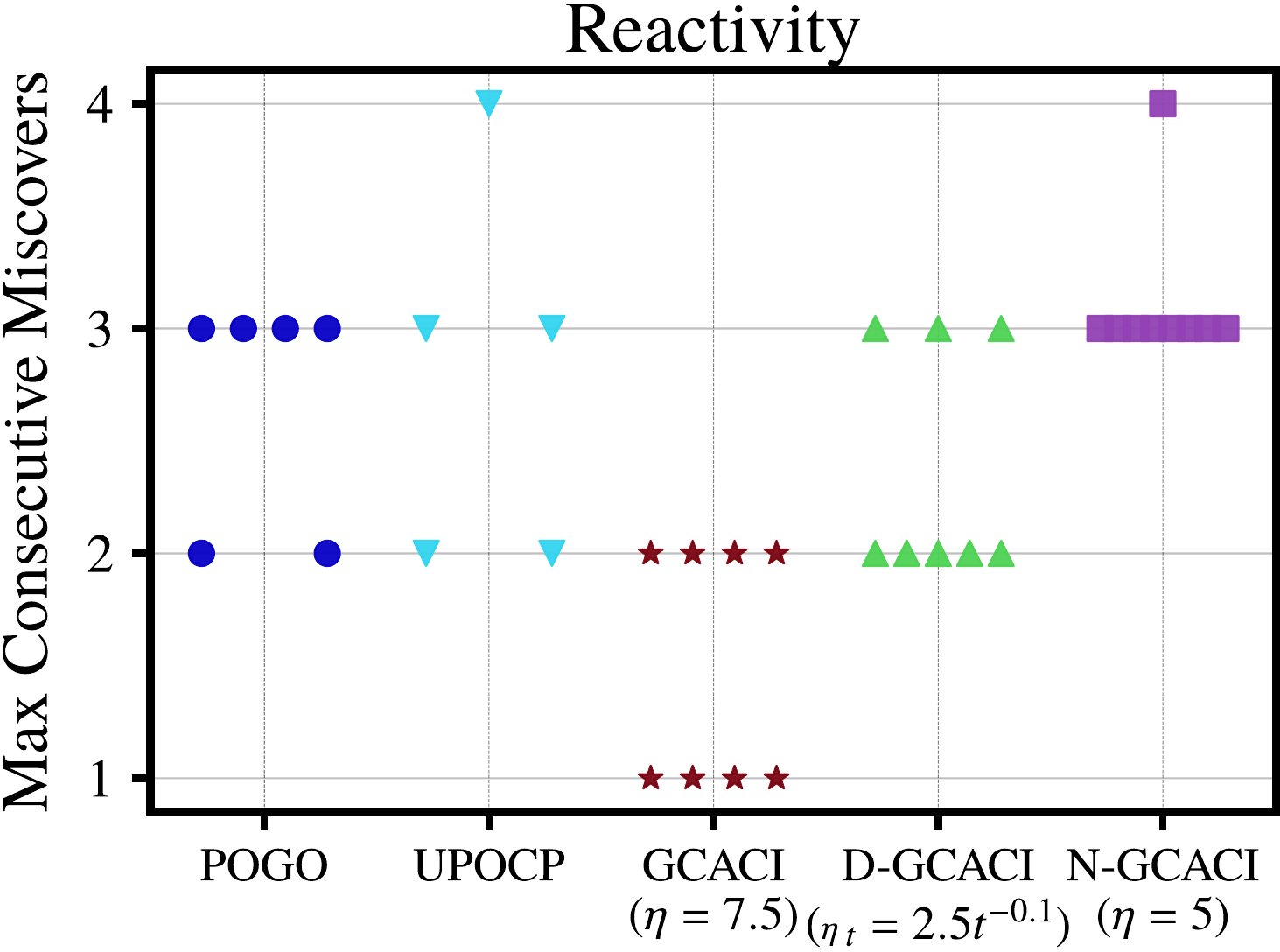}
            \end{subfigure}
            \caption{Results for Nvidia (\texttt{NVDA}) stock.}
        \end{subfigure}
        \vspace{0.5em}
        
        \begin{subfigure}[t]{\linewidth}
            \begin{subfigure}[t]{0.31\linewidth}
                \includegraphics[width=\linewidth]{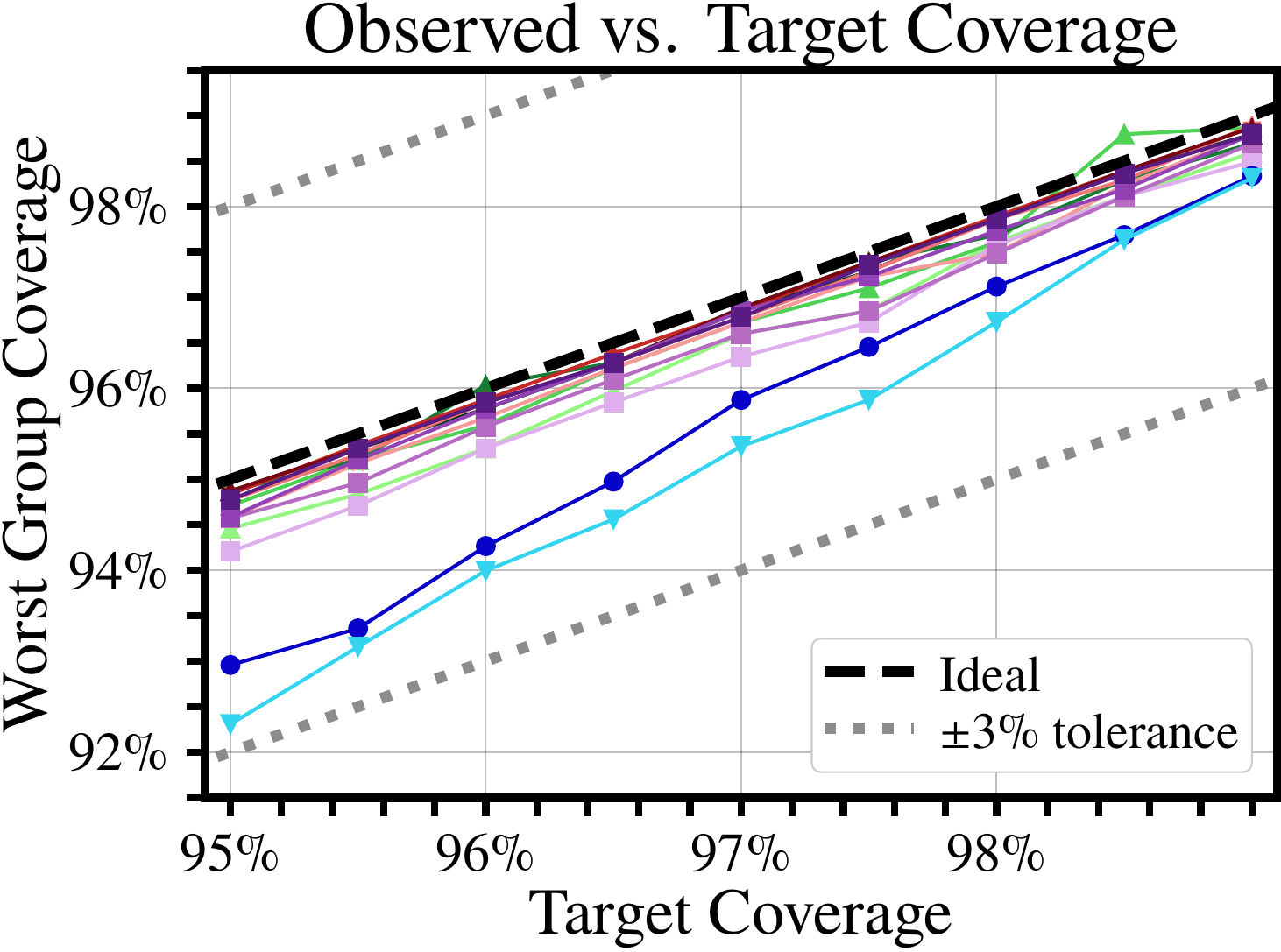}
                
            \end{subfigure}
            \hspace{0.2em}
            \begin{subfigure}[t]{0.31\linewidth}
                \includegraphics[width=\linewidth]{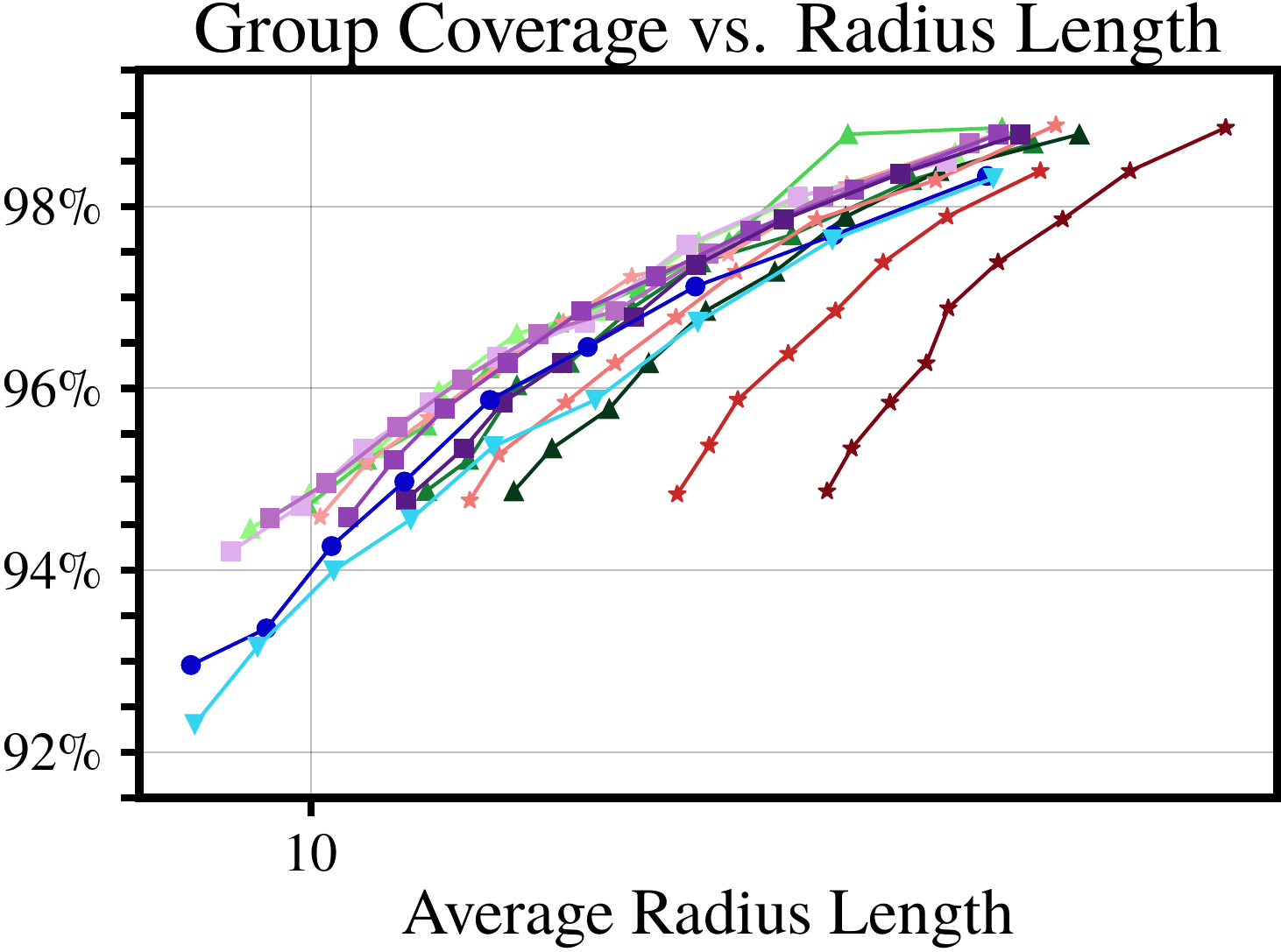}
            \end{subfigure}
            \hspace{0.2em}
            \begin{subfigure}[t]{0.31\linewidth}
                \includegraphics[width=\linewidth]{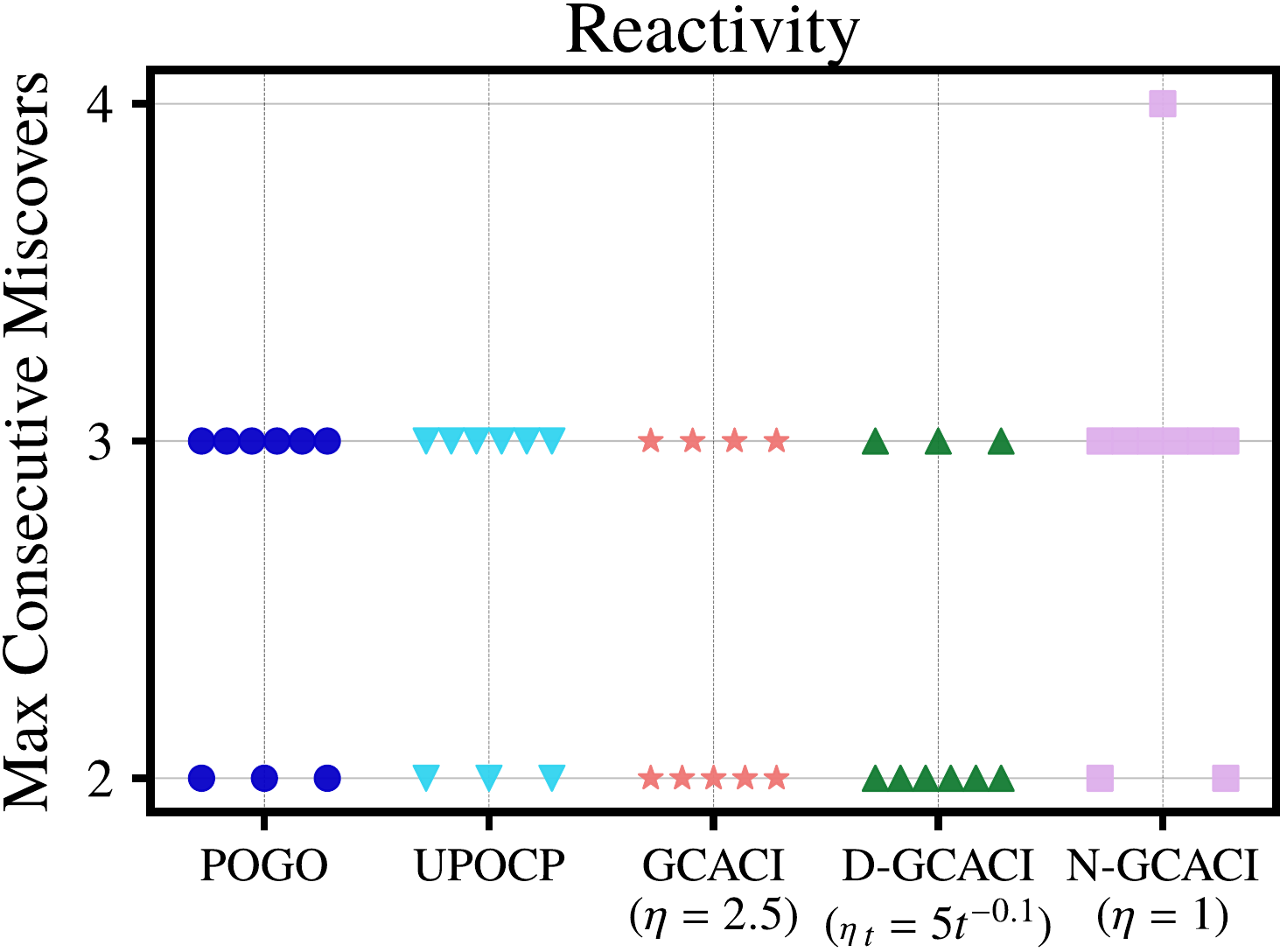}
            \end{subfigure}
            \caption{Results for MIMIC-IV.}
        \end{subfigure}
    \end{minipage}%
    \hspace{-1em}
    \begin{minipage}[c]{0.13\textwidth}
    \raisebox{-4cm}{
        \includegraphics[
            width=1.2\linewidth,
        ]{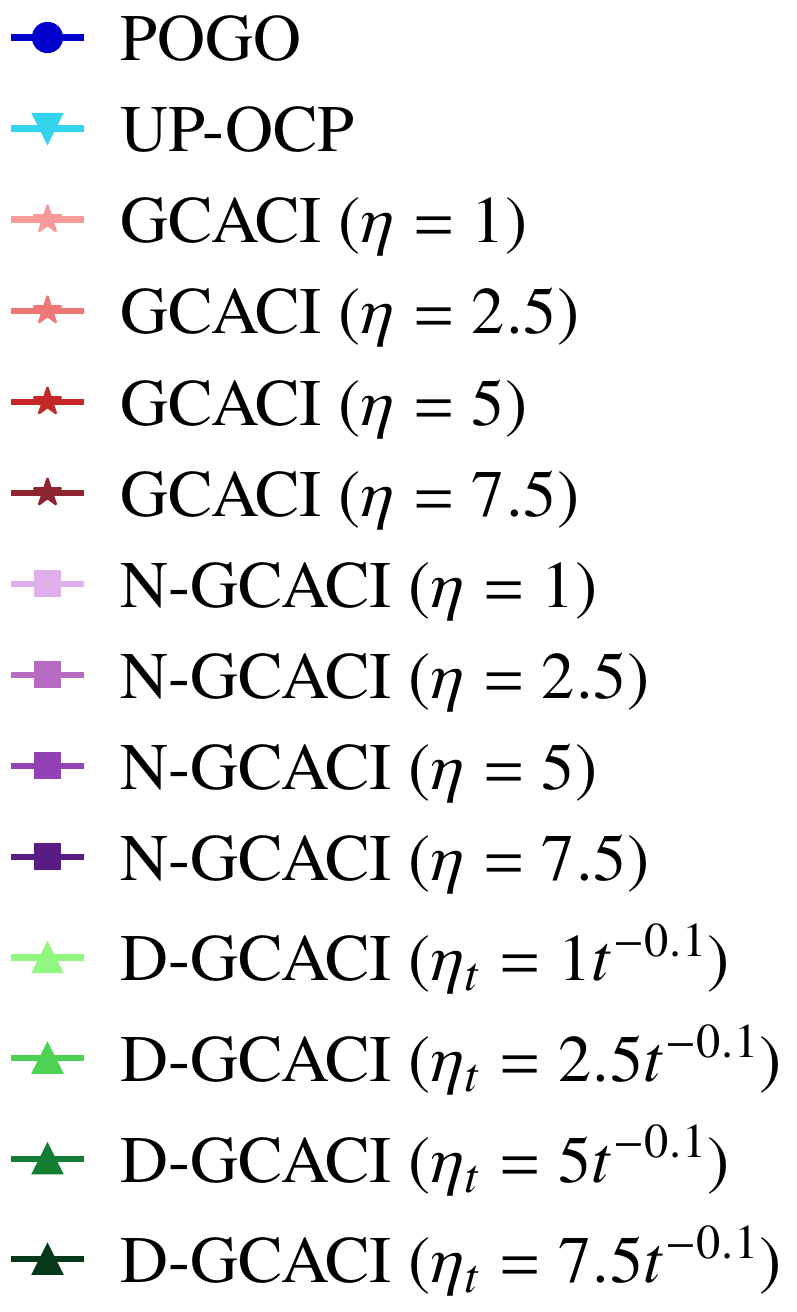}
    }
    \end{minipage}
    \caption{Performance across target coverage levels \(1-\alpha \in \{0.95, 0.955, \ldots, 0.99\}\). {\bf Left:} worst-group coverage vs. target coverage. {\bf Middle:} worst-group coverage vs. radius length (curves closer to the top-left indicate better performance). {\bf Right:} Max consecutive miscovers across target coverage levels. For clarity, among the GCACI, N-GCACI, and D-GCACI families, we show only the variant with the lowest median max consecutive miscovers. Data point shown if coverage \(\geq 0.92\).}
    \label{fig:real_data_results}
\end{figure}

For our real data experiments, we first learn a base model $\hat{Y} = f(X)$ (training) and then compare POGO, UP-OCP \footnote{UP-OCP only provides marginal coverage -- to compute group quantities, we run UP-OCP at a target level $\alpha$ to generate radii $(\tau_t)_{t\geq1}$. Then, for each group $j \in [k]$, we compute $|\frac{1}{T_j}\sum^T_{t=1}\mathbf{1}\{S_t \leq \tau_t \} c_j(X_t)|$.}, and GCACI on a test dataset. We run all methods for target coverage levels $1-\alpha \in \{0.95, 0.955,..., 0.99\}$. Since, a priori, we do not know the bounds of the non-conformity scores, we consider $\eta \in \{1, 2.5, 5, 7.5\}$ for GCACI. 

\paragraph{Length-of-Stay in the Intensive Care Unit.} We consider an MLP model that predicts the length of stay in the intensive care unit. We use the MIMIC-IV dataset \citep{johnson2023mimic}, a de-identified dataset of patients at BIDMC, collected within 24 hours of admission. We define groups using self-reported race, insurance status, and biological sex. Additional details are provided in Appendix~\ref{supp:experiments_mimic}.

\paragraph{Stock Market.} We consider a model that predicts the daily opening price of a stock. We use stock data of Nvidia (\texttt{NVDA}), Meta (\texttt{META}), and others from the last 15 years \citep{Nugent_2018}. We use Prophet \citep{taylor2018forecasting}, a popular model to produce daily opening-price forecasts, trained on the log opening prices of the first 100 trading days and then refit daily \citep{nyseHolidaysamp}. Groups are defined using market indicators such as volatility relative to a rolling median, calendar-year markers such as day-of-week, and indices of trading days from the first Monday \citep{bastani2022practical, ramalingam2025relationship}. Additional details and results for other stocks are in Appendix~\ref{supp:experiments_stock}.

\cref{fig:real_data_results} shows that POGO is competitive with the other methods across both experiments. POGO provides slightly smaller radii but undercovers by \(1\)--\(2\%\). Among the GCACI-based methods, those with smaller \(\eta\) and \(\eta_0\) offer favorable trade-offs among coverage, radius length, and reactivity. Importantly, \emph{these learning rates cannot be selected reliably in advance}. Being learning-rate-free, POGO avoids this problem while remaining competitive with these favorable GCACI variants.

\section{Conclusion}
\label{sec:conclusion}
We introduced POGO, the first parameter-free algorithm for group-conditional online conformal prediction. By combining ideas from online conformal prediction and parameter-free online learning, POGO provides rigorous coverage guarantees for all groups without tuning learning rates. Our analysis establishes the strongest known finite-time group-coverage bounds, and experiments on synthetic and real-world data show that POGO consistently attains the target coverage level for every group while maintaining competitive prediction interval lengths and reactivity relative to non–parameter-free baselines. Looking ahead, it remains unknown if stronger group-coverage guarantees are possible using richer models for the radius beyond the linear parameterization. Extending our theory and tools to other forms of online uncertainty quantification, such as calibration, is also an interesting direction. Finally, studying how POGO can support online decision-making in safety-critical applications like clinical decision support systems is a promising avenue for impact.

\bibliographystyle{plainnat}
\bibliography{arxiv/bibliography}

\begin{thebibliography}{52}
\providecommand{\natexlab}[1]{#1}
\providecommand{\url}[1]{\texttt{#1}}
\expandafter\ifx\csname urlstyle\endcsname\relax
  \providecommand{\doi}[1]{doi: #1}\else
  \providecommand{\doi}{doi: \begingroup \urlstyle{rm}\Url}\fi

\bibitem[Anderer(2025)]{anderer20252024}
Samantha Anderer.
\newblock 2024-2025 flu season marked highest hospitalization rate in nearly 15 years, 2025.

\bibitem[Angelopoulos et~al.(2023)Angelopoulos, Candes, and Tibshirani]{angelopoulos2023conformalpid}
Anastasios Angelopoulos, Emmanuel Candes, and Ryan~J Tibshirani.
\newblock Conformal pid control for time series prediction.
\newblock \emph{Advances in neural information processing systems}, 36:\penalty0 23047--23074, 2023.

\bibitem[Angelopoulos and Bates(2023)]{angelopoulos2023conformal}
Anastasios~N Angelopoulos and Stephen Bates.
\newblock Conformal prediction: A gentle introduction.
\newblock \emph{Foundations and Trends in Machine Learning}, 16\penalty0 (4):\penalty0 494--591, 2023.

\bibitem[Angelopoulos et~al.(2025)Angelopoulos, Jordan, and Tibshirani]{angelopoulos2025gradient}
Anastasios~N. Angelopoulos, Michael~I. Jordan, and Ryan~J. Tibshirani.
\newblock Gradient equilibrium in online learning: Theory and applications.
\newblock \emph{Journal of Machine Learning Research}, 26\penalty0 (305):\penalty0 1--68, 2025.
\newblock URL \url{http://jmlr.org/papers/v26/25-0356.html}.

\bibitem[Angelopoulos et~al.(2024)Angelopoulos, Barber, and Bates]{angelopoulos2024online}
Anastasios~Nikolas Angelopoulos, Rina Barber, and Stephen Bates.
\newblock Online conformal prediction with decaying step sizes.
\newblock In \emph{International Conference on Machine Learning}, pages 1616--1630. PMLR, 2024.

\bibitem[Areces et~al.(2025)Areces, Mohri, Hashimoto, and Duchi]{areces2025online}
Felipe Areces, Christopher Mohri, Tatsunori Hashimoto, and John Duchi.
\newblock Online conformal prediction via online optimization.
\newblock In \emph{International Conference on Machine Learning}, pages 1604--1649. PMLR, 2025.

\bibitem[Bastani et~al.(2022)Bastani, Gupta, Jung, Noarov, Ramalingam, and Roth]{bastani2022practical}
Osbert Bastani, Varun Gupta, Christopher Jung, Georgy Noarov, Ramya Ramalingam, and Aaron Roth.
\newblock Practical adversarial multivalid conformai prediction.
\newblock In \emph{Proceedings of the 36th International Conference on Neural Information Processing Systems}, pages 29362--29373, 2022.

\bibitem[Bhatnagar et~al.(2023)Bhatnagar, Wang, Xiong, and Bai]{bhatnagar2023improved}
Aadyot Bhatnagar, Huan Wang, Caiming Xiong, and Yu~Bai.
\newblock Improved online conformal prediction via strongly adaptive online learning.
\newblock In \emph{International Conference on Machine Learning}, pages 2337--2363. PMLR, 2023.

\bibitem[Bhatore et~al.(2020)Bhatore, Mohan, and Reddy]{bhatore2020machine}
Siddharth Bhatore, Lalit Mohan, and Y~Raghu Reddy.
\newblock Machine learning techniques for credit risk evaluation: a systematic literature review.
\newblock \emph{Journal of Banking and Financial Technology}, 4\penalty0 (1):\penalty0 111--138, 2020.

\bibitem[Cesa-Bianchi and Lugosi(2006)]{cesa2006prediction}
Nicolo Cesa-Bianchi and G{\'a}bor Lugosi.
\newblock \emph{Prediction, learning, and games}.
\newblock Cambridge university press, 2006.

\bibitem[Cover(1991)]{cover1991universal}
Thomas~M Cover.
\newblock Universal portfolios.
\newblock \emph{Mathematical finance}, 1\penalty0 (1):\penalty0 1--29, 1991.

\bibitem[Cover and Ordentlich(1996)]{cover1996universal}
Thomas~M Cover and Erik Ordentlich.
\newblock Universal portfolios with side information.
\newblock \emph{IEEE Transactions on Information Theory}, 42\penalty0 (2):\penalty0 348--363, 1996.

\bibitem[Foygel~Barber et~al.(2021)Foygel~Barber, Candes, Ramdas, and Tibshirani]{foygel2021limits}
Rina Foygel~Barber, Emmanuel~J Candes, Aaditya Ramdas, and Ryan~J Tibshirani.
\newblock The limits of distribution-free conditional predictive inference.
\newblock \emph{Information and Inference: A Journal of the IMA}, 10\penalty0 (2):\penalty0 455--482, 2021.

\bibitem[Gibbs and Cand{\`e}s(2021)]{gibbs2021adaptive}
Isaac Gibbs and Emmanuel~J Cand{\`e}s.
\newblock Adaptive conformal inference under distribution shift.
\newblock In \emph{Proceedings of the 35th International Conference on Neural Information Processing Systems}, pages 1660--1672, 2021.

\bibitem[Gibbs and Cand{\`e}s(2024)]{gibbs2024conformal}
Isaac Gibbs and Emmanuel~J Cand{\`e}s.
\newblock Conformal inference for online prediction with arbitrary distribution shifts.
\newblock \emph{Journal of Machine Learning Research}, 25\penalty0 (162):\penalty0 1--36, 2024.

\bibitem[Gibbs et~al.(2025)Gibbs, Cherian, and Cand{\`e}s]{gibbs2025conformal}
Isaac Gibbs, John~J Cherian, and Emmanuel~J Cand{\`e}s.
\newblock Conformal prediction with conditional guarantees.
\newblock \emph{Journal of the Royal Statistical Society Series B: Statistical Methodology}, 87\penalty0 (4):\penalty0 1100--1126, 2025.

\bibitem[Gneiting and Raftery(2007)]{gneiting2007strictly}
Tilmann Gneiting and Adrian~E Raftery.
\newblock Strictly proper scoring rules, prediction, and estimation.
\newblock \emph{Journal of the American statistical Association}, 102\penalty0 (477):\penalty0 359--378, 2007.

\bibitem[Grigorescu et~al.(2020)Grigorescu, Trasnea, Cocias, and Macesanu]{grigorescu2020survey}
Sorin Grigorescu, Bogdan Trasnea, Tiberiu Cocias, and Gigel Macesanu.
\newblock A survey of deep learning techniques for autonomous driving.
\newblock \emph{Journal of field robotics}, 37\penalty0 (3):\penalty0 362--386, 2020.

\bibitem[Gupta et~al.(2022)Gupta, Jung, Noarov, Pai, and Roth]{gupta2022online}
Varun Gupta, Christopher Jung, Georgy Noarov, Mallesh~M Pai, and Aaron Roth.
\newblock Online multivalid learning: Means, moments, and prediction intervals.
\newblock \emph{Innovations in Theoretical Computer Science (ITCS)}, 2022.

\bibitem[Hong et~al.(2022)Hong, Liu, Gao, Han, Chang, Gong, and Su]{hong2022state}
Na~Hong, Chun Liu, Jianwei Gao, Lin Han, Fengxiang Chang, Mengchun Gong, and Longxiang Su.
\newblock State of the art of machine learning--enabled clinical decision support in intensive care units: literature review.
\newblock \emph{JMIR medical informatics}, 10\penalty0 (3):\penalty0 e28781, 2022.

\bibitem[Huber(1992)]{huber1992robust}
Peter~J Huber.
\newblock Robust estimation of a location parameter.
\newblock In \emph{Breakthroughs in statistics: Methodology and distribution}, pages 492--518. Springer, 1992.

\bibitem[Johnson et~al.(2023)Johnson, Bulgarelli, Shen, Gayles, Shammout, Horng, Pollard, Hao, Moody, Gow, et~al.]{johnson2023mimic}
Alistair~EW Johnson, Lucas Bulgarelli, Lu~Shen, Alvin Gayles, Ayad Shammout, Steven Horng, Tom~J Pollard, Sicheng Hao, Benjamin Moody, Brian Gow, et~al.
\newblock Mimic-iv, a freely accessible electronic health record dataset.
\newblock \emph{Scientific data}, 10\penalty0 (1):\penalty0 1, 2023.

\bibitem[Jun et~al.(2017{\natexlab{a}})Jun, Orabona, Wright, and Willett]{jun2017improved}
Kwang-Sung Jun, Francesco Orabona, Stephen Wright, and Rebecca Willett.
\newblock Improved strongly adaptive online learning using coin betting.
\newblock In \emph{Artificial Intelligence and Statistics}, pages 943--951. PMLR, 2017{\natexlab{a}}.

\bibitem[Jun et~al.(2017{\natexlab{b}})Jun, Orabona, Wright, and Willett]{jun2017online}
Kwang-Sung Jun, Francesco Orabona, Stephen Wright, and Rebecca Willett.
\newblock Online learning for changing environments using coin betting.
\newblock \emph{Electronic Journal of Statistics}, 11:\penalty0 5282--5310, 2017{\natexlab{b}}.

\bibitem[Jung et~al.(2021)Jung, Lee, Pai, Roth, and Vohra]{jung2021moment}
Christopher Jung, Changhwa Lee, Mallesh Pai, Aaron Roth, and Rakesh Vohra.
\newblock Moment multicalibration for uncertainty estimation.
\newblock In \emph{Conference on Learning Theory}, pages 2634--2678. PMLR, 2021.

\bibitem[Jung et~al.(2023)Jung, Noarov, Ramalingam, and Roth]{jung2023batch}
Christopher Jung, Georgy Noarov, Ramya Ramalingam, and Aaron Roth.
\newblock Batch multivalid conformal prediction.
\newblock In \emph{International Conference on Learning Representations (ICLR)}, 2023.

\bibitem[Kingma and Ba(2014)]{kingma2014adam}
Diederik~P Kingma and Jimmy Ba.
\newblock Adam: A method for stochastic optimization.
\newblock \emph{arXiv preprint arXiv:1412.6980}, 2014.

\bibitem[Koenker and Hallock(2001)]{koenker2001quantile}
Roger Koenker and Kevin~F Hallock.
\newblock Quantile regression.
\newblock \emph{Journal of economic perspectives}, 15\penalty0 (4):\penalty0 143--156, 2001.

\bibitem[Lei and Wasserman(2014)]{lei2014distribution}
Jing Lei and Larry Wasserman.
\newblock Distribution-free prediction bands for non-parametric regression.
\newblock \emph{Journal of the Royal Statistical Society Series B: Statistical Methodology}, 76\penalty0 (1):\penalty0 71--96, 2014.

\bibitem[Liu et~al.(2026)Liu, Dobriban, and Orabona]{liu2026online}
Tuo Liu, Edgar Dobriban, and Francesco Orabona.
\newblock Online conformal prediction via universal portfolio algorithms.
\newblock \emph{arXiv preprint arXiv:2602.03168}, 2026.

\bibitem[Lyu and Cheung(2026)]{lyu2026normalization}
Lixing Lyu and Wang~Chi Cheung.
\newblock Normalization is enough: An improved algorithm for online group-wise conformal prediction.
\newblock In \emph{ICML 2026 Statistical Frameworks for Uncertainty in Agentic Systems}, 2026.

\bibitem[Nugent(2018)]{Nugent_2018}
Cam Nugent.
\newblock S\&p 500 stock data, Feb 2018.
\newblock URL \url{https://www.kaggle.com/datasets/camnugent/sandp500}.

\bibitem[{NYSE}(2026)]{nyseHolidaysamp}
{NYSE}.
\newblock Holidays \& trading hours, 2026.
\newblock URL \url{https://www.nyse.com/trade/hours-calendars}.

\bibitem[Orabona(2019)]{orabona2019modern}
Francesco Orabona.
\newblock A modern introduction to online learning.
\newblock \emph{arXiv preprint arXiv:1912.13213}, 2019.

\bibitem[Orabona and P{\'a}l(2016)]{orabona2016coin}
Francesco Orabona and D{\'a}vid P{\'a}l.
\newblock Coin betting and parameter-free online learning.
\newblock In \emph{Proceedings of the 30th International Conference on Neural Information Processing Systems}, pages 577--585, 2016.

\bibitem[Orabona and P{\'a}l(2021)]{orabona2021parameter}
Francesco Orabona and D{\'a}vid P{\'a}l.
\newblock Parameter-free stochastic optimization of variationally coherent functions.
\newblock \emph{arXiv preprint arXiv:2102.00236}, 2021.

\bibitem[Papadopoulos et~al.(2002)Papadopoulos, Proedrou, Vovk, and Gammerman]{papadopoulos2002inductive}
Harris Papadopoulos, Kostas Proedrou, Volodya Vovk, and Alex Gammerman.
\newblock Inductive confidence machines for regression.
\newblock In \emph{European conference on machine learning}, pages 345--356. Springer, 2002.

\bibitem[Podkopaev et~al.(2024)Podkopaev, Xu, and Lee]{podkopaev2024adaptive}
Aleksandr Podkopaev, Darren Xu, and Kuang-chih Lee.
\newblock Adaptive conformal inference by betting.
\newblock In \emph{Proceedings of the 41st International Conference on Machine Learning}, pages 40886--40907, 2024.

\bibitem[Ramalingam et~al.(2025)Ramalingam, Kiyani, and Roth]{ramalingam2025relationship}
Ramya Ramalingam, Shayan Kiyani, and Aaron Roth.
\newblock The relationship between no-regret learning and online conformal prediction.
\newblock \emph{arXiv preprint arXiv:2502.10947}, 2025.

\bibitem[Romano et~al.(2020)Romano, Barber, Sabatti, and Cand{\` e}s]{romano2020malice}
Yaniv Romano, Rina~Foygel Barber, Chiara Sabatti, and Emmanuel Cand{\` e}s.
\newblock With {Malice} {Toward} {None}: Assessing {Uncertainty} via {Equalized} {Coverage}.
\newblock \emph{Harvard Data Science Review}, 2\penalty0 (2), apr 30 2020.
\newblock https://hdsr.mitpress.mit.edu/pub/qedrwcz3.

\bibitem[Rozenberg et~al.(2017)Rozenberg, Danish, Dombrovskiy, and Vogel]{rozenberg2017outcomes}
Aleksandr Rozenberg, Timothy Danish, Viktor~Y Dombrovskiy, and Todd~R Vogel.
\newblock Outcomes after motor vehicle trauma: transfers to level i trauma centers compared with direct admissions.
\newblock \emph{The Journal of Emergency Medicine}, 53\penalty0 (3):\penalty0 295--301, 2017.

\bibitem[Saunders et~al.(1999)Saunders, Gammerman, and Vovk]{saunders1999transduction}
C~Saunders, A~Gammerman, and V~Vovk.
\newblock Transduction with confidence and credibility.
\newblock In \emph{Proceedings of the 16th international joint conference on Artificial intelligence-Volume 2}, pages 722--726, 1999.

\bibitem[Shafer and Vovk(2008)]{shafer2008tutorial}
Glenn Shafer and Vladimir Vovk.
\newblock A tutorial on conformal prediction.
\newblock \emph{Journal of machine learning research}, 9\penalty0 (3), 2008.

\bibitem[Taylor and Letham(2018)]{taylor2018forecasting}
Sean~J Taylor and Benjamin Letham.
\newblock Forecasting at scale.
\newblock \emph{The American Statistician}, 72\penalty0 (1):\penalty0 37--45, 2018.

\bibitem[Teasdale and Jennett(1974)]{teasdale1974assessment}
Graham Teasdale and Bryan Jennett.
\newblock Assessment of coma and impaired consciousness: a practical scale.
\newblock \emph{The lancet}, 304\penalty0 (7872):\penalty0 81--84, 1974.

\bibitem[Vincent et~al.(1996)Vincent, Moreno, Takala, Willatts, De~Mendon{\c{c}}a, Bruining, Reinhart, Suter, and Thijs]{vincent1996sofa}
J-L Vincent, Rui Moreno, Jukka Takala, Sheila Willatts, Arnaldo De~Mendon{\c{c}}a, Hajo Bruining, C~Kathryn Reinhart, PeterM Suter, and Lambertius~G Thijs.
\newblock The sofa (sepsis-related organ failure assessment) score to describe organ dysfunction/failure: On behalf of the working group on sepsis-related problems of the european society of intensive care medicine (see contributors to the project in the appendix).
\newblock \emph{Intensive care medicine}, 22\penalty0 (7):\penalty0 707--710, 1996.

\bibitem[Vovk et~al.(2003)Vovk, Lindsay, Nouretdinov, Gammerman, et~al.]{vovk2003mondrian}
Vladimir Vovk, David Lindsay, Ilia Nouretdinov, Alex Gammerman, et~al.
\newblock Mondrian confidence machine.
\newblock \emph{Technical Report}, 2003.

\bibitem[Vovk et~al.(2005)Vovk, Gammerman, and Shafer]{vovk2005algorithmic}
Vladimir Vovk, Alexander Gammerman, and Glenn Shafer.
\newblock \emph{Algorithmic learning in a random world}.
\newblock Springer, 2005.

\bibitem[Yang et~al.(2024)Yang, Cand{\`e}s, and Lei]{yang2024bellman}
Zitong Yang, Emmanuel Cand{\`e}s, and Lihua Lei.
\newblock Bellman conformal inference: Calibrating prediction intervals for time series.
\newblock \emph{arXiv preprint arXiv:2402.05203}, 2024.

\bibitem[Yu and Neely(2020)]{yu2020low}
Hao Yu and Michael~J Neely.
\newblock A low complexity algorithm with o ($\sqrt{}$ t) regret and o (1) constraint violations for online convex optimization with long term constraints.
\newblock \emph{Journal of Machine Learning Research}, 21\penalty0 (1):\penalty0 1--24, 2020.

\bibitem[Zaffran et~al.(2022)Zaffran, F{\'e}ron, Goude, Josse, and Dieuleveut]{zaffran2022adaptive}
Margaux Zaffran, Olivier F{\'e}ron, Yannig Goude, Julie Josse, and Aymeric Dieuleveut.
\newblock Adaptive conformal predictions for time series.
\newblock In \emph{ICML 2022-39th International Conference on Machine Learning}, volume 162, 2022.

\bibitem[Zhang et~al.(2024)Zhang, Bombara, and Yang]{zhang2024discounted}
Zhiyu Zhang, David Bombara, and Heng Yang.
\newblock Discounted adaptive online learning: towards better regularization.
\newblock In \emph{Proceedings of the 41st International Conference on Machine Learning}, pages 58631--58661, 2024.

\end{thebibliography}

\appendix
\section{Proofs \label{app:proofs}}
\subsection{Proof of \cref{theorem:CW_group_cov} \label{app:CW_group_cov_proof}}

\coordcov*
\begin{proof}
Recall $\bc_t = (c_{t,1}, \dots, c_{t,k}) \in [0,1]^k$ and the subgradient vector $\bg_t$ at time $t$ is defined as 
\begin{align}
    \bg_t = Z_t\bc_t, \quad \text{where} \quad Z_t = \mathbf{1}\{S_t \leq \langle \theta_t, \bc_t\rangle\} - (1-\alpha).
\end{align} 
Let $g_{t,j}$ denote the $j^{\text{th}}$ entry of the subgradient vector $\bg_t$, which, by definition above, is $g_{t,j} = c_{t,j} Z_t$. Define $w_{t,j,1} = 1 - \frac{g_{t,j}}{\alpha}$ and $w_{t,j,2} = 1 + \frac{g_{t,j}}{1-\alpha}$ to be the returns of two stocks and the $j^{th}$ wealth process
\begin{align}
    W_{t, j} = W_{t-1, j}\left(\lambda_{t,j}w_{t,j,1} + (1-\lambda_{t,j})w_{t,j,2}\right) \quad \text{with} \quad W_{0,j} = \frac{1}{k}.
\end{align}
where $(\lambda_{t,j})^T_{t=1}$ is a sequence of portfolios in $[0,1]$. By this recursive definition of $W_{t,j}$, it is clear that $W_{T,j} = \frac{1}{k}\prod^T_{t=1}\left(\lambda_{t,j}w_{t,j,1} + (1-\lambda_{t,j})w_{t,j,2}\right)$. As a result,
\begin{align}
    \ln(W_{T,j}) &= \ln\left(\frac{1}{k}\prod^T_{t=1}\left(\lambda_{t,j}w_{t,j,1} + (1-\lambda_{t,j})w_{t,j,2}\right)\right)\\
    &= \ln\left(\prod^T_{t=1}\left(\lambda_{t,j}w_{t,j,1} + (1-\lambda_{t,j})w_{t,j,2}\right)\right) + \ln\left(\frac{1}{k}\right).
\end{align}
Since the $\lambda_{t,j}$ are determined by Universal Portfolio \citep{cover1991universal, cover1996universal} with Jeffrey's prior, we have 
\begin{align}
     \ln\left(\prod^T_{t=1}\left(\lambda_{t,j}w_{t,j,1} + (1-\lambda_{t,j})w_{t,j,2}\right)\right) &\geq
     \max_{\lambda \in [0,1]} \ln\left(\prod^T_{s=1}\left(\lambda w_{t,j,1} + (1-\lambda) w_{t,j,2}\right)\right)& \\
     &\qquad \qquad \qquad \qquad - \frac{1}{2}\ln(\pi(T+1)).
\end{align}
By the definitions of $w_{t,j,1}$ and $w_{t,j,2}$, along with noting that 
$g_{t,j} = c_{t,j} Z_t$, where $c_{t,j} \in [0,1]$ and $Z_t \in \{\alpha - 1, \alpha\}$, by \cref{lemma:opt_wealth_lb} 
\begin{align}
    \max_{\lambda \in [0,1]} \ln\left(\prod^T_{t=1}\left(\lambda w_{t,j,1} + (1-\lambda)w_{t,j,2}\right)\right) \geq \frac{\left(\sum^T_{t=1} g_{t,j}\right)^2}{{2}T_j\alpha(1-\alpha) + {\frac{2}{3}}\left|\sum^T_{t=1} g_{t,j}\right|},
\end{align}
where $T_j = \sum_{t = 1}^{T} c_{t, j}$. Therefore,
\begin{align}
    \ln(W_{T,j}) &\geq  \frac{\left(\sum^T_{t=1} g_{t,j}\right)^2}{{2}T_j\alpha(1-\alpha) + {\frac{2}{3}}\left|\sum^T_{t=1} g_{t,j}\right|} + \ln\left(\frac{1}{k}\right) - \frac{1}{2}\ln(\pi(T+1)).
\end{align}
Applying $\exp(\cdot)$ to both sides and summing the $k$ wealth processes $(W_{t,1})^T_{t=1}, \dots (W_{t,k})^T_{t=1}$, gives
\begin{align}
    \sum^k_{j=1} W_{T,j} \geq \sum^k_{j=1}\frac{1}{k\sqrt{\pi}(T+1)^{1/2}}\exp\left(\frac{\left(\sum^T_{t=1} g_{t,j}\right)^2}{{2}T_j\alpha(1-\alpha) + {\frac{2}{3}}\left|\sum^T_{t=1} g_{t,j}\right|}\right).
\end{align}
Thus we have a lower bound on $\sum^k_{j=1} W_{T,j}$.

Now, we have that for any $j \in [k]$, by the definitions of $w_{t,j,1}$ and $w_{t,j,2}$,
\begin{align}
    \lambda_{t,j}w_{t,j,1} + (1-\lambda_{t,j})w_{t,j,2} &= \lambda_{t,j} - \frac{\lambda_{t,j}}{\alpha}g_{t,j} + 1-\lambda_{t,j} + \frac{1-\lambda_{t,j}}{1-\alpha}g_{t,j}\\
    &= 1 - \left(\frac{\lambda_{t,j}}{\alpha} - \frac{1-\lambda_{t,j}}{1-\alpha}\right)g_{t,j}\\
    &= 1 - \left(\frac{\lambda_{t,j} - \alpha}{\alpha(1-\alpha)}\right)g_{t,j}.
\end{align}
By defining $\theta_{t,j} = W_{t-1,j}\left(\frac{\lambda_{t,j} - \alpha}{\alpha(1-\alpha)}\right)$, we have
\begin{align}
    W_{T,j} 
    &= W_{T-1,j}( \lambda_{t,j}w_{t,j,1} + (1-\lambda_{t,j})w_{t,j,2})\\
    &= W_{T-1,j}\left(1 - \left(\frac{\lambda_{t,j} - \alpha}{\alpha(1-\alpha)}\right)g_{t,j}\right) \\
    &= W_{T-1,j} - \theta_{T,j}g_{T,j}\\
    &= W_{T-2,j} - \theta_{T-1,j} g_{T-1,j} - \theta_{T,j} g_{T,j} \\
    &= W_{0,j} - \sum^T_{t=1}\theta_{t,j}g_{t,j} = \frac{1}{k} - \sum^T_{t=1} \theta_{t,j}g_{t,j}.
\end{align}
Summing the individual $W_{T,j}$, we obtain
\begin{align}
    \sum^k_{j=1}W_{T,j} &= \sum^k_{j=1} \left(\frac{1}{k} - \sum^T_{t=1} \theta_{t,j}g_{t,j} \right)\\
    &= 1 - \sum^k_{j=1} \left(\sum^T_{t=1} \theta_{t,j}g_{t,j} \right)\\
    &= 1 - \sum^T_{t=1} \sum^k_{j=1}\theta_{t,j}g_{t,j} =  1 - \sum^T_{t=1}\langle \bm{\theta}_t , \bg_t\rangle,
\end{align}
where we define $(\bm{\theta}_t = (\theta_{t,1}, \dots \theta_{t,k}))$. By \cref{lemma:wealth_ub}, and the assumption $S_t \leq Dt^q$ for $D > 0$ and $q \geq 0$, we obtain
\begin{align}
    1 - \sum^T_{t=1}\langle \theta_t , \bg_t\rangle \leq 1 + (1-\alpha)\sum^T_{t=1} S_t \leq 1 + (1-\alpha)\frac{D(T+1)^{q+1}}{q+1},
\end{align}
where we have used the integral inequality $\sum_{t=1}^{T} t^q \cdot 1 \leq \int_{0}^{T+1} t^q dt = \frac{(T+1)^{q+1}}{q + 1}$.
Therefore, we have an upper bound for $\sum^k_{j=1}W_{T,j}$.

Comparing the upper and the lower bounds, 
\begin{align}
    \sum^k_{j=1}\frac{1}{k\sqrt{\pi}(T+1)^{1/2}}\exp\left(\frac{\left(\sum^T_{t=1} g_{t,j}\right)^2}{{2}T_j\alpha(1-\alpha) + {\frac{2}{3}}\left|\sum^T_{t=1} g_{t,j}\right|}\right) \leq 1 + (1-\alpha)\frac{D(T+1)^{q+1}}{q+1}
\end{align}
Every term in the left hand side equation is non-negative. Therefore, for any $j \in [k]$,
\begin{align}
    \label{eq:eq1}
    \frac{1}{k\sqrt{\pi}(T+1)^{1/2}}\exp\left(\frac{\left(\sum^T_{t=1} g_{t,j}\right)^2}{{2}T_j\alpha(1-\alpha) + {\frac{2}{3}}\left|\sum^T_{t=1} g_{t,j}\right|}\right) \leq 1 + (1-\alpha)\frac{D(T+1)^{q+1}}{q+1}.
\end{align}
Multiplying both sides of \eqref{eq:eq1} by $k\sqrt{\pi}(T+1)^{1/2}$ and then applying $\ln(\cdot)$ to both sides, we obtain
\begin{align}
\frac{\left(\sum_{t=1}^T g_{t,j}\right)^2}{2T_j\alpha(1-\alpha) + \frac{2}{3}\left|\sum_{t=1}^T g_{t,j}\right|} \leq U_T(k),
\end{align}
where we define
\begin{align}
    U_T(k) \coloneqq \ln\left(1 + (1-\alpha)\frac{D(T+1)^{q+1}}{q+1}\right) + \frac{1}{2}\ln(\pi(T+1)) + \ln(k).
\end{align}
Equivalently,
\begin{align}
\left(\sum_{t=1}^T g_{t,j}\right)^2 - \frac{2}{3}U_T(k)\left|\sum_{t=1}^T g_{t,j}\right| - 2U_T(k)T_j\alpha(1-\alpha) \leq 0.
\end{align}
The inequality above is of the form $z^2 + Q|z| + R \leq 0$, and the LHS might be non-convex depending on $Q, R$. Fortunately, we only require an upper bound on $|z|$, and we can look at the part of the curve in $z \geq 0$, which is a quadratic, and remains non-positive in $0 \leq z \leq z^\star$, where $z^\star = \frac{-Q + \sqrt{Q^2 - 4R}}{2}$.
\begin{align}
\left|\sum_{t=1}^T g_{t,j}\right|
&\leq \frac{1}{3}U_T(k) + \sqrt{2U_T(k)T_j\alpha(1-\alpha) + \frac{1}{9}U_T(k)^2} \\
&\leq U_T(k) + \sqrt{2T_j\alpha(1-\alpha)U_T(k)},
\end{align}
where we have used the inequality $\sqrt{Q + R} \leq \sqrt{Q} + \sqrt{R}$, with $Q = 2T_j\alpha(1-\alpha)U_T(k)$, $R = \frac{1}{9}U_T(k)^2$. Finally, dividing both sides by $T_j$ obtains the desired result.
\end{proof}

\section{Useful Lemmas and Inequalities \label{app:useful_lemmas}}

In this section, we present  results that are used throughout the proofs in \cref{app:proofs}. For cited results, their proofs can be found in the referenced works; otherwise, proofs are provided here.

\subsection{Lower bound on the optimal wealth process}
\begin{lemma} 
\label{lemma:opt_wealth_lb}
Let $\alpha \in (0,1)$. Consider any integer $T \geq 1$ and a sequence $(c_tZ_t)^T_{t=1}$ where $c_t \in [0,1]$ and $Z_t \in \{\alpha-1, \alpha\}$. Define $T_c = \sum^T_{t=1} c_t$ and assume $T_c > 0$. Then
\begin{align}
    \max_{\lambda  \in [0,1]} \ln \left(\prod^T_{t=1} \lambda \left(1 - \frac{c_tZ_t}{\alpha}\right) + (1-\lambda)\left(1 + \frac{c_tZ_t}{1-\alpha}\right)\right) \geq \frac{\left(\sum^T_{t=1} c_tZ_t\right)^2}{2T_c\alpha(1-\alpha) + \frac{2}{3}\left| \sum^T_{t=1} c_tZ_t\right|}.
\end{align}
\end{lemma}
\begin{proof}
First, observe
\begin{align}
    \lambda \left(1 - \frac{c_tZ_t}{\alpha}\right) + (1-\lambda)\left(1 + \frac{c_tZ_t}{1-\alpha}\right) &= \lambda  - \frac{\lambda}{\alpha}c_tZ_t + 1 - \lambda + \frac{(1-\lambda)}{1-\alpha}c_tZ_t\\
    &= 1 - \left(\frac{\lambda}{\alpha} - \frac{1-\lambda}{1-\alpha}\right)c_tZ_t\\
    &= 1 - \left(\frac{\lambda-\alpha}{\alpha(1-\alpha)}\right)c_tZ_t.
\end{align}
Then, rearrange to obtain
\begin{align}
    1 - \frac{\lambda - \alpha}{\alpha(1 - \alpha)} c_t Z_t
    &= 1 - c_t + c_t - \frac{\lambda - \alpha}{\alpha(1 - \alpha)} c_t Z_t \\ 
    &= 1 - c_t + \left(1 + \frac{\lambda - \alpha}{\alpha (\alpha - 1)} Z_t \right) \cdot c_t \\
    &= 1 - c_t + u_t \cdot c_t,
\end{align}
where $u_t = \frac{\lambda}{\alpha}$ when $Z_t = \alpha - 1$, and $u_t = \frac{1-\lambda}{1-\alpha}$ when $Z_t = \alpha$.

Since $c_t \in [0,1]$, by the concavity of $\ln(x)$,
\begin{align}
    \ln((1 - c_t)(1) + c_t u_t) &\geq (1-c_t)\ln(1) + c_t\ln(u_t) = c_t\ln(u_t).
\end{align}
Therefore, for any $\lambda \in [0,1]$
\begin{align}
    \max_{\lambda  \in [0,1]} &\ln \left(\prod^T_{t=1} \lambda \left(1 - \frac{c_tZ_t}{\alpha}\right) + (1-\lambda)\left(1 + \frac{c_tZ_t}{1-\alpha}\right)\right)\\
     &\geq \ln \left(\prod^T_{t=1} \lambda \left(1 - \frac{c_tZ_t}{\alpha}\right) + (1-\lambda)\left(1 + \frac{c_tZ_t}{1-\alpha}\right)\right)\\
    &= \sum^T_{t=1}\ln \left( 1 - c_t + c_t u_t\right) \\
    & \geq \sum^T_{t=1}c_t\ln(u_t) \\
    &= \sum_{t: Z_t = \alpha }c_t\ln(u_t) + \sum_{t: Z_t = \alpha-1 }c_t\ln(u_t)\\
    &=  \ln\left(\frac{1-\lambda}{1-\alpha}\right) T_1 + \ln\left(\frac{\lambda}{\alpha}\right) T_2, \label{eq:lnlowerbd}
\end{align}
where $T_1 = \sum_{\{t \colon Z_t = \alpha\}} c_t$ and $T_2 = \sum_{\{t \colon Z_t=\alpha-1\}} c_t$. Recall that $T_c = T_1 + T_2$.

The lower bound in \eqref{eq:lnlowerbd} holds for any $\lambda \in [0,1]$: we instantiate it for $\lambda^*$, which is the maximizer of \eqref{eq:lnlowerbd}. To maximize \eqref{eq:lnlowerbd}, set the partial derivative with respect to $\lambda$ to $0$, to obtain that 
\begin{align}
    \lambda^\star = \frac{T_2}{T_c} \in [0,1]. \label{eq:lmbstar}
\end{align}
Now $T_1 = (1 - \lambda^\star) T_c$ and $T_2 = \lambda^\star T_c$. Substituting this,  \eqref{eq:lnlowerbd} can be rewritten as
\begin{align}
     T_c\left((1-\lambda^*) \ln \left(\frac{1-\lambda^*}{1-\alpha}\right) + \lambda^*\ln \left(\frac{\lambda^*}{\alpha}\right)\right)
     = T_c\cdot\text{KL}\left(\text{Ber}(\lambda^*) ||\text{Ber}(\alpha)\right).
\end{align}
Further, we can express $\lambda^\star$ as a function of $\alpha$ and $\sum c_t Z_t$ using \eqref{eq:lmbstar}:
\begin{align}
    \sum_{t = 1}^T c_t Z_t = \alpha T_1 + (\alpha - 1) T_2 = \alpha T_c - T_2 = \alpha T_c - \lambda^\star T_c.
\end{align}
Thus
\begin{align}
T_c\cdot\text{KL}\left(\text{Ber}(\lambda^*) ~|| \text{Ber}(\alpha)\right) = T_c\cdot\text{KL}\left(\text{Ber}\left(\alpha - \frac{\sum^T_{t=1} c_t Z_t}{T_c}\right) \bigg|\bigg| ~\text{Ber}(\alpha)\right)
\end{align}
By \cref{lemma:kl_lb}
\begin{align}
    T_c \cdot \text{KL}\left(\text{Ber}\left(\alpha - \frac{\sum^T_{t=1} c_t Z_t}{T_c}\right)|| \text{Ber}(\alpha)\right) &\geq T_c \frac{\left(\frac{\sum^T_{t=1} c_t Z_t}{T_c}\right)^2}{2\alpha(1-\alpha) + \frac{2}{3}\left|\frac{\sum^T_{t=1} c_t Z_t}{T_c}\right|}\\
    &= \frac{\left(\sum^T_{t=1} c_t Z_t\right)^2}{2T_c\alpha(1-\alpha) + \frac{2}{3}\left|\sum^T_{t=1} c_t Z_t\right|}
\end{align}
Thus we have shown the result
\begin{align}
     \max_{\lambda  \in [0,1]} \ln \left(\prod^T_{t=1} \lambda \left(1 - \frac{c_tZ_t}{\alpha}\right) + (1-\lambda)\left(1 + \frac{c_tZ_t}{1-\alpha}\right)\right) 
    &\geq \frac{\left(\sum^T_{t=1} c_t Z_t\right)^2}{2T_c\alpha(1-\alpha) + \frac{2}{3}\left|\sum^T_{t=1} c_t Z_t\right|}
\end{align}
\end{proof}

\subsection{Upper bound on the wealth process}
\begin{lemma} 
\label{lemma:wealth_ub} 
Let $\alpha \in (0,1)$. Consider any integer $T \geq 1$ and let $(S_t)^T_{t=1}$ and $(\bc_t)^T_{t=1}$ be a sequence of non-conformity scores and group membership vectors. At every time $t$ consider a vector $\bm{\theta}_t \in \bR^k$ and define
\begin{align}
    \bg_t = [{\bf 1}\{S_t \leq \langle \theta_t, \bc_t \rangle \} - (1-\alpha)]\bc_t.
\end{align}
Then 
\begin{align}
    1 - \sum^T_{t=1} \langle \bm{\theta}_t, \bg_t \rangle \leq 1 + (1-\alpha)\sum^T_{t=1} S_t.
\end{align}
\end{lemma}
\begin{proof}
It suffices to show
\begin{align}
    - \langle \bm{\theta}, \bg_t \rangle \leq (1-\alpha)S_t, \quad \forall t \geq 1
\end{align}
because, if true, then $-\sum^T_{t=1}\langle \bm{\theta}, \bg_t \rangle \leq  \sum^T_{t=1}(1-\alpha)S_t$, and the result is immediate.

By definition of $\bg_t$, we have
\begin{align}
    \langle \theta_t, \bg_t \rangle = \langle \theta_t, \bc_t \rangle\cdot [\mathbf{1}\{S_t \leq \langle \theta_t, \bc_t\rangle\} - (1-\alpha)].
\end{align}
The range of $\langle \theta_t, \bg_t \rangle$ is $[-\infty, \infty]$. We will analyze the three exhaustive cases of where this quantity is less than $0$, in between $0$ and $S_t$, and greater than zero. We will show that no matter the scenario, the inequality $\langle \bm{\theta}, \bg_t \rangle \leq (1-\alpha)S_t$ holds  $\forall t \geq 1$.

\paragraph{Case 1:} Suppose $\langle \theta_t, \bc_t \rangle < 0$. Then, since $S_t \geq 0$ because it is a non-conformity score, we have
\begin{align}
    [\mathbf{1}\{S_t \leq \langle \theta_t, \bc_t\rangle\} - (1-\alpha)] = -(1-\alpha).
\end{align}
Thus, $-\langle \theta_t, \bg_t \rangle = \langle \theta_t, \bc_t \rangle(1-\alpha) \leq 0 \leq (1-\alpha)S_t$.

\paragraph{Case 2:} Suppose $0 \leq \langle \theta_t, \bc_t \rangle < S_t$. Then, 
\begin{align}
    [\mathbf{1}\{S_t \leq \langle \theta_t, \bc_t\rangle\} - (1-\alpha)] = -(1-\alpha)
\end{align}
and, as a result, $-\langle \theta_t, \bg_t \rangle = \langle \theta_t, \bc_t \rangle(1-\alpha) \leq (1-\alpha)S_t$.

\paragraph{Case 3:} Suppose $\langle \theta_t, \bc_t \rangle \geq  S_t$. Then,
\begin{align}
    [\mathbf{1}\{S_t \leq \langle \theta_t, \bc_t\rangle\} - (1-\alpha)] = \alpha.
\end{align}
Thus,  $-\langle \theta_t, \bg_t \rangle = -\langle \theta_t, \bc_t \rangle\alpha \leq 0 \leq (1-\alpha)S_t$.
\end{proof}

\subsection{A lower bound on the KL divergence between two Bernoulli random variables}
\begin{lemma}[\citep{liu2026online} Lemma F.1] 
\label{lemma:kl_lb}
Let $p, q \in [0,1]$. Then
\begin{align}
    \text{\emph{KL}}(\text{\emph{Ber}}(p)~||\text{\emph{Ber}}(q)) = p\ln \frac{p}{q} + (1-p)\ln \frac{1-p}{1-q} \geq \frac{(p-q)^2}{2q(1-q) + \frac{2}{3}|p-q|}.
\end{align}  
\end{lemma}
\begin{proof} See \citet{liu2026online} Appendix F.1.
\end{proof}

\newpage
\section{\label{supp:experiments}Additional Experimental Details and Results}
\setcounter{table}{0}

\subsection{\label{supp:syn_experiment_details}Additional synthetic experiment details}
In the synthetic experiments presented in \cref{sec:syn_experiments}, we directly generate non-conformity scores. We model the scores as
\begin{align}
S_t
&= \max\{0, S^{\mathrm{base}}_t
+ \left\langle \mathbf{b}(t), \bc(X_t)\right\rangle \}.
\end{align}
Here, $S^{\mathrm{base}}_t \sim \mathrm{Beta}(1,20)$ and 
the vector $\mathbf b(t)= (b_1(t), \dots, b_k(t))$ introduces time-varying, group-dependent shifts. Specifically, $\mathbf b(t) $ is modeled as a first-order autoregressive (AR(1)) process with bounded uniform noise:
\begin{align}
\label{eq:bias_dynamics}
b_j(t)
&=
(1-\gamma_j)\,b_j(t-1)
+ \gamma_j\,m_j
+ \mu_j
+ \sigma_j\,\xi_{t,j},
\qquad b_j(0)=0, ~~ \xi_{t,j}\overset{iid}{\sim}\mathrm{Unif}(-1,1),
\end{align}
At a high level, \(\mu_j\) controls the overall drift, \(\sigma_j\) controls the volatility, and  $\gamma_j$ governs persistence. When $\gamma_j=0$, \eqref{eq:bias_dynamics} reduces to a random walk, whereas $\gamma_j>0$ yields a shift that stabilizes toward $m_j$. With this model, scores belonging to different groups exhibit different time-varying shifts.

In the experiments, we set $\gamma_j=0.1$, \(m_j=0.2\), $\mu_j=10^{-3}$, and \(\sigma_j=10^{-2}\), yielding stochastic shifts that start at 0 and fluctuate around 0.2. Time-varying biases are applied \emph{only} to the under-represented groups: for $j\in\calU$, $b_j(t)$ follows ~\eqref{eq:bias_dynamics}$,$ whereas for $j\notin\calU$, $b_j(t)=\sigma_j\xi_{t,j}$ $\forall t$.
\newpage

\subsection{\label{supp:additional_synthetic_results}Additional synthetic experiment results}

{\bf Results for $k=25$ number of groups.}
\begin{figure}[h!]
    \begin{subfigure}[t]{\linewidth}
    \centering
        \begin{subfigure}[t]{0.31\linewidth}
        \raisebox{0.12em}{
            \includegraphics[width=\linewidth]{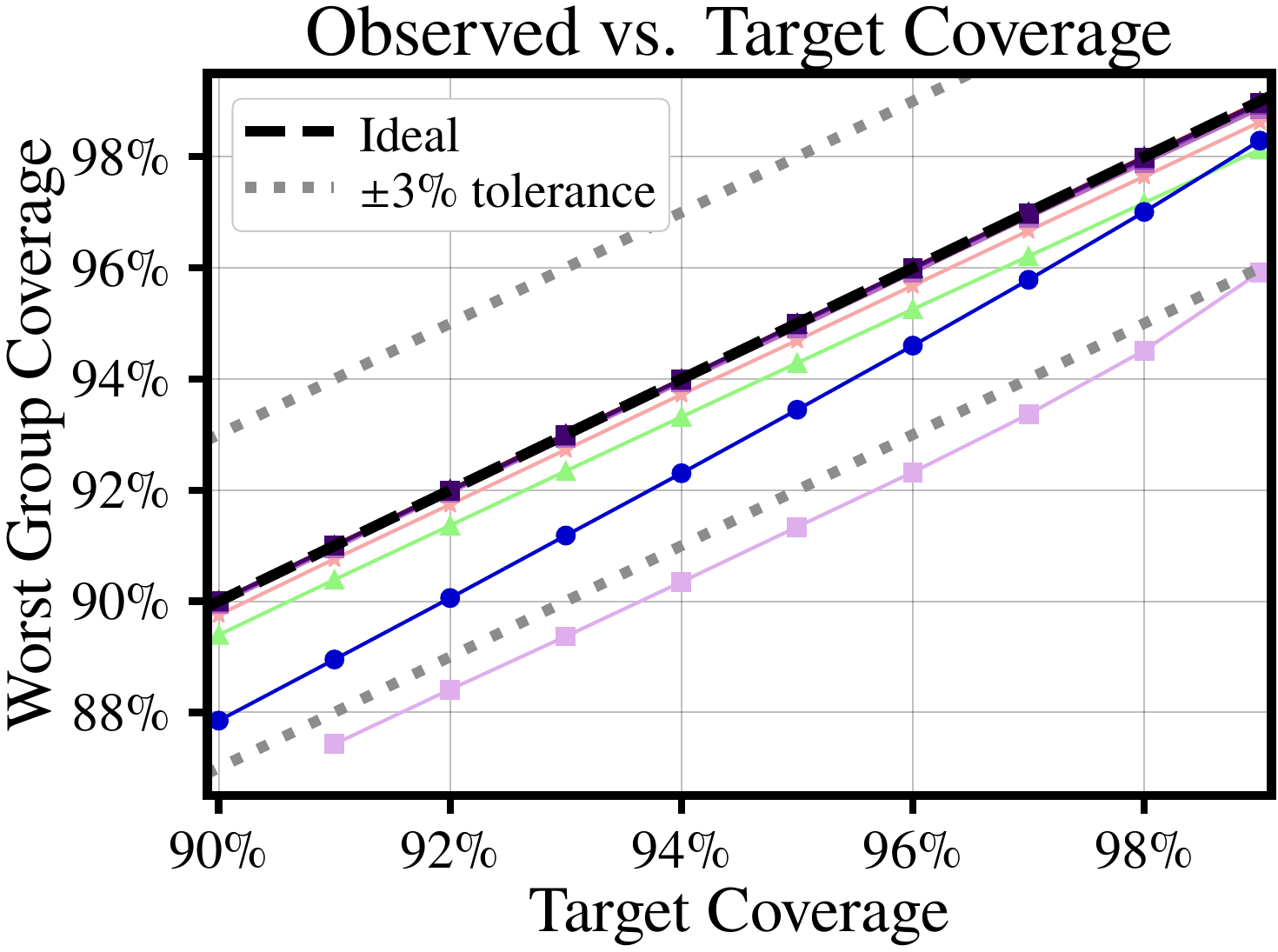}
            }
        \end{subfigure}
        \hspace{0.2em}
        \begin{subfigure}[t]{0.31\linewidth}
            \includegraphics[width=\linewidth]{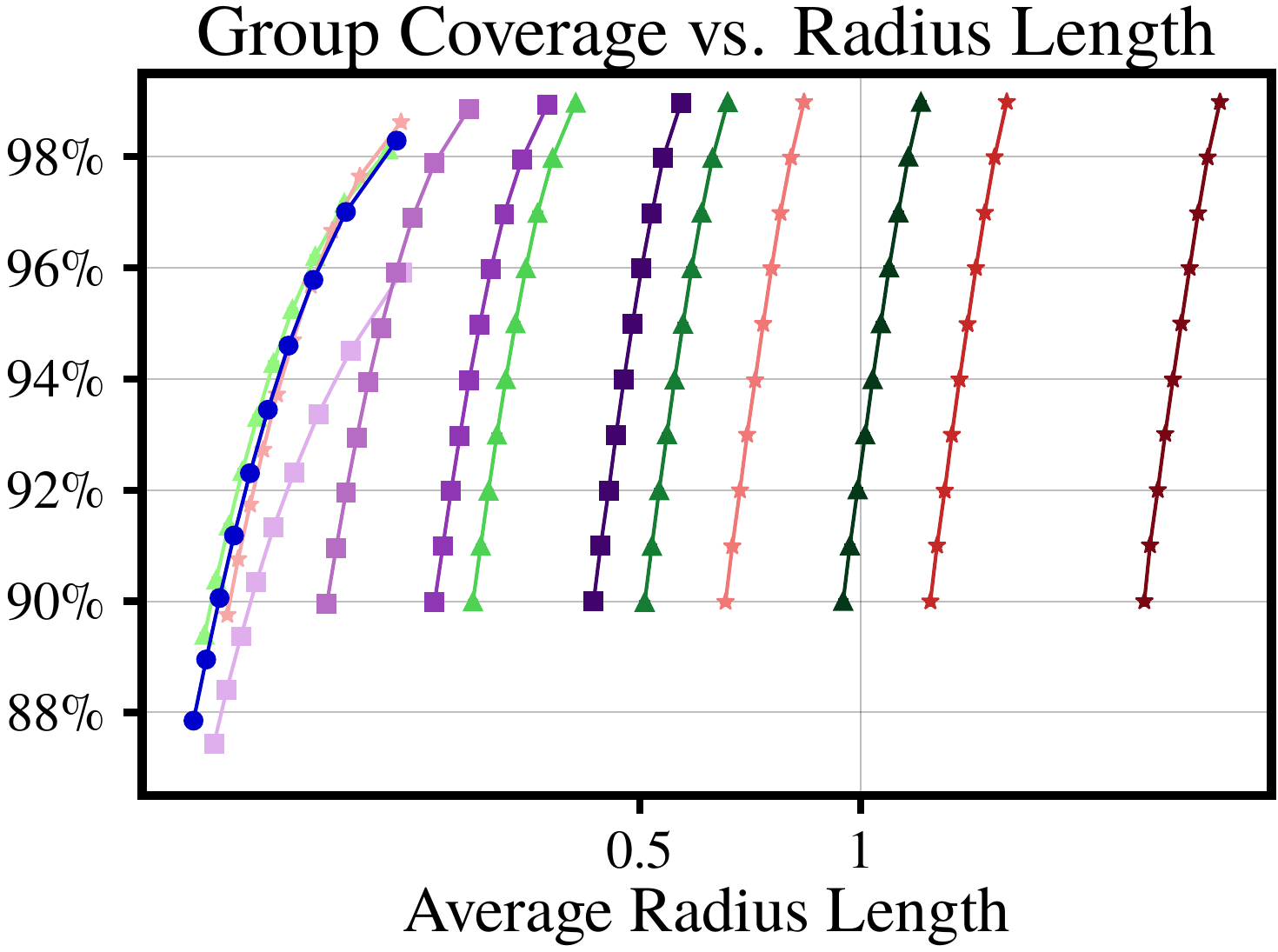}
        \end{subfigure}
        \hspace{0.2em}
        \begin{subfigure}[t]{0.31\linewidth}
            \includegraphics[width=\linewidth]{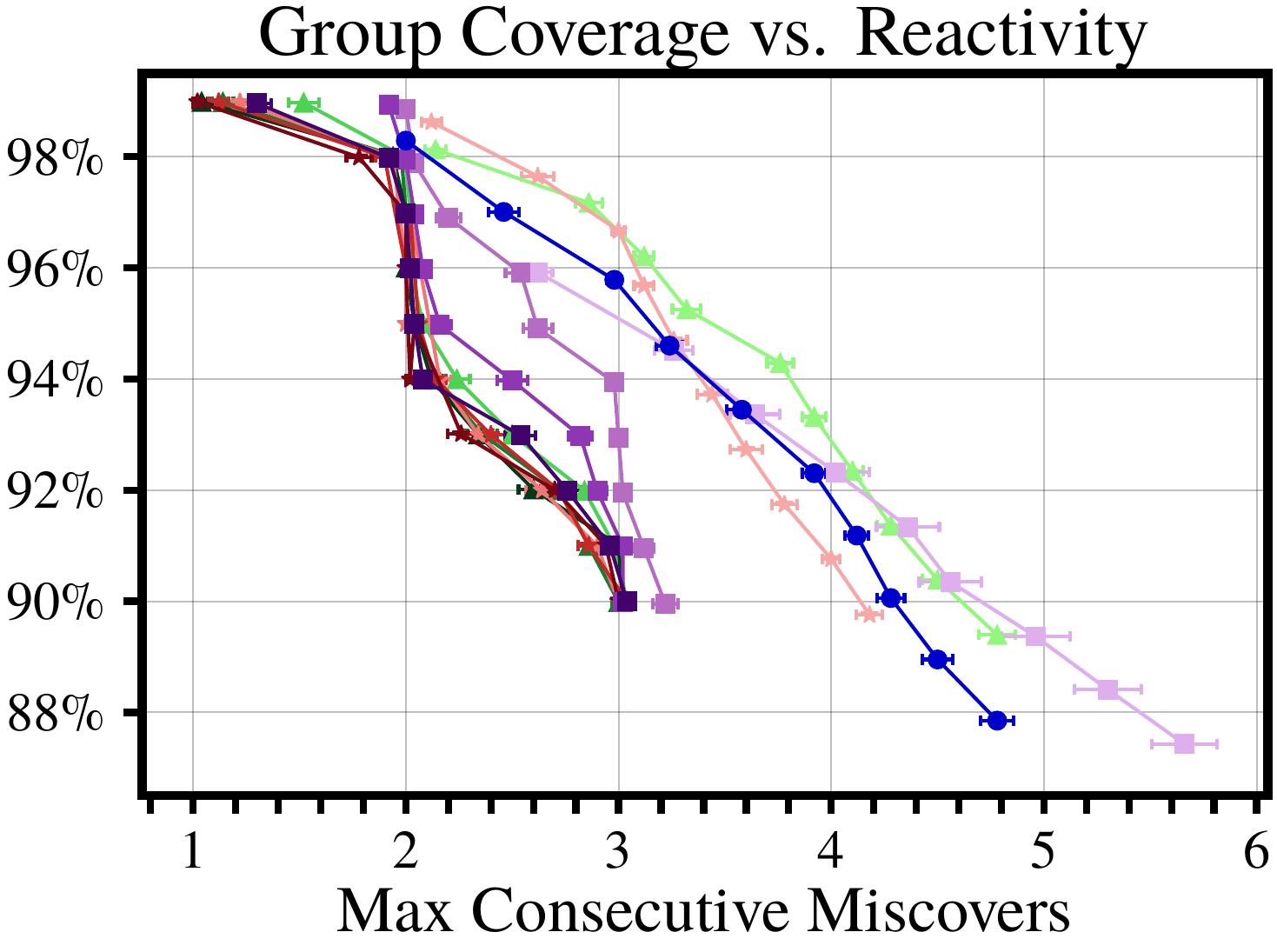}
        \end{subfigure}
        \caption{Results for synthetic setting with bounded, gradually varying scores.}
    \end{subfigure}
    \vspace{0.5em}
        
    \begin{subfigure}[t]{\linewidth}
    \centering
        \begin{subfigure}[t]{0.31\linewidth}
        \raisebox{0.05em}{
            \includegraphics[width=\linewidth]{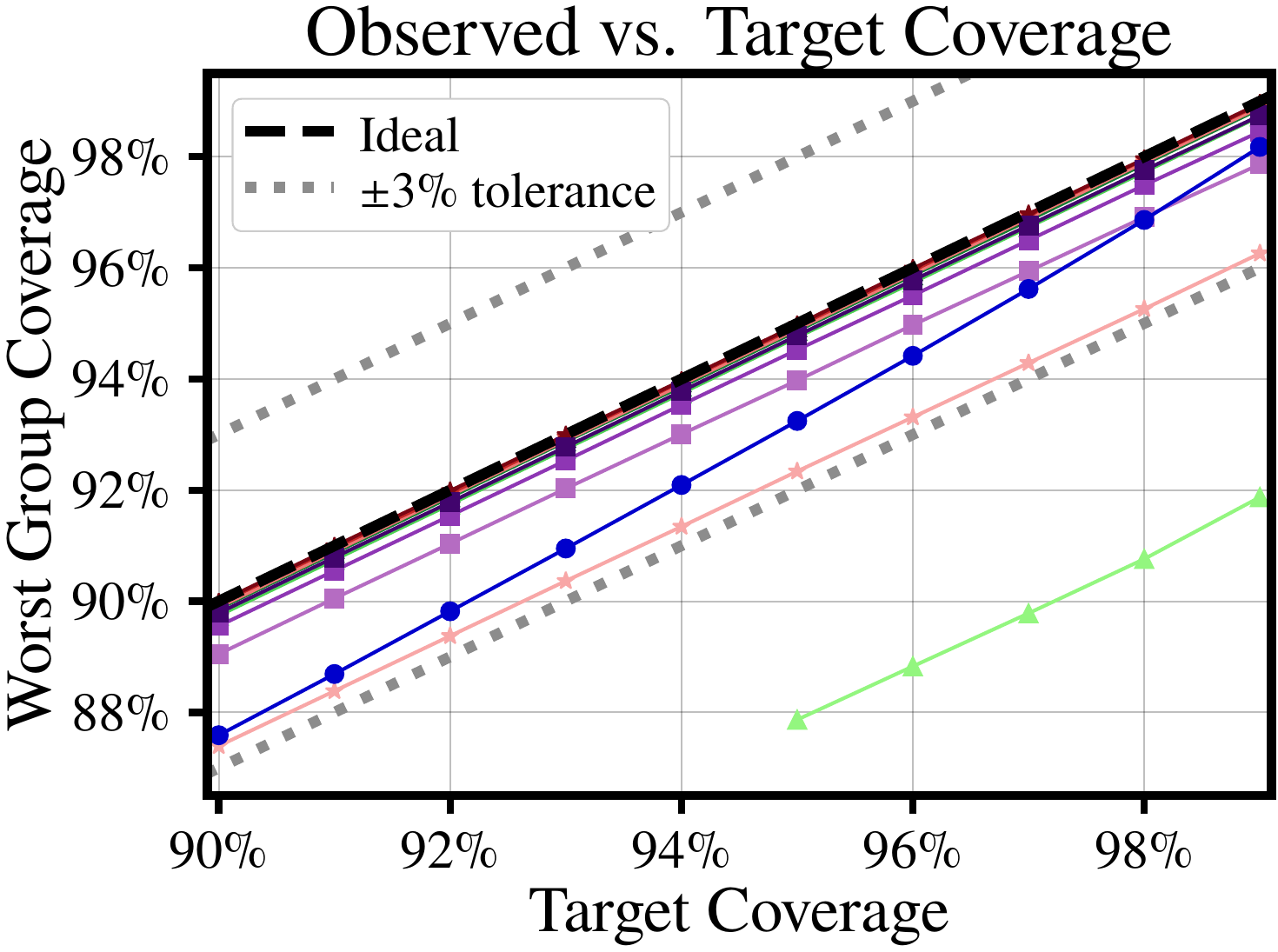}
            }
        \end{subfigure}
        \hspace{0.3em}
        \begin{subfigure}[t]{0.31\linewidth}
            \includegraphics[width=\linewidth]{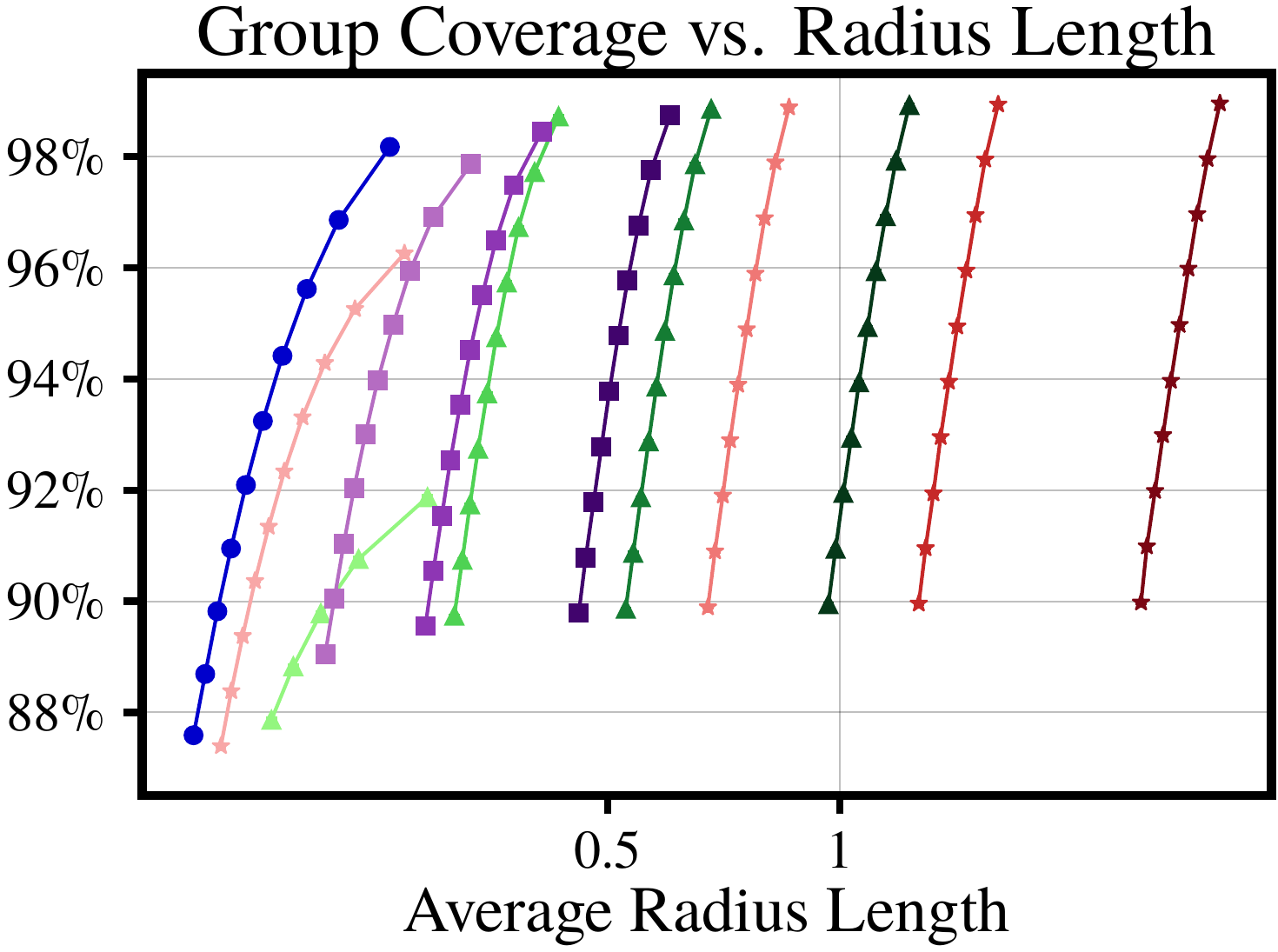}
        \end{subfigure}
        \hspace{0.2em}
        \begin{subfigure}[t]{0.31\linewidth}
            \includegraphics[width=\linewidth]{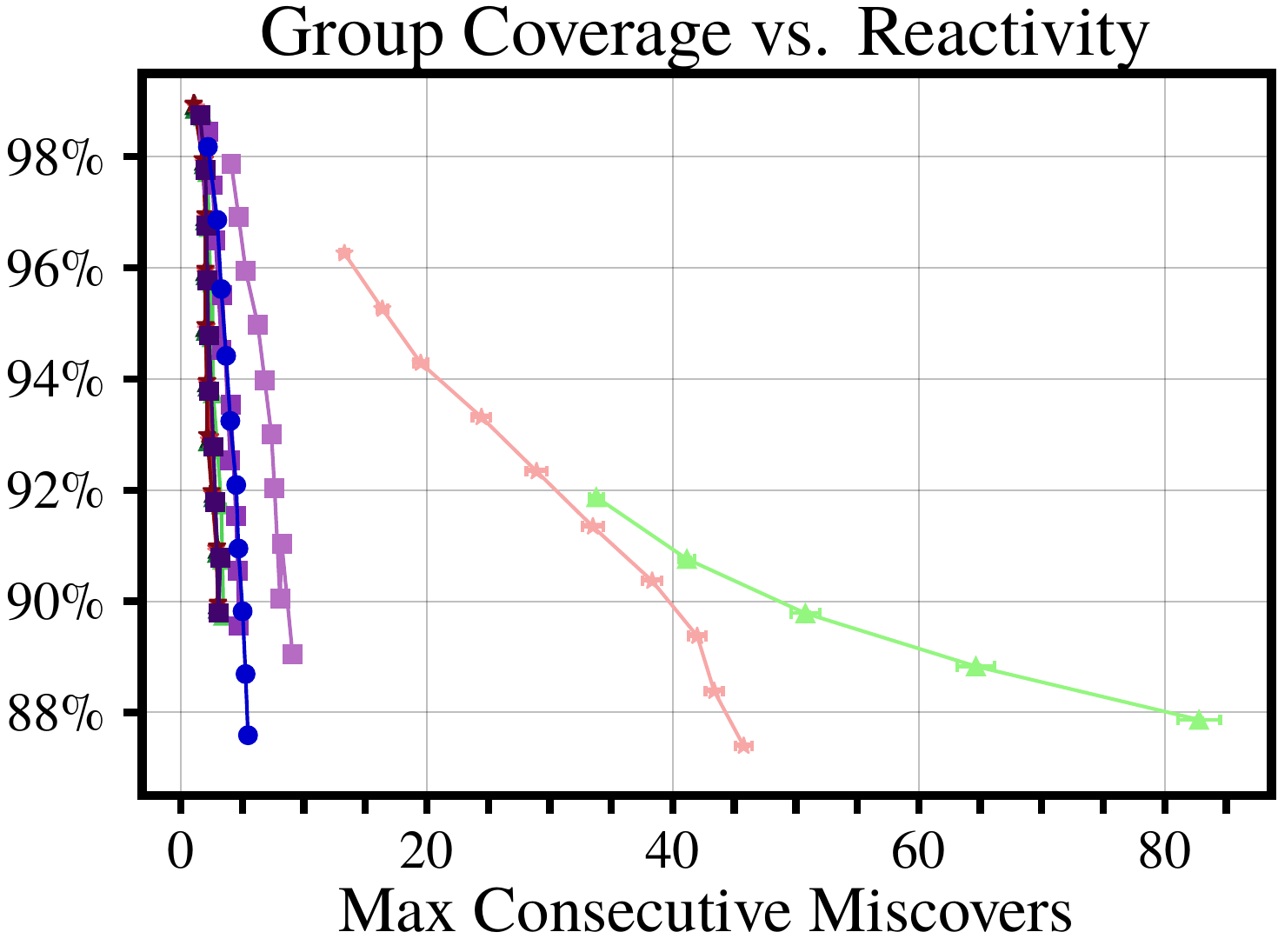}
        \end{subfigure}
        \caption{Results for synthetic setting with bounded scores with a changepoint.}
    \end{subfigure}
    \vspace{0.5em}

    \centering
    \begin{subfigure}{1\textwidth}
        \includegraphics[width=1\linewidth]{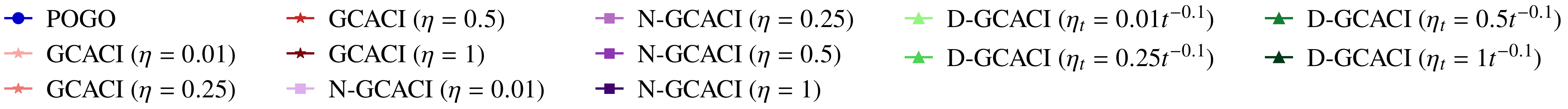}
    \end{subfigure}

    \caption{Performance across target coverage levels \(1-\alpha \in \{0.90, \ldots, 0.99\}\). {\bf Left:} worst-group coverage vs. target coverage. {\bf Middle/right:} worst-group coverage vs. radius length / max consecutive miscovers across groups (curves closer to the top-left indicate better performance; data point shown if observed coverage $\geq 0.87$). Error bars show \(x\)-axis standard error of the mean (SEM).}
    \vspace{-1em}
\end{figure}
\newpage

{\bf Results for $k=25$ number of groups, quadratic growing scores.}
\begin{figure}[h!]
    \begin{subfigure}[t]{\linewidth}
        \begin{subfigure}[t]{0.31\linewidth}
        \raisebox{0.08em}{
            \includegraphics[width=\linewidth]{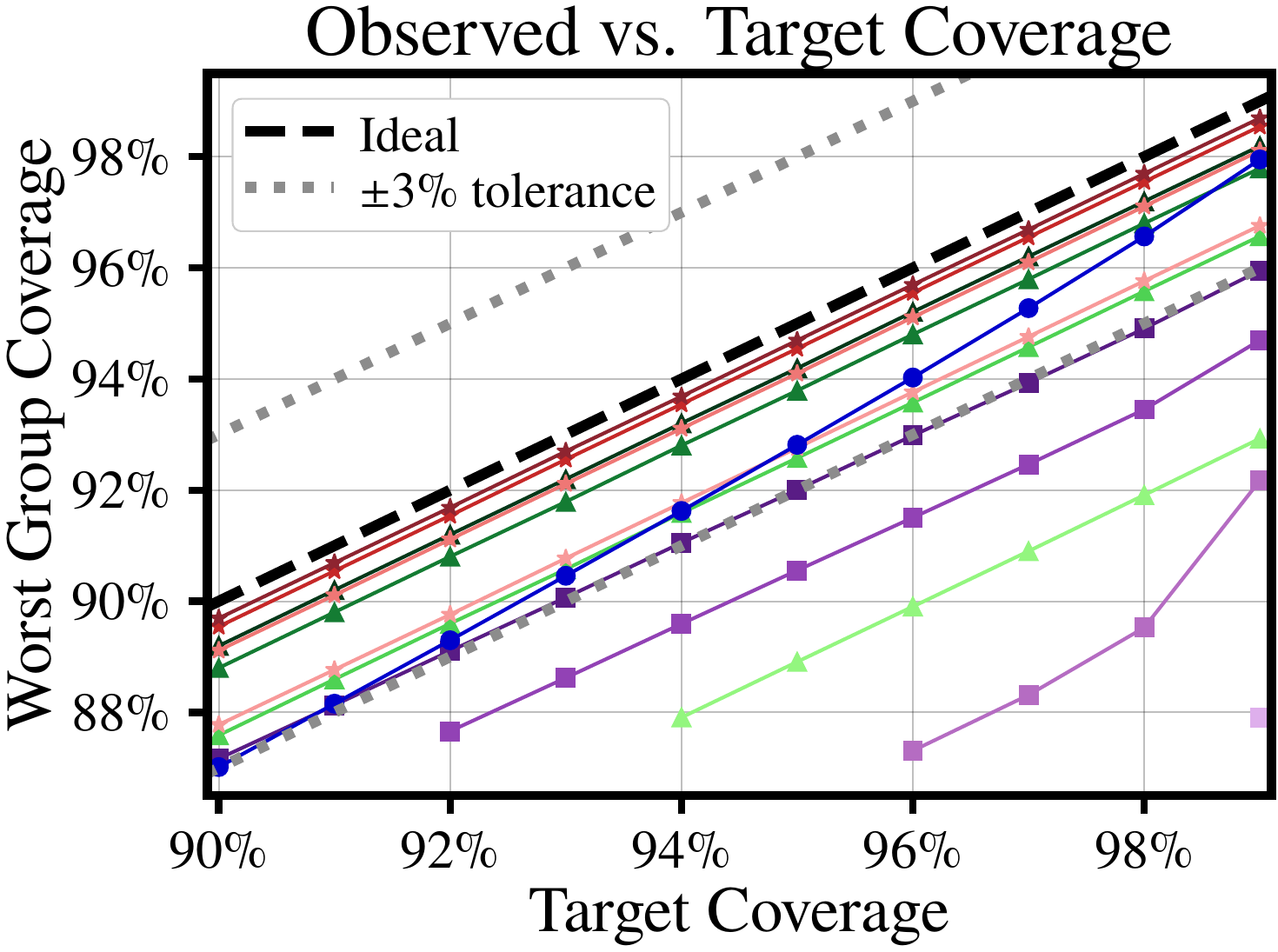}
            }
        \end{subfigure}
        \hspace{0.2em}
        \begin{subfigure}[t]{0.31\linewidth}
            \includegraphics[width=\linewidth]{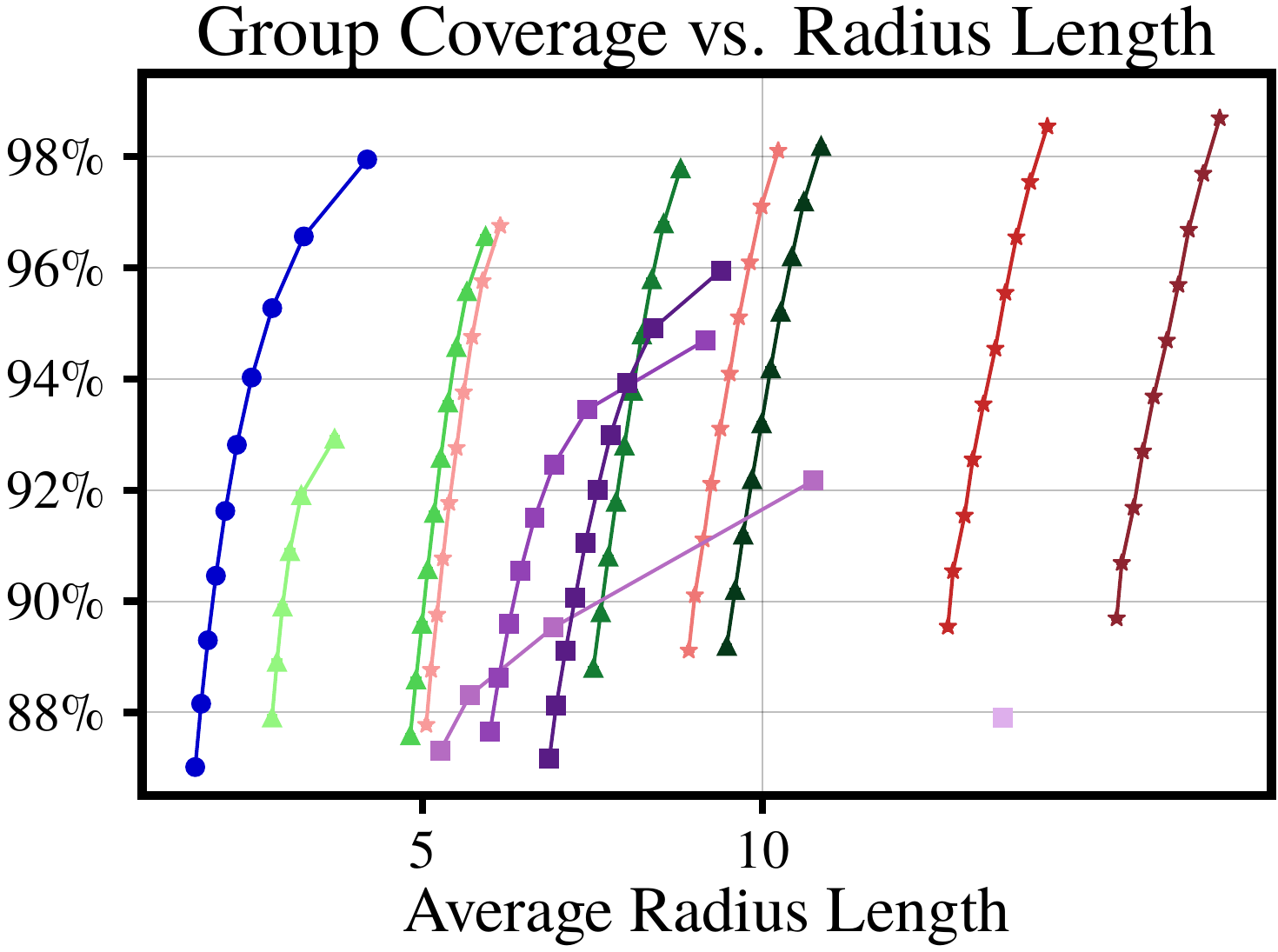}
        \end{subfigure}
        \hspace{0.2em}
        \begin{subfigure}[t]{0.31\linewidth}
            \includegraphics[width=\linewidth]{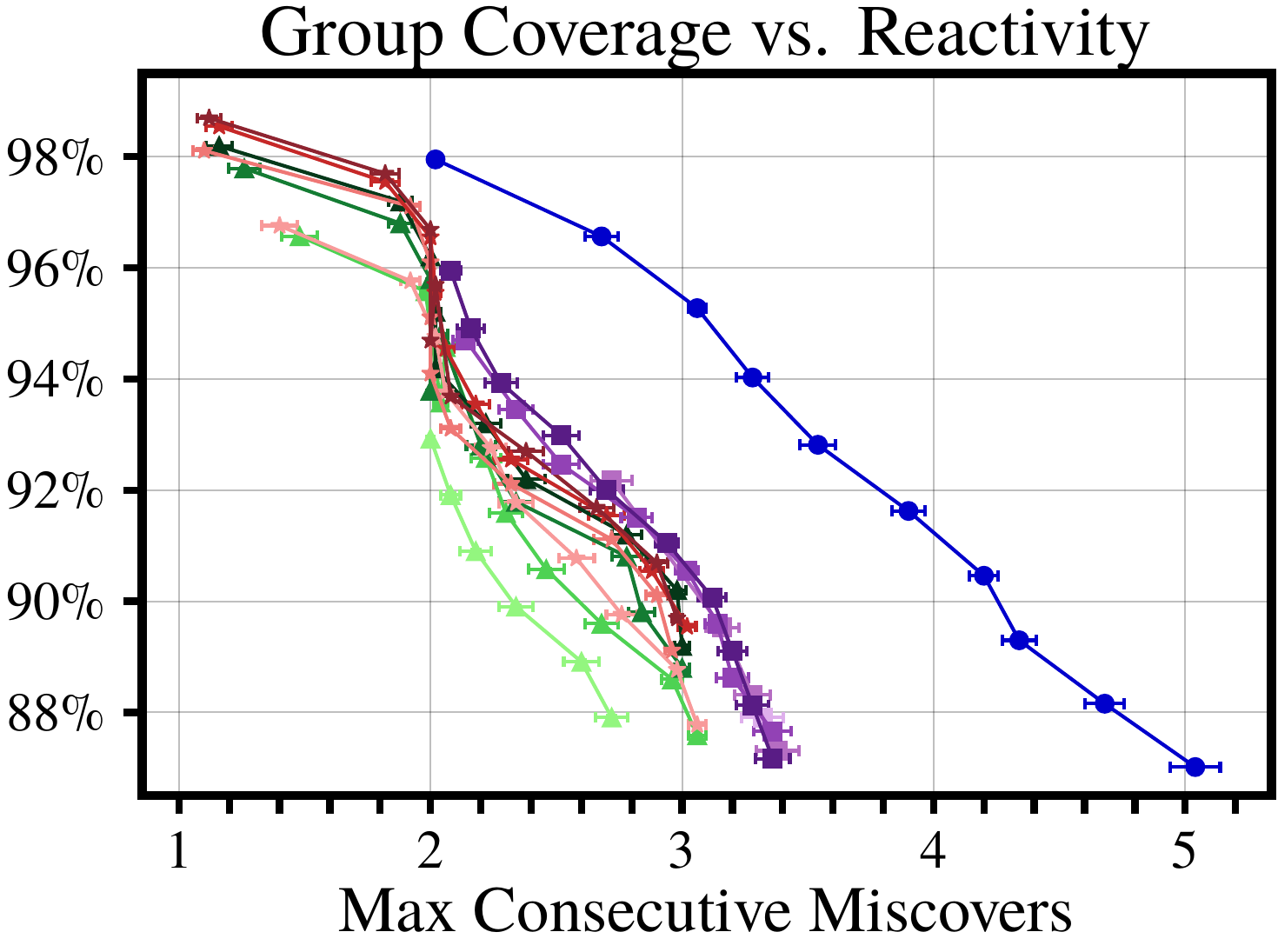}
        \end{subfigure}
        \caption{Results for the synthetic setting with slow growing quadratic scores ($q=2, A=5$).}
    \end{subfigure}
    \vspace{0.5em}
        
    \begin{subfigure}[t]{\linewidth}
        \begin{subfigure}[t]{0.31\linewidth}
        \raisebox{0.12em}{
            \includegraphics[width=\linewidth]{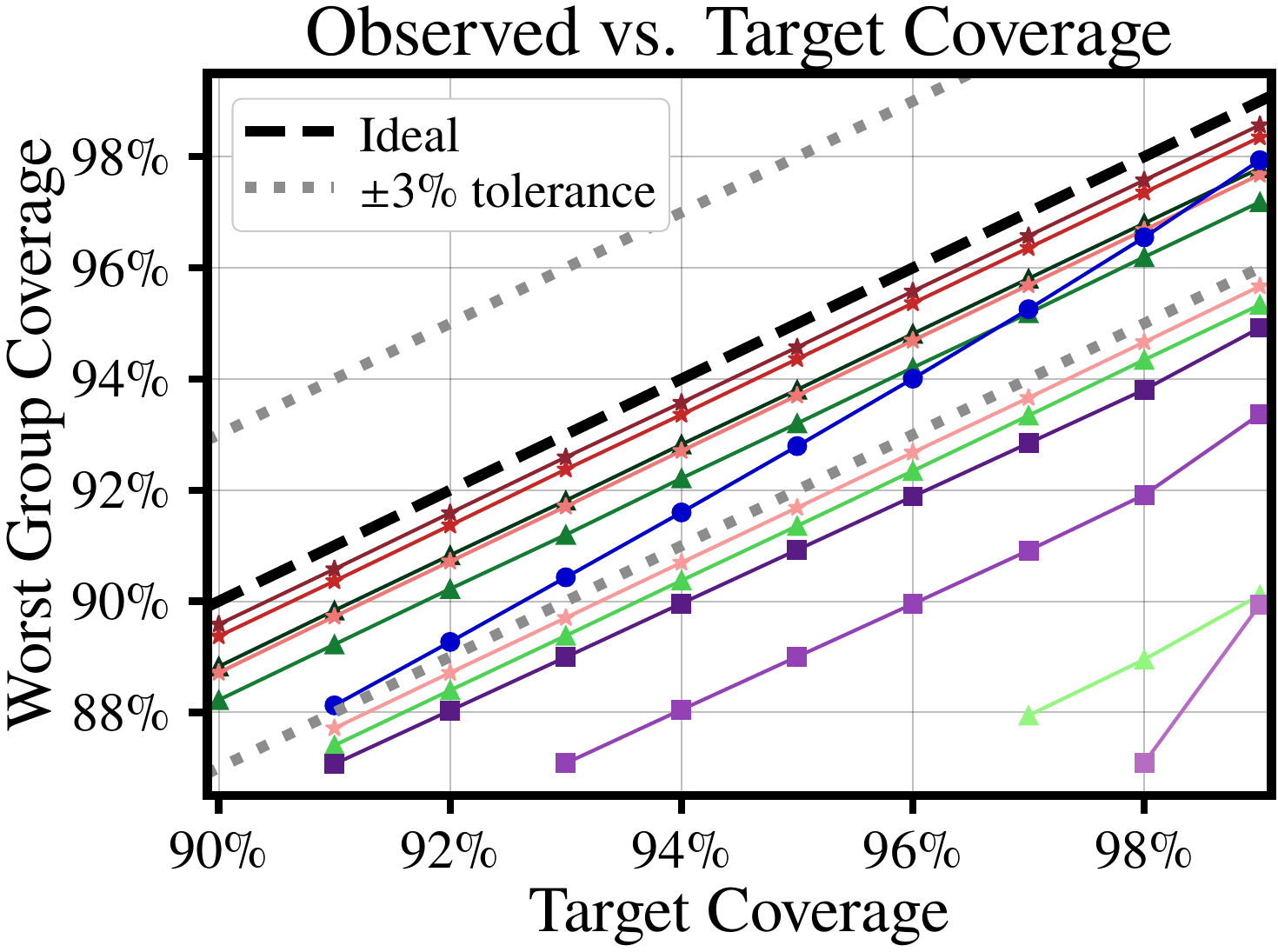}
            }
        \end{subfigure}
        \hspace{0.2em}
        \begin{subfigure}[t]{0.31\linewidth}
            \includegraphics[width=\linewidth]{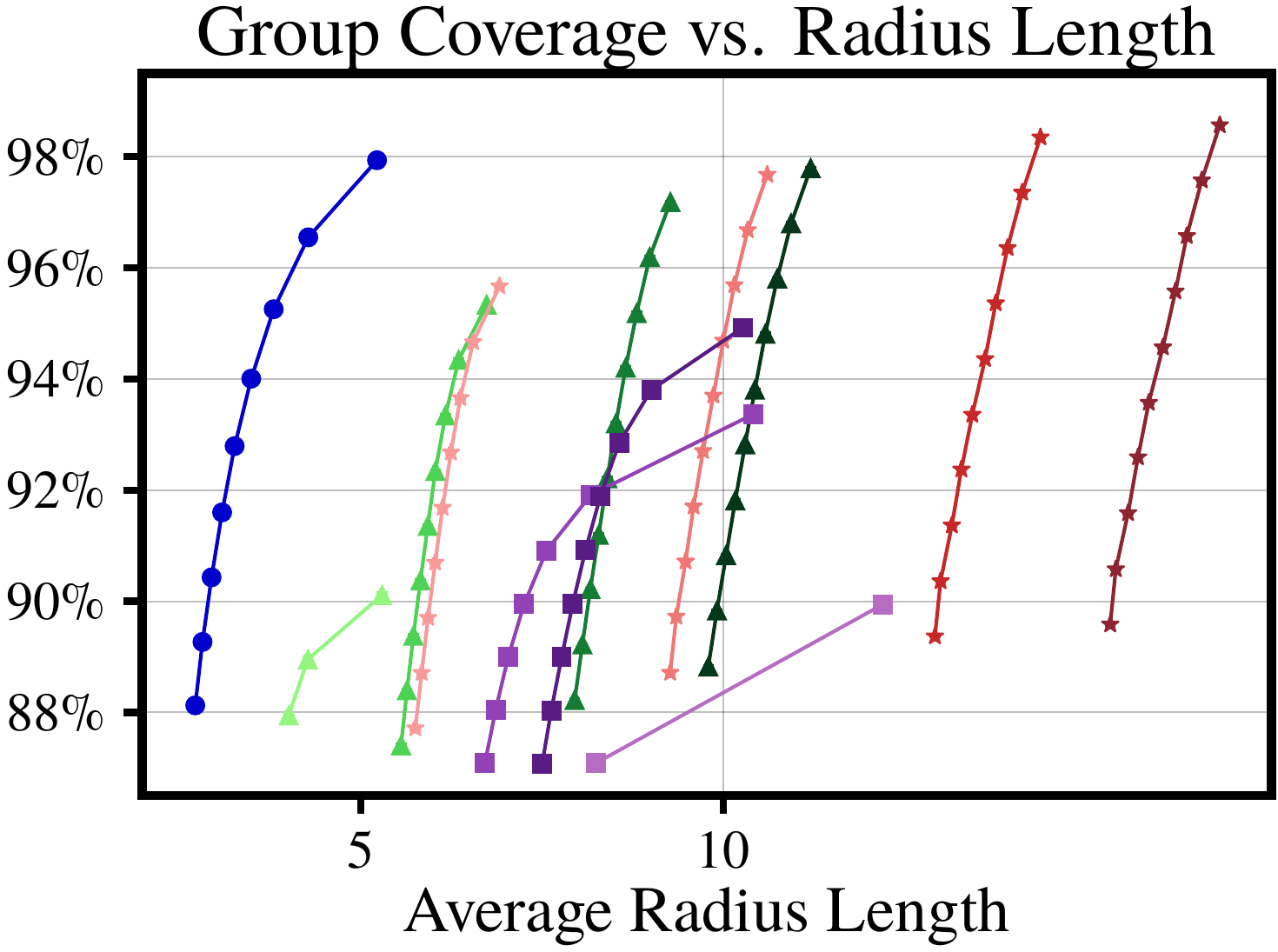}
        \end{subfigure}
        \hspace{0.2em}
        \begin{subfigure}[t]{0.31\linewidth}
            \includegraphics[width=\linewidth]{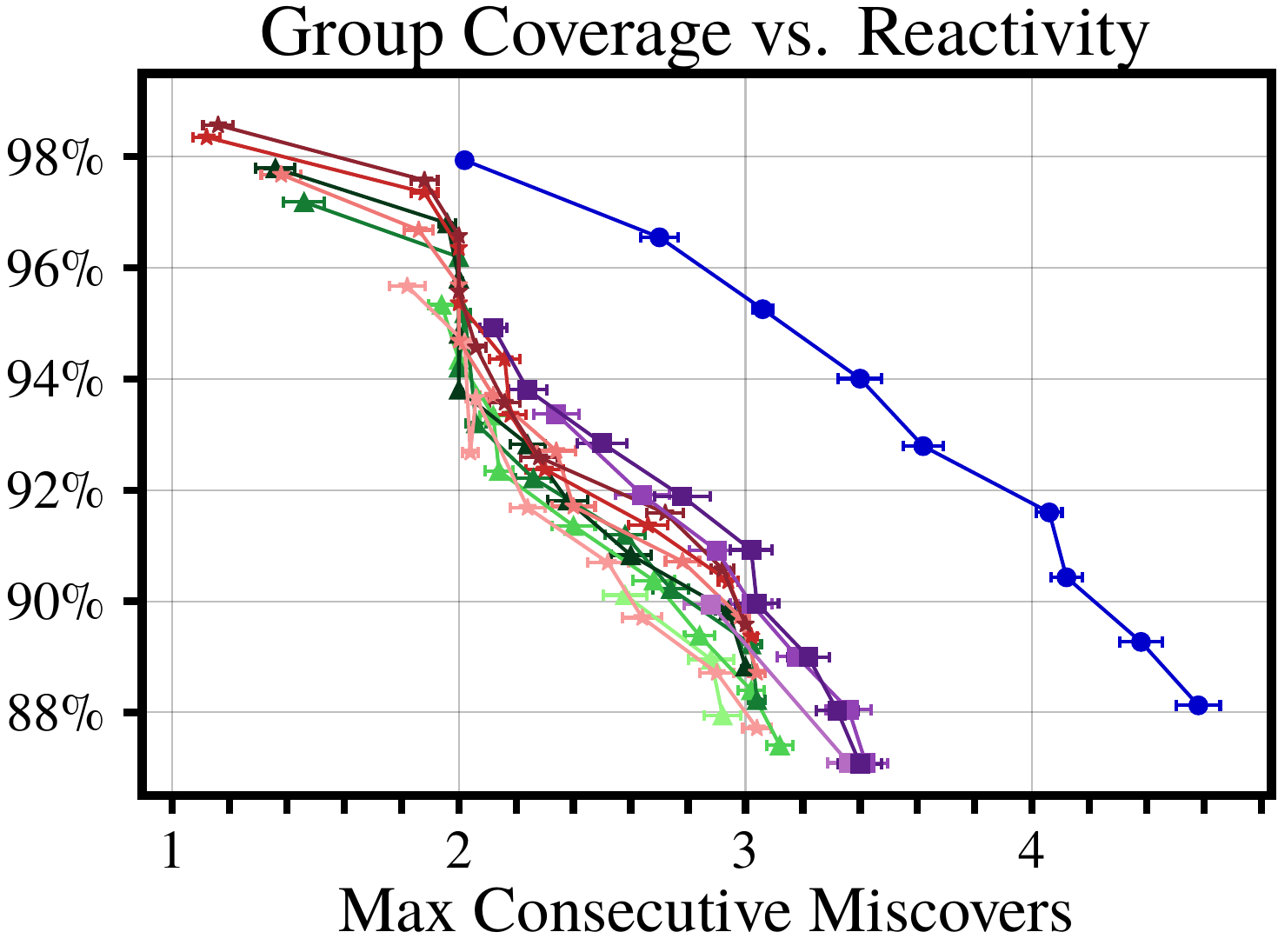}
        \end{subfigure}
        \caption{Results for the synthetic setting with moderately growing quadratic scores ($q=2, A=25$).}
    \end{subfigure}
    \vspace{0.5em}

    \begin{subfigure}[t]{\linewidth}
        \begin{subfigure}[t]{0.31\linewidth}
        \raisebox{0.12em}{
            \includegraphics[width=\linewidth]{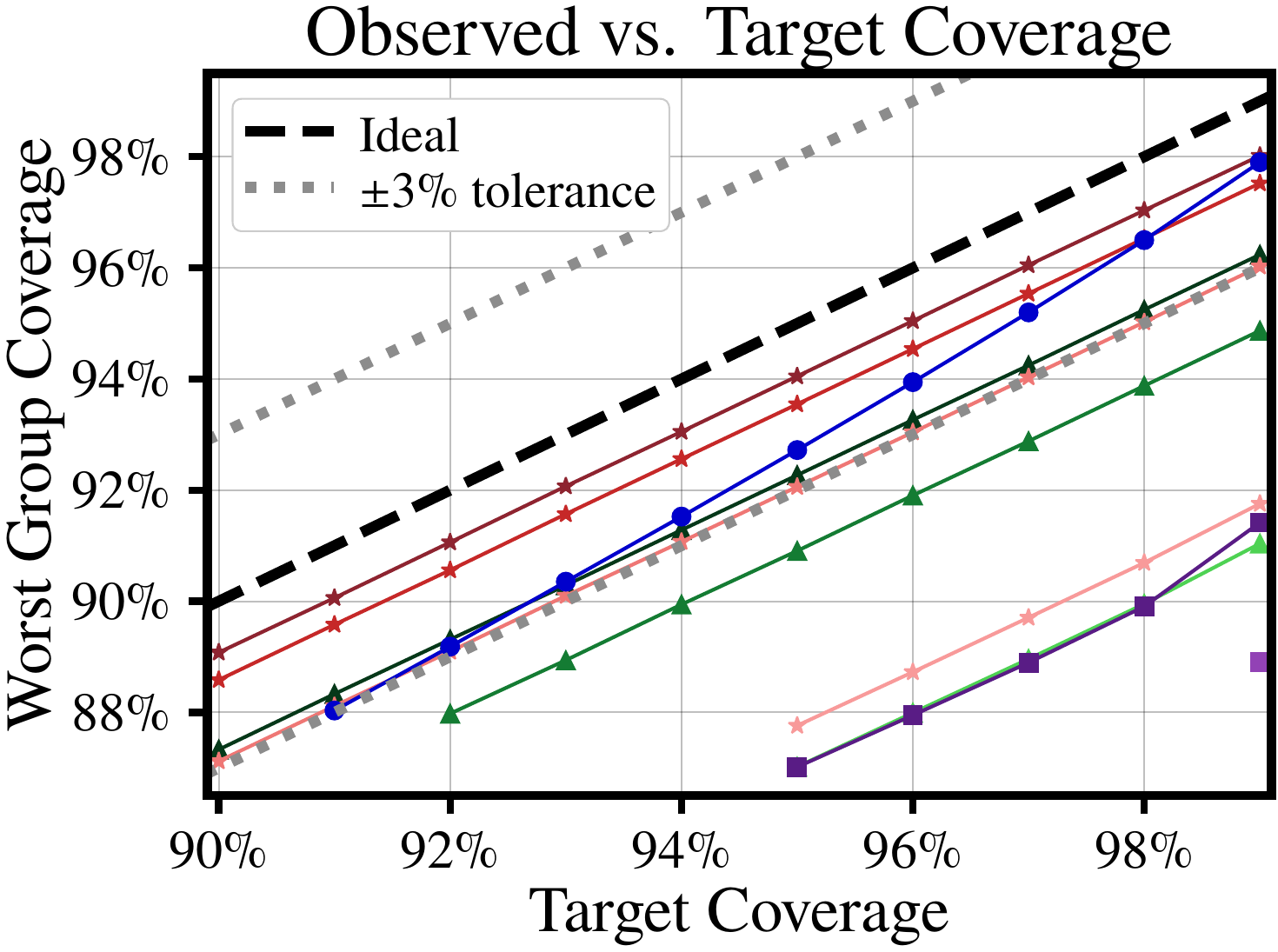}
            }
        \end{subfigure}
        \hspace{0.2em}
        \begin{subfigure}[t]{0.31\linewidth}
            \includegraphics[width=\linewidth]{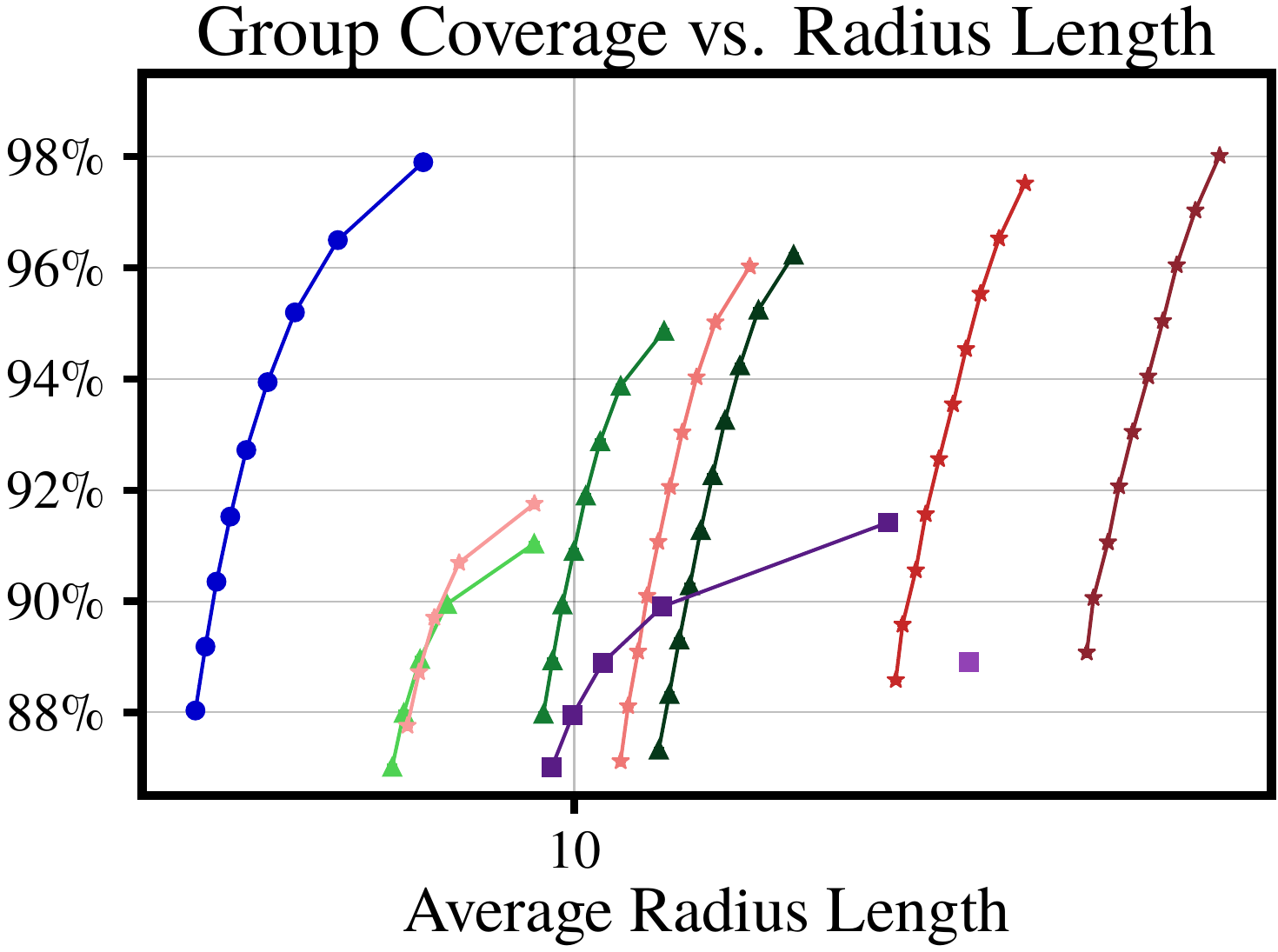}
        \end{subfigure}
        \hspace{0.2em}
        \begin{subfigure}[t]{0.31\linewidth}
            \includegraphics[width=\linewidth]{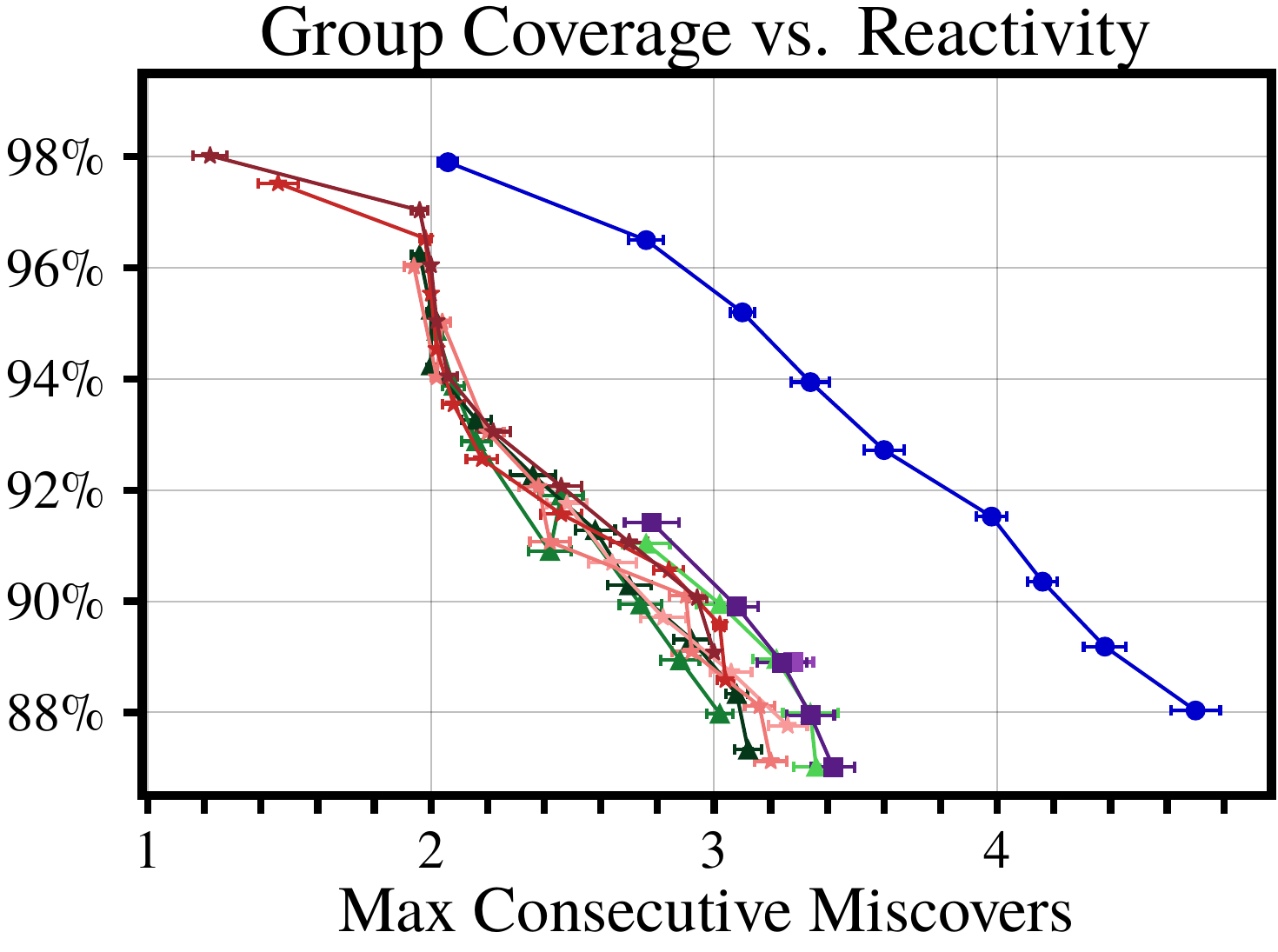}
        \end{subfigure}
        \caption{Results for the synthetic setting with extreme growing quadratic scores ($q=2, A=100$).}
    \end{subfigure}
    \vspace{0.5em}

    \centering
    \begin{subfigure}{1\textwidth}
        \includegraphics[width=\linewidth]{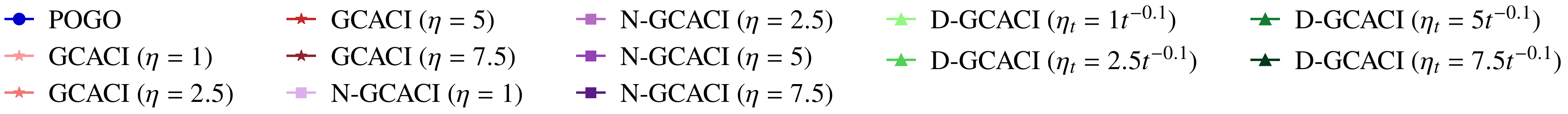}
    \end{subfigure}
    
    \caption{Performance across target coverage levels \(1-\alpha \in \{0.90, \ldots, 0.99\}\). {\bf Left:} worst-group coverage vs. target coverage. {\bf Middle/right:} worst-group coverage vs. radius length / max consecutive miscovers across groups (curves closer to the top-left indicate better performance; data point shown if observed coverage $\geq 0.87$). Error bars show \(x\)-axis standard error of the mean (SEM).}
    \vspace{-1em}
\end{figure}
\newpage

{\bf Results for $k=25$ number of groups, cubic growing scores.}
\begin{figure}[h!]
    \begin{subfigure}[t]{\linewidth}
        \begin{subfigure}[t]{0.31\linewidth}
        \raisebox{0.08em}{
            \includegraphics[width=\linewidth]{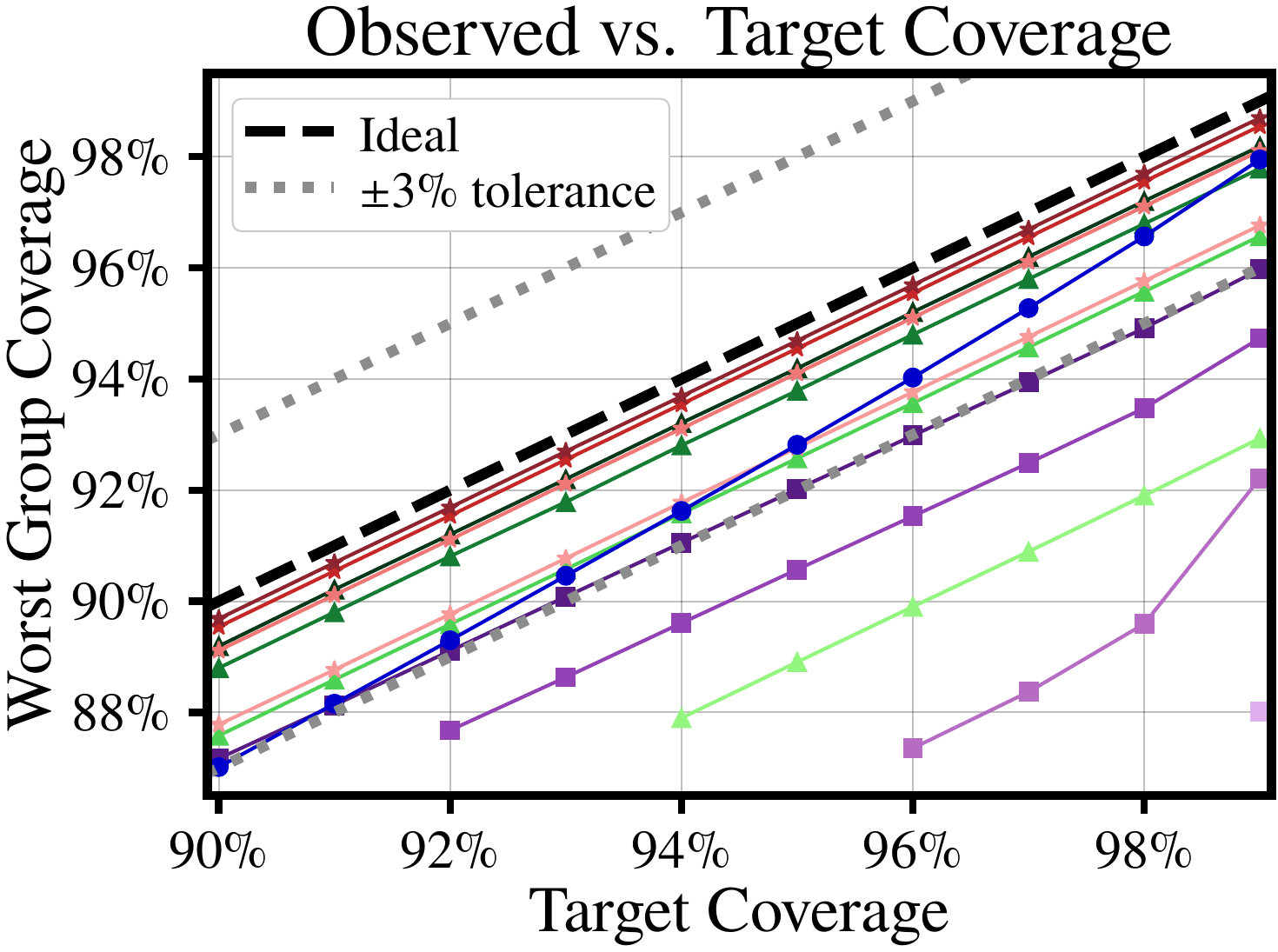}
            }
        \end{subfigure}
        \hspace{0.2em}
        \begin{subfigure}[t]{0.31\linewidth}
            \includegraphics[width=\linewidth]{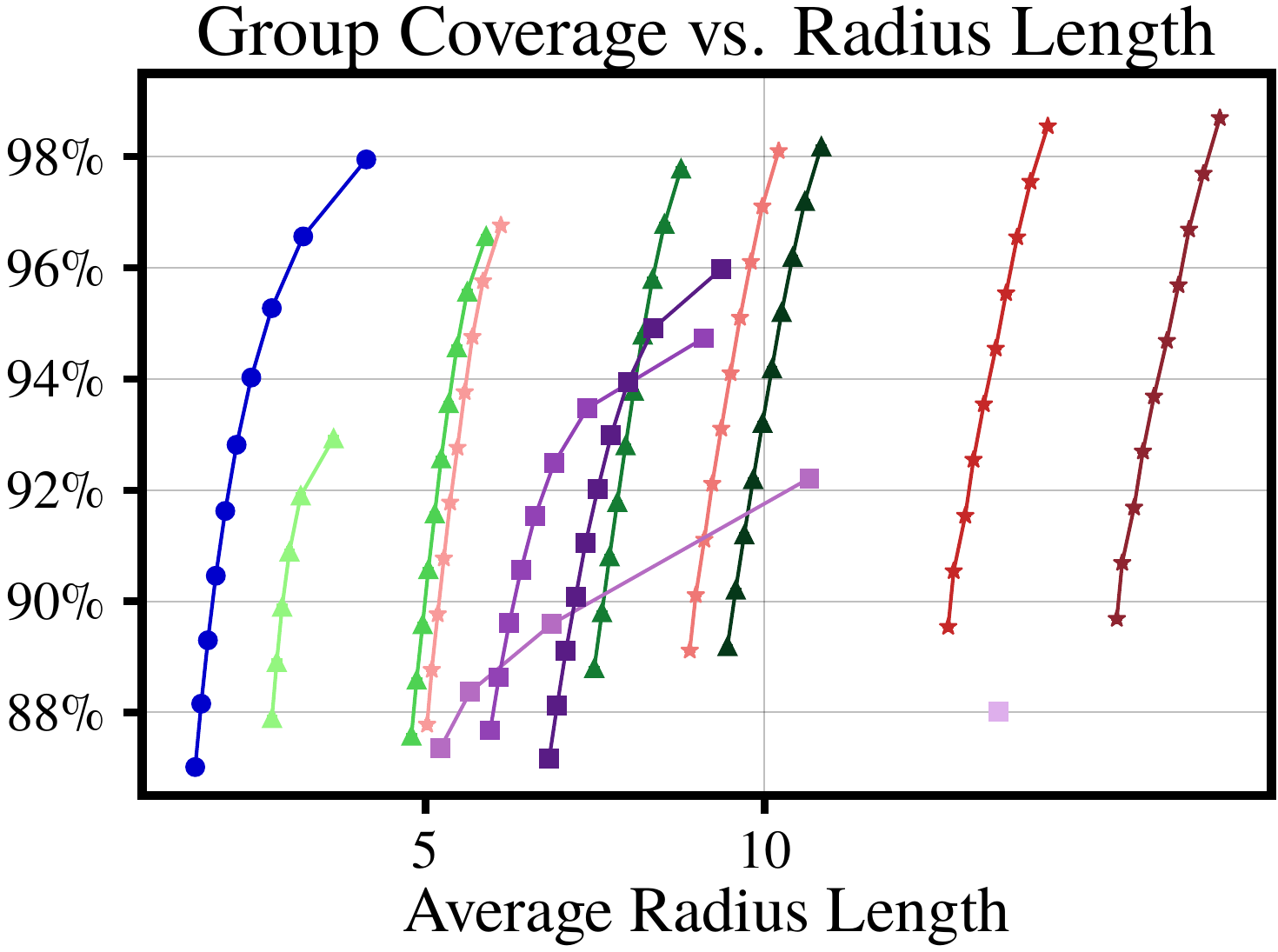}
        \end{subfigure}
        \hspace{0.2em}
        \begin{subfigure}[t]{0.31\linewidth}
            \includegraphics[width=\linewidth]{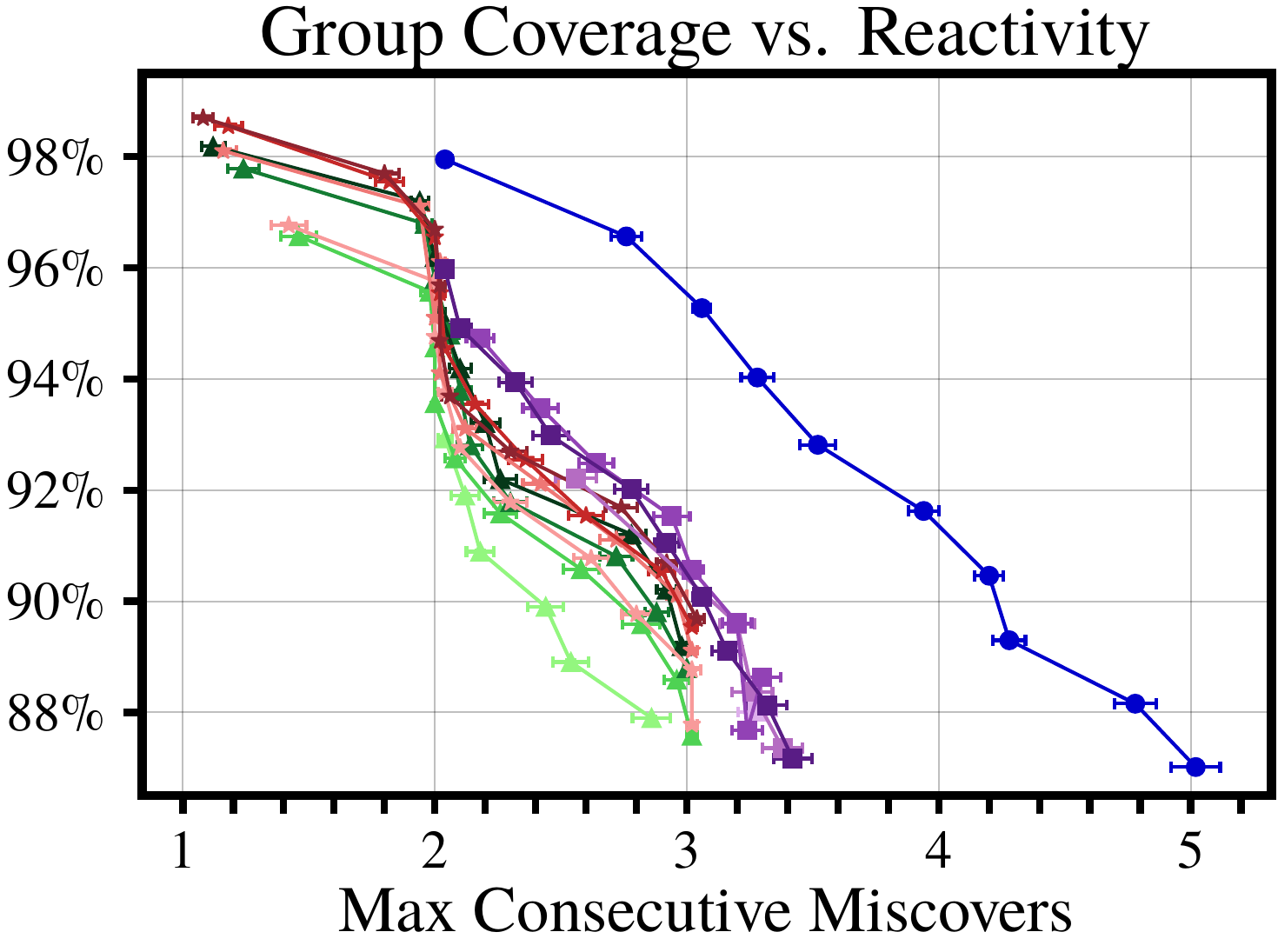}
        \end{subfigure}
        \caption{Results for the synthetic setting with slow growing cubic scores ($q=3, A=5$).}
    \end{subfigure}
    \vspace{0.5em}
        
    \begin{subfigure}[t]{\linewidth}
        \begin{subfigure}[t]{0.31\linewidth}
        \raisebox{0.12em}{
            \includegraphics[width=\linewidth]{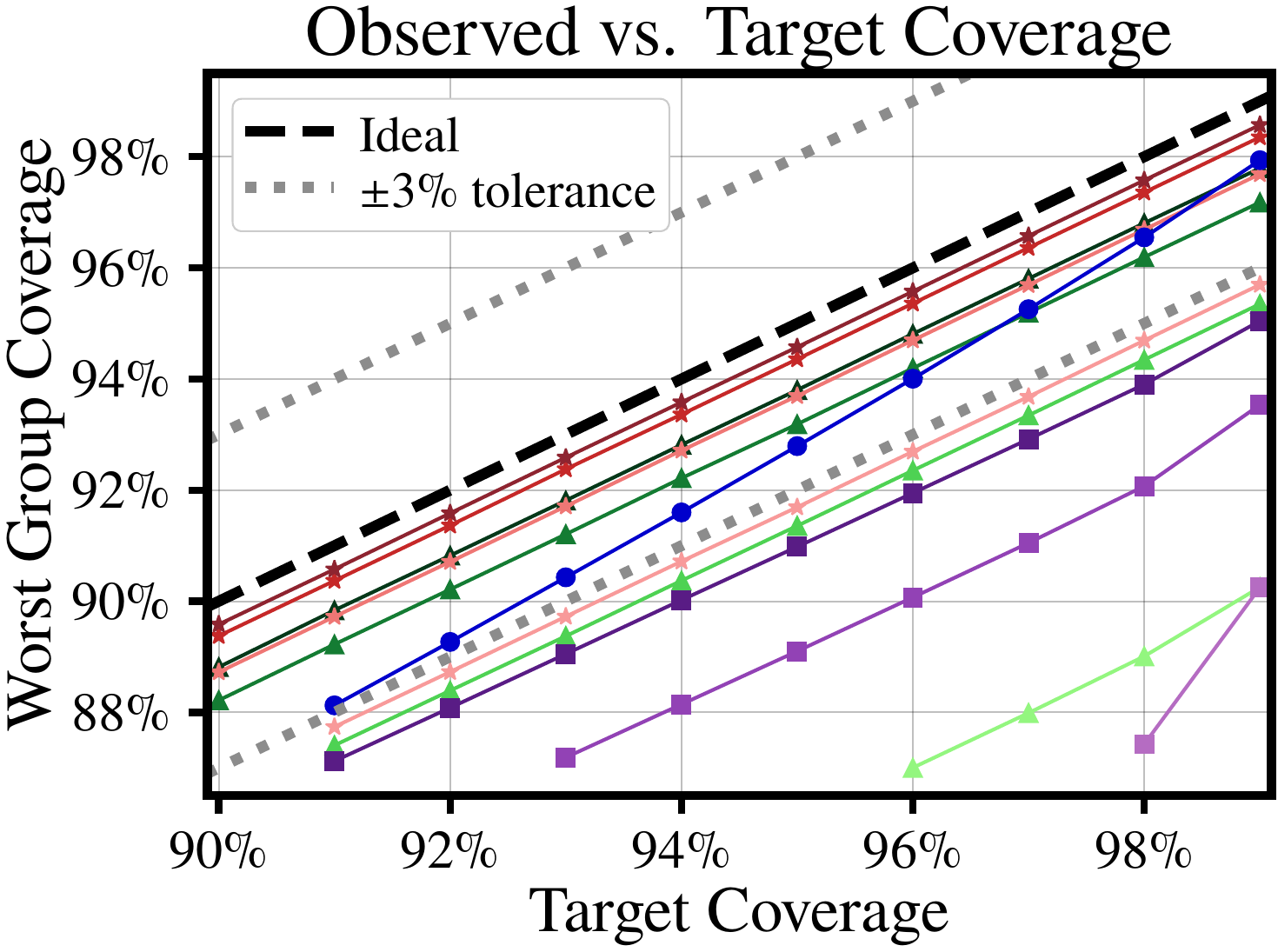}
            }
        \end{subfigure}
        \hspace{0.2em}
        \begin{subfigure}[t]{0.31\linewidth}
            \includegraphics[width=\linewidth]{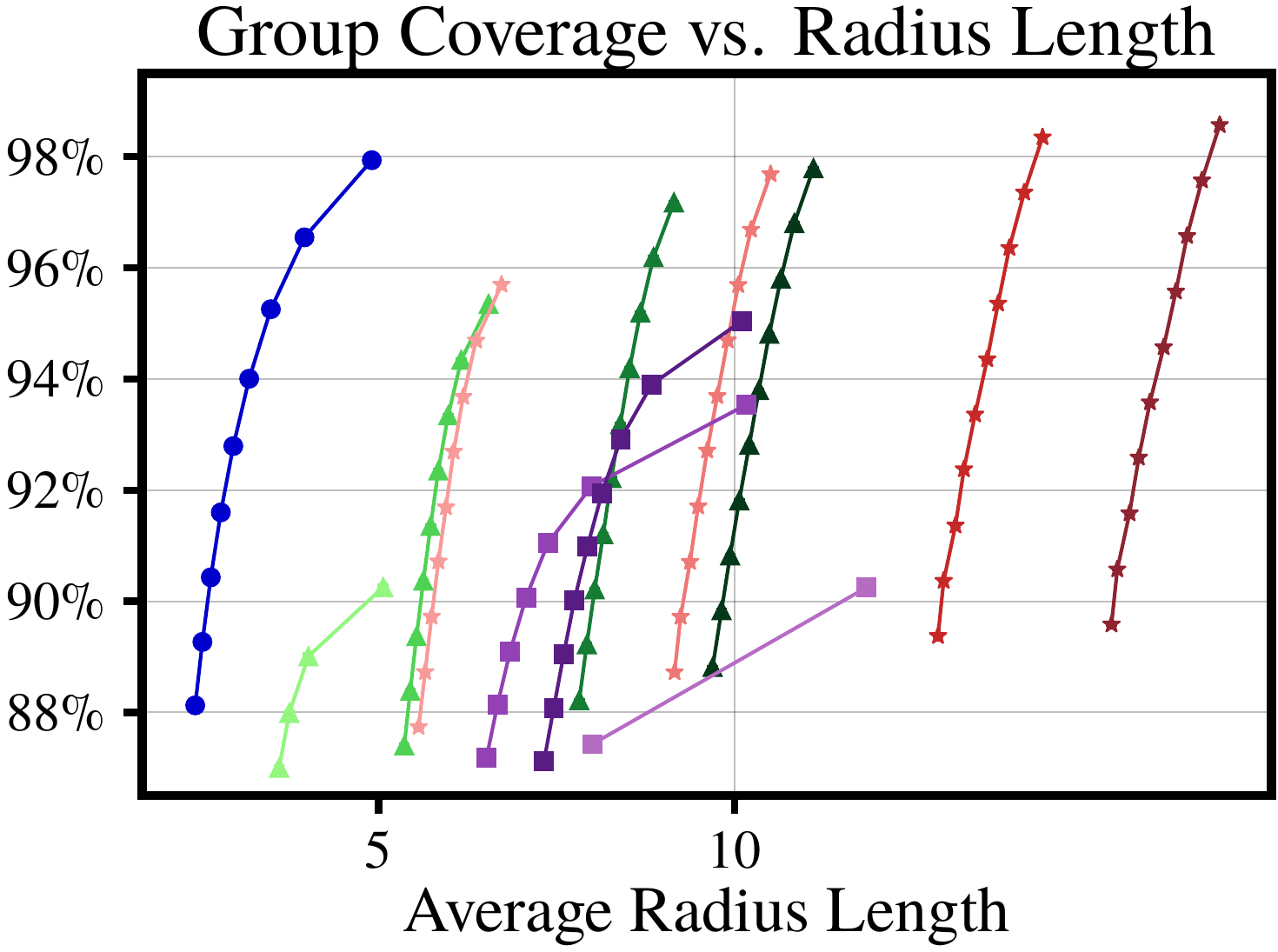}
        \end{subfigure}
        \hspace{0.2em}
        \begin{subfigure}[t]{0.31\linewidth}
            \includegraphics[width=\linewidth]{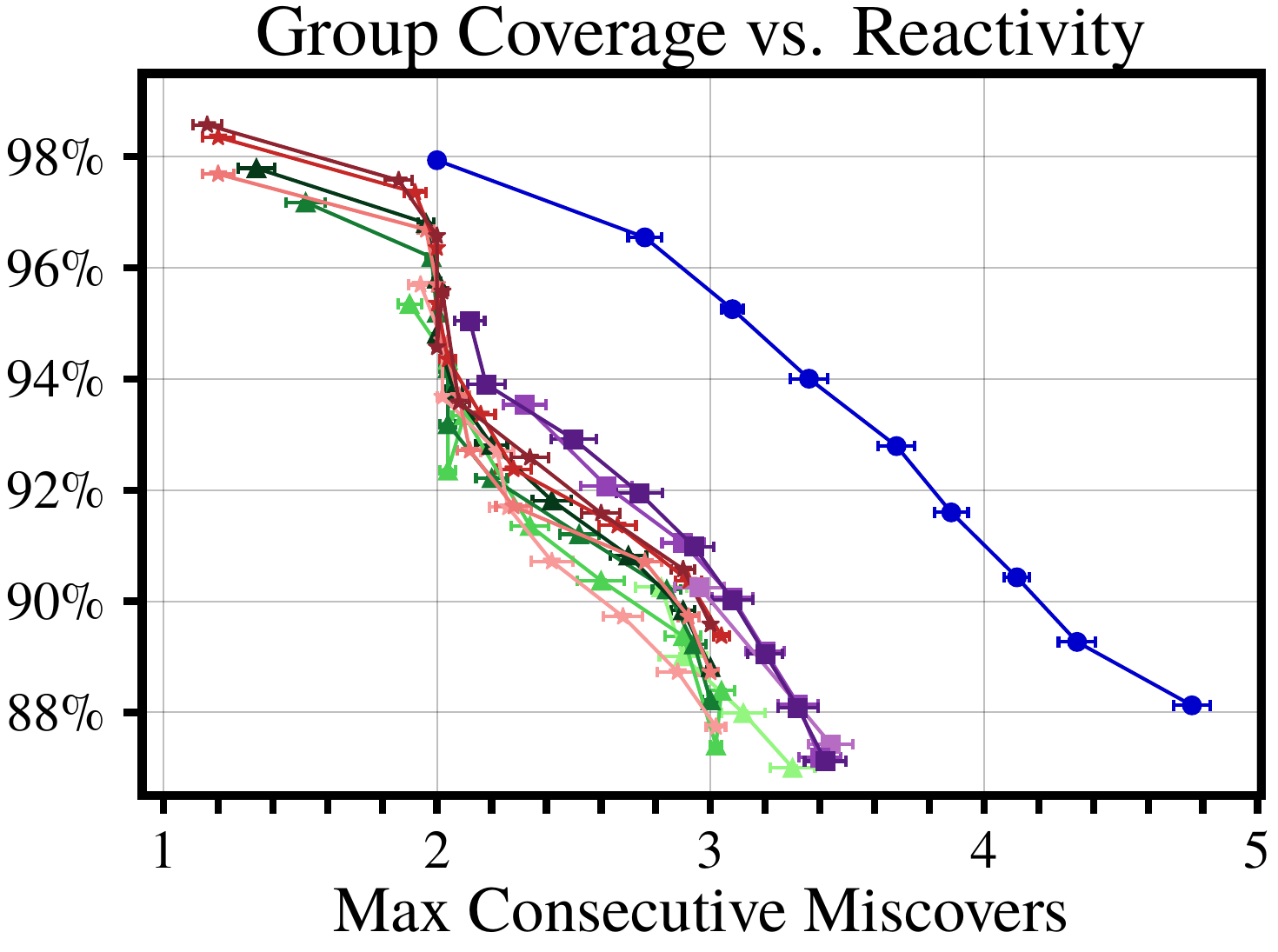}
        \end{subfigure}
        \caption{Results for the synthetic setting with moderately growing cubic scores ($q=3, A=25$).}
    \end{subfigure}
    \vspace{0.5em}

    \begin{subfigure}[t]{\linewidth}
        \begin{subfigure}[t]{0.31\linewidth}
        \raisebox{0.12em}{
            \includegraphics[width=\linewidth]{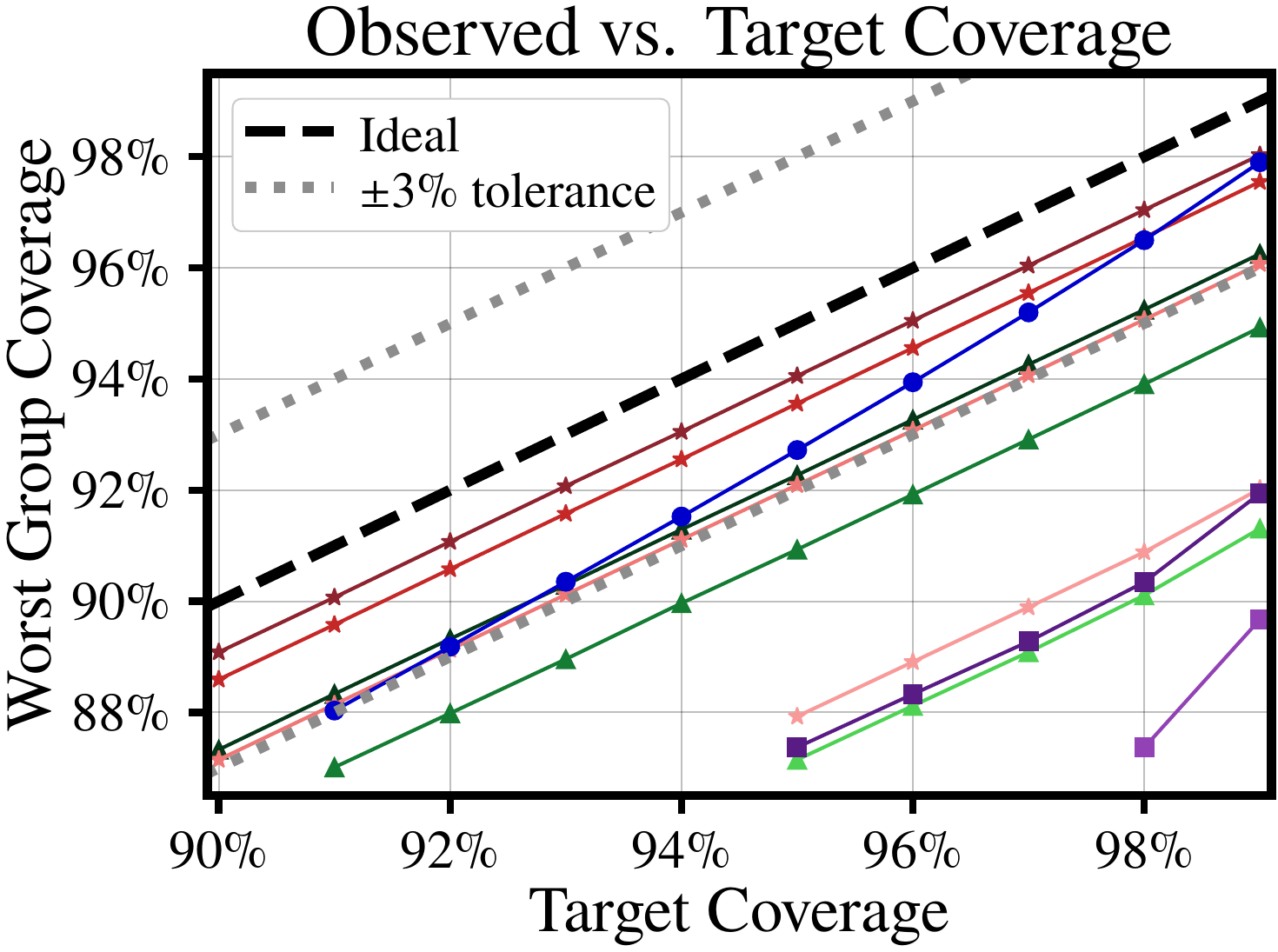}
            }
        \end{subfigure}
        \hspace{0.2em}
        \begin{subfigure}[t]{0.31\linewidth}
            \includegraphics[width=\linewidth]{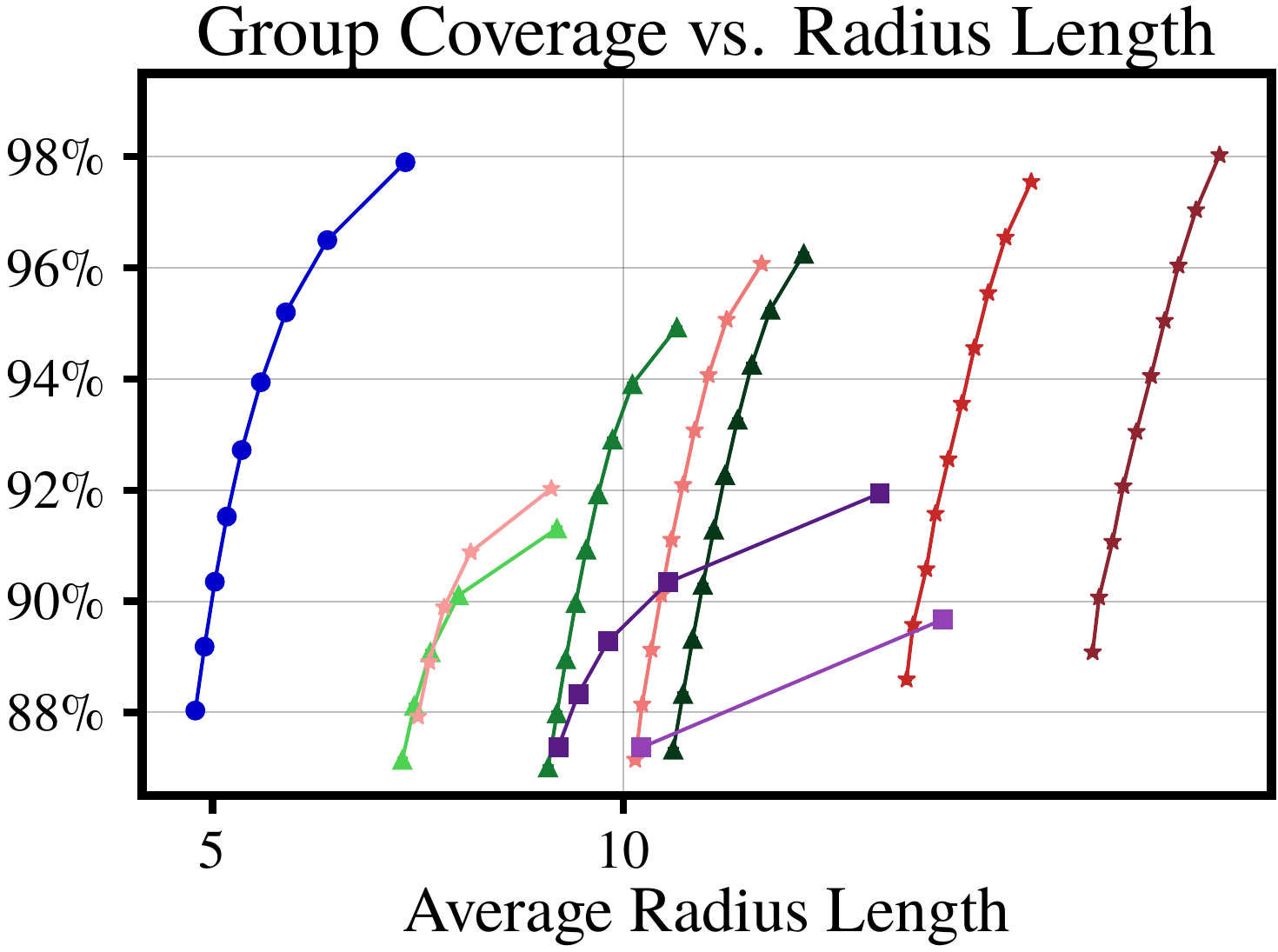}
        \end{subfigure}
        \hspace{0.2em}
        \begin{subfigure}[t]{0.31\linewidth}
            \includegraphics[width=\linewidth]{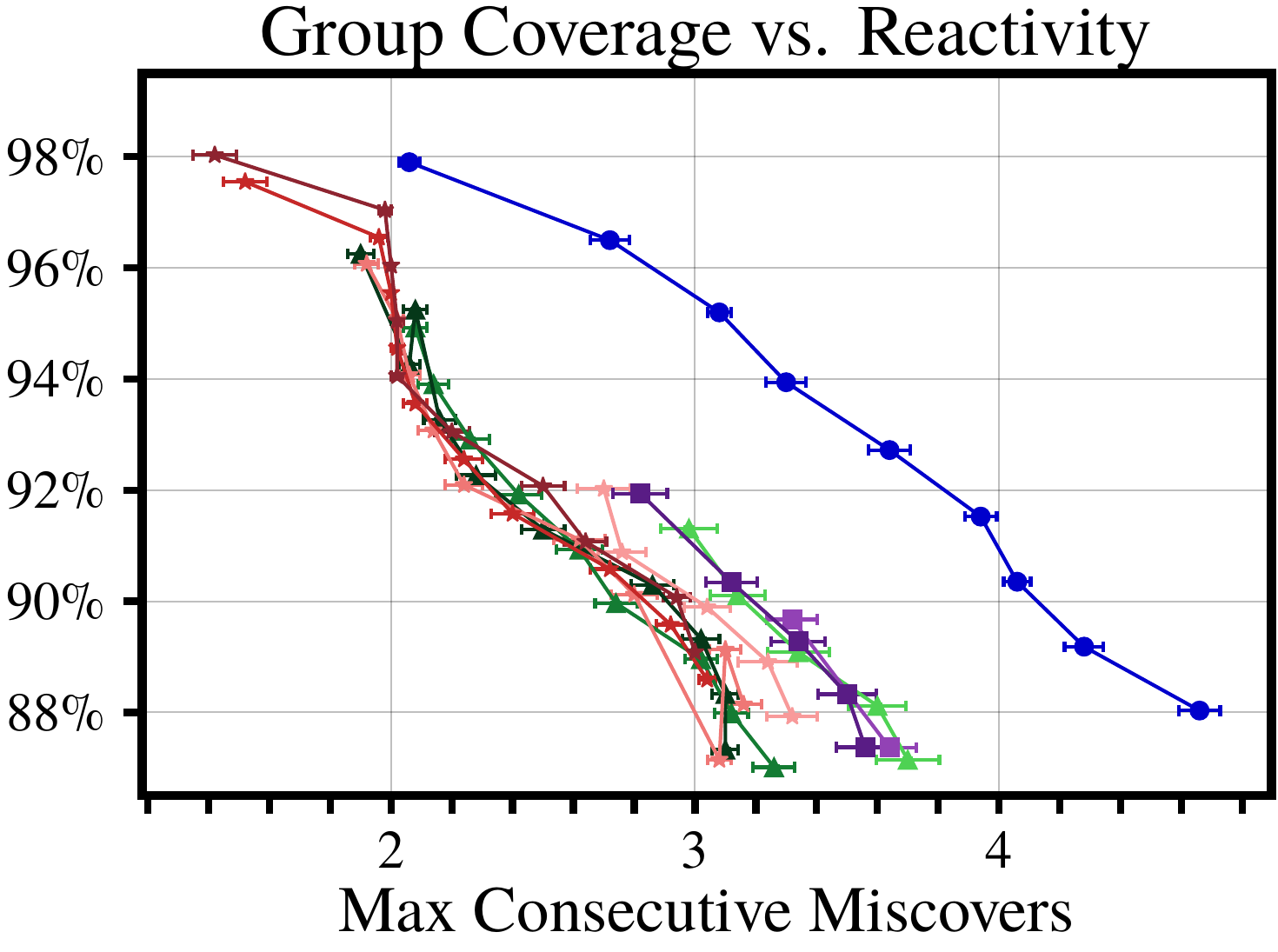}
        \end{subfigure}
        \caption{Results for the synthetic setting with extreme growing cubic scores ($q=3, A=100$).}
    \end{subfigure}
    \vspace{0.5em}

    \centering
    \begin{subfigure}{1\textwidth}
        \includegraphics[width=\linewidth]{figures/iclr/legend_syn_2_horz.pdf}
    \end{subfigure}

    \caption{Performance across target coverage levels \(1-\alpha \in \{0.90, \ldots, 0.99\}\). {\bf Left:} worst-group coverage vs. target coverage. {\bf Middle/right:} worst-group coverage vs. radius length / max consecutive miscovers across groups (curves closer to the top-left indicate better performance; data point shown if observed coverage $\geq 0.87$). Error bars show \(x\)-axis standard error of the mean (SEM).}
    \vspace{-1em}
\end{figure}
\newpage

{\bf Results for $k=50$ number of groups, bounded scores}
\begin{figure}[h!]
    \begin{subfigure}[t]{\linewidth}
        \begin{subfigure}[t]{0.31\linewidth}
        \raisebox{0.12em}{
            \includegraphics[width=\linewidth]{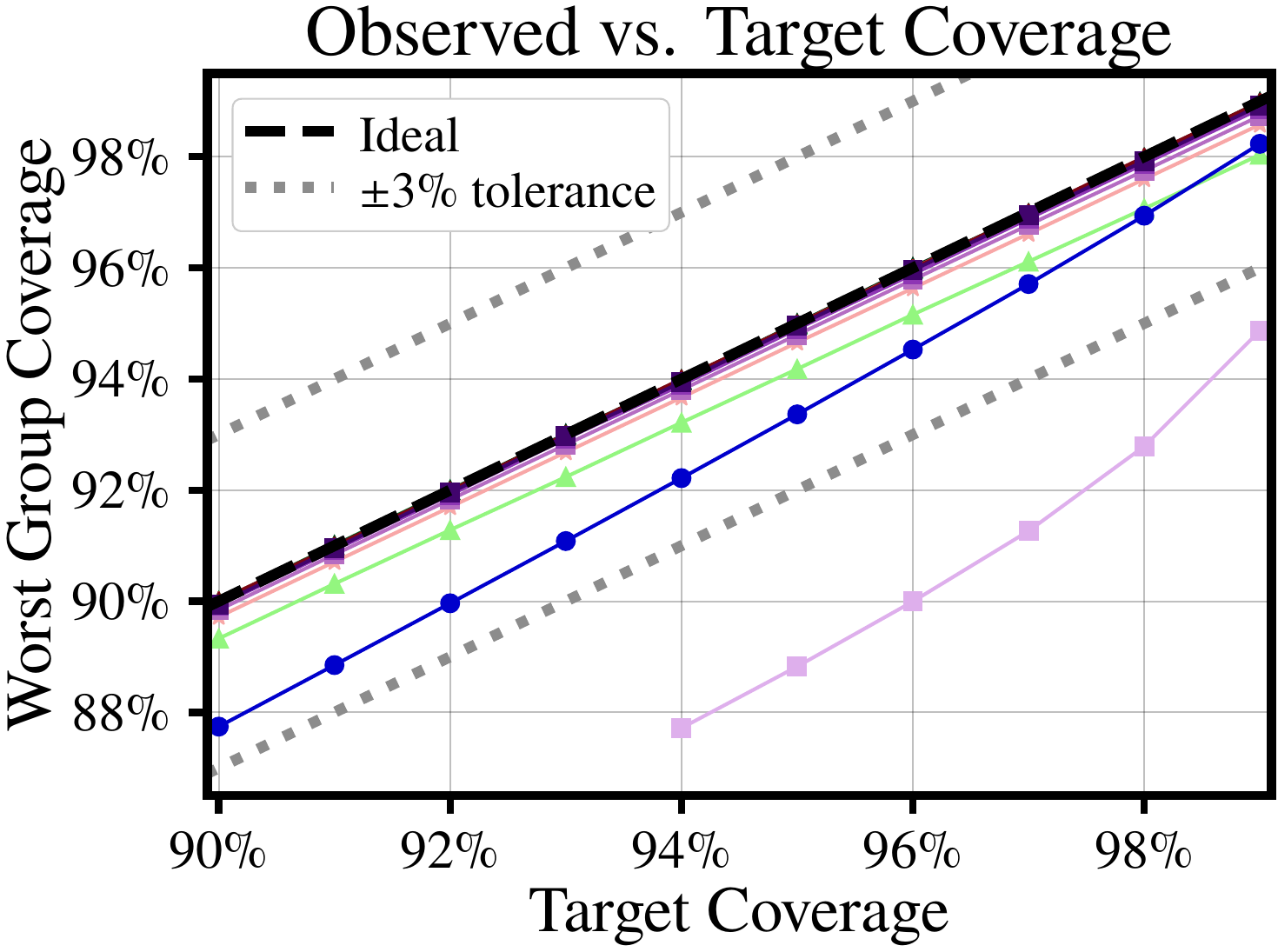}
            }
        \end{subfigure}
        \hspace{0.2em}
        \begin{subfigure}[t]{0.31\linewidth}
            \includegraphics[width=\linewidth]{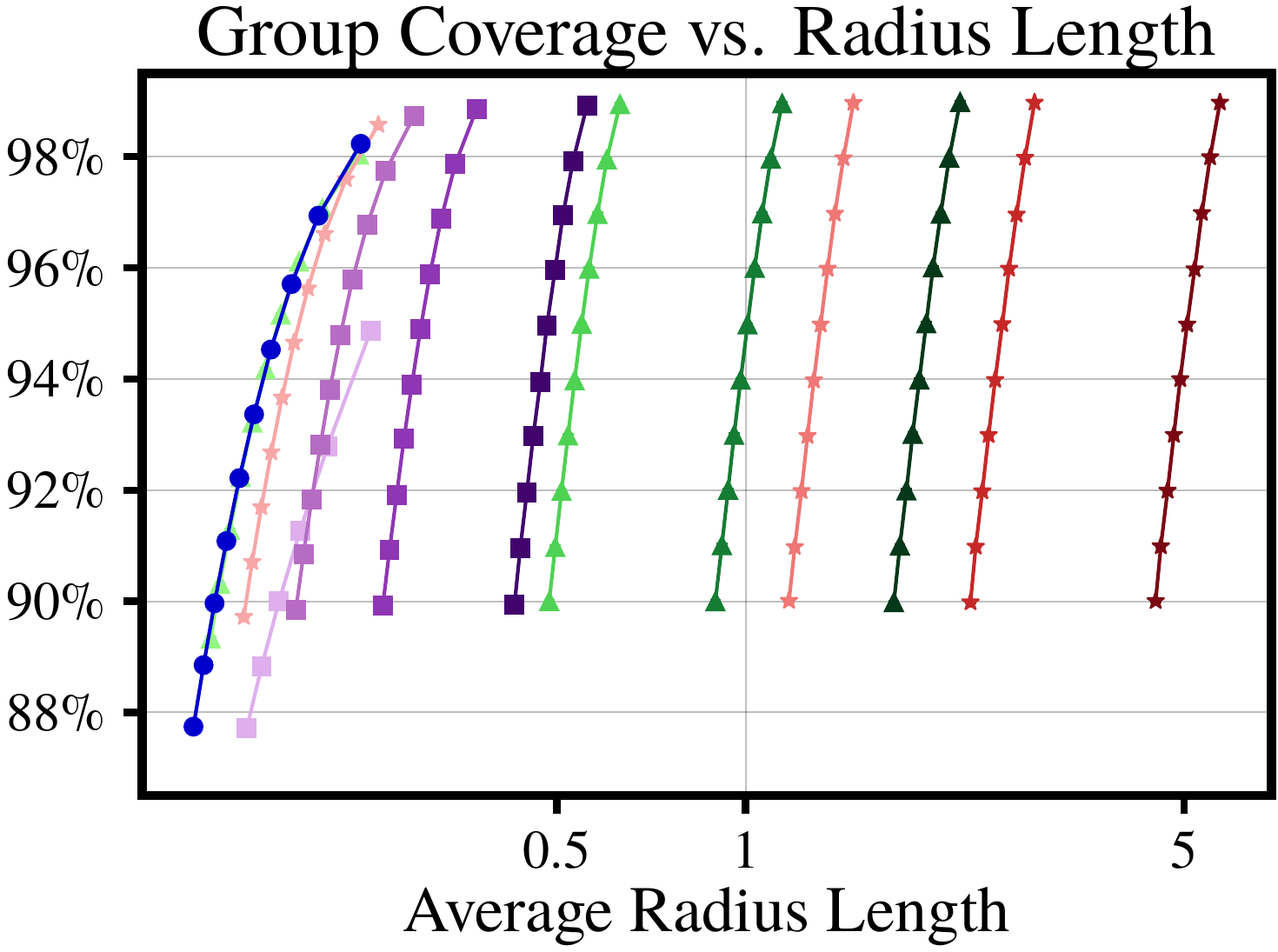}
        \end{subfigure}
        \hspace{0.2em}
        \begin{subfigure}[t]{0.31\linewidth}
            \includegraphics[width=\linewidth]{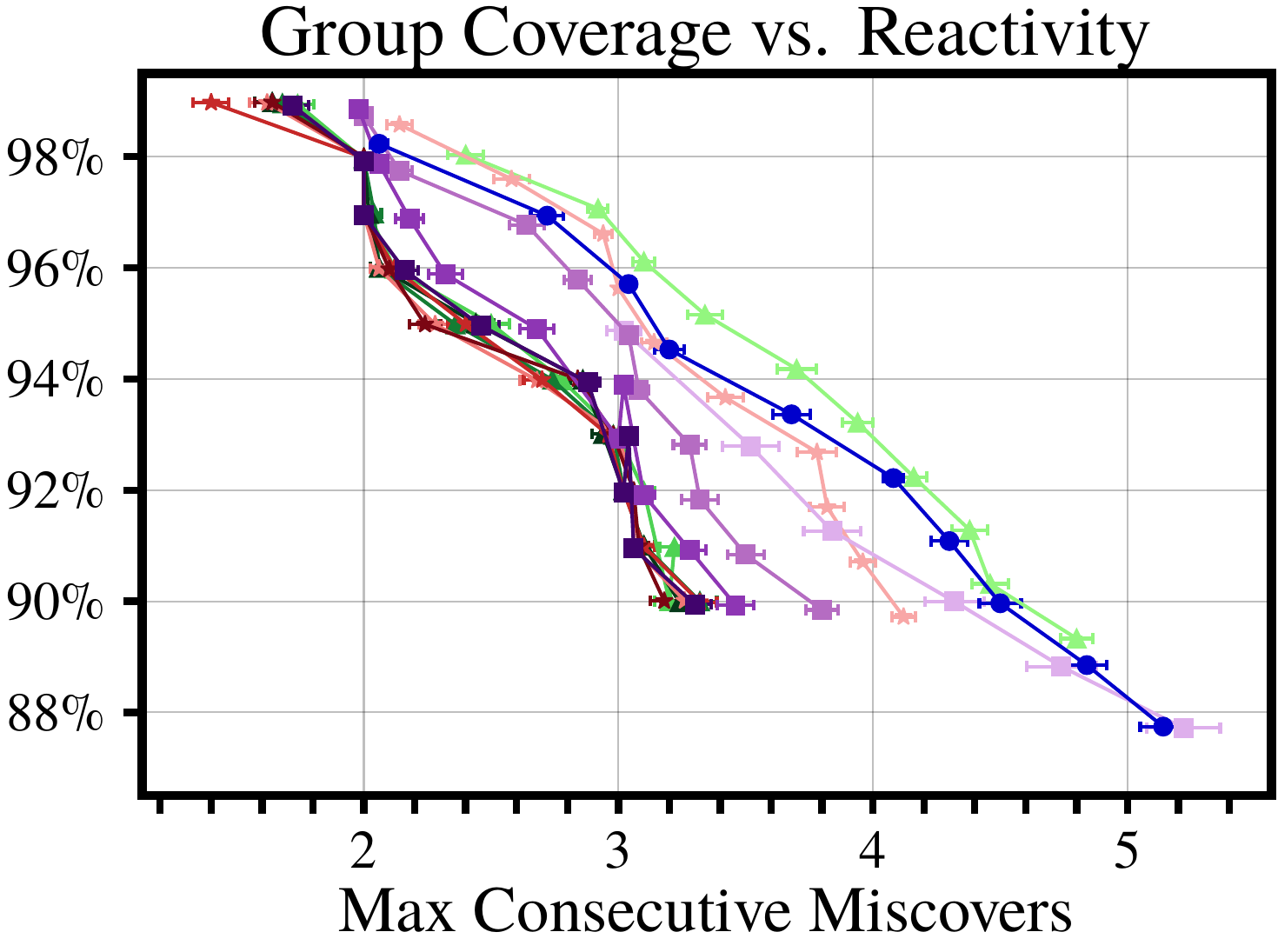}
        \end{subfigure}
        \caption{Results for synthetic setting with bounded, gradually varying scores.}
    \end{subfigure}
    \vspace{0.5em}
        
    \begin{subfigure}[t]{\linewidth}
        \begin{subfigure}[t]{0.31\linewidth}
        \raisebox{0.05em}{
            \includegraphics[width=\linewidth]{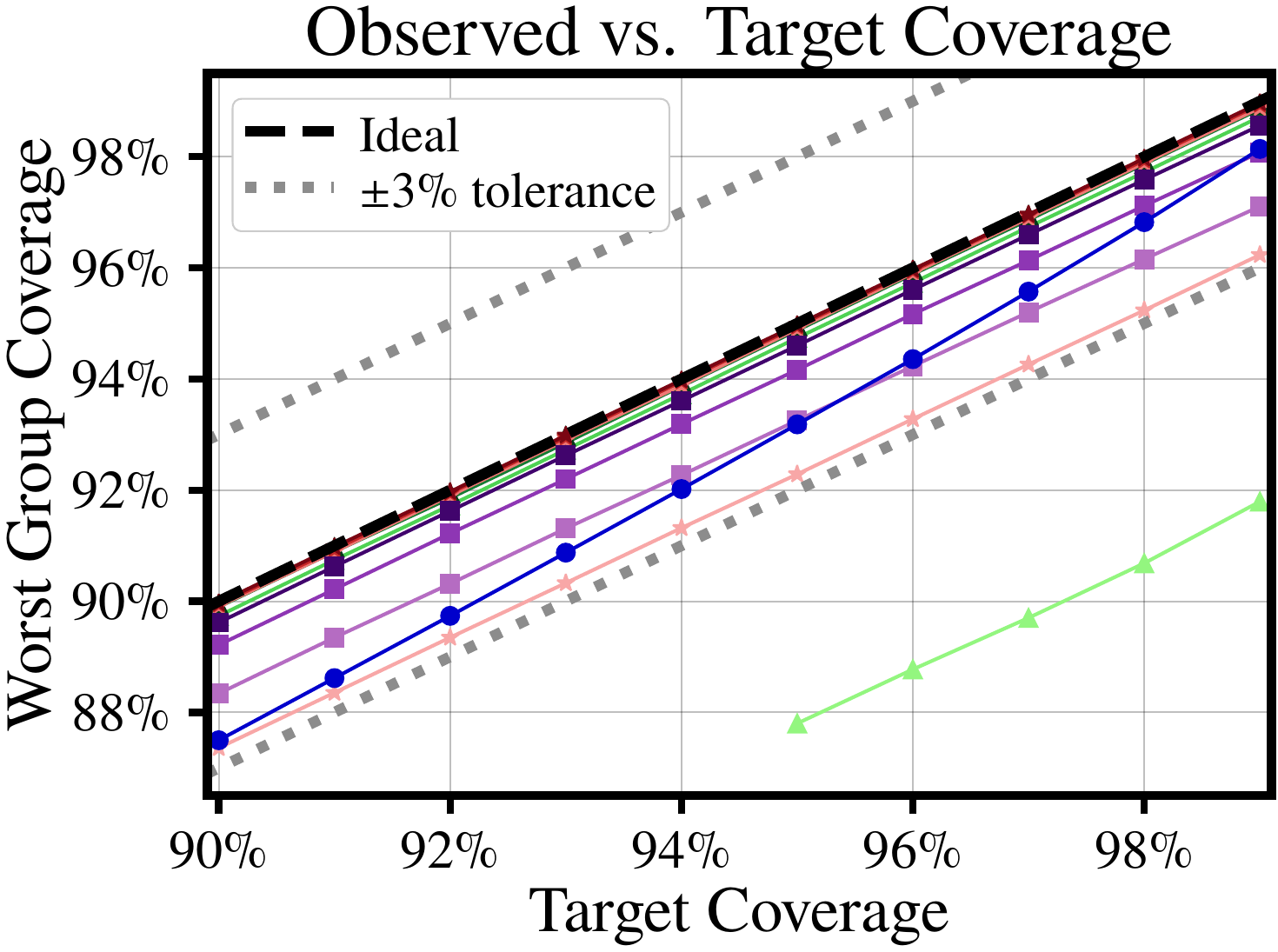}
            }
        \end{subfigure}
        \hspace{0.3em}
        \begin{subfigure}[t]{0.31\linewidth}
            \includegraphics[width=\linewidth]{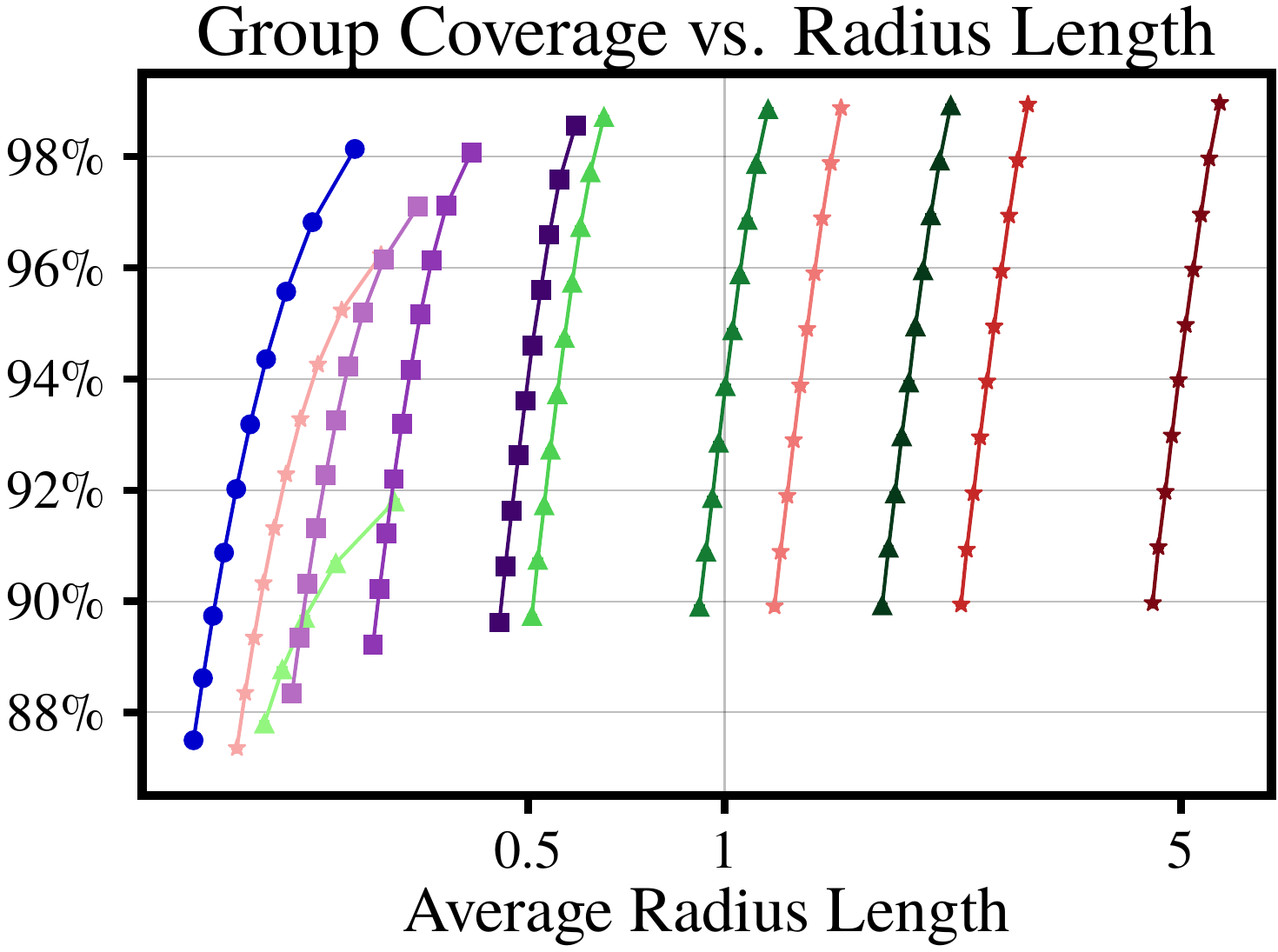}
        \end{subfigure}
        \hspace{0.2em}
        \begin{subfigure}[t]{0.31\linewidth}
            \includegraphics[width=\linewidth]{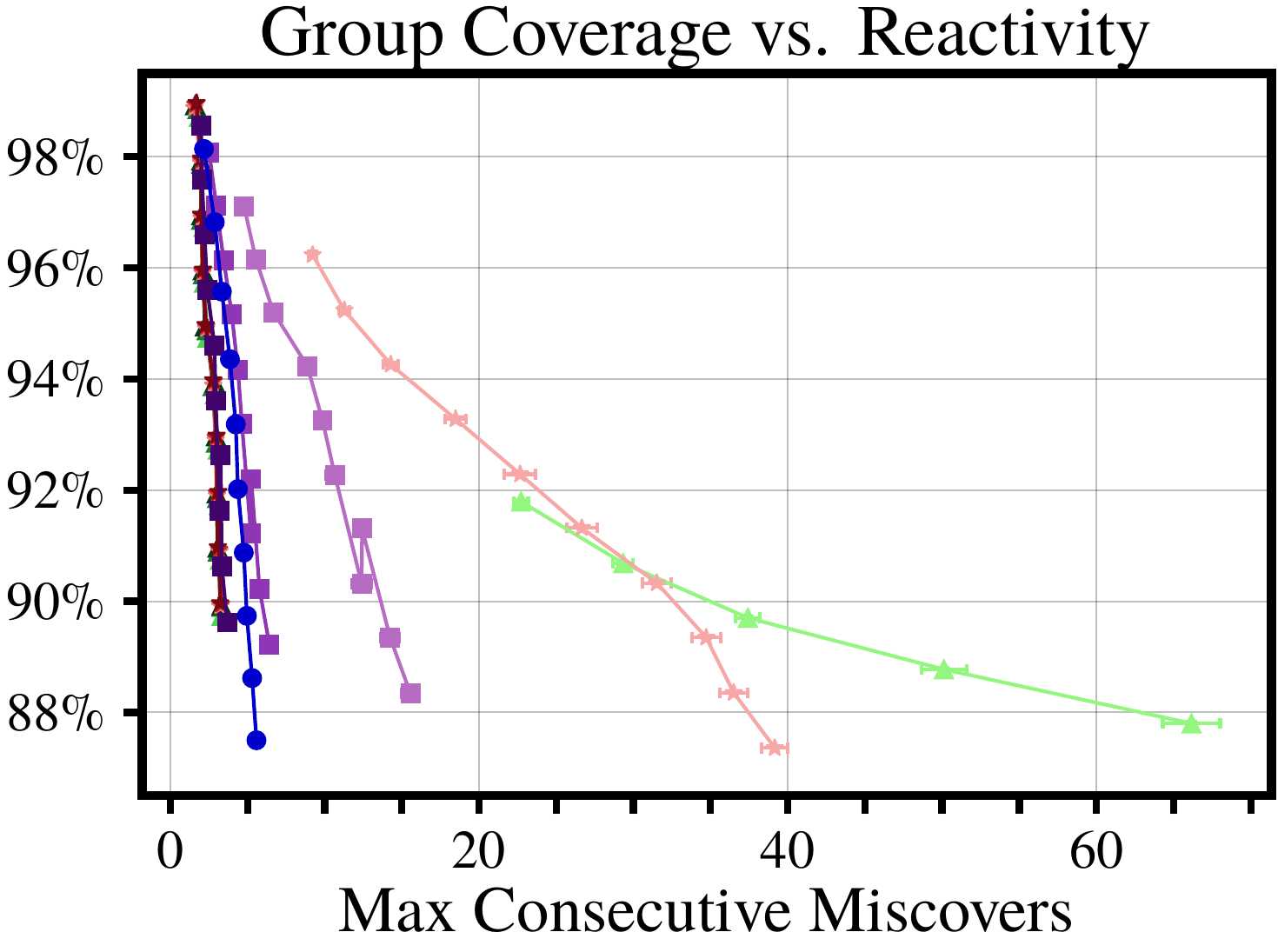}
        \end{subfigure}
        \caption{Results for synthetic setting with bounded scores with a changepoint.}
    \end{subfigure}
    \vspace{0.5em}

    \centering
    \begin{subfigure}{1\textwidth}
        \includegraphics[width=\linewidth]{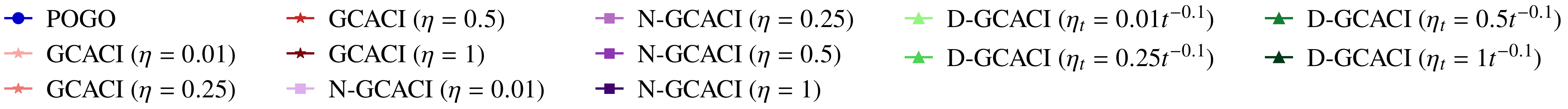}
    \end{subfigure}
 
    \caption{Performance across target coverage levels \(1-\alpha \in \{0.90, \ldots, 0.99\}\). {\bf Left:} worst-group coverage vs. target coverage. {\bf Middle/right:} worst-group coverage vs. radius length / max consecutive miscovers across groups (curves closer to the top-left indicate better performance; data point shown if observed coverage $\geq 0.87$). Error bars show \(x\)-axis standard error of the mean (SEM).}
    \vspace{-1em}
\end{figure}
\newpage

{\bf Results for $k=50$ number of groups, quadratic growing scores.}
\begin{figure}[h!]
    \begin{subfigure}[t]{\linewidth}
        \begin{subfigure}[t]{0.31\linewidth}
        \raisebox{0.08em}{
            \includegraphics[width=\linewidth]{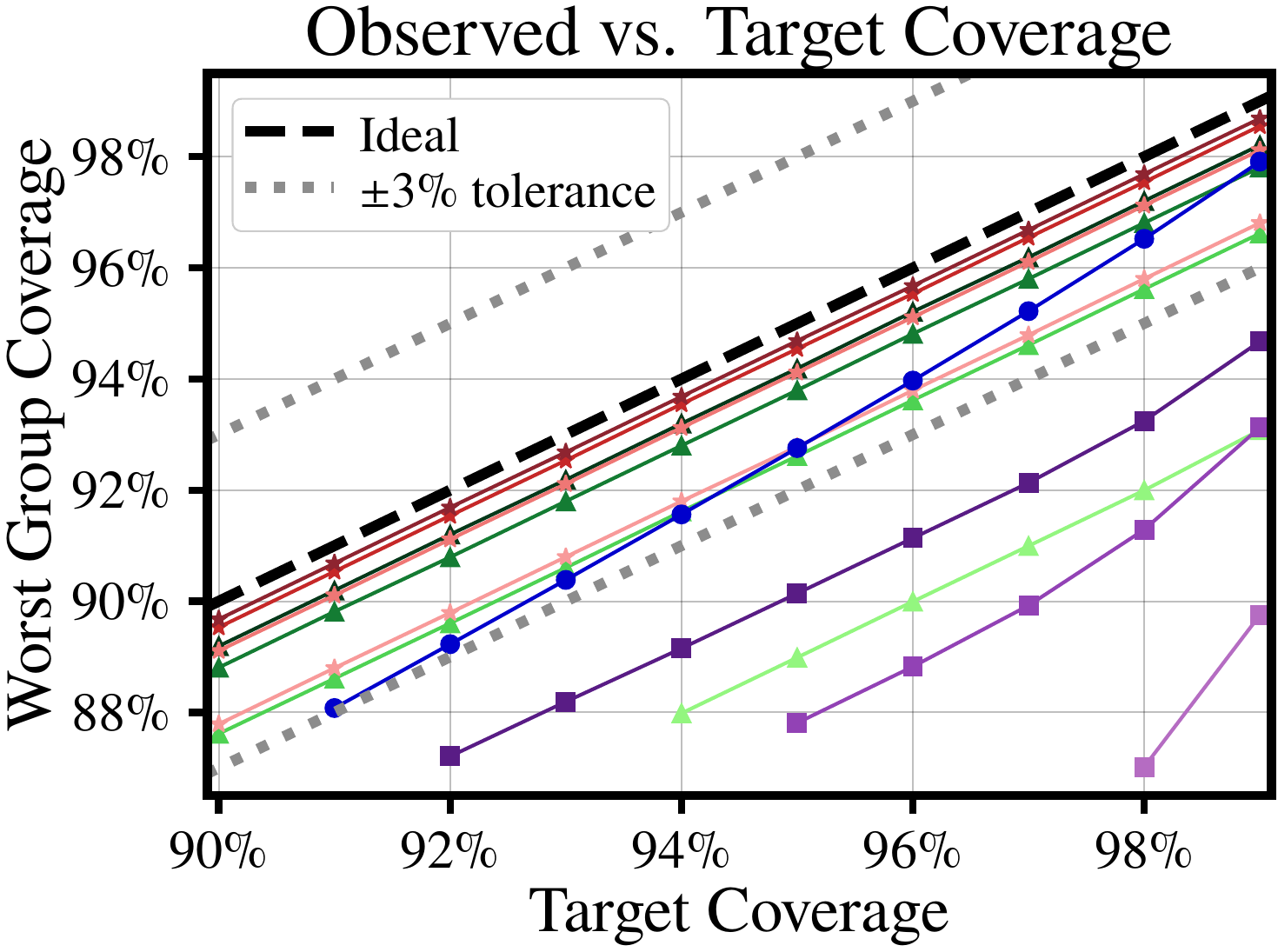}
            }
        \end{subfigure}
        \hspace{0.2em}
        \begin{subfigure}[t]{0.31\linewidth}
            \includegraphics[width=\linewidth]{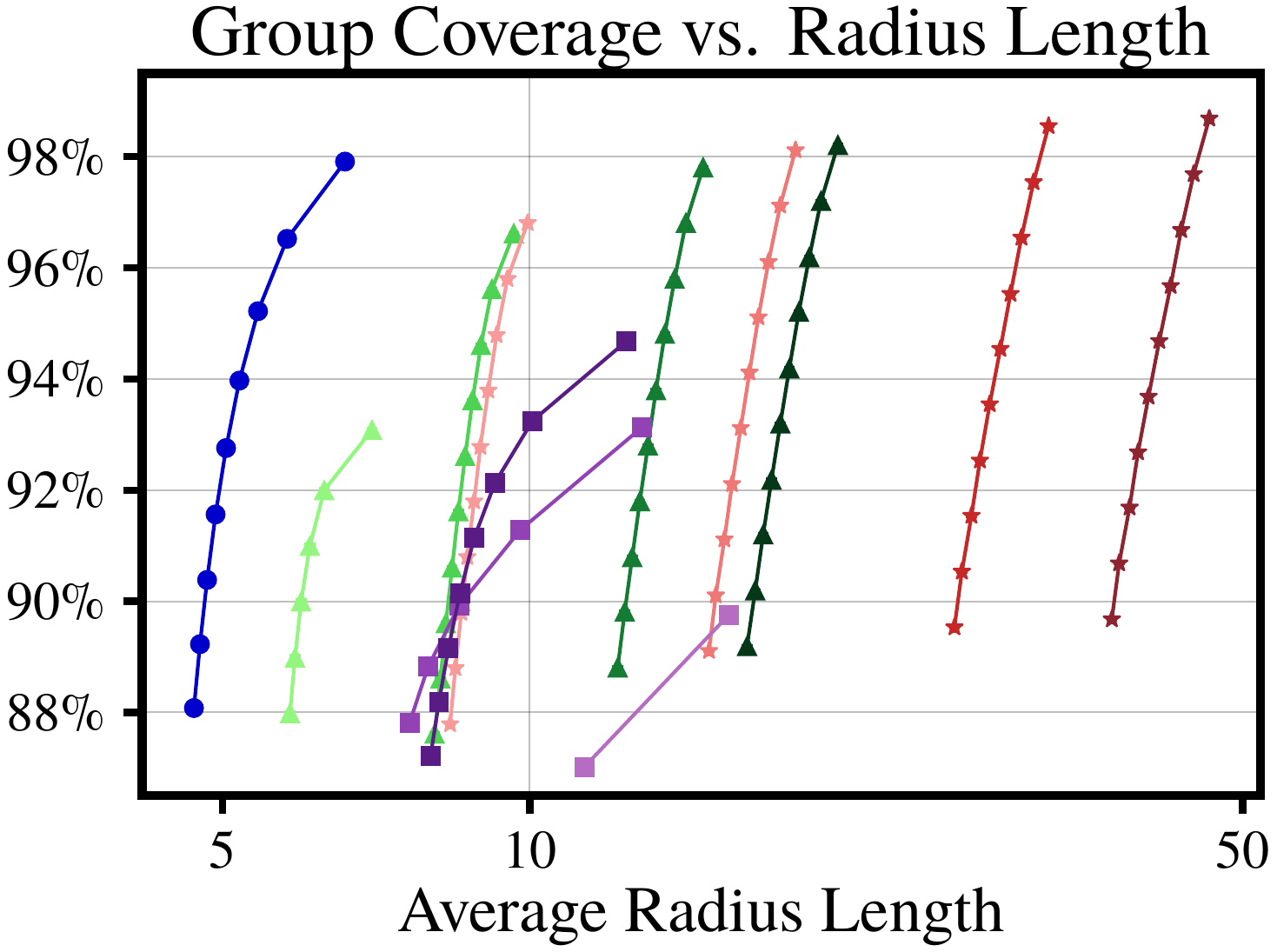}
        \end{subfigure}
        \hspace{0.2em}
        \begin{subfigure}[t]{0.31\linewidth}
            \includegraphics[width=\linewidth]{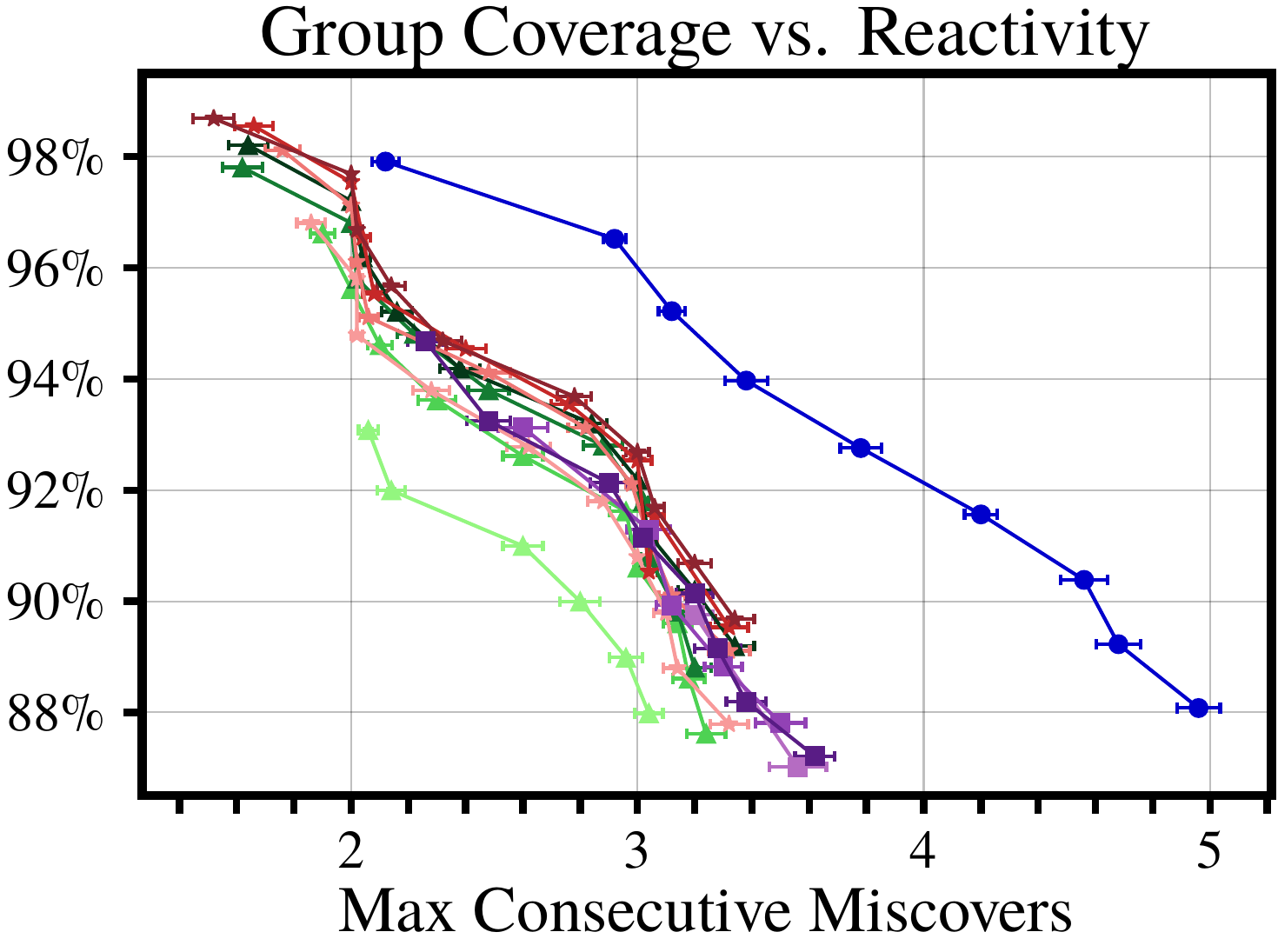}
        \end{subfigure}
        \caption{Results for the synthetic setting with slow growing quadratic scores ($q=2, A=5$).}
    \end{subfigure}
    \vspace{0.5em}
        
    \begin{subfigure}[t]{\linewidth}
        \begin{subfigure}[t]{0.31\linewidth}
        \raisebox{0.12em}{
            \includegraphics[width=\linewidth]{figures/iclr/synthetic/k=50/quad_growth_moderate_cov_vs_dcov.pdf}
            }
        \end{subfigure}
        \hspace{0.2em}
        \begin{subfigure}[t]{0.31\linewidth}
            \includegraphics[width=\linewidth]{figures/iclr/synthetic/k=50/quad_growth_moderate_cov_vs_rad.pdf}
        \end{subfigure}
        \hspace{0.2em}
        \begin{subfigure}[t]{0.31\linewidth}
            \includegraphics[width=\linewidth]{figures/iclr/synthetic/k=50/quad_growth_moderate_cov_vs_adapt.pdf}
        \end{subfigure}
        \caption{Results for the synthetic setting with moderately growing quadratic scores ($q=2, A=25$).}
    \end{subfigure}
    \vspace{0.5em}

    \begin{subfigure}[t]{\linewidth}
        \begin{subfigure}[t]{0.31\linewidth}
        \raisebox{0.12em}{
            \includegraphics[width=\linewidth]{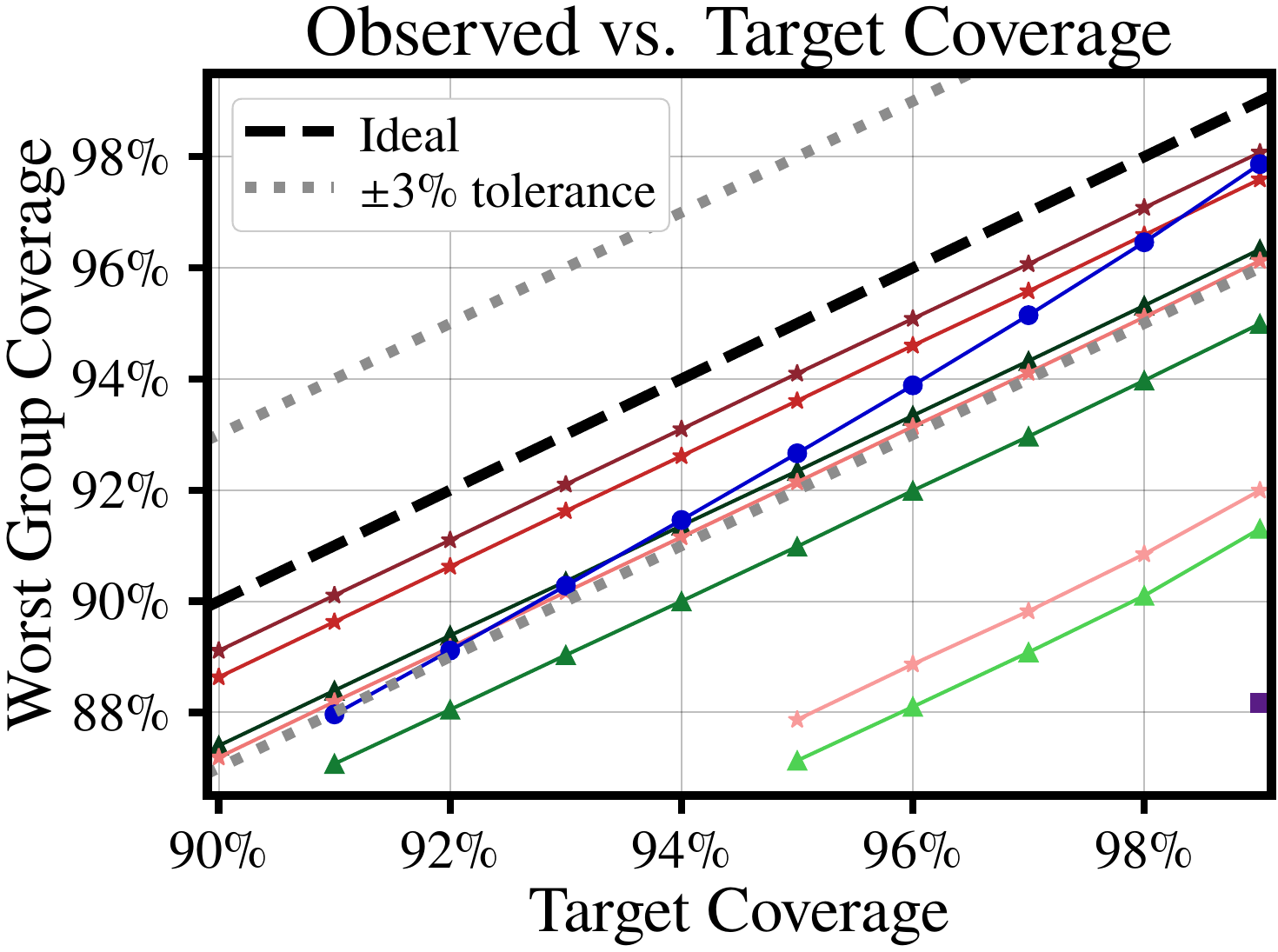}
            }
        \end{subfigure}
        \hspace{0.2em}
        \begin{subfigure}[t]{0.31\linewidth}
            \includegraphics[width=\linewidth]{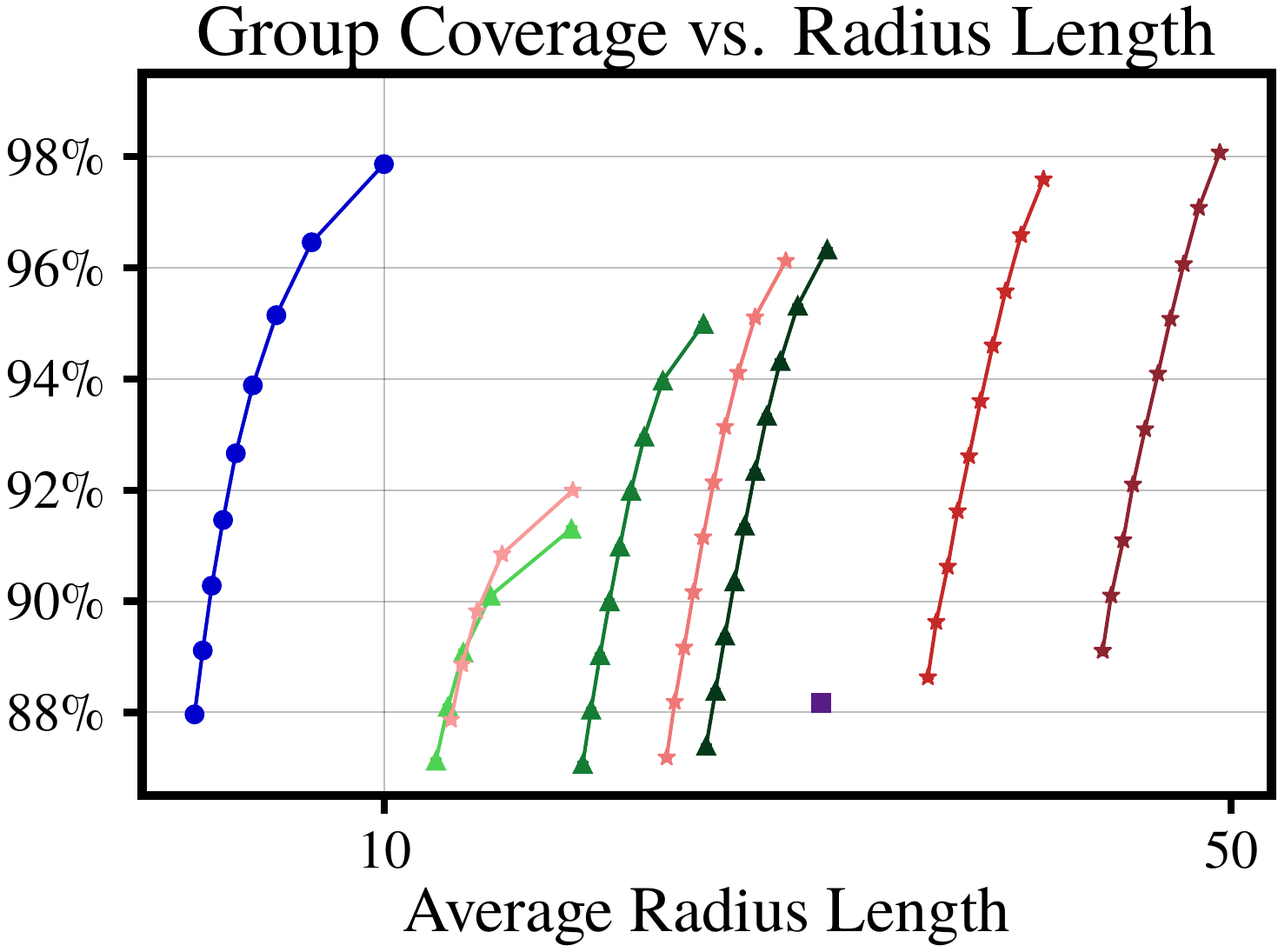}
        \end{subfigure}
        \hspace{0.2em}
        \begin{subfigure}[t]{0.31\linewidth}
            \includegraphics[width=\linewidth]{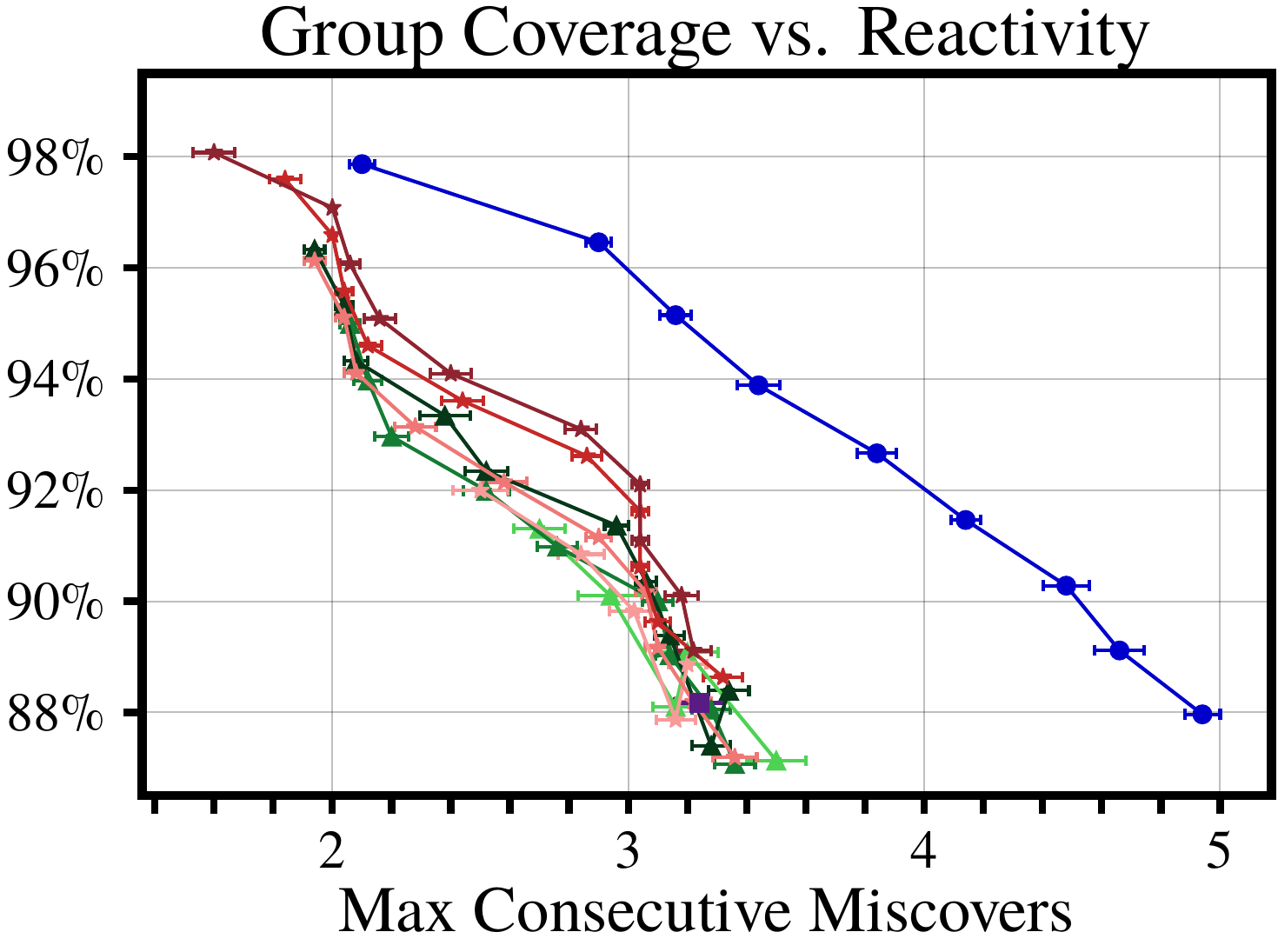}
        \end{subfigure}
        \caption{Results for the synthetic setting with extreme growing quadratic scores ($q=2, A=100$).}
    \end{subfigure}
    \vspace{0.5em}

    \centering
    \begin{subfigure}{1\textwidth}
        \includegraphics[width=\linewidth]{figures/iclr/legend_syn_2_horz.pdf}
    \end{subfigure}
 
    \caption{Performance across target coverage levels \(1-\alpha \in \{0.90, \ldots, 0.99\}\). {\bf Left:} worst-group coverage vs. target coverage. {\bf Middle/right:} worst-group coverage vs. radius length / max consecutive miscovers across groups (curves closer to the top-left indicate better performance; data point shown if observed coverage $\geq 0.87$). Error bars show \(x\)-axis standard error of the mean (SEM).}
    \vspace{-1em}
\end{figure}
\newpage

{\bf Results for $k=50$ number of groups, cubic growing scores.}
\begin{figure}[h!]
    \begin{subfigure}[t]{\linewidth}
        \begin{subfigure}[t]{0.31\linewidth}
        \raisebox{0.08em}{
            \includegraphics[width=\linewidth]{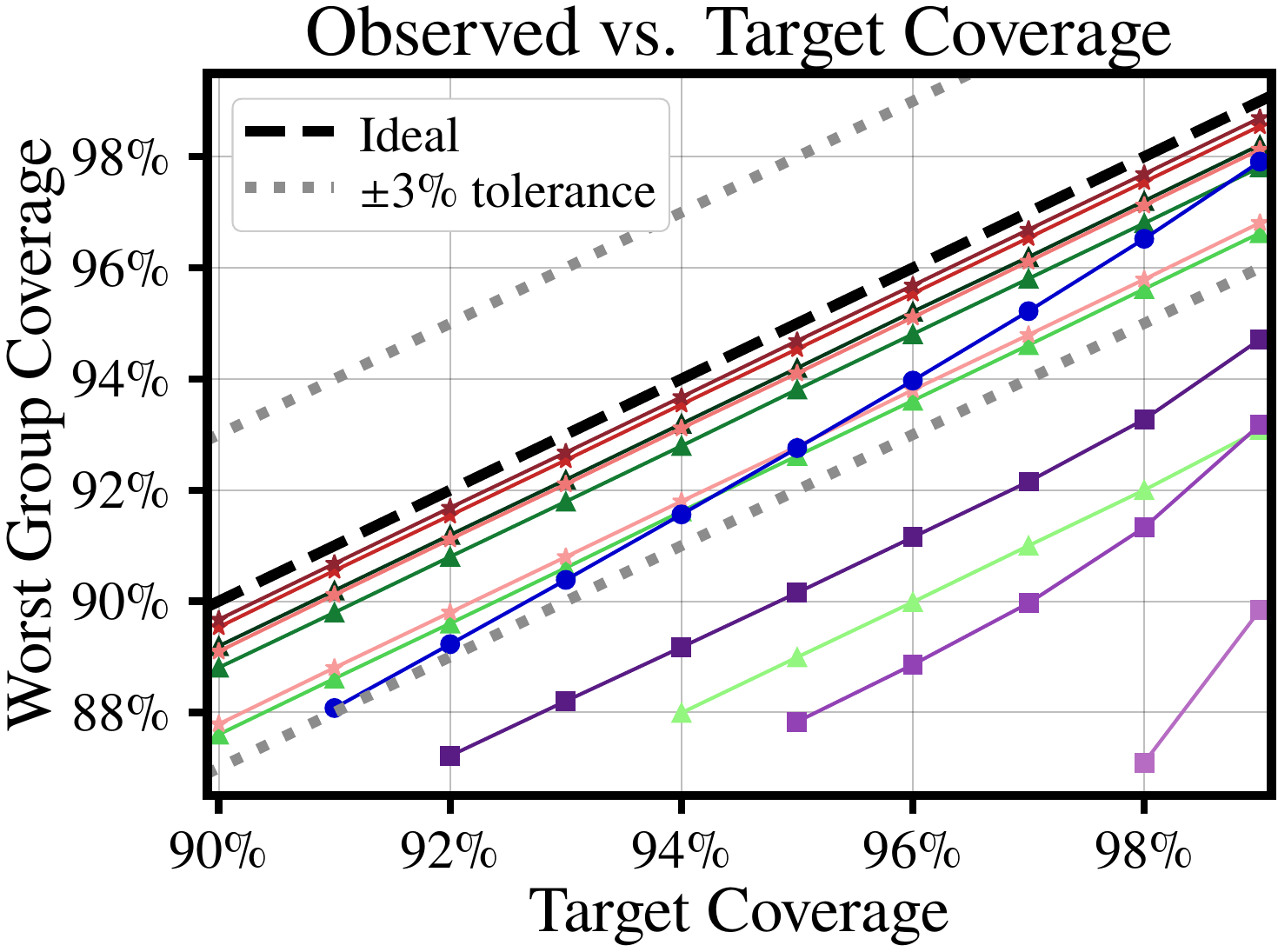}
            }
        \end{subfigure}
        \hspace{0.2em}
        \begin{subfigure}[t]{0.31\linewidth}
            \includegraphics[width=\linewidth]{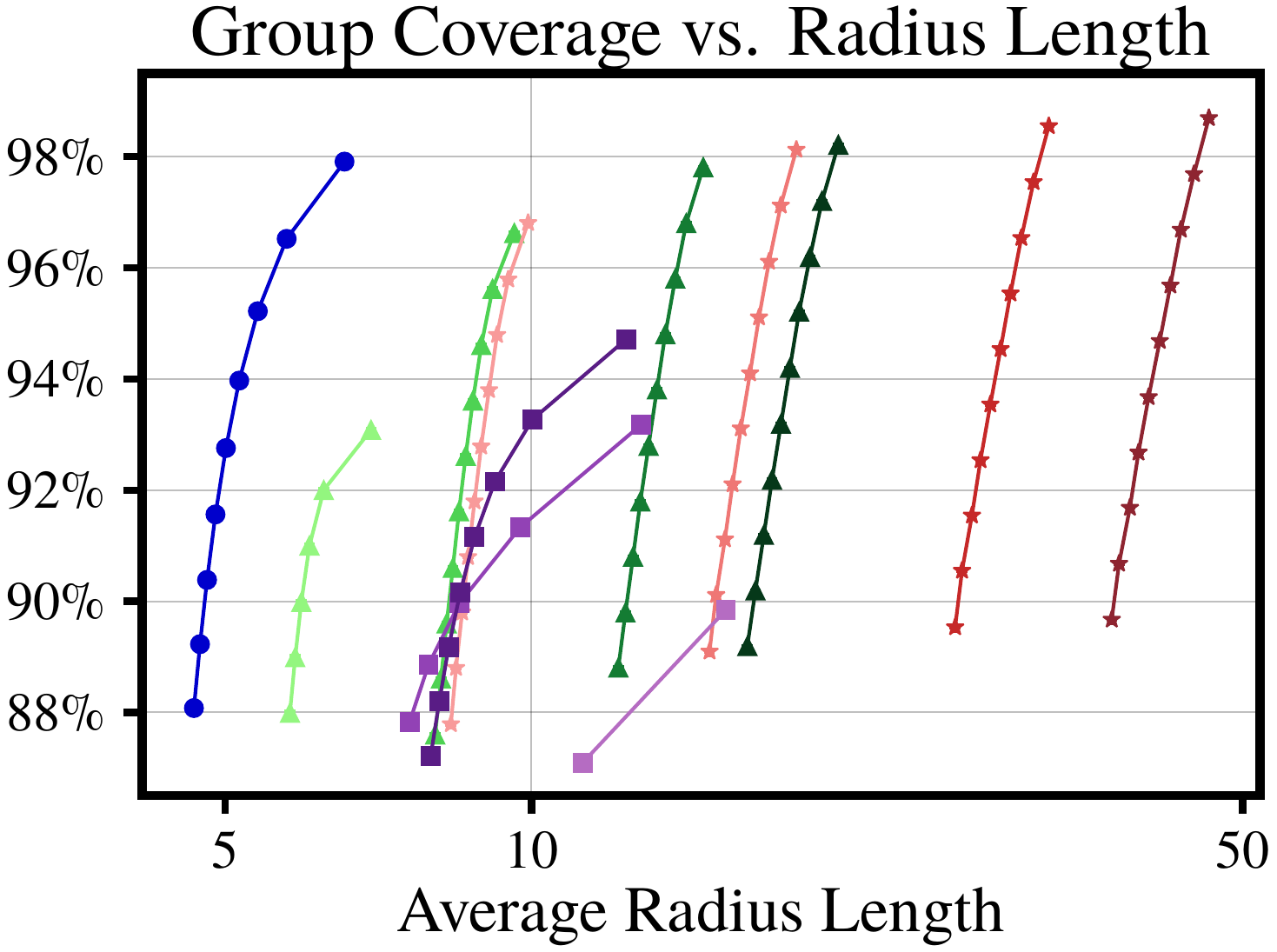}
        \end{subfigure}
        \hspace{0.2em}
        \begin{subfigure}[t]{0.31\linewidth}
            \includegraphics[width=\linewidth]{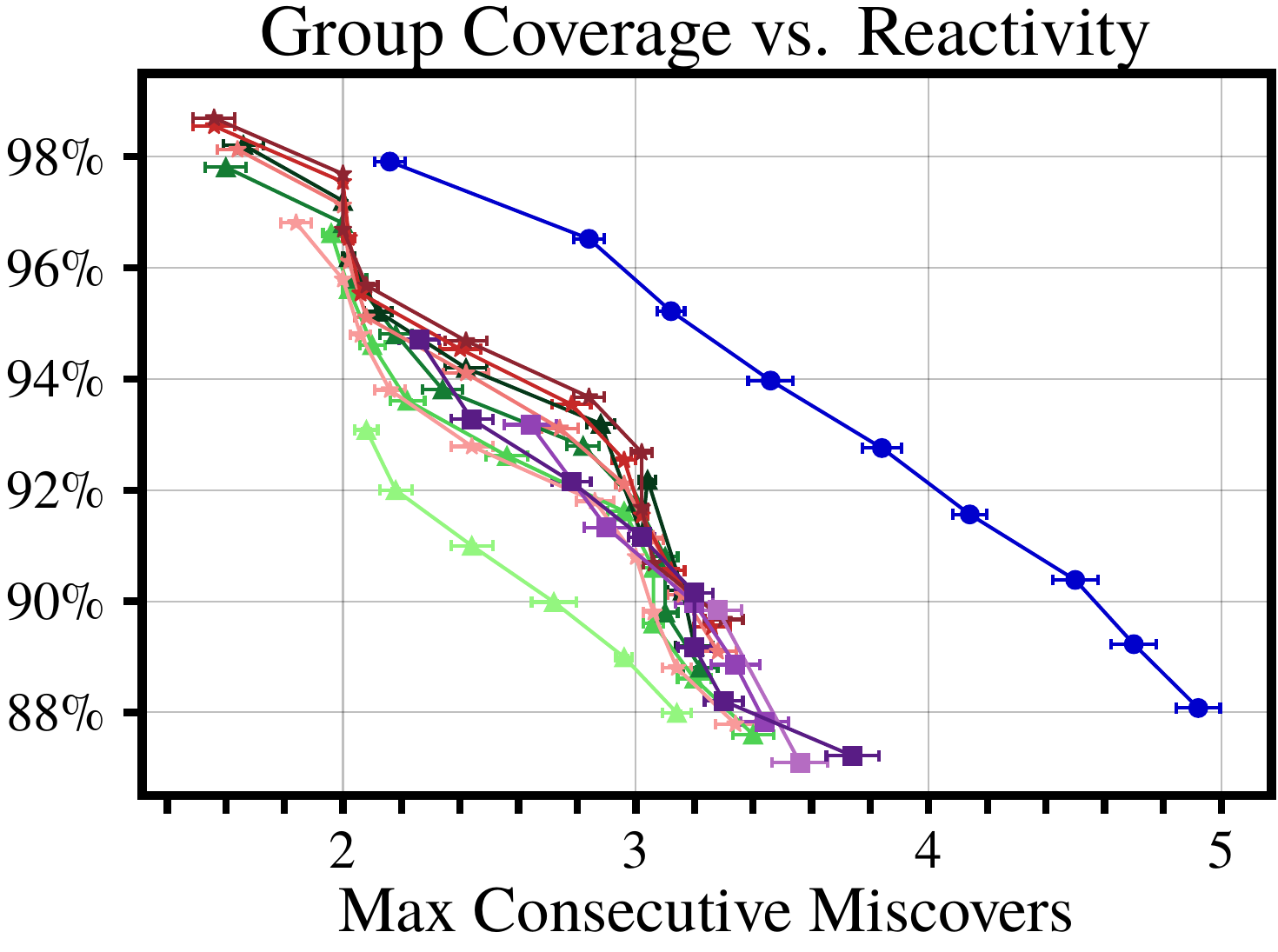}
        \end{subfigure}
        \caption{Results for the synthetic setting with slow growing cubic scores ($q=3, A=5$).}
    \end{subfigure}
    \vspace{0.5em}
        
    \begin{subfigure}[t]{\linewidth}
        \begin{subfigure}[t]{0.31\linewidth}
        \raisebox{0.12em}{
            \includegraphics[width=\linewidth]{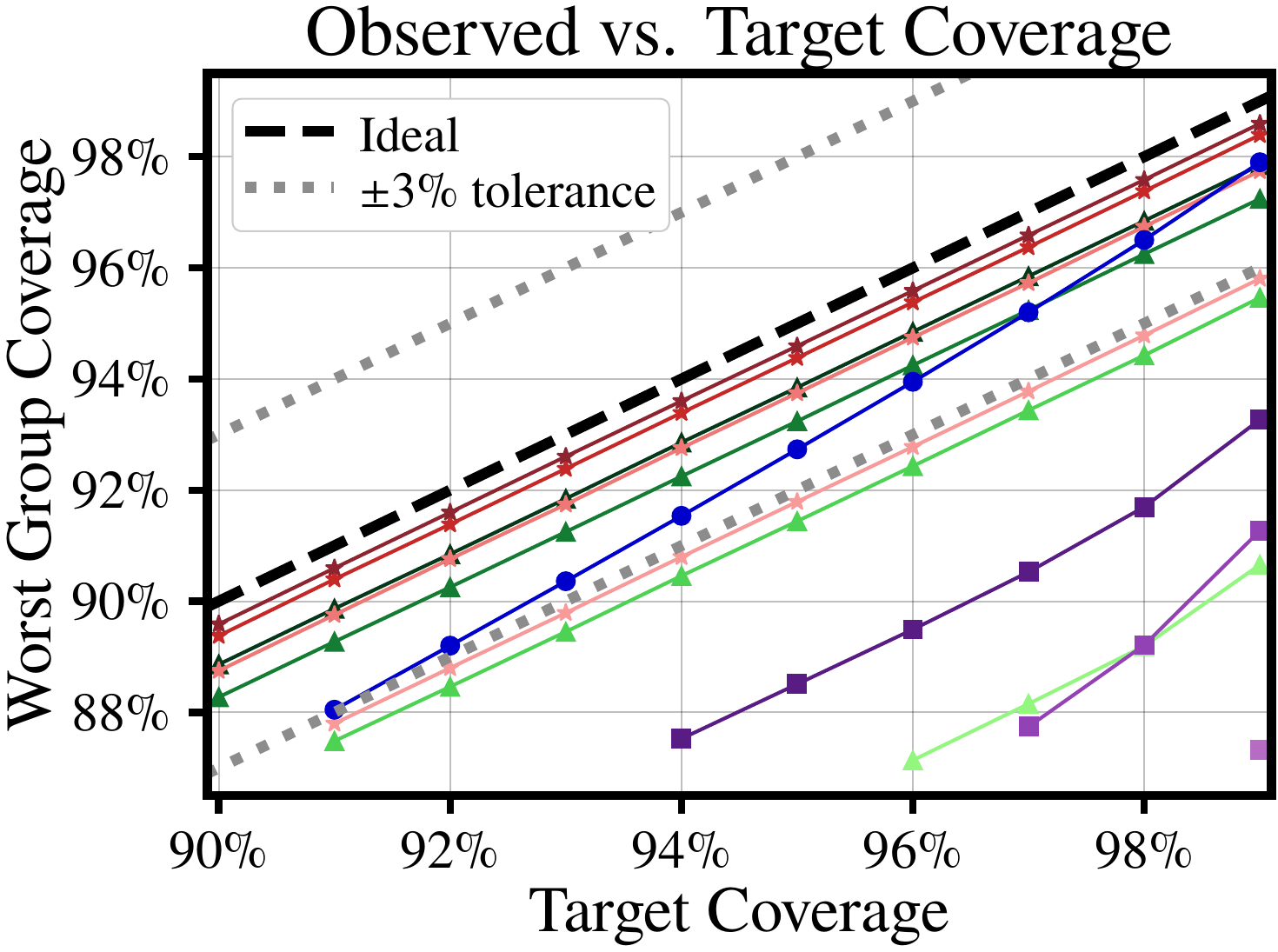}
            }
        \end{subfigure}
        \hspace{0.2em}
        \begin{subfigure}[t]{0.31\linewidth}
            \includegraphics[width=\linewidth]{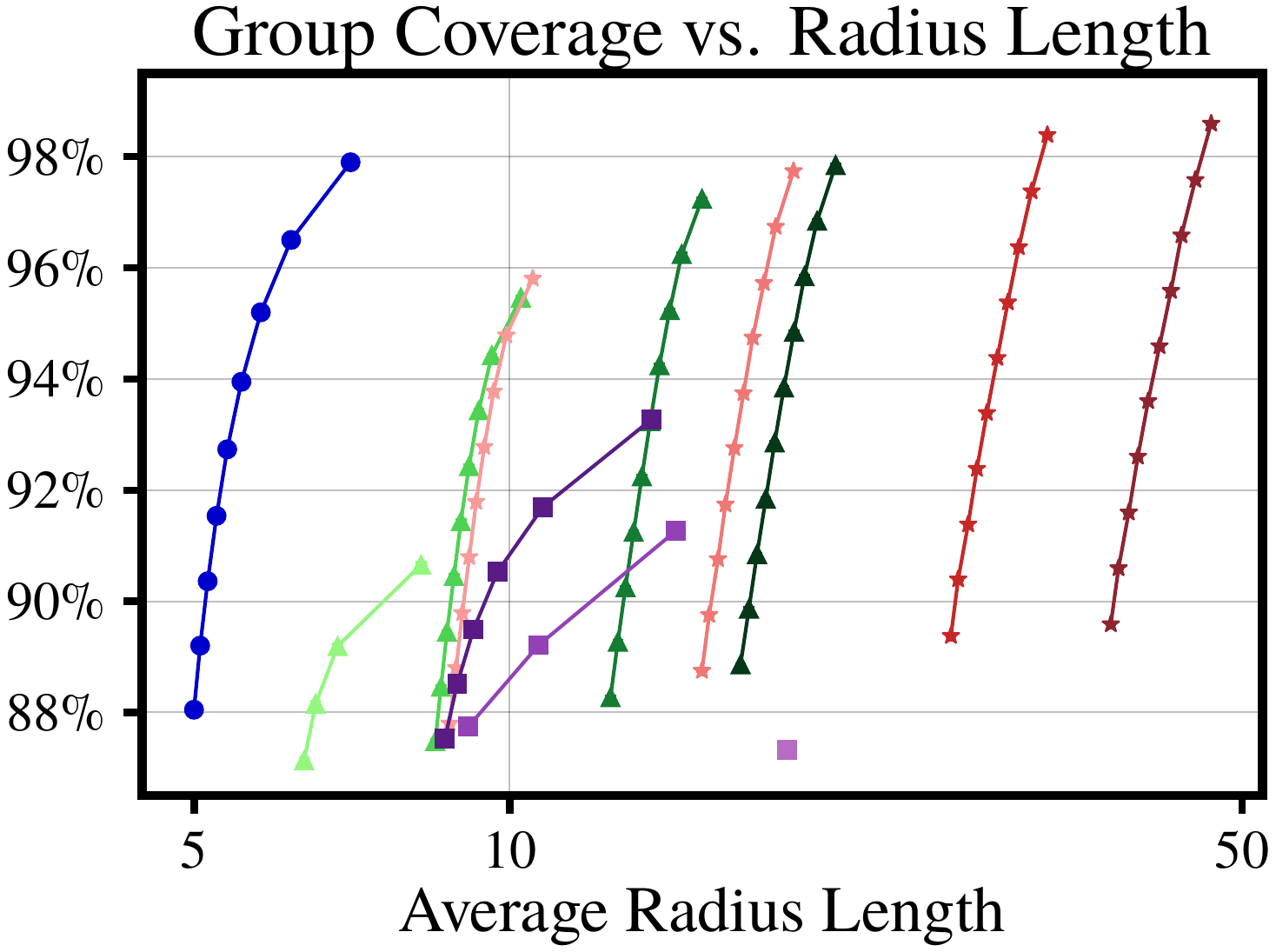}
        \end{subfigure}
        \hspace{0.2em}
        \begin{subfigure}[t]{0.31\linewidth}
            \includegraphics[width=\linewidth]{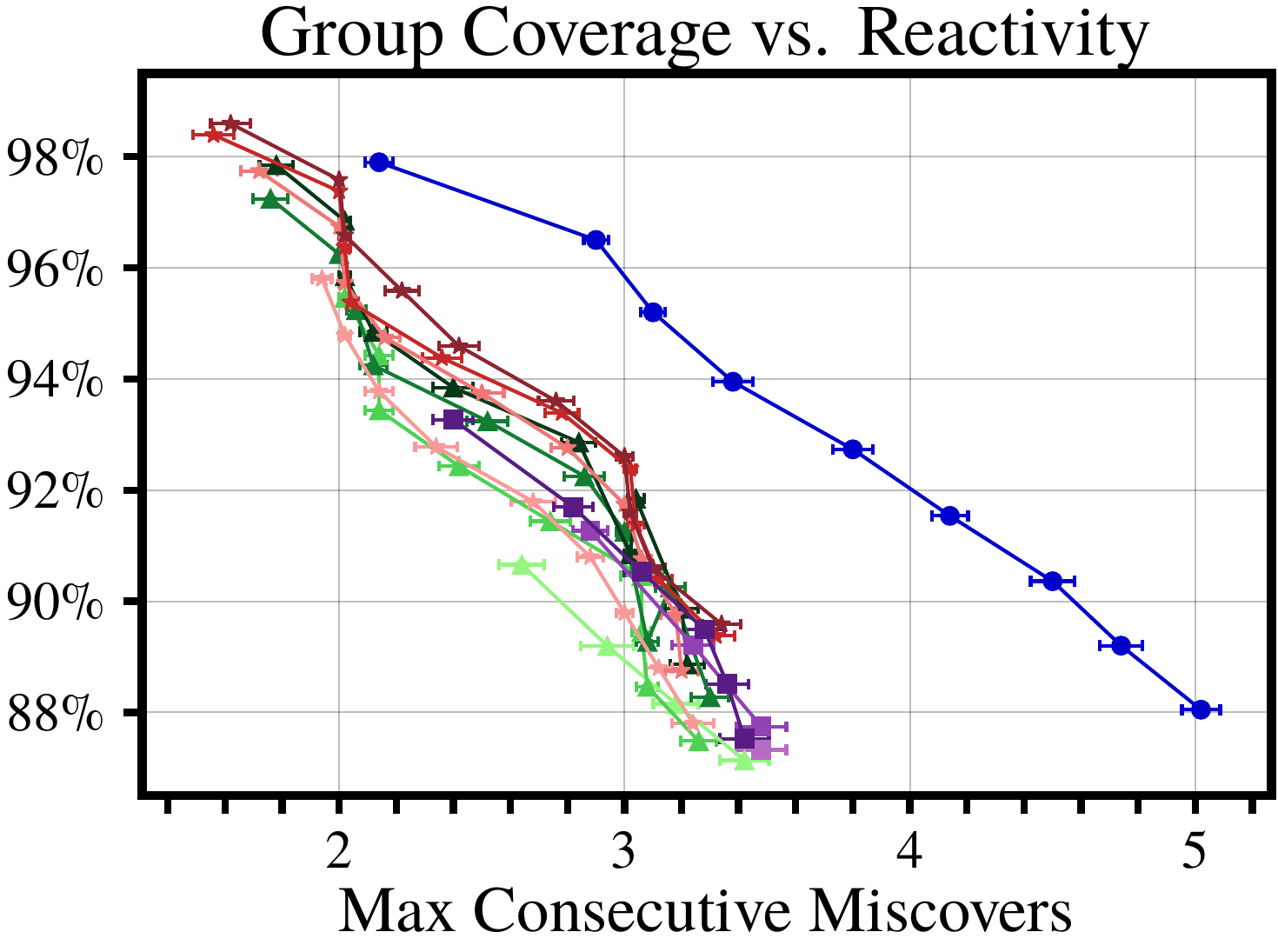}
        \end{subfigure}
        \caption{Results for the synthetic setting with moderately growing cubic scores ($q=3, A=25$).}
    \end{subfigure}
    \vspace{0.5em}

    \begin{subfigure}[t]{\linewidth}
        \begin{subfigure}[t]{0.31\linewidth}
        \raisebox{0.12em}{
            \includegraphics[width=\linewidth]{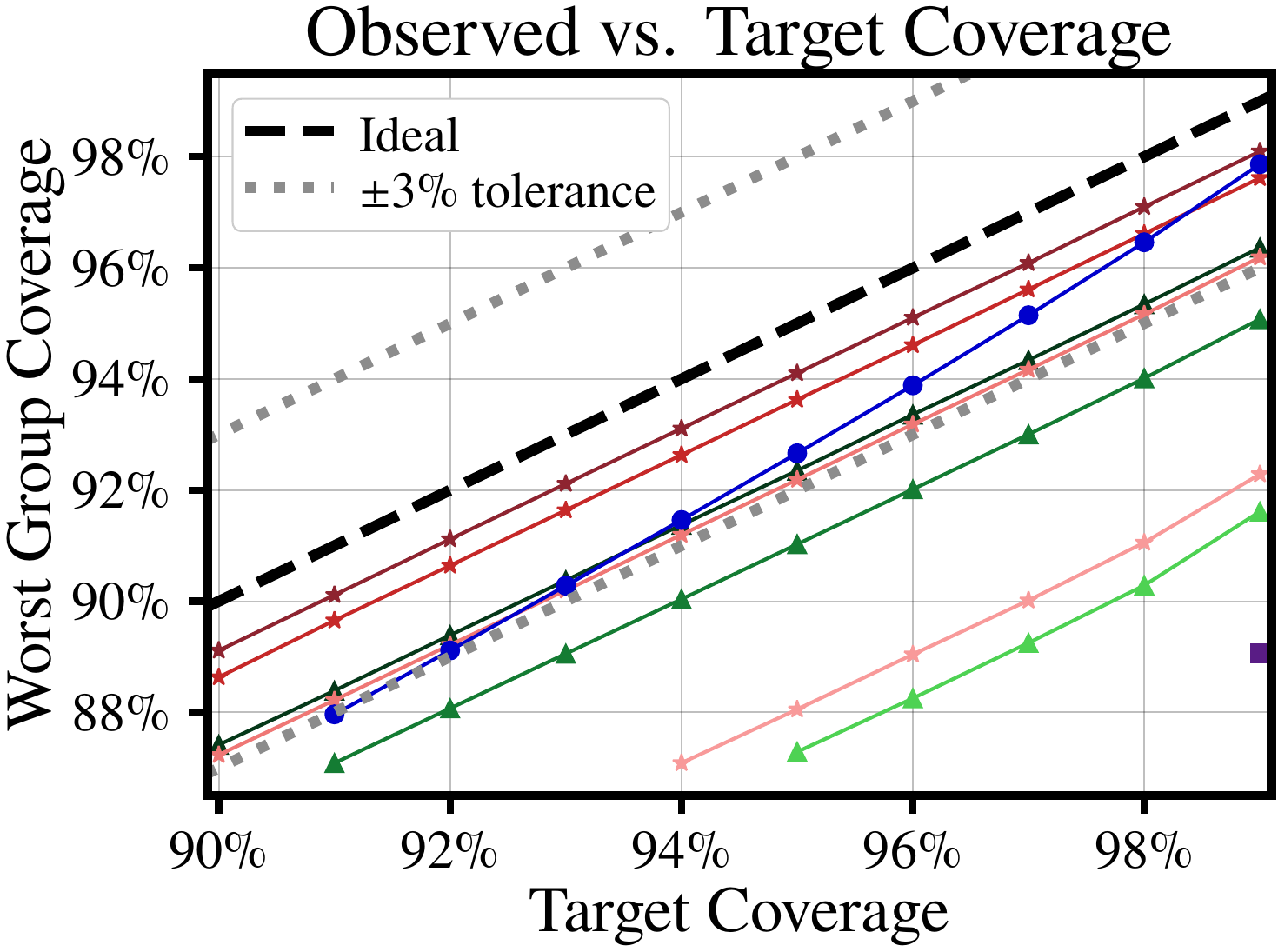}
            }
        \end{subfigure}
        \hspace{0.2em}
        \begin{subfigure}[t]{0.31\linewidth}
            \includegraphics[width=\linewidth]{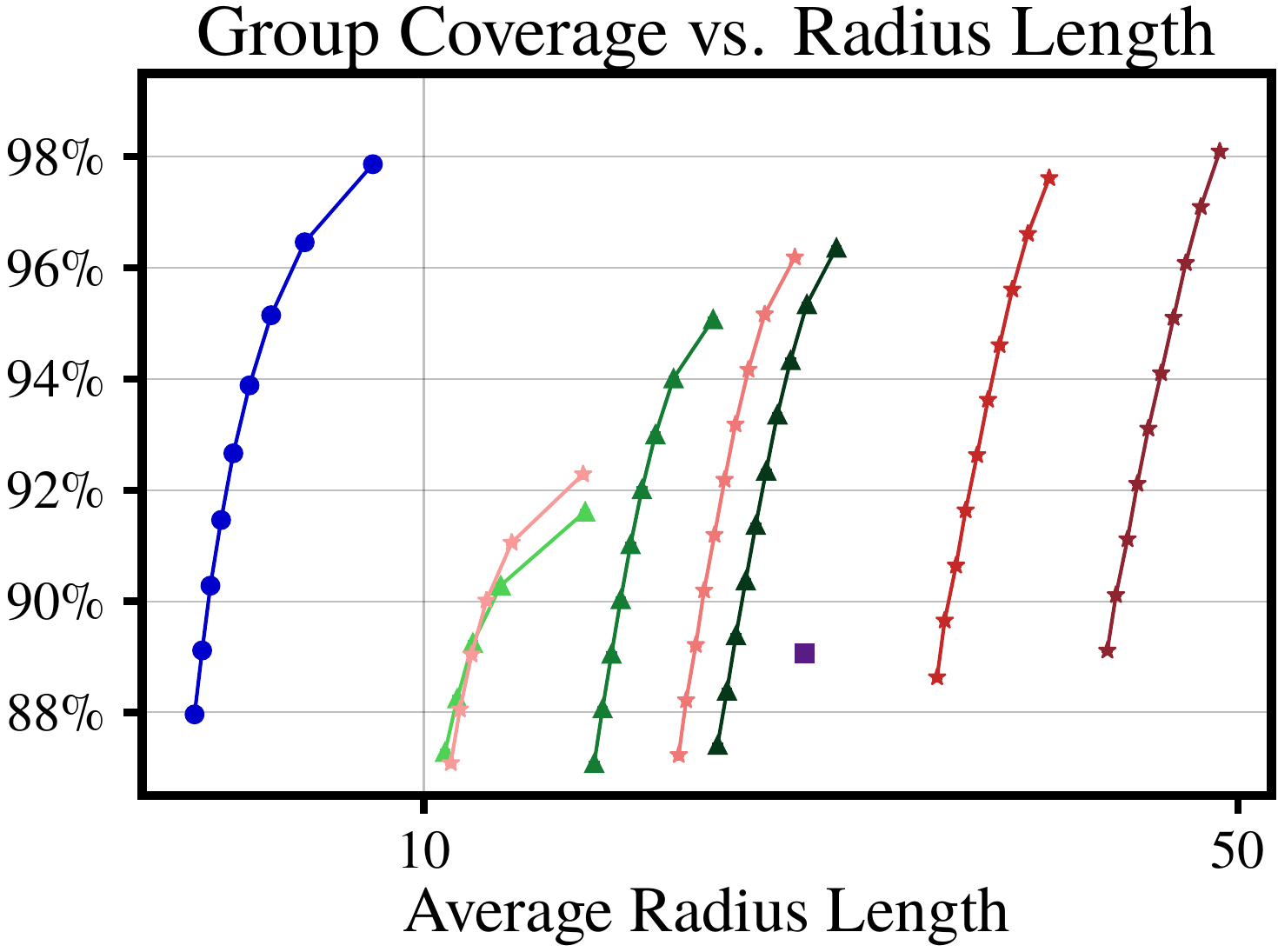}
        \end{subfigure}
        \hspace{0.2em}
        \begin{subfigure}[t]{0.31\linewidth}
            \includegraphics[width=\linewidth]{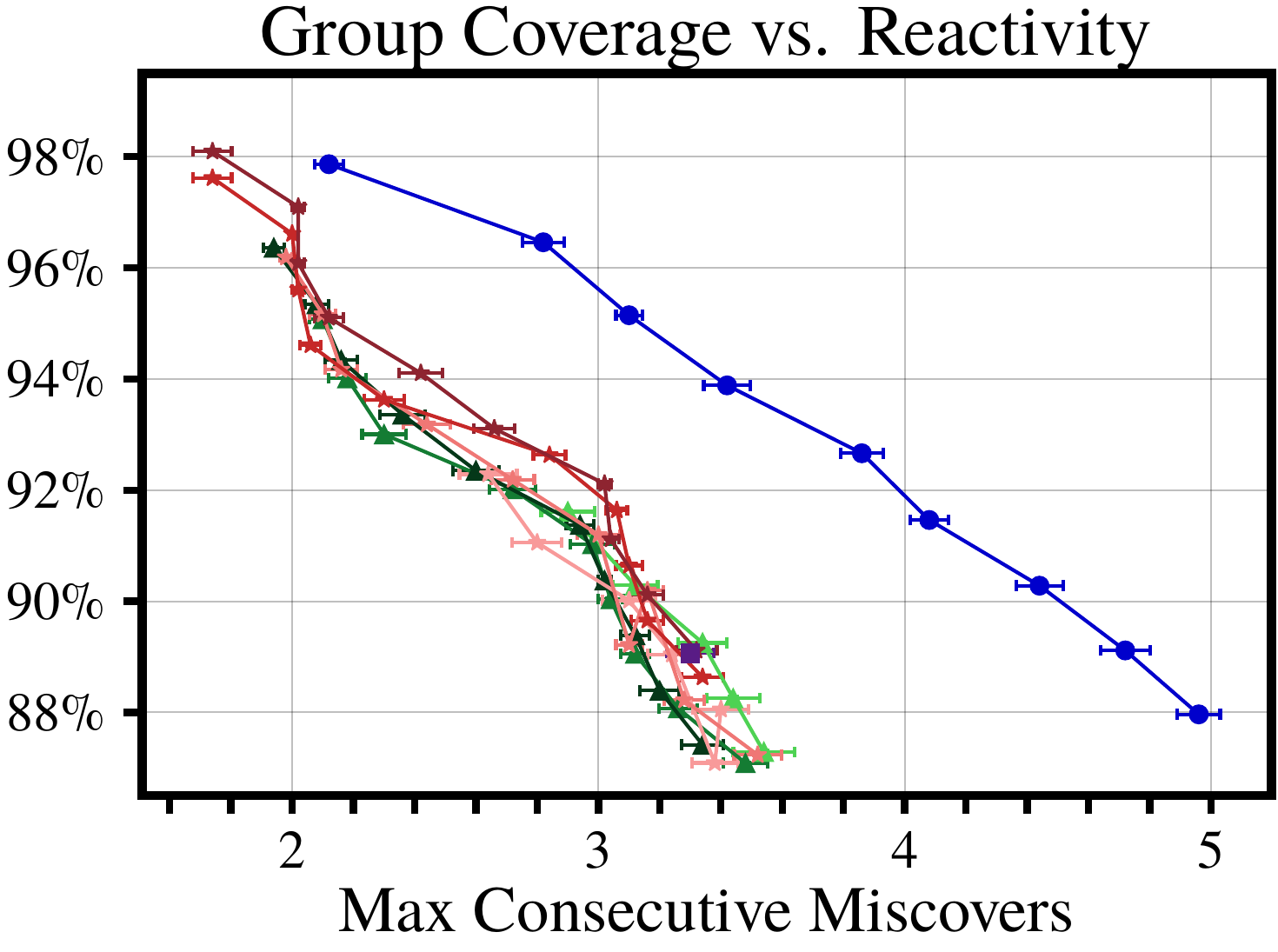}
        \end{subfigure}
        \caption{Results for the synthetic setting with extreme growing cubic scores ($q=3, A=100$).}
    \end{subfigure}
    \vspace{0.5em}

    \centering
    \begin{subfigure}{1\textwidth}
        \includegraphics[width=\linewidth]{figures/iclr/legend_syn_2_horz.pdf}
    \end{subfigure}

    \caption{Performance across target coverage levels \(1-\alpha \in \{0.90, \ldots, 0.99\}\). {\bf Left:} worst-group coverage vs. target coverage. {\bf Middle/right:} worst-group coverage vs. radius length / max consecutive miscovers across groups (curves closer to the top-left indicate better performance; data point shown if observed coverage $\geq 0.87$). Error bars show \(x\)-axis standard error of the mean (SEM).}
    \vspace{-1em}
\end{figure}
\newpage

{\bf Results for $k=100$ number of groups, bounded scores}
\begin{figure}[h!]
    \begin{subfigure}[t]{\linewidth}
        \begin{subfigure}[t]{0.31\linewidth}
        \raisebox{0.12em}{
            \includegraphics[width=\linewidth]{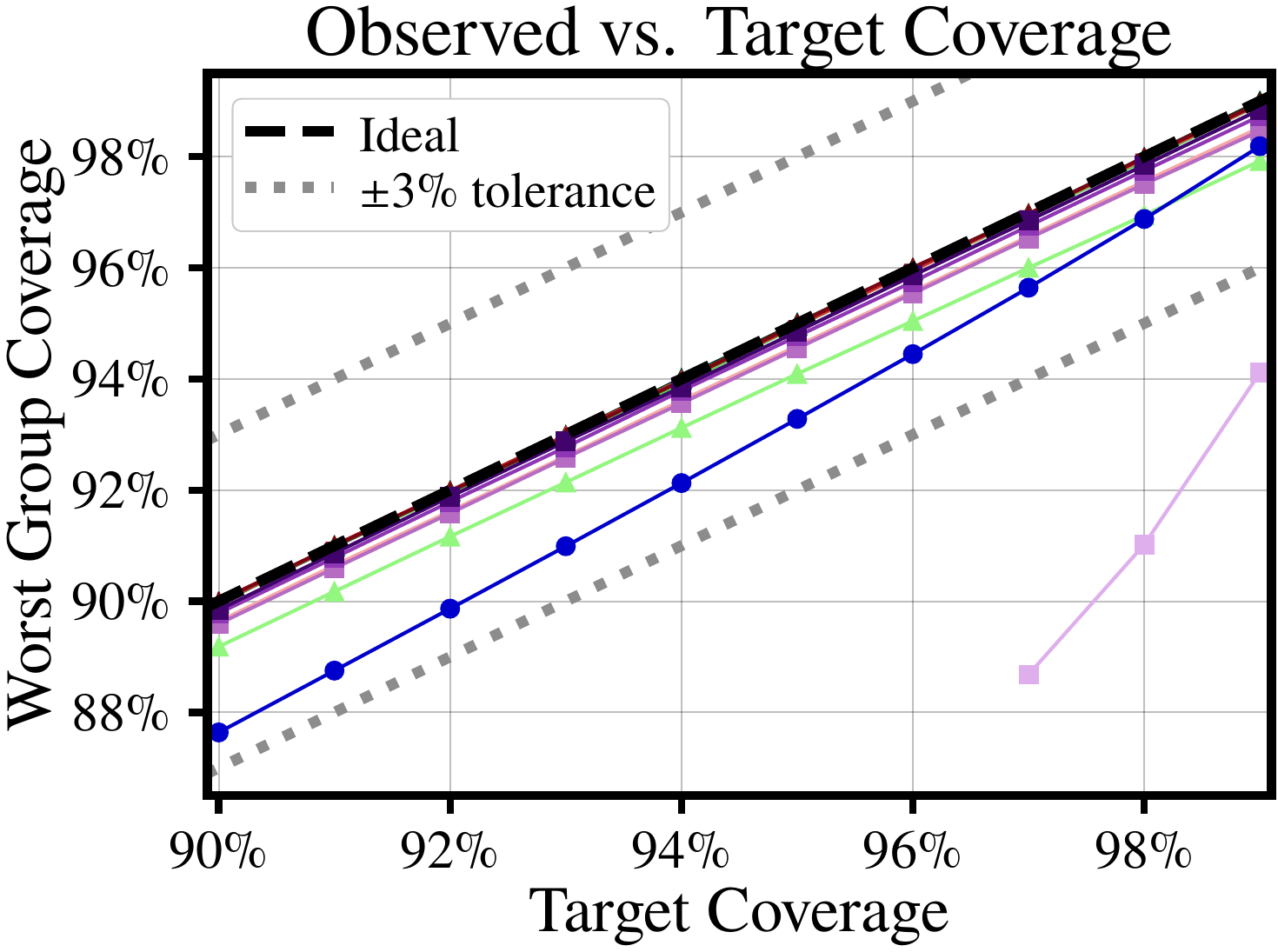}
            }
        \end{subfigure}
        \hspace{0.2em}
        \begin{subfigure}[t]{0.31\linewidth}
            \includegraphics[width=\linewidth]{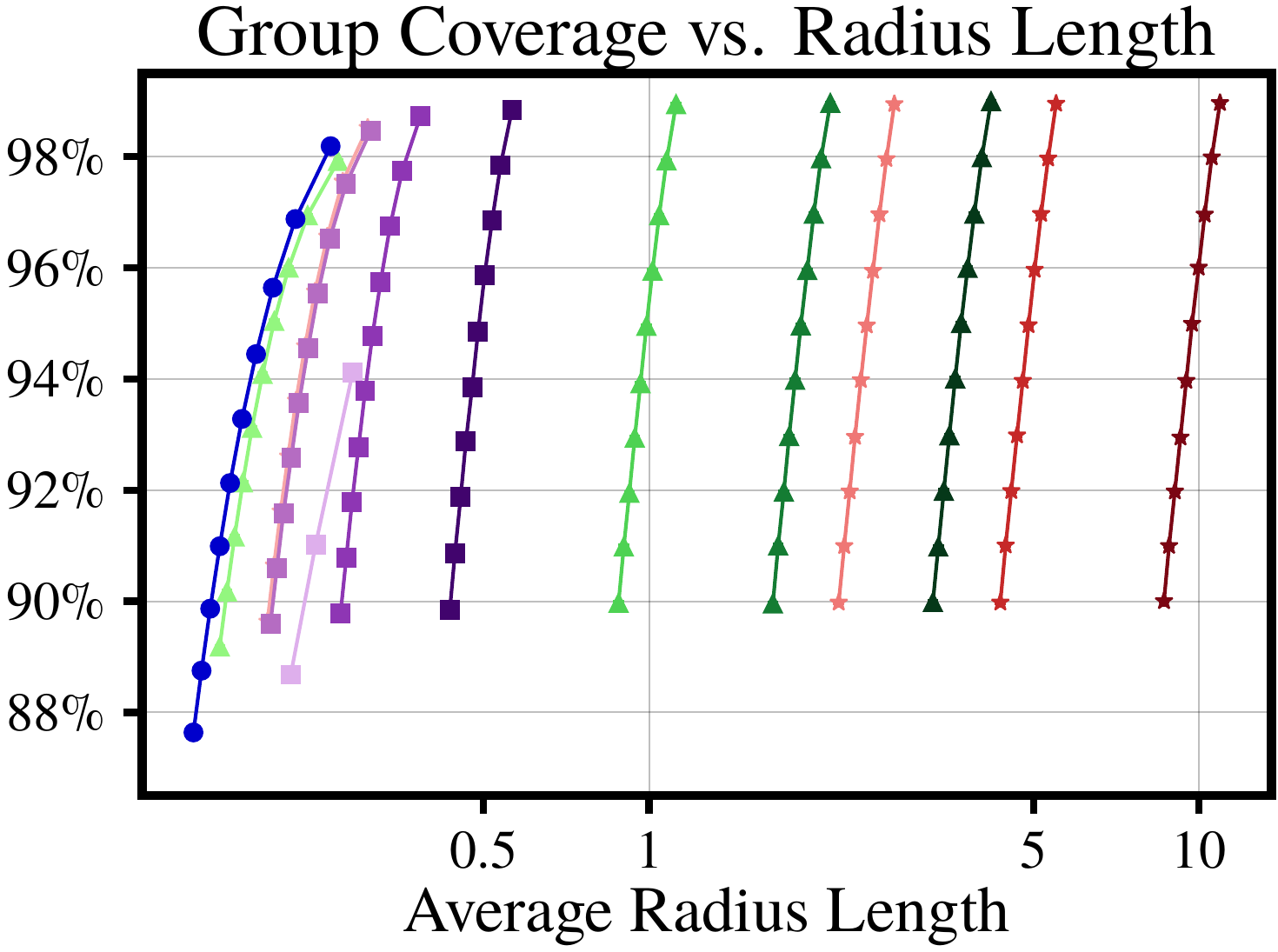}
        \end{subfigure}
        \hspace{0.2em}
        \begin{subfigure}[t]{0.31\linewidth}
            \includegraphics[width=\linewidth]{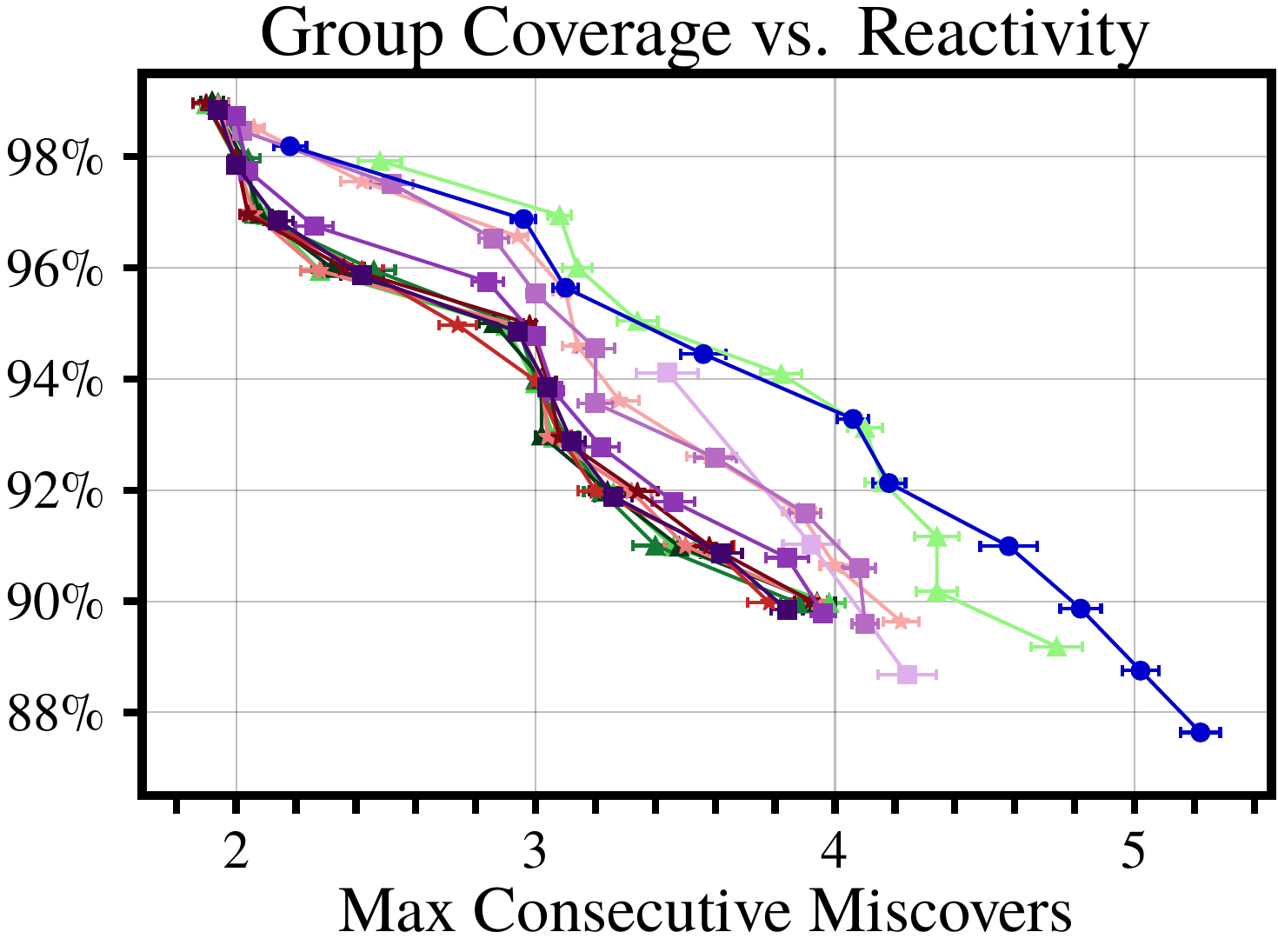}
        \end{subfigure}
        \caption{Results for synthetic setting with bounded, gradually varying scores.}
    \end{subfigure}
    \vspace{0.5em}
        
    \begin{subfigure}[t]{\linewidth}
        \begin{subfigure}[t]{0.31\linewidth}
        \raisebox{0.05em}{
            \includegraphics[width=\linewidth]{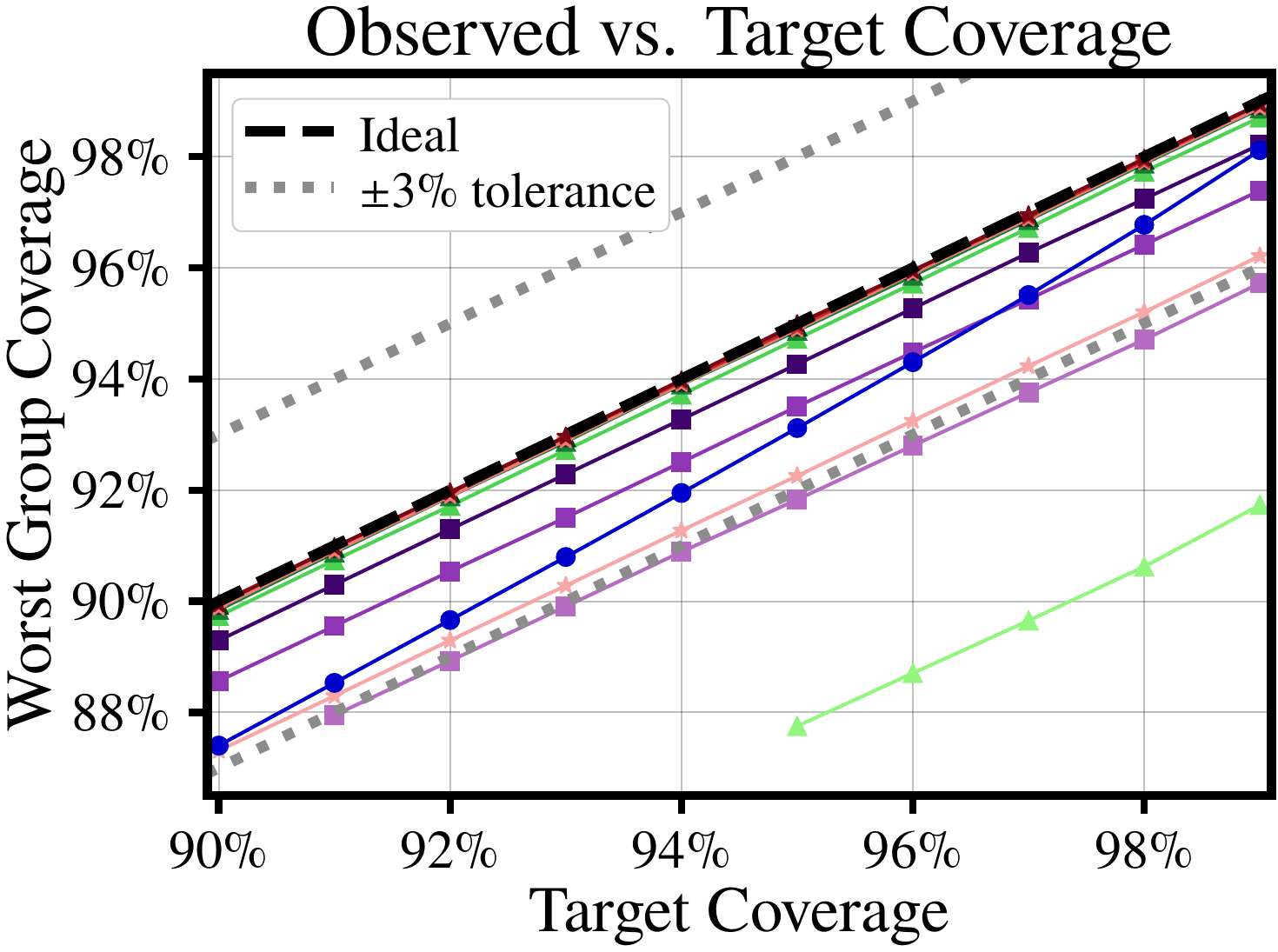}
            }
        \end{subfigure}
        \hspace{0.3em}
        \begin{subfigure}[t]{0.31\linewidth}
            \includegraphics[width=\linewidth]{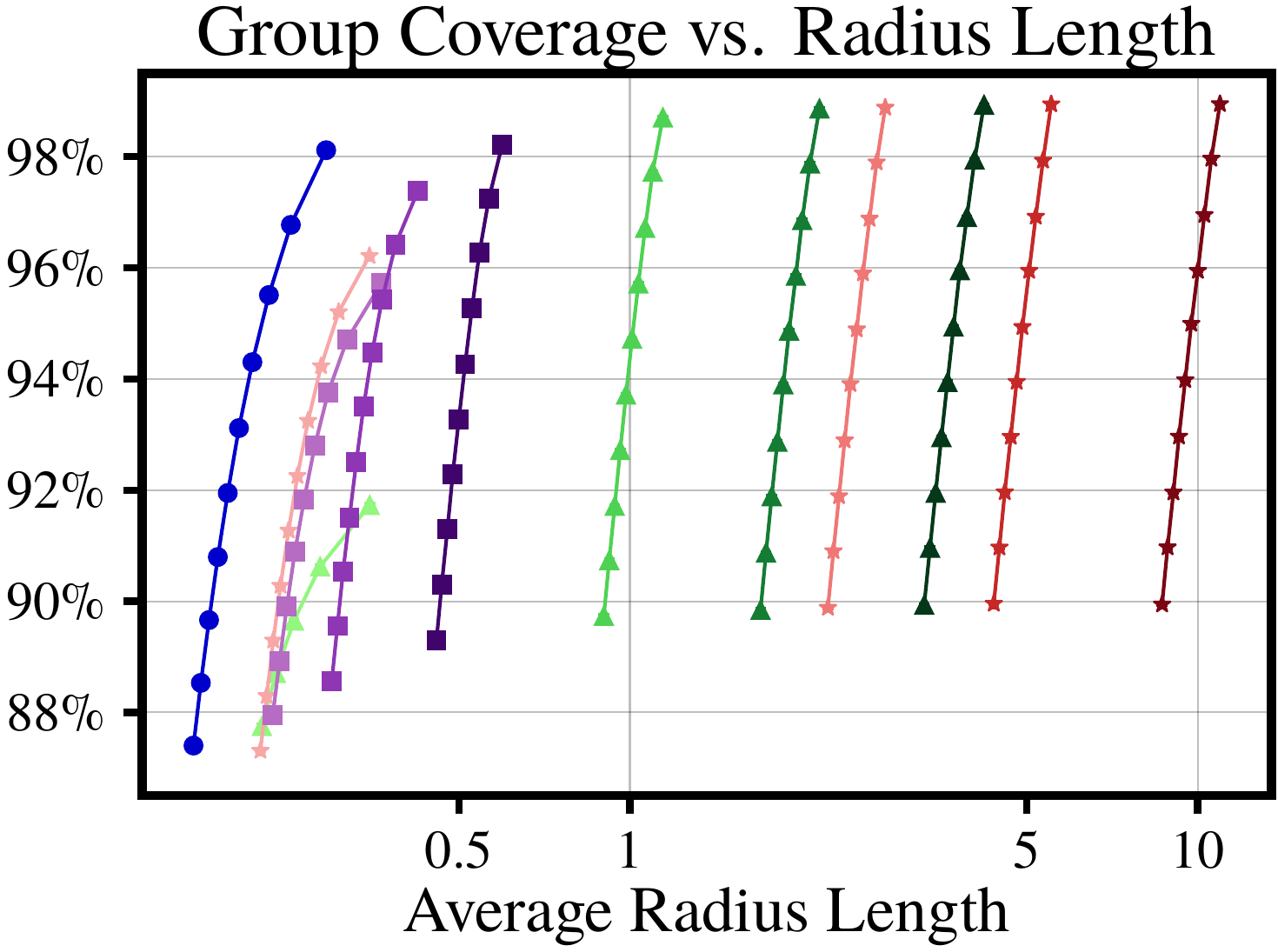}
        \end{subfigure}
        \hspace{0.2em}
        \begin{subfigure}[t]{0.31\linewidth}
            \includegraphics[width=\linewidth]{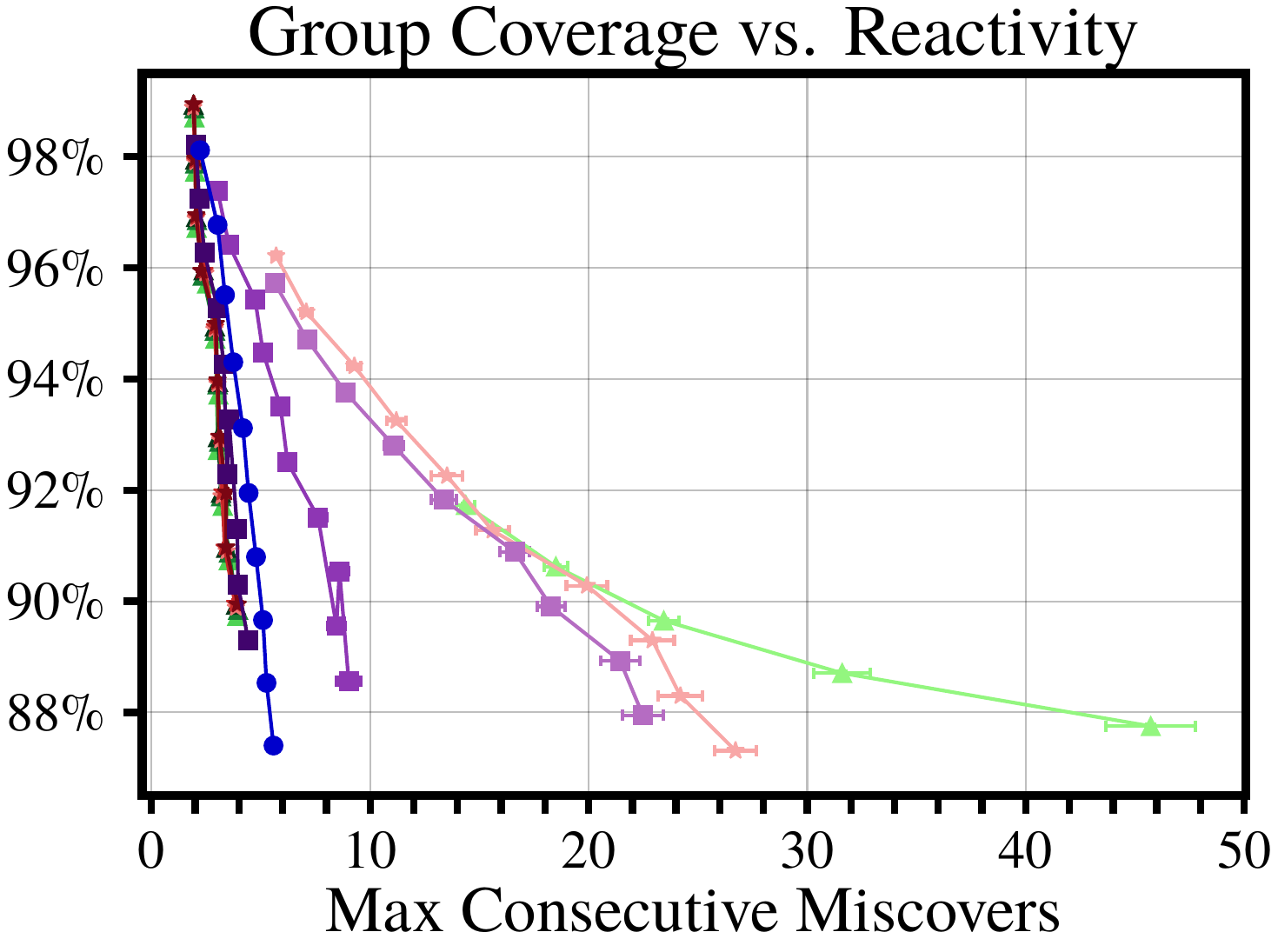}
        \end{subfigure}
        \caption{Results for synthetic setting with bounded scores with a changepoint.}

    \end{subfigure}
    \vspace{0.5em}

    \centering
    \begin{subfigure}{1\textwidth}
        \includegraphics[width=\linewidth]{figures/iclr/legend_syn_1_horz.pdf}
    \end{subfigure}
 
    \caption{Performance across target coverage levels \(1-\alpha \in \{0.90, \ldots, 0.99\}\). {\bf Left:} worst-group coverage vs. target coverage. {\bf Middle/right:} worst-group coverage vs. radius length / max consecutive miscovers across groups (curves closer to the top-left indicate better performance; data point shown if observed coverage $\geq 0.87$). Error bars show \(x\)-axis standard error of the mean (SEM).}
    \vspace{-1em}
\end{figure}
\newpage

{\bf Results for $k=100$ number of groups, quadratic growing scores.}
\begin{figure}[h!]
    \begin{subfigure}[t]{\linewidth}
        \begin{subfigure}[t]{0.31\linewidth}
        \raisebox{0.08em}{
            \includegraphics[width=\linewidth]{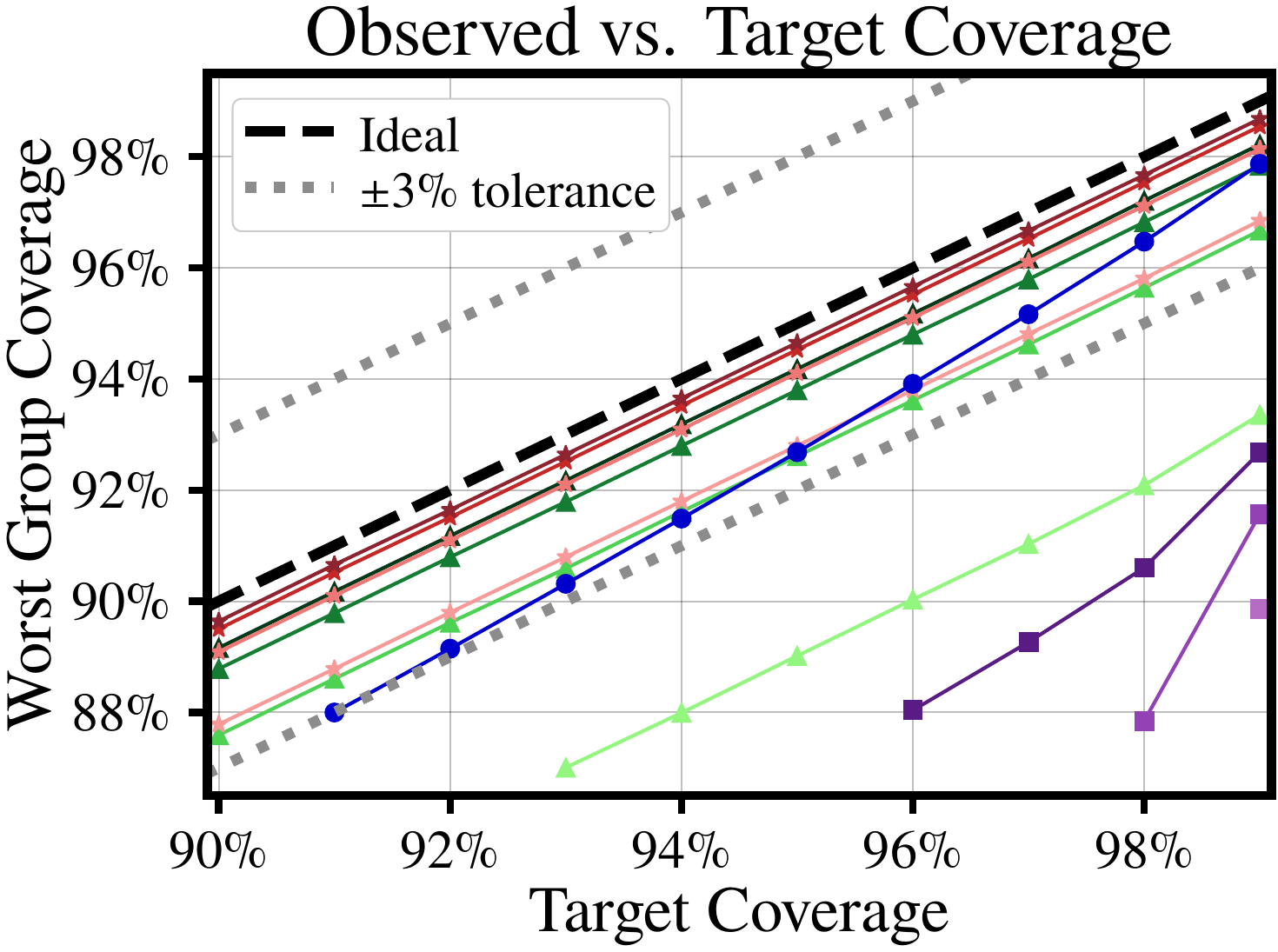}
            }
        \end{subfigure}
        \hspace{0.2em}
        \begin{subfigure}[t]{0.31\linewidth}
            \includegraphics[width=\linewidth]{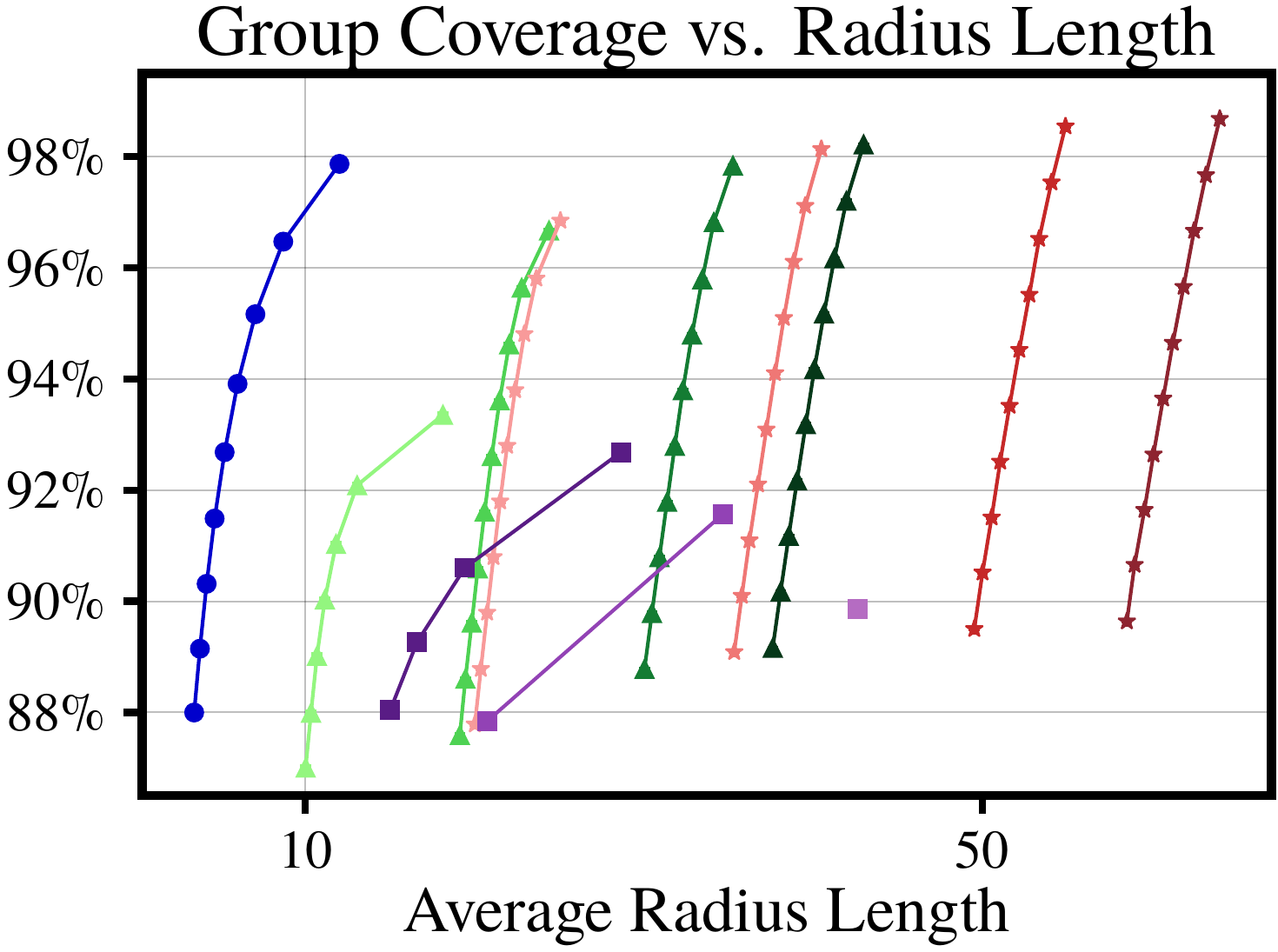}
        \end{subfigure}
        \hspace{0.2em}
        \begin{subfigure}[t]{0.31\linewidth}
            \includegraphics[width=\linewidth]{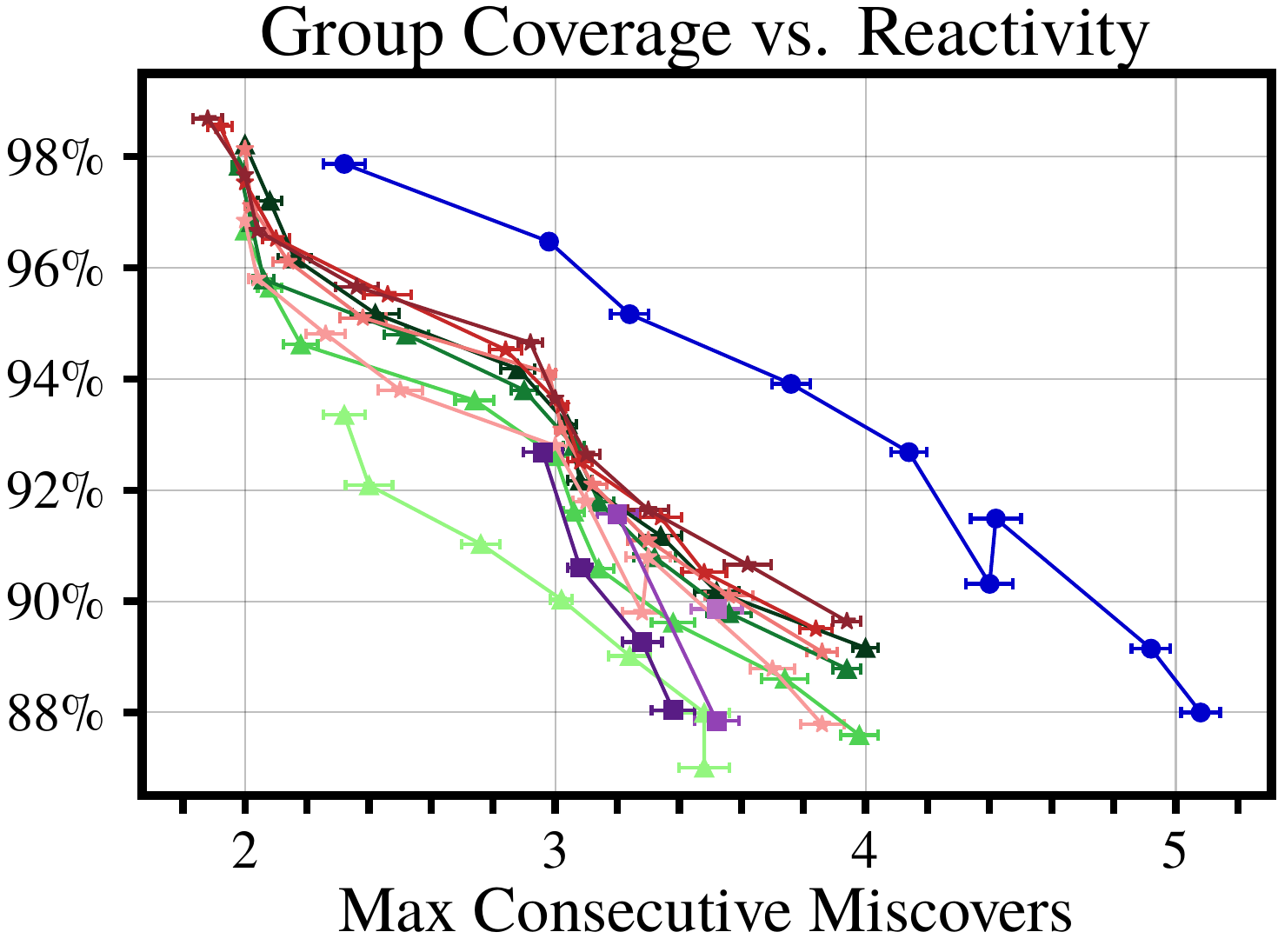}
        \end{subfigure}
        \caption{Results for the synthetic setting with slow growing quadratic scores ($q=2, A=5$).}
    \end{subfigure}
    \vspace{0.5em}
        
    \begin{subfigure}[t]{\linewidth}
        \begin{subfigure}[t]{0.31\linewidth}
        \raisebox{0.12em}{
            \includegraphics[width=\linewidth]{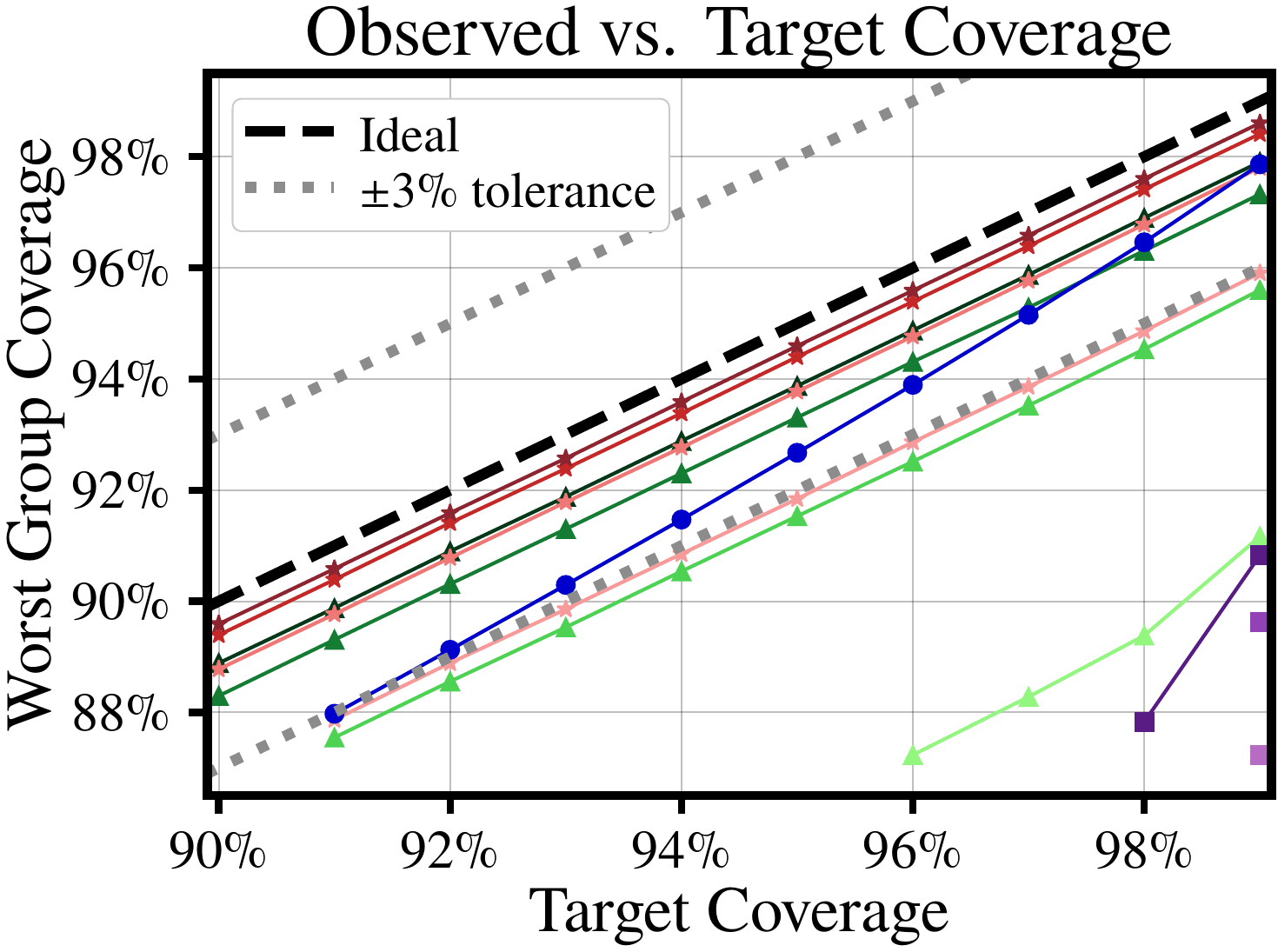}
            }
        \end{subfigure}
        \hspace{0.2em}
        \begin{subfigure}[t]{0.31\linewidth}
            \includegraphics[width=\linewidth]{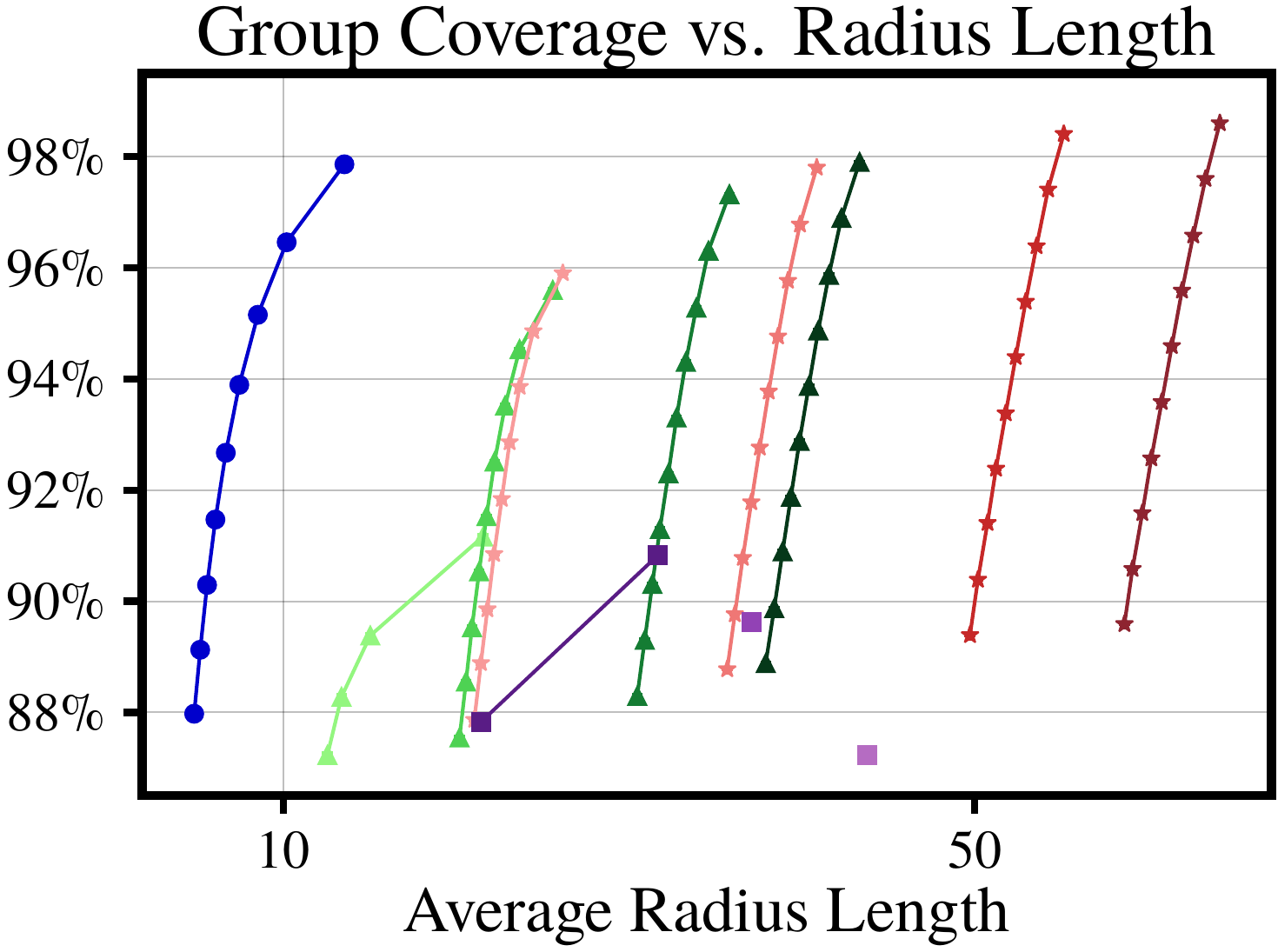}
        \end{subfigure}
        \hspace{0.2em}
        \begin{subfigure}[t]{0.31\linewidth}
            \includegraphics[width=\linewidth]{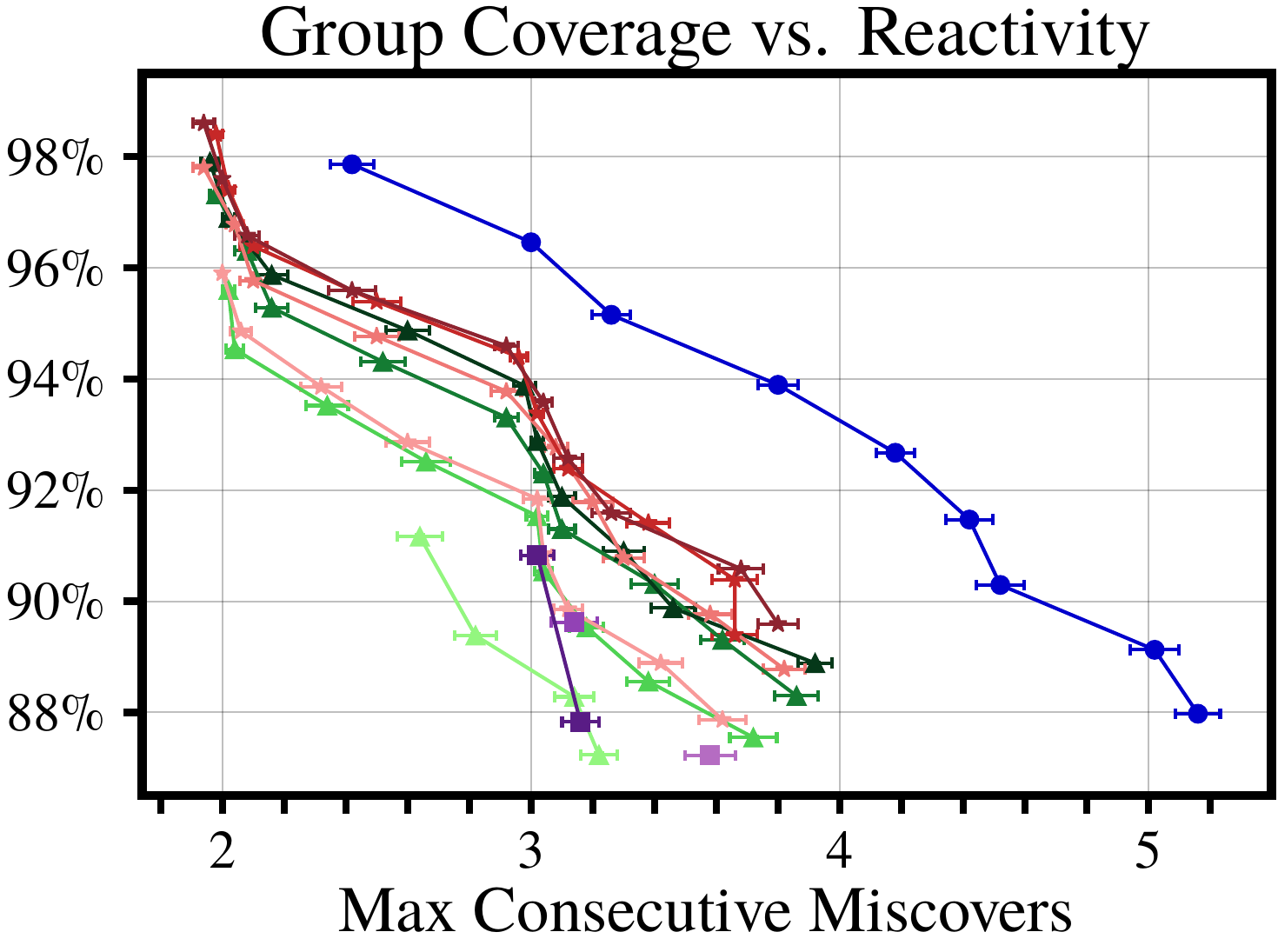}
        \end{subfigure}
        \caption{Results for the synthetic setting with moderately growing quadratic scores ($q=2, A=25$).}
    \end{subfigure}
    \vspace{0.5em}

    \begin{subfigure}[t]{\linewidth}
        \begin{subfigure}[t]{0.31\linewidth}
        \raisebox{0.12em}{
            \includegraphics[width=\linewidth]{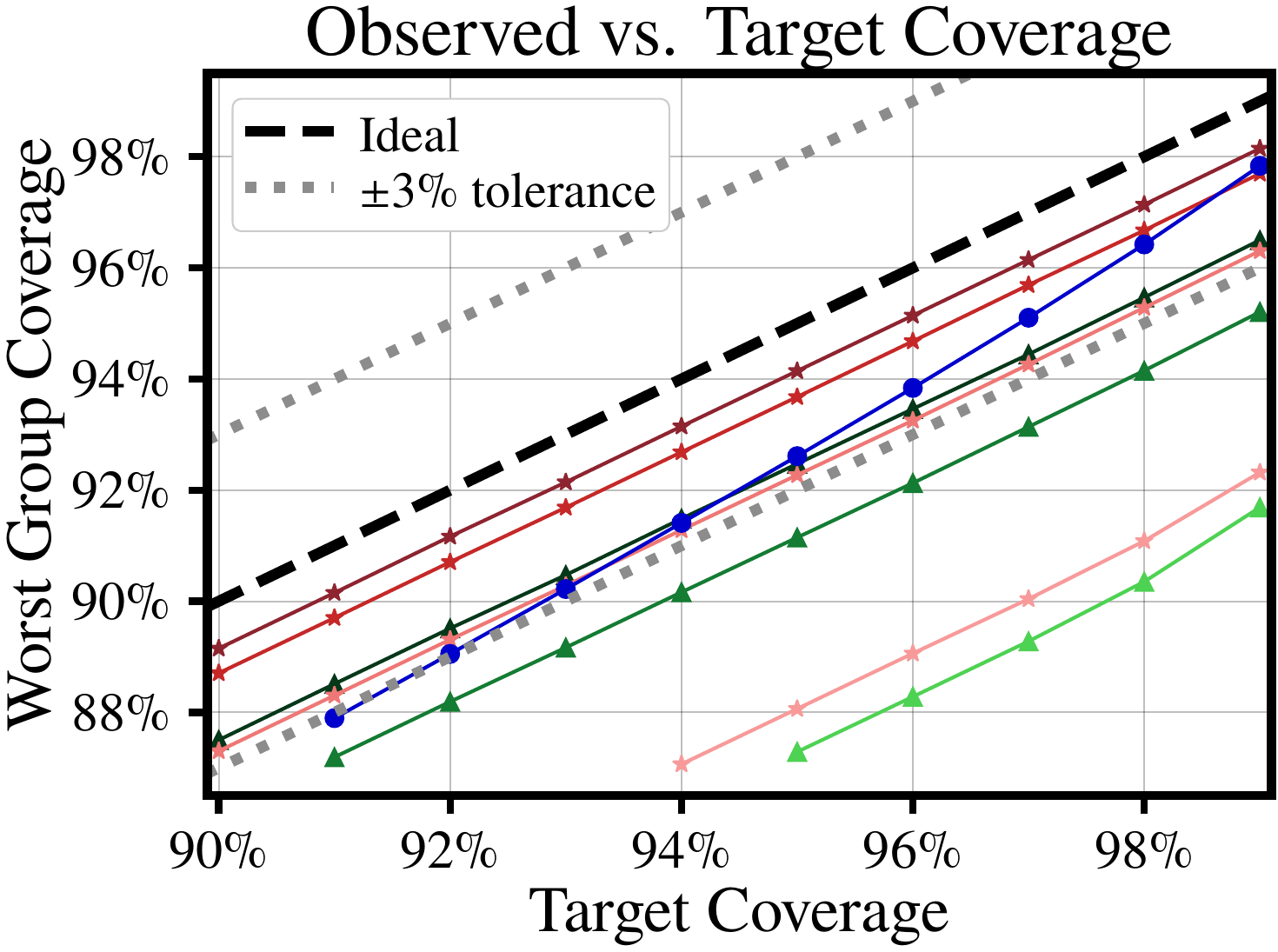}
            }
        \end{subfigure}
        \hspace{0.2em}
        \begin{subfigure}[t]{0.31\linewidth}
            \includegraphics[width=\linewidth]{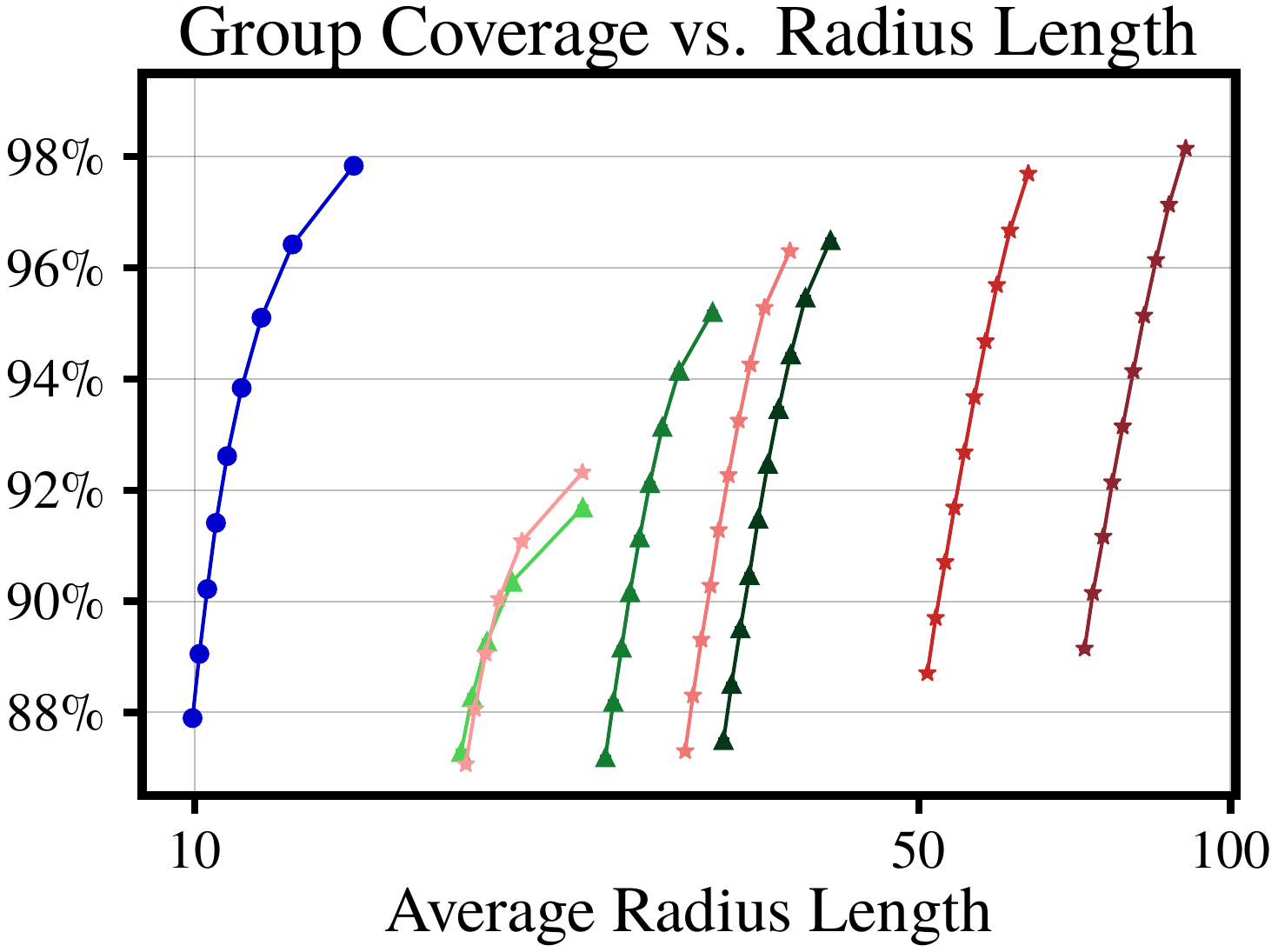}
        \end{subfigure}
        \hspace{0.2em}
        \begin{subfigure}[t]{0.31\linewidth}
            \includegraphics[width=\linewidth]{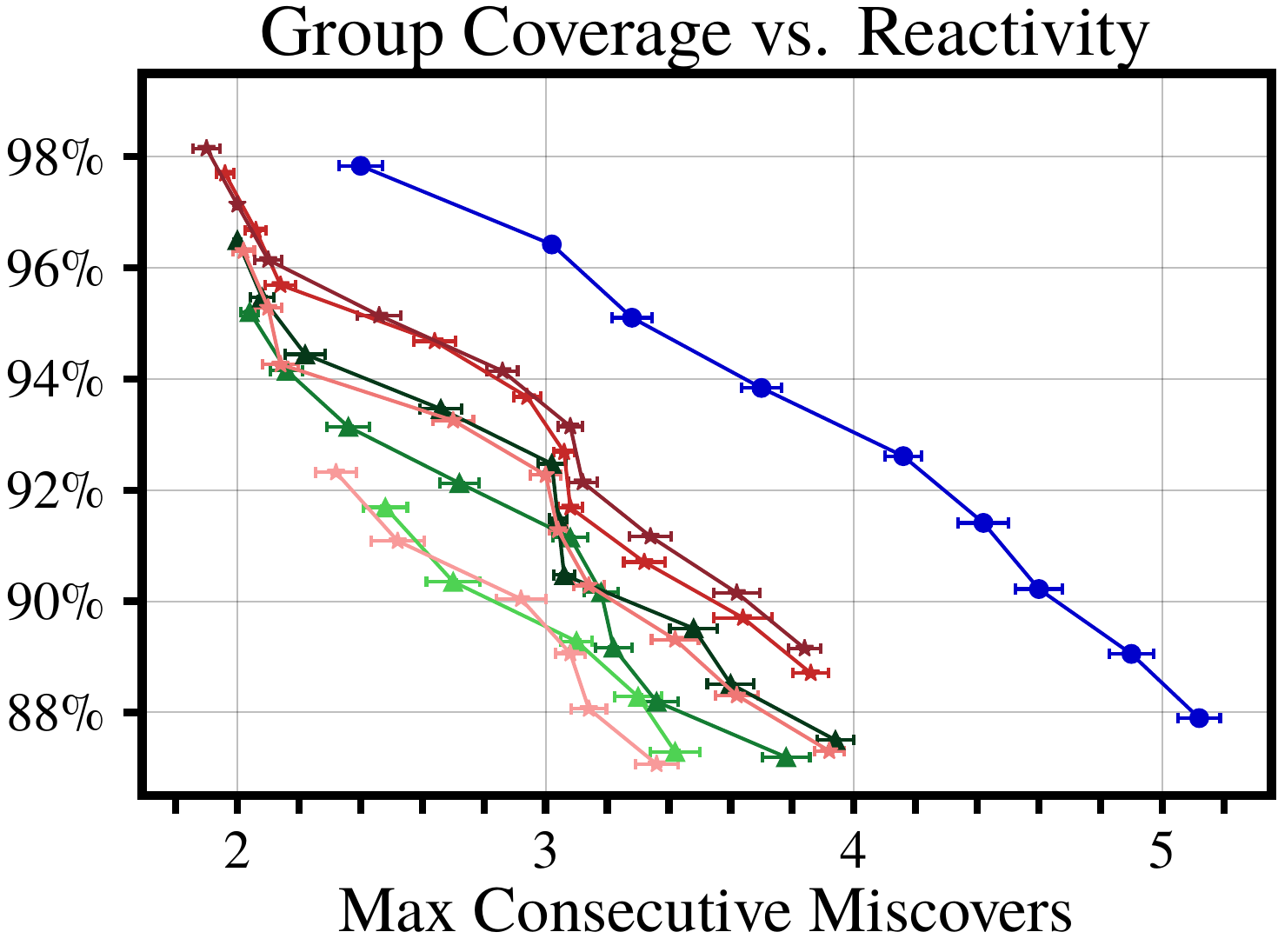}
        \end{subfigure}
        \caption{Results for the synthetic setting with extreme growing quadratic scores ($q=2, A=100$).}
    \end{subfigure}
    \vspace{0.5em}

    \centering
    \begin{subfigure}{1\textwidth}
        \includegraphics[width=\linewidth]{figures/iclr/legend_syn_2_horz.pdf}
    \end{subfigure}
 
    \caption{Performance across target coverage levels \(1-\alpha \in \{0.90, \ldots, 0.99\}\). {\bf Left:} worst-group coverage vs. target coverage. {\bf Middle/right:} worst-group coverage vs. radius length / max consecutive miscovers across groups (curves closer to the top-left indicate better performance; data point shown if observed coverage $\geq 0.87$). Error bars show \(x\)-axis standard error of the mean (SEM).}
    \vspace{-1em}
\end{figure}
\newpage

{\bf Results for $k=100$ number of groups, cubic growing scores.}
\begin{figure}[h!]
    \begin{subfigure}[t]{\linewidth}
        \begin{subfigure}[t]{0.31\linewidth}
        \raisebox{0.08em}{
            \includegraphics[width=\linewidth]{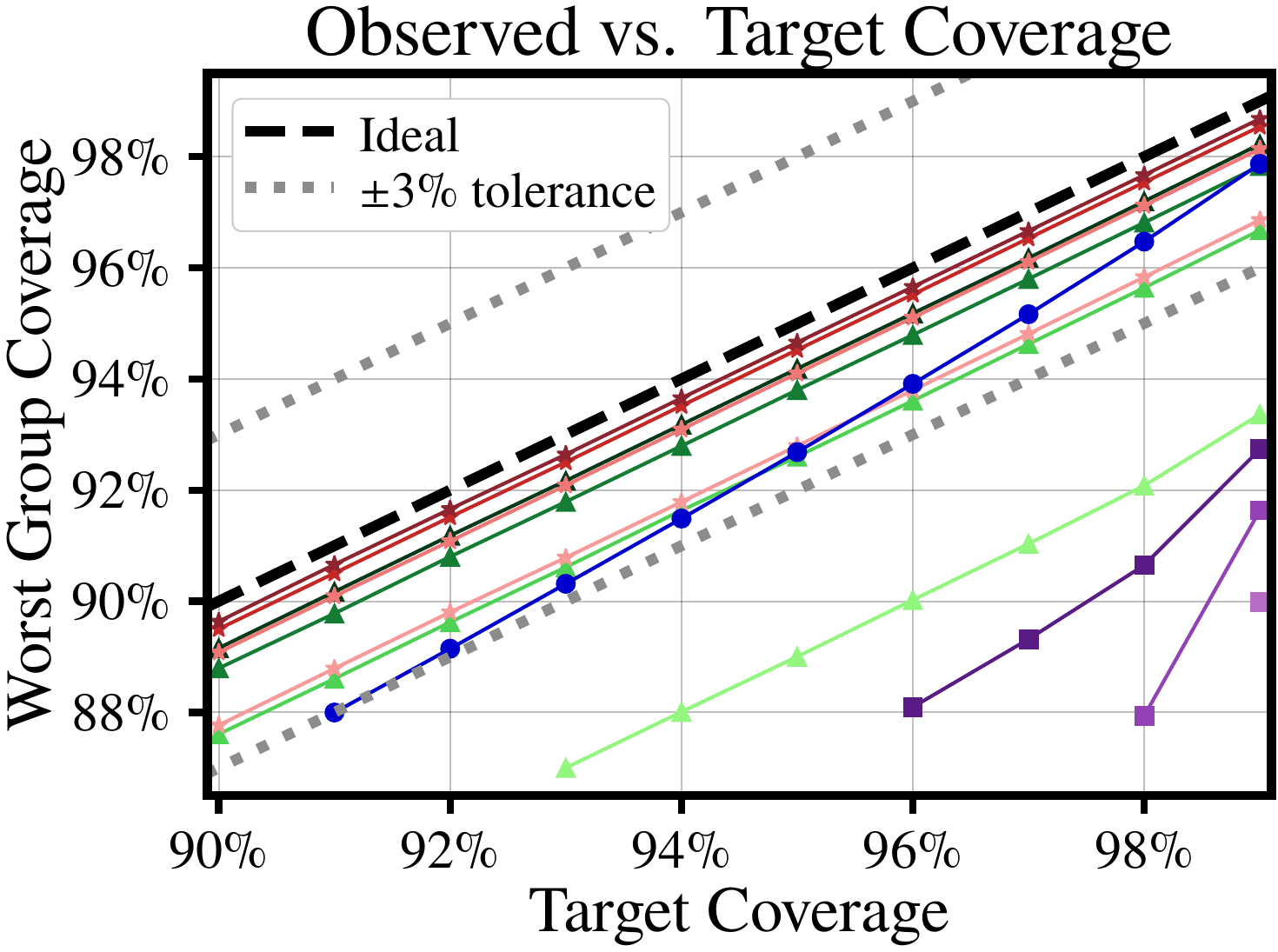}
            }
        \end{subfigure}
        \hspace{0.2em}
        \begin{subfigure}[t]{0.31\linewidth}
            \includegraphics[width=\linewidth]{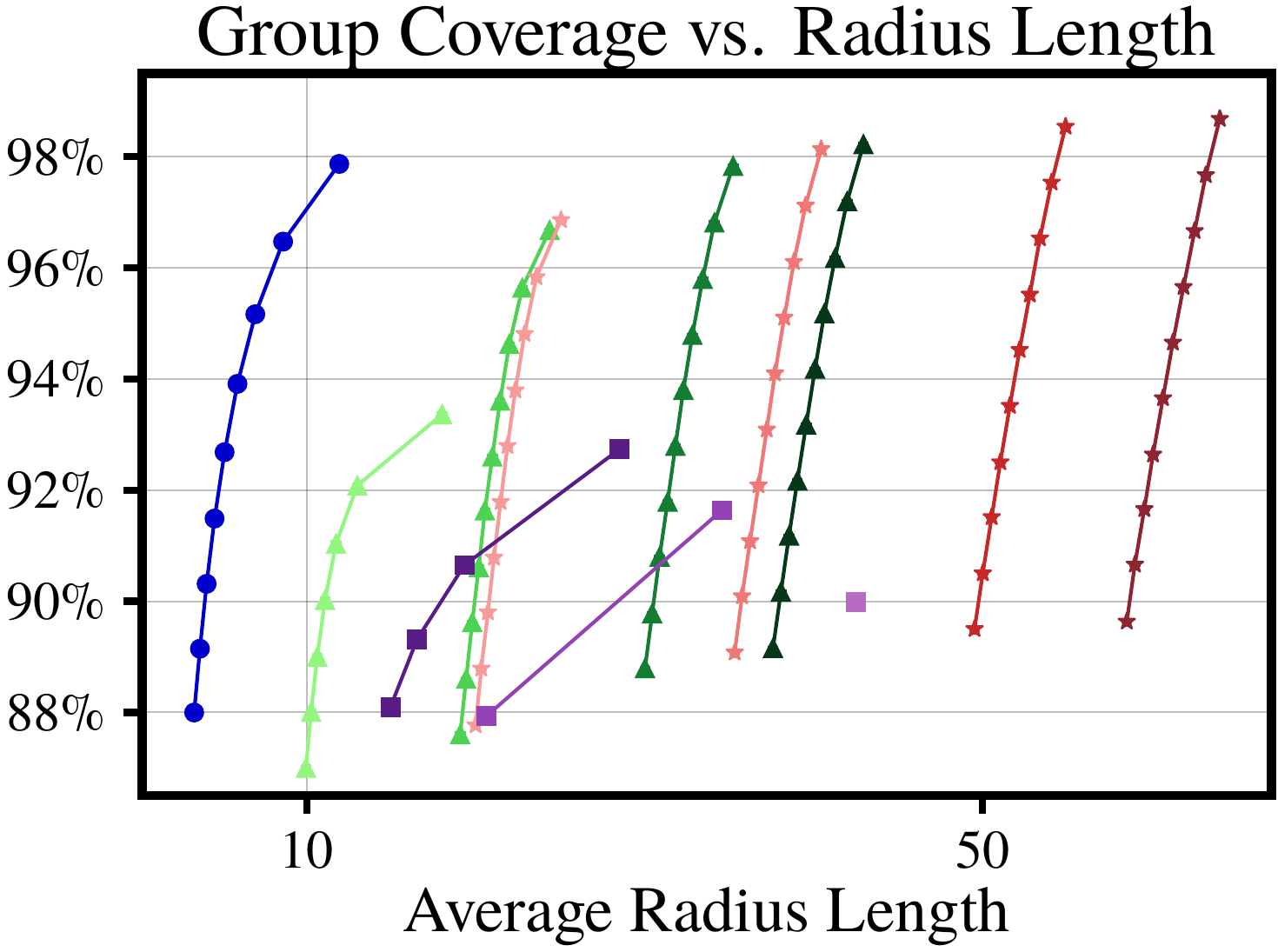}
        \end{subfigure}
        \hspace{0.2em}
        \begin{subfigure}[t]{0.31\linewidth}
            \includegraphics[width=\linewidth]{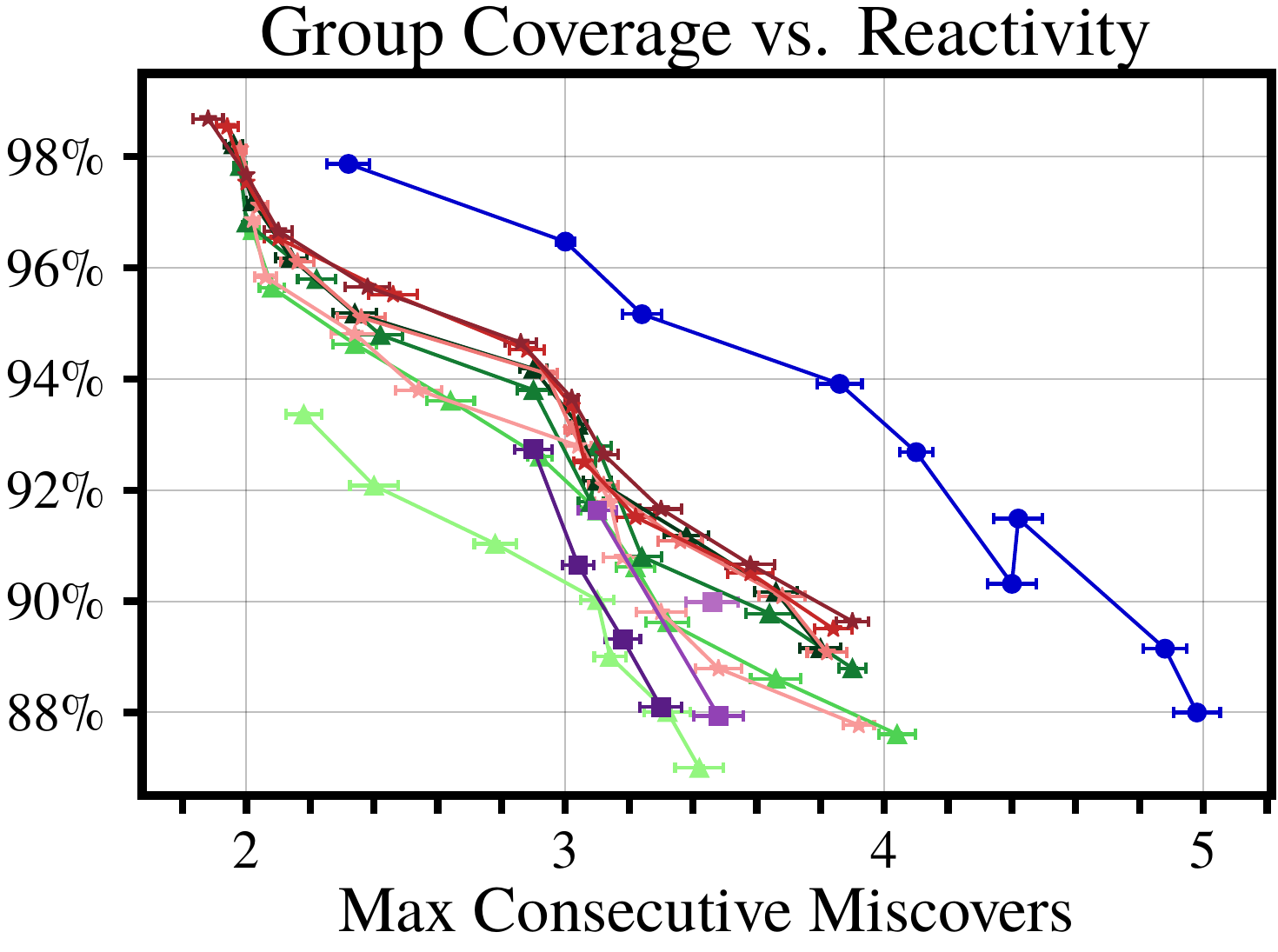}
        \end{subfigure}
        \caption{Results for the synthetic setting with slow growing cubic scores ($q=3, A=5$).}
    \end{subfigure}
    \vspace{0.5em}
        
    \begin{subfigure}[t]{\linewidth}
        \begin{subfigure}[t]{0.31\linewidth}
        \raisebox{0.12em}{
            \includegraphics[width=\linewidth]{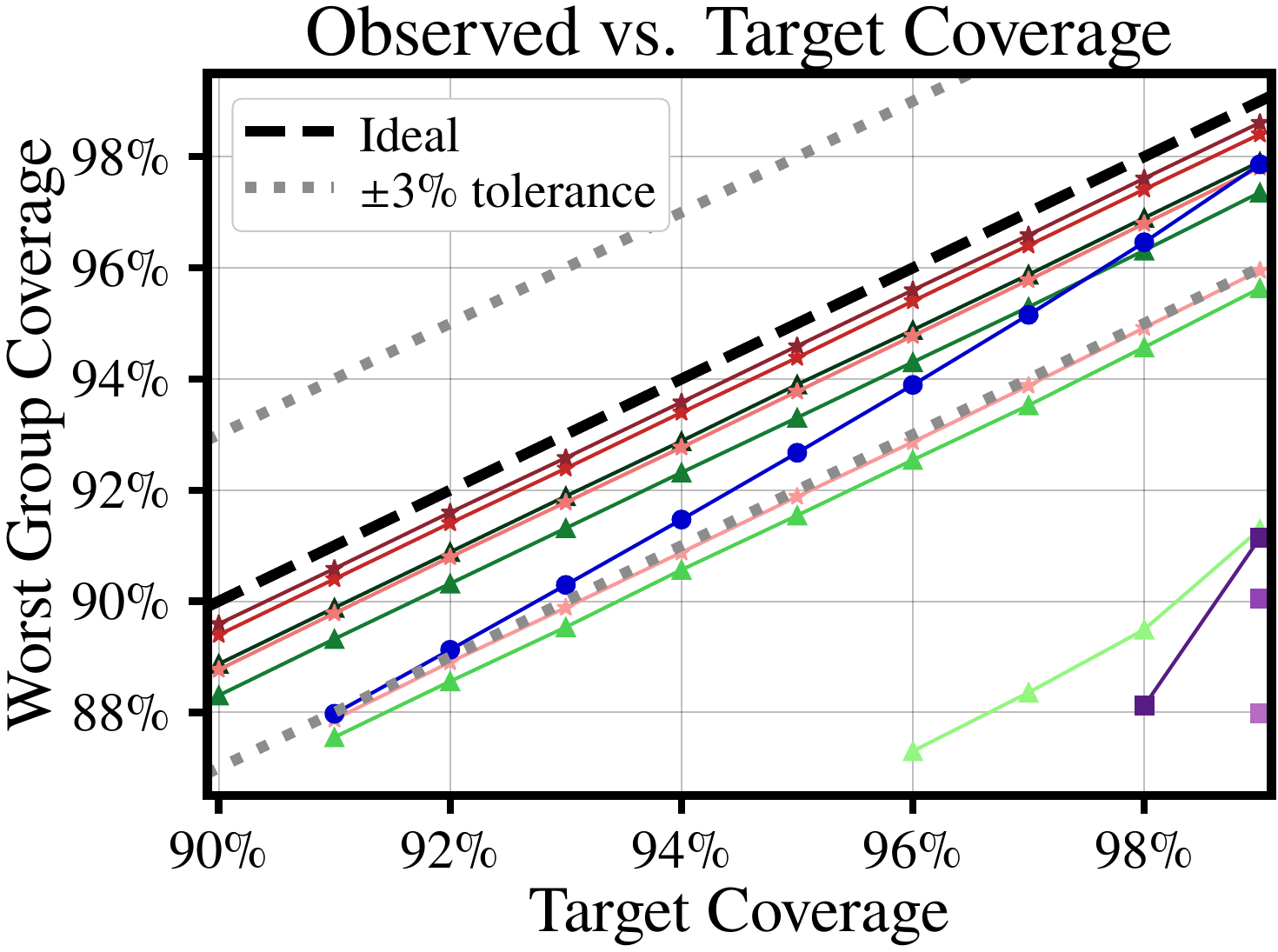}
            }
        \end{subfigure}
        \hspace{0.2em}
        \begin{subfigure}[t]{0.31\linewidth}
            \includegraphics[width=\linewidth]{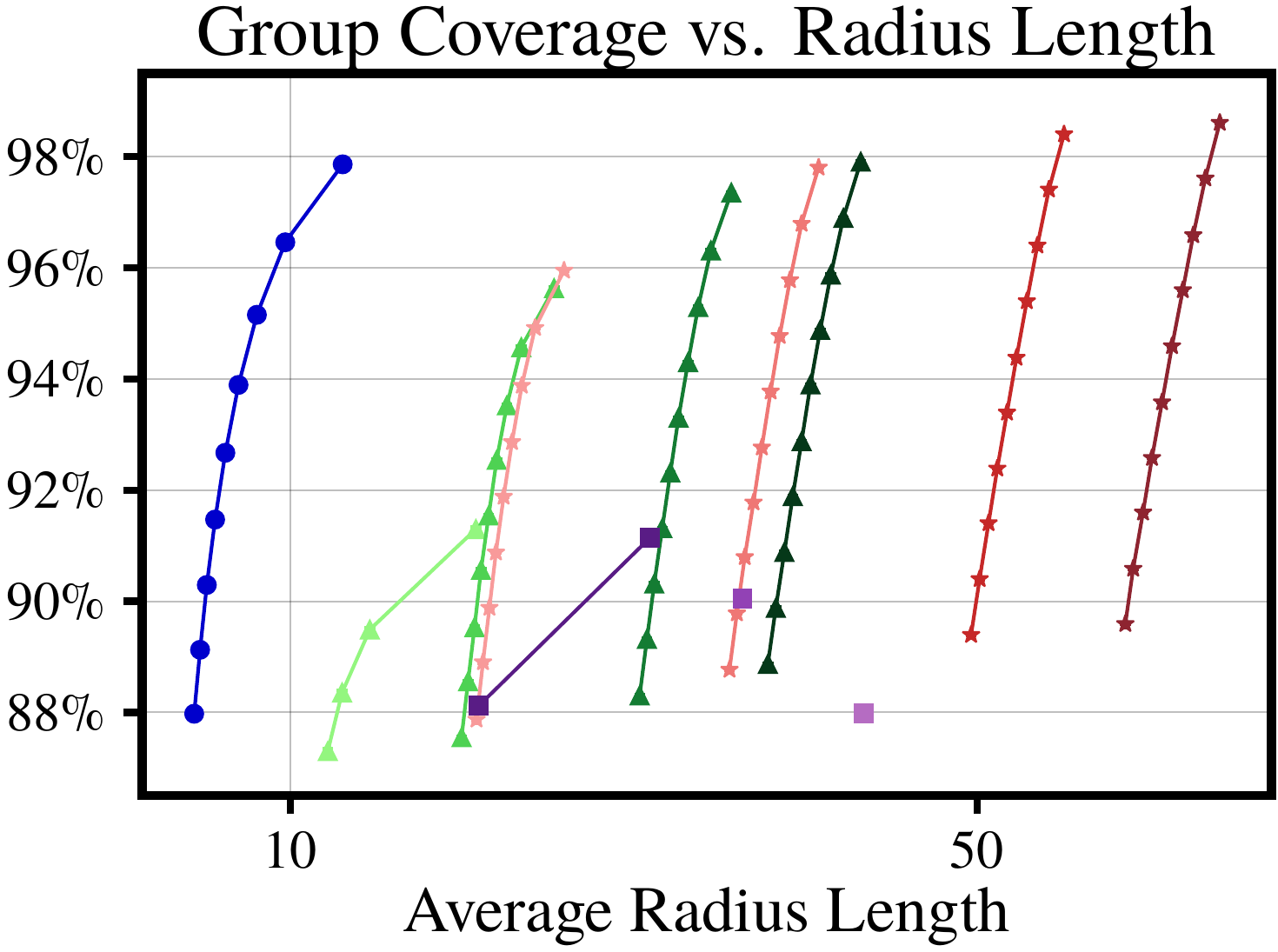}
        \end{subfigure}
        \hspace{0.2em}
        \begin{subfigure}[t]{0.31\linewidth}
            \includegraphics[width=\linewidth]{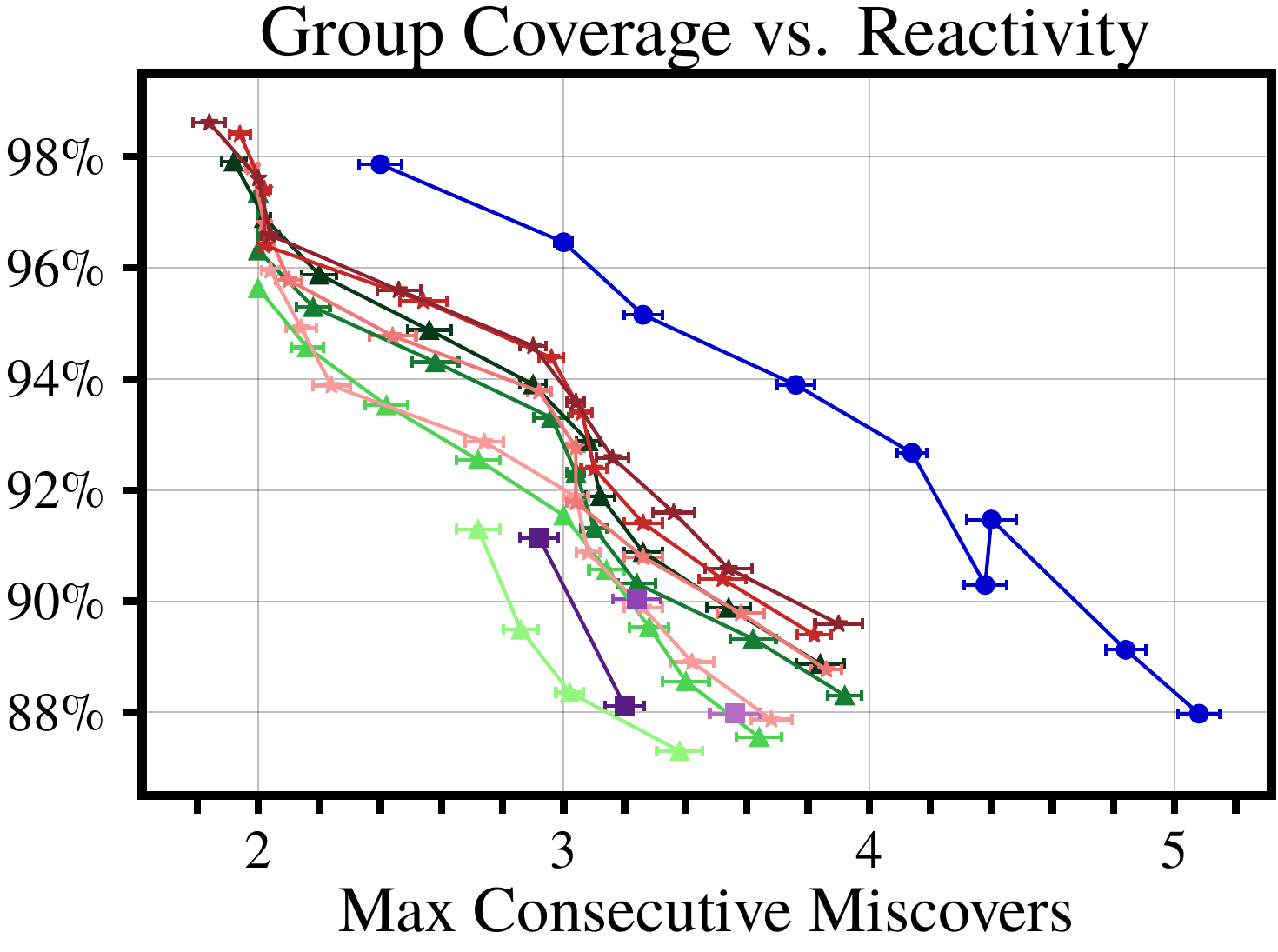}
        \end{subfigure}
        \caption{Results for the synthetic setting with moderately growing cubic scores ($q=3, A=25$).}
    \end{subfigure}
    \vspace{0.5em}

    \begin{subfigure}[t]{\linewidth}
        \begin{subfigure}[t]{0.31\linewidth}
        \raisebox{0.12em}{
            \includegraphics[width=\linewidth]{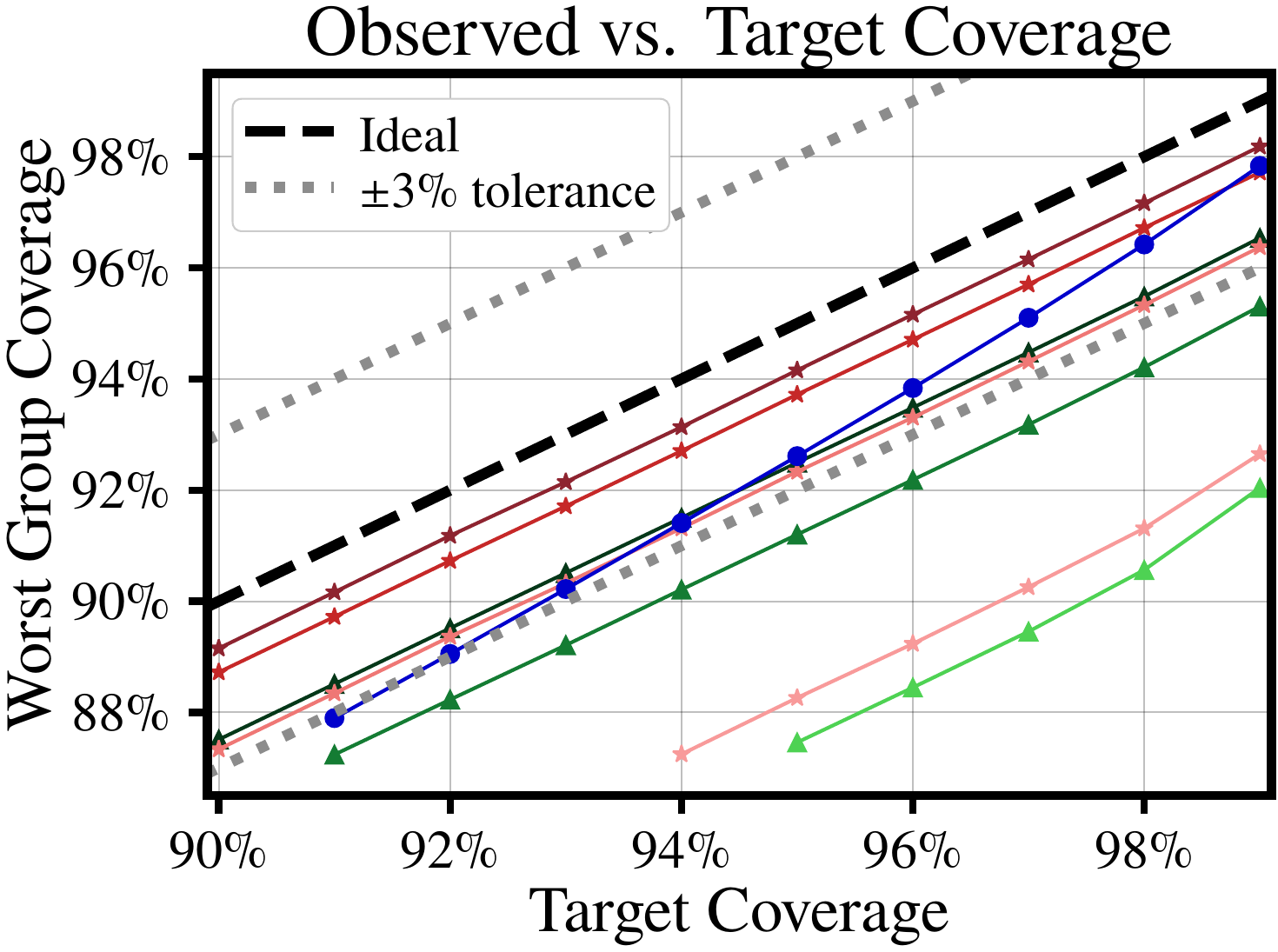}
            }
        \end{subfigure}
        \hspace{0.2em}
        \begin{subfigure}[t]{0.31\linewidth}
            \includegraphics[width=\linewidth]{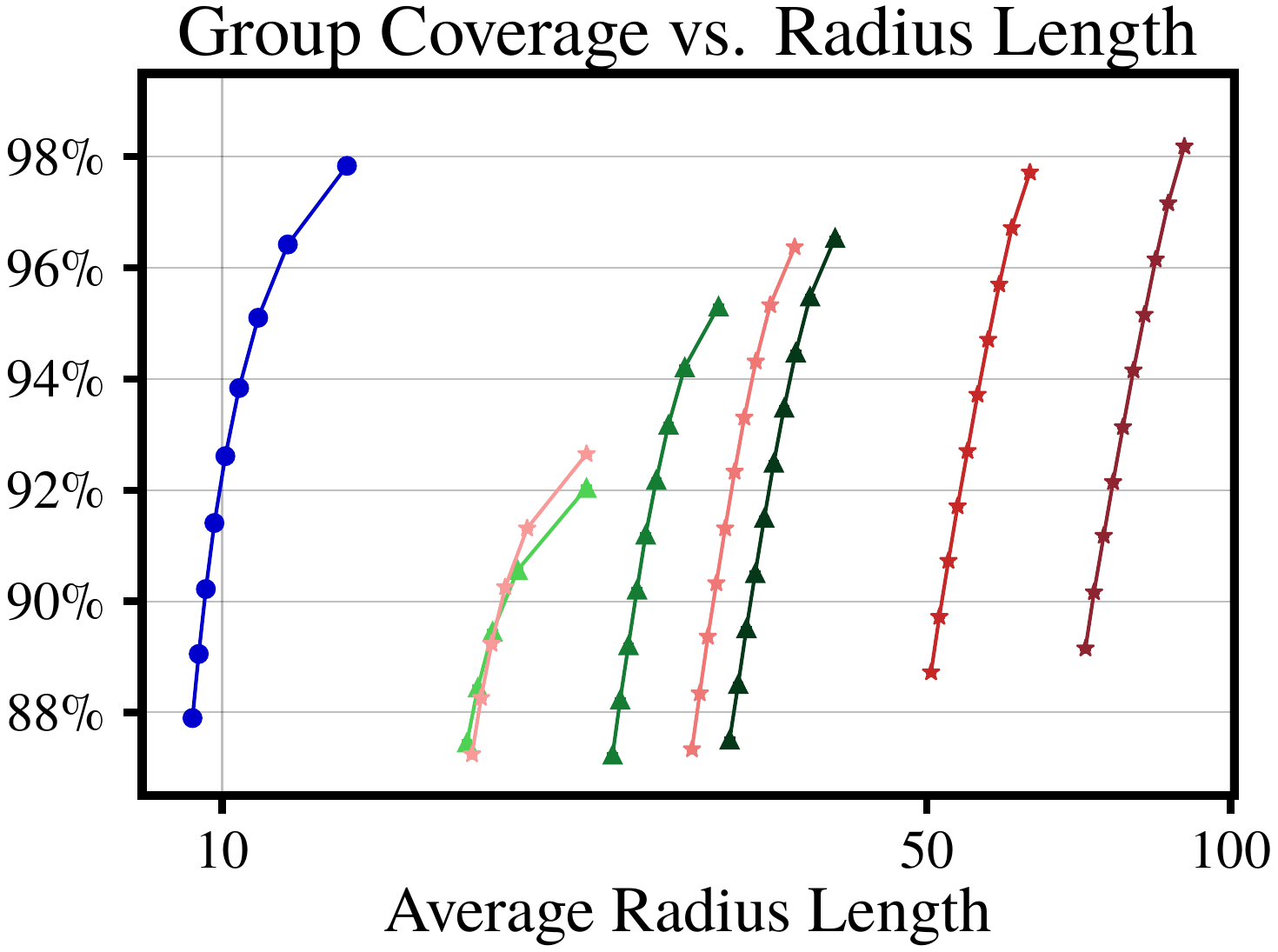}
        \end{subfigure}
        \hspace{0.2em}
        \begin{subfigure}[t]{0.31\linewidth}
            \includegraphics[width=\linewidth]{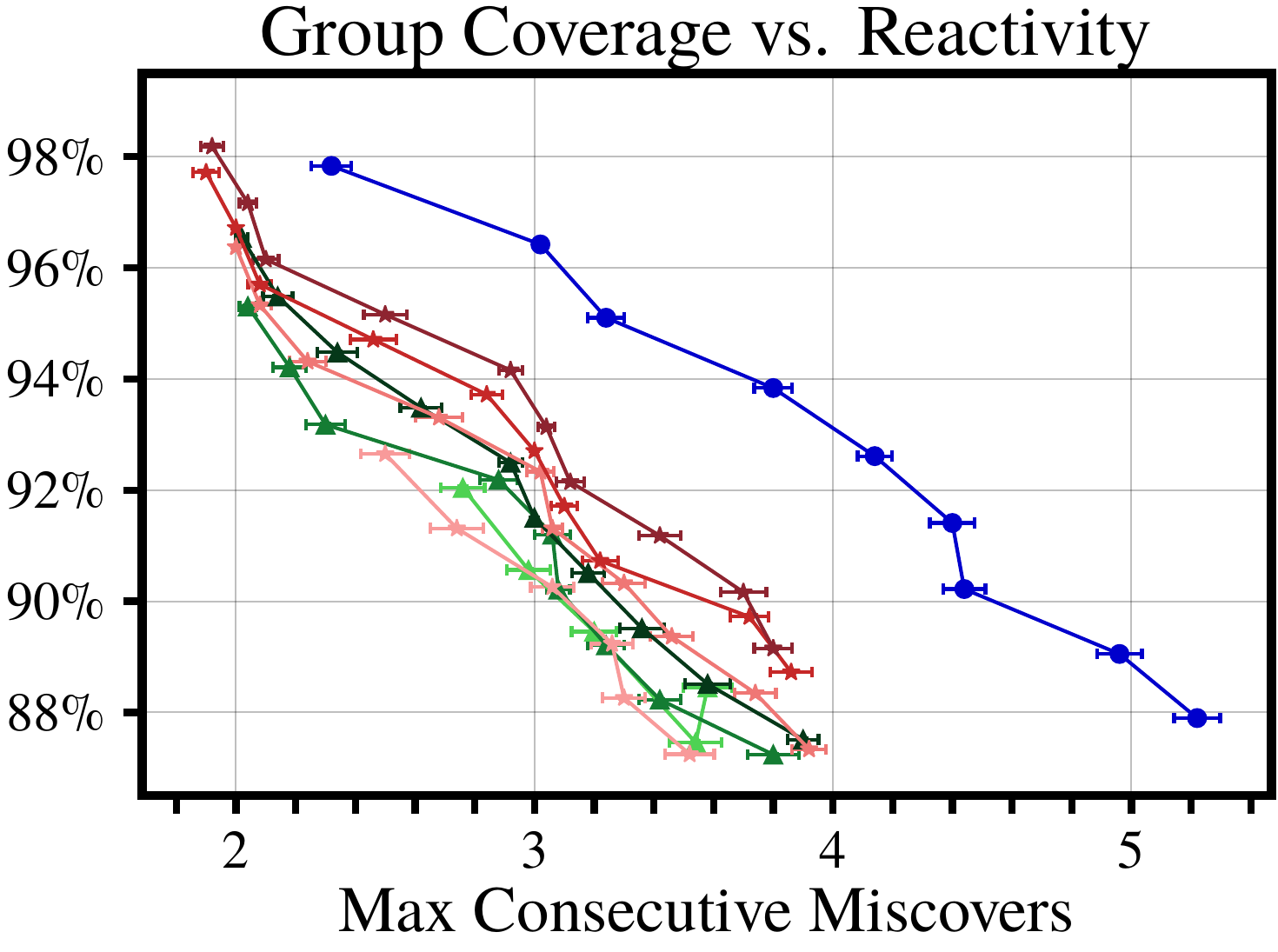}
        \end{subfigure}
        \caption{Results for the synthetic setting with extreme growing cubic scores ($q=3, A=100$).}
    \end{subfigure}
    \vspace{0.5em}

    \centering
    \begin{subfigure}{1\textwidth}
        \includegraphics[width=\linewidth]{figures/iclr/legend_syn_2_horz.pdf}
    \end{subfigure}

    \caption{Performance across target coverage levels \(1-\alpha \in \{0.90, \ldots, 0.99\}\). {\bf Left:} worst-group coverage vs. target coverage. {\bf Middle/right:} worst-group coverage vs. radius length / max consecutive miscovers across groups (curves closer to the top-left indicate better performance; data point shown if observed coverage $\geq 0.87$). Error bars show \(x\)-axis standard error of the mean (SEM).}
    \vspace{-1em}
\end{figure}

\newpage
\subsection{\label{supp:experiments_stock}Additional experimental details: Stock Market}
\paragraph{Group definitions.} We define groups in the following manner:

\begin{itemize}

\item {\bf Market-regime groups (computed from past returns; no leakage)}

First, we compute the following quantities.
\begin{itemize}
    \item Daily returns $r_t = \frac{Y_t - Y_{t-1}}{Y_{t-1}}$ (we use lagged returns $Y_{t-1}$ to not use data from time $t$)
    \item Over a rolling window of length $W$ (with full-window requirement), we compute
    \begin{itemize}
        \item Volatility: $\sigma_t = \mathrm{Std}\big(r_{t-1}, r_{t-2}, \ldots, r_{t-W}\big)$
        \item Trend: $\mu_t = \mathrm{Mean}\big(r_{t-1}, r_{t-2}, \ldots, r_{t-W}\big)$
    \end{itemize}
    \item The rolling median volatility over a window of length $W_m$: $
        \tilde{\sigma}_t = \mathrm{Median}\big(\sigma_{t-1}, \ldots, \sigma_{t-W_m}\big).
        $
\end{itemize}
Then, we define binary groups for each time $t$
    \begin{itemize}
        \item $\texttt{is\_high\_vol}_t = \mathbb{I}\{\sigma_t > \tilde{\sigma}_t\}$, \quad
              $\texttt{is\_low\_vol}_t = \mathbb{I}\{\sigma_t \le \tilde{\sigma}_t\}$
        \item $\texttt{is\_uptrend}_t = \mathbb{I}\{\mu_t > 0\}$, \quad
              $\texttt{is\_downtrend}_t = \mathbb{I}\{\mu_t \le 0\}$
    \end{itemize}

\item {\bf Calendar groups (derived from the forecast date).}
    \begin{itemize}
        \item Day-of-week one-hot indicators: \texttt{is\_monday}, \ldots, \texttt{is\_friday}.
        \item Quarter one-hot indicators: \texttt{is\_q1}, \ldots, \texttt{is\_q4}.
        \item Month one-hot indicators: \texttt{is\_january}, \ldots, \texttt{is\_december}.
    \end{itemize}
\end{itemize}

\newpage
\subsection{\label{supp:experiments_stock_results}Additional results for Stock Market Experiment}

\begin{figure}[h!]
    \begin{subfigure}[t]{\linewidth}
        \begin{subfigure}[t]{0.31\linewidth}
        \raisebox{0.08em}{
            \includegraphics[width=\linewidth]{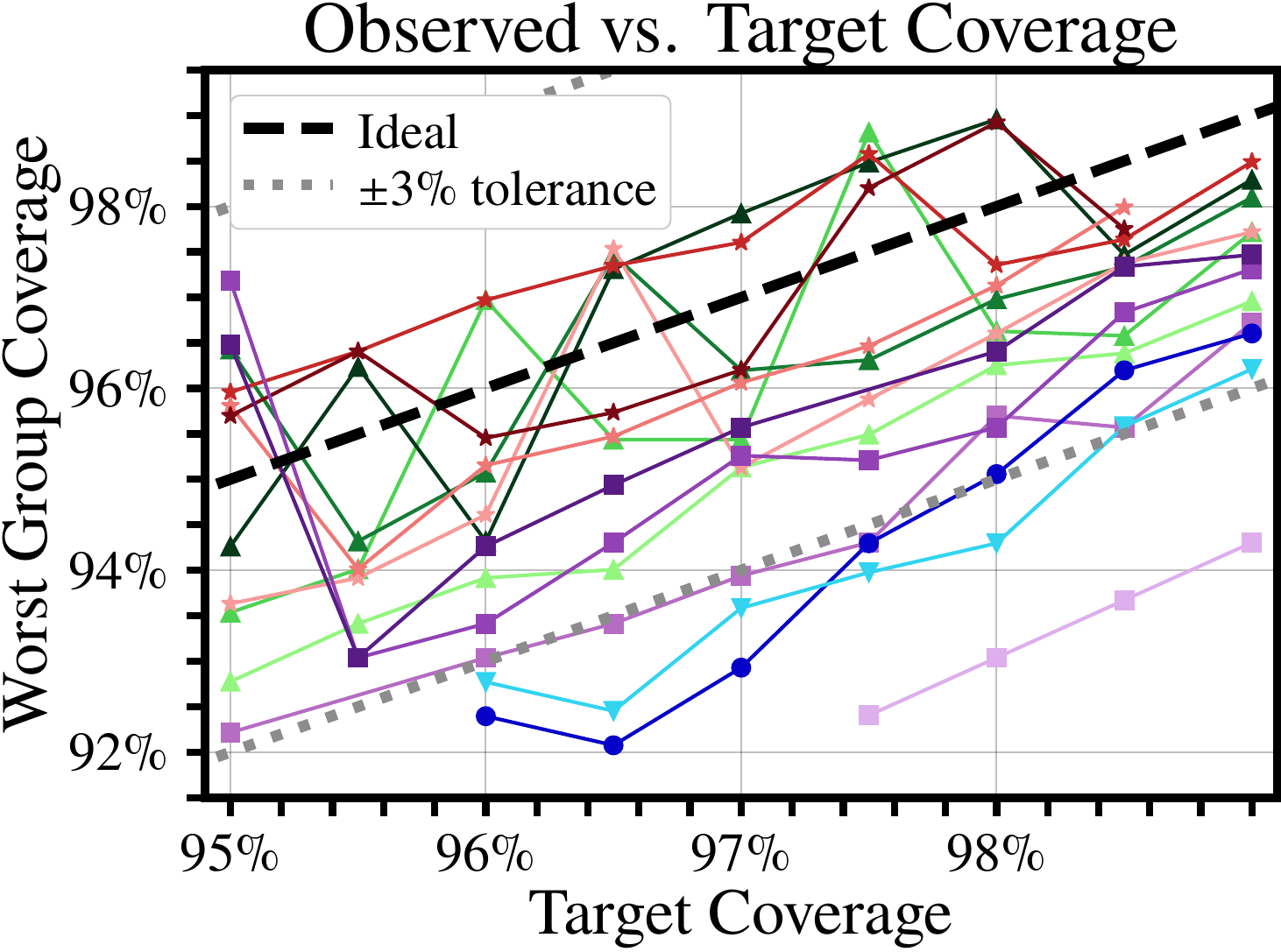}
            }
        \end{subfigure}
        \hspace{0.2em}
        \begin{subfigure}[t]{0.31\linewidth}
            \includegraphics[width=\linewidth]{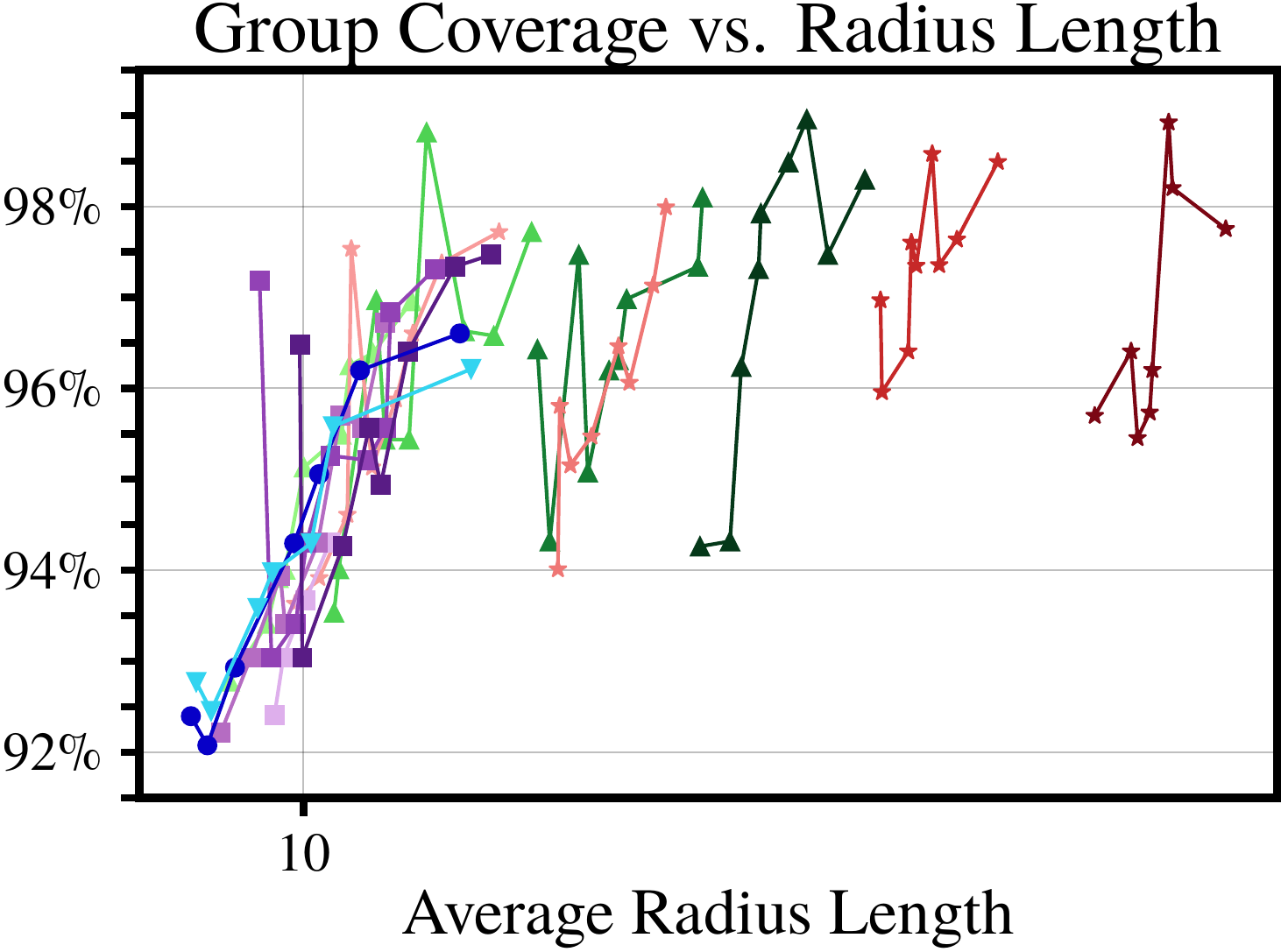}
        \end{subfigure}
        \hspace{0.2em}
        \begin{subfigure}[t]{0.31\linewidth}
            \includegraphics[width=\linewidth]{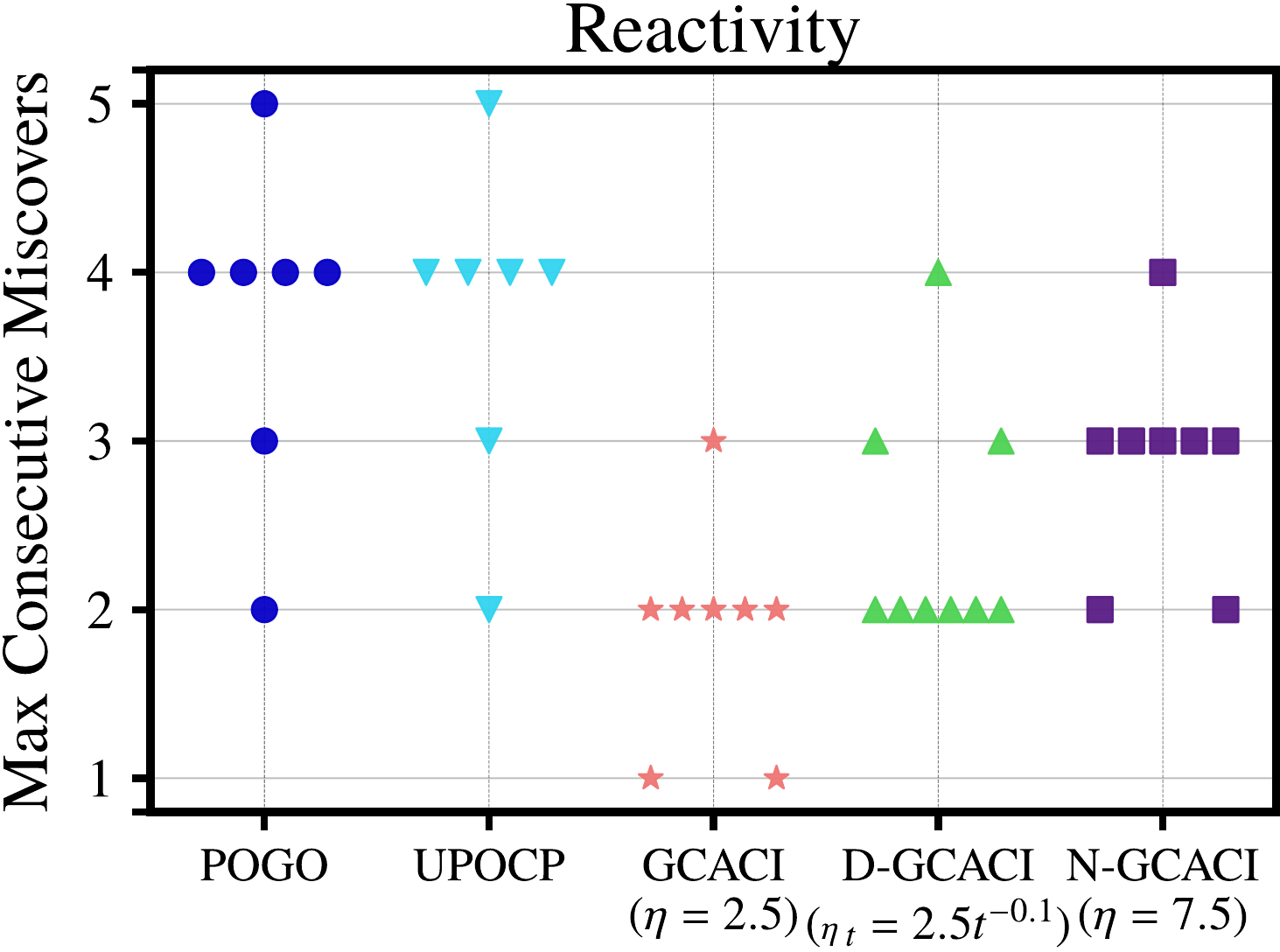}
        \end{subfigure}
        \caption{Results for Apple (\texttt{AAPL}) stock.}
    \end{subfigure}
    \vspace{0.5em}
        
    \begin{subfigure}[t]{\linewidth}
        \begin{subfigure}[t]{0.31\linewidth}
        \raisebox{0.08em}{
            \includegraphics[width=\linewidth]{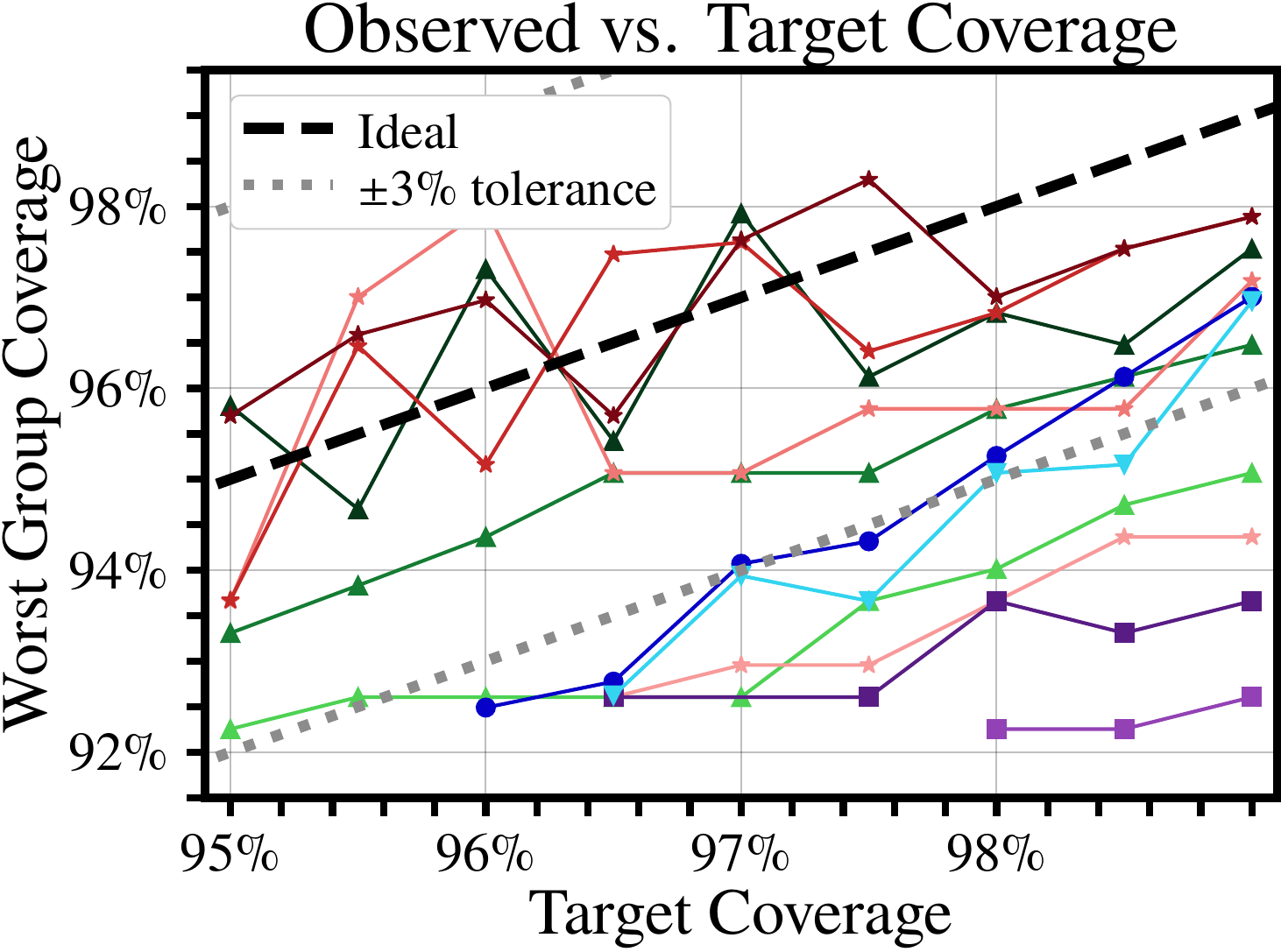}
            }
        \end{subfigure}
        \hspace{0.2em}
        \raisebox{0.5em}{
        \begin{subfigure}[t]{0.3\linewidth}
            \includegraphics[width=\linewidth]{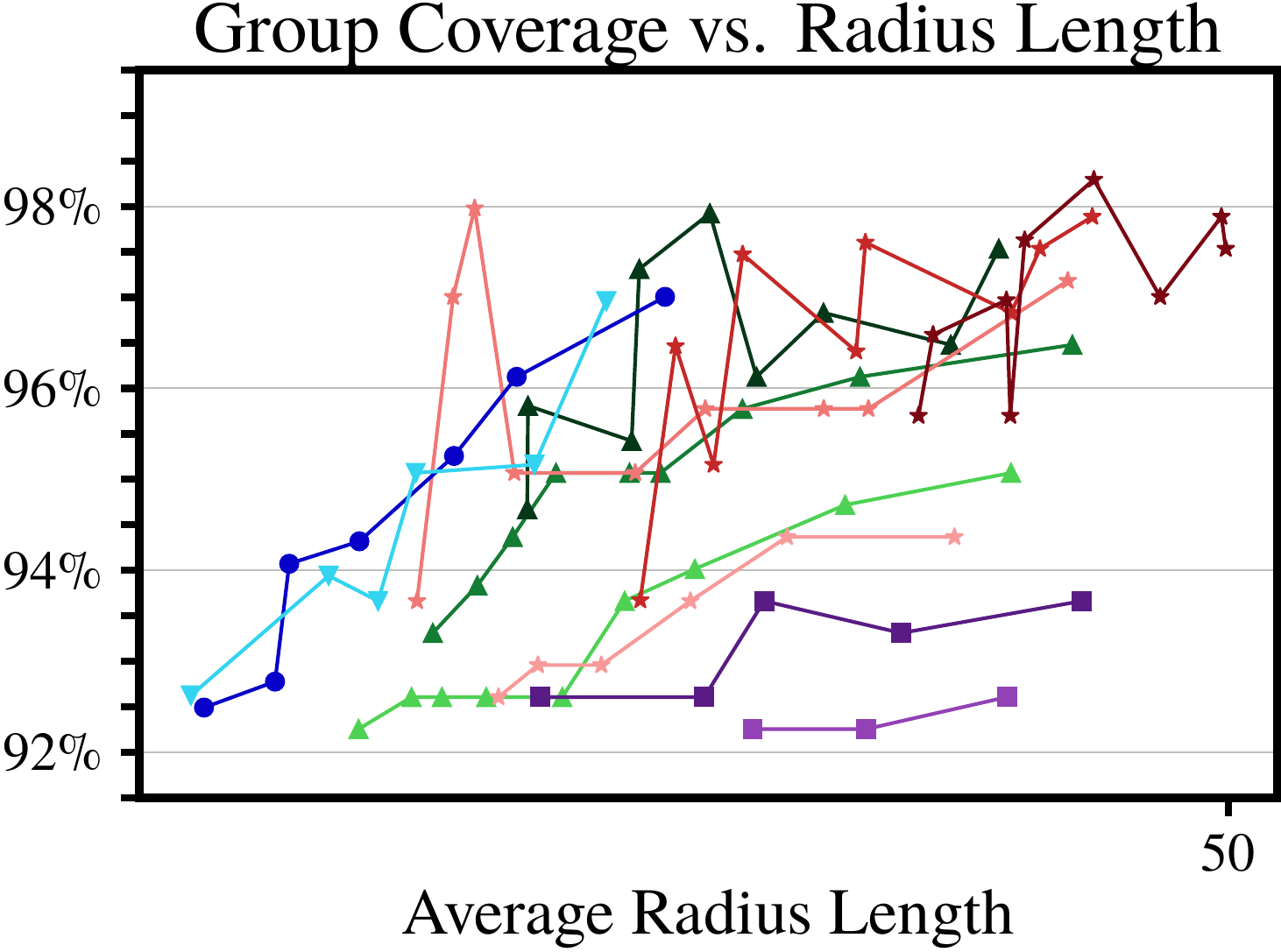}
        \end{subfigure}
        }
        \hspace{0.2em}
        \begin{subfigure}[t]{0.31\linewidth}
            \includegraphics[width=\linewidth]{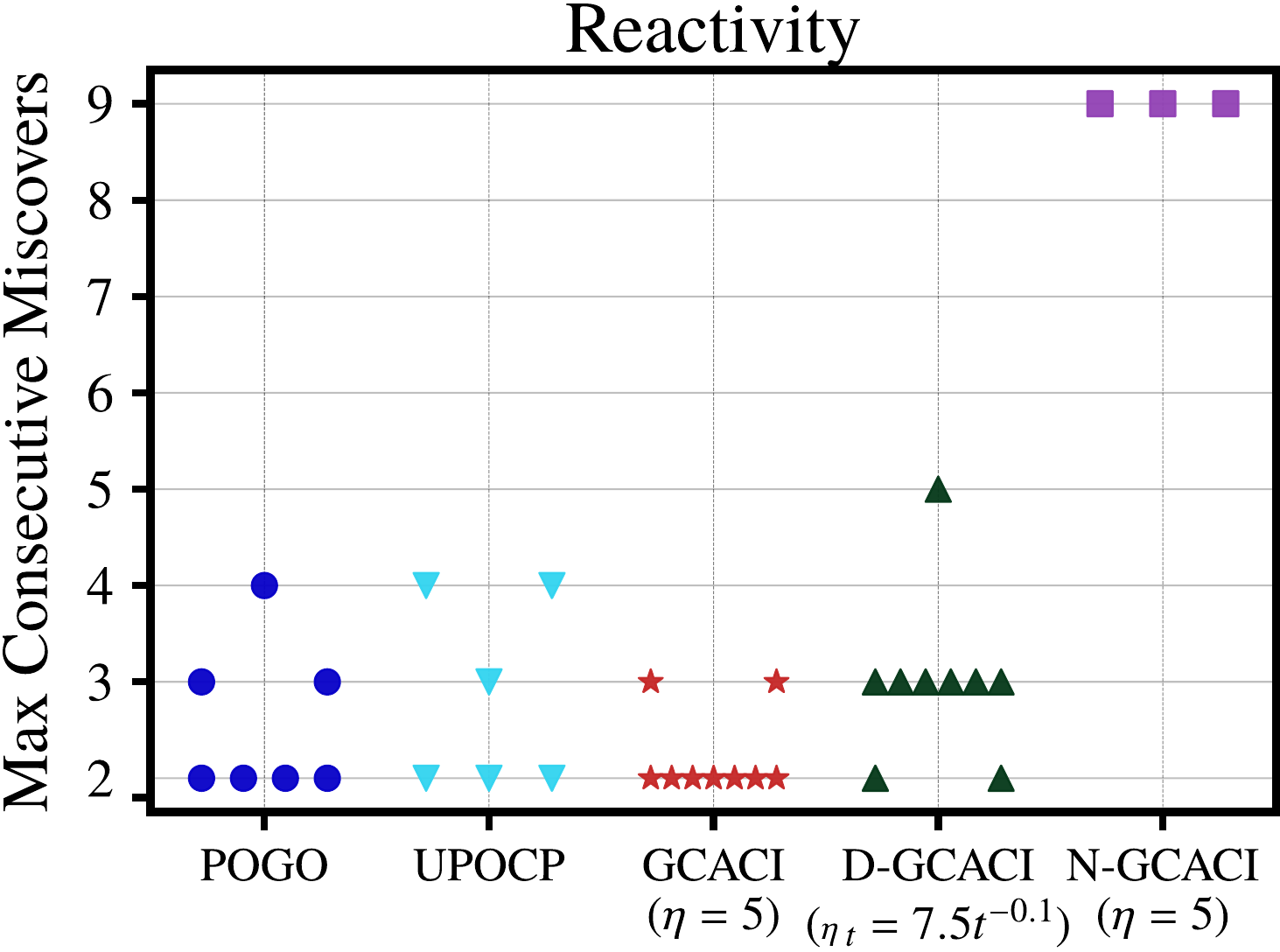}
        \end{subfigure}
        \caption{Results for Boeing Company (\texttt{BA}) stock.}
    \end{subfigure}
    \vspace{0.5em}

    \begin{subfigure}[t]{\linewidth}
        \begin{subfigure}[t]{0.31\linewidth}
        \raisebox{0.08em}{
            \includegraphics[width=\linewidth]{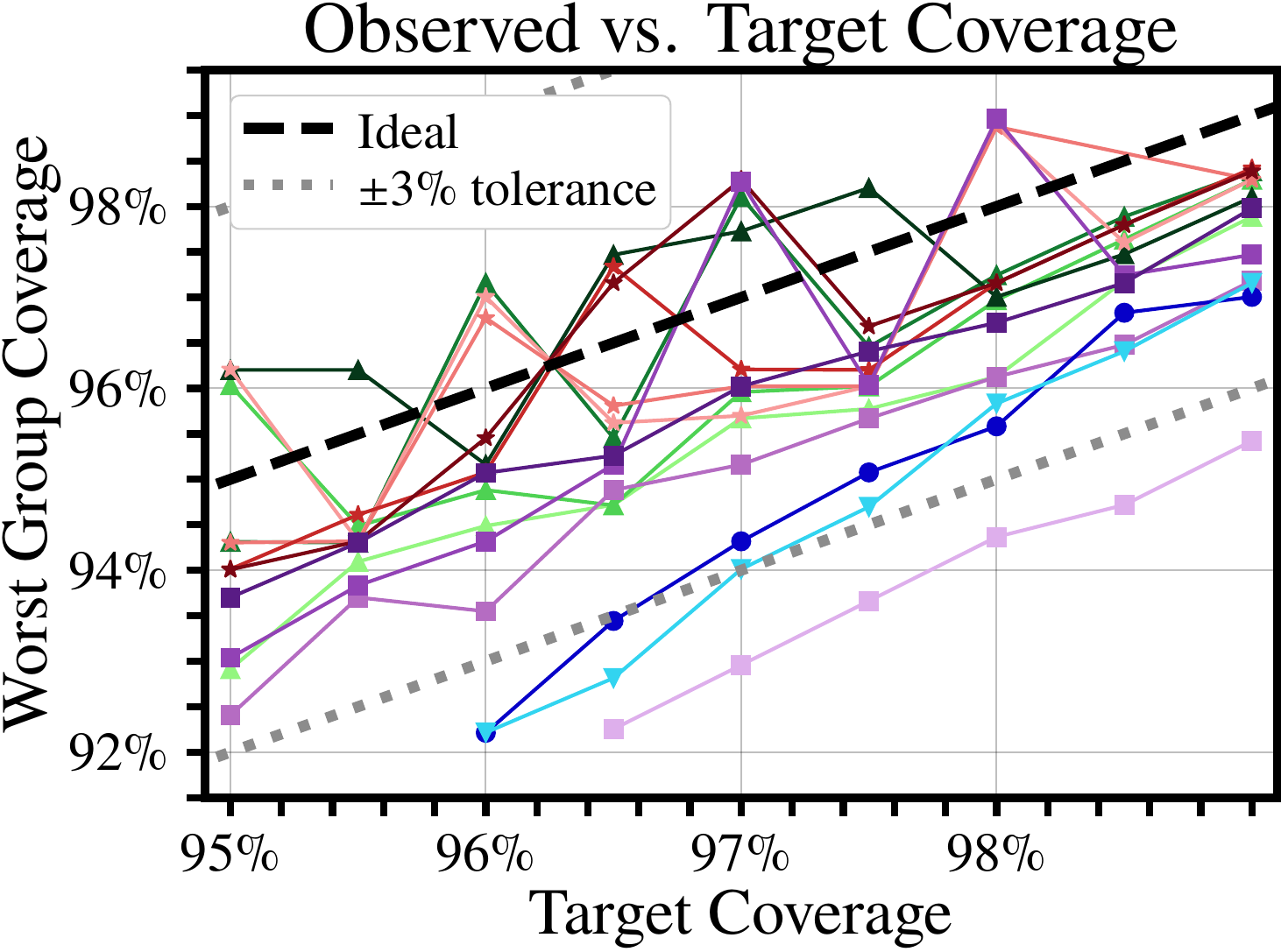}
            }
        \end{subfigure}
        \hspace{0.2em}
        \begin{subfigure}[t]{0.31\linewidth}
            \includegraphics[width=\linewidth]{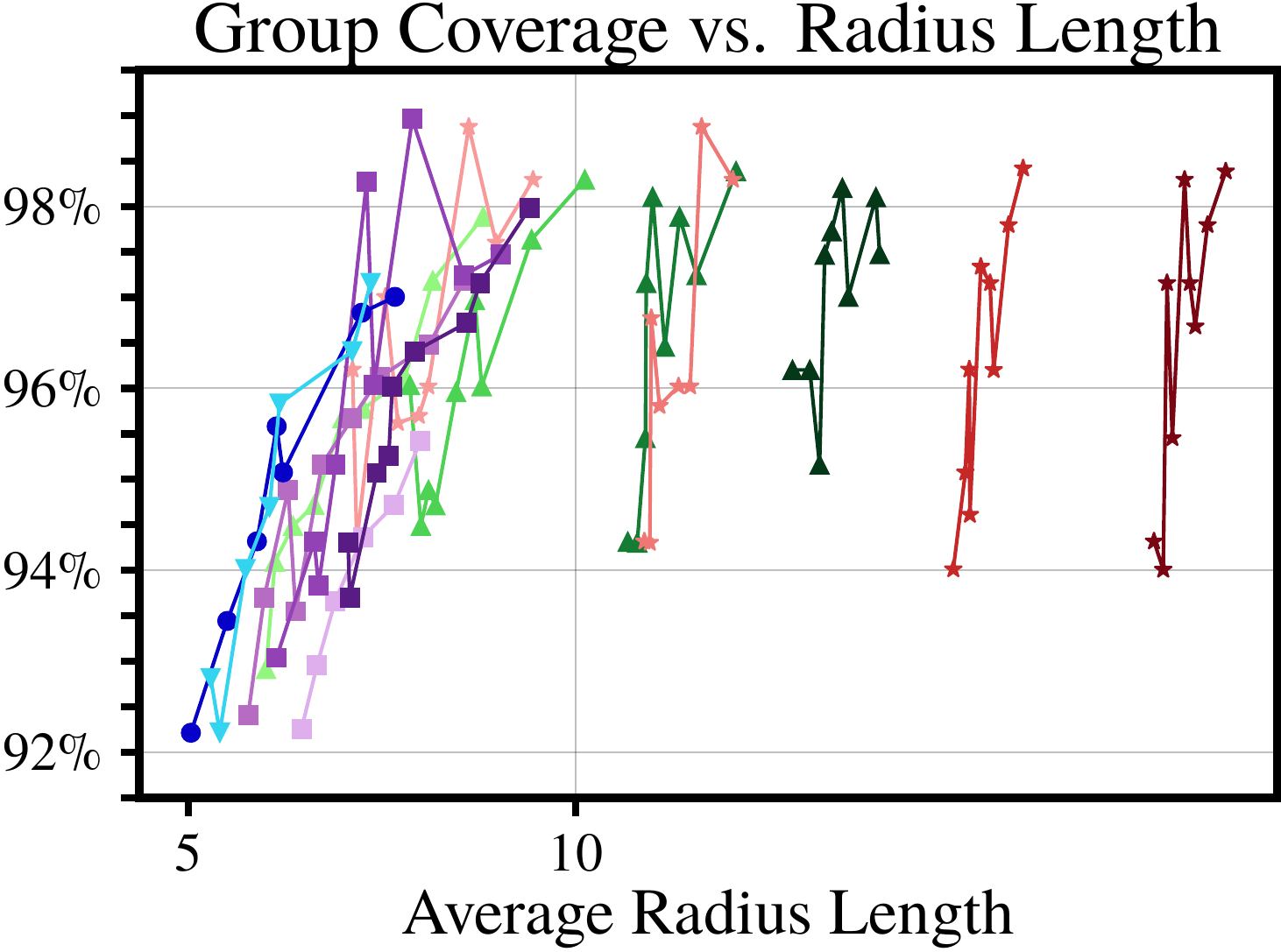}
        \end{subfigure}
        \hspace{0.2em}
        \begin{subfigure}[t]{0.31\linewidth}
            \includegraphics[width=\linewidth]{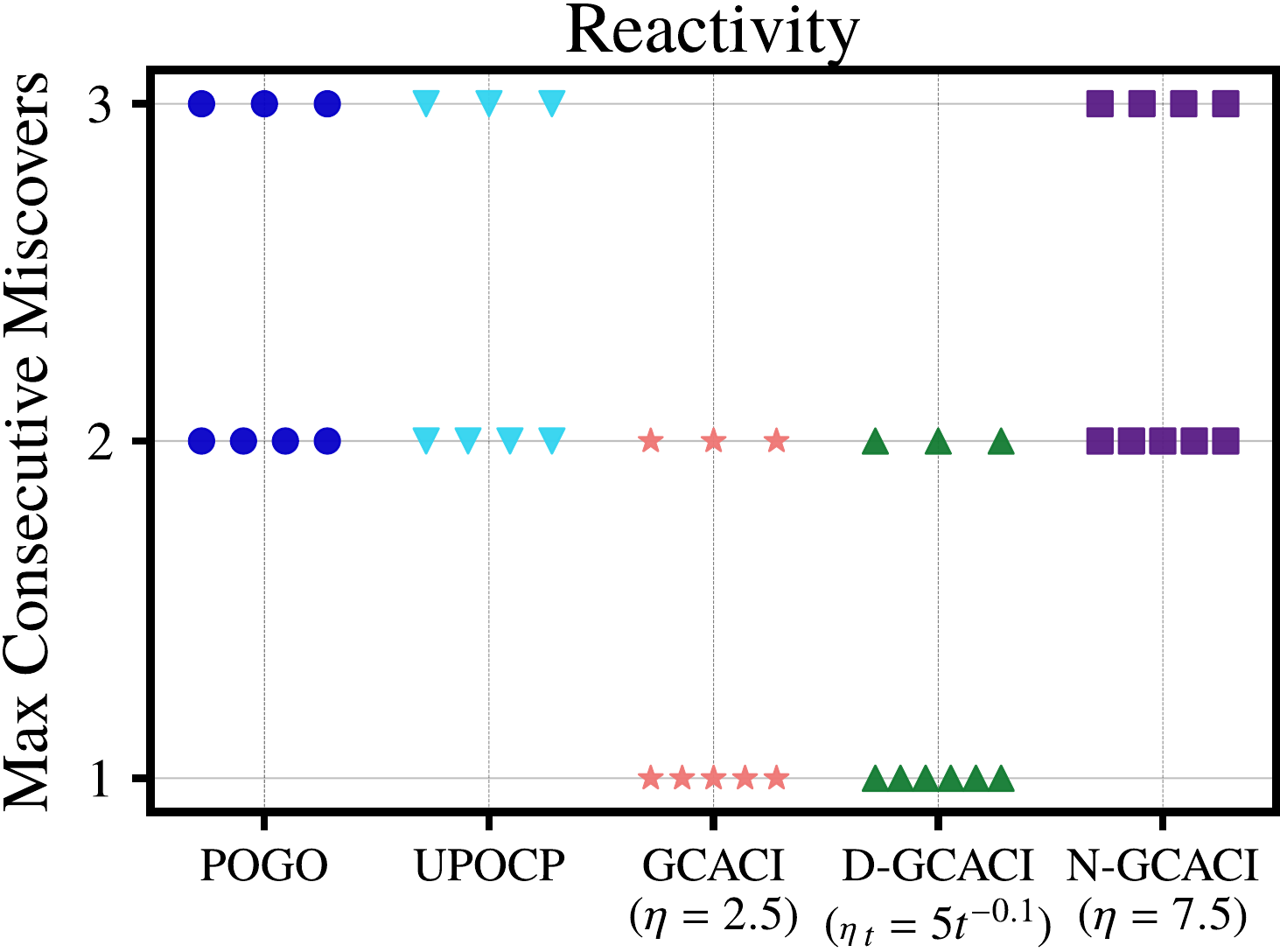}
        \end{subfigure}
        \caption{Results for Delta Air Lines Inc (\texttt{DAL}) stock.}
    \end{subfigure}
    \vspace{0.5em}

    \centering
    \begin{subfigure}{1\textwidth}
        \includegraphics[width=\linewidth]{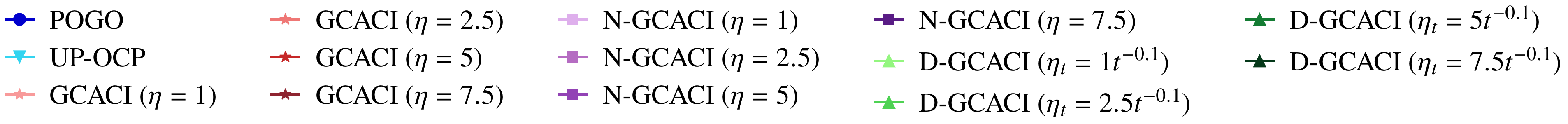}
    \end{subfigure}

    \caption{Performance across target coverage levels \(1-\alpha \in \{0.95, 0.955, \ldots, 0.99\}\). {\bf Left:} worst-group coverage vs. target coverage. {\bf Middle:} worst-group coverage vs. radius length (curves closer to the top-left indicate better performance). {\bf Right:} Max consecutive miscovers across target coverage levels. For clarity, among the GCACI, N-GCACI, and D-GCACI families, we show only the variant with the lowest median max consecutive miscovers. Data point shown if coverage \(\geq 0.92\).}
    \vspace{-1em}
\end{figure}

\begin{figure}[h!]
    \begin{subfigure}[t]{\linewidth}
        \begin{subfigure}[t]{0.31\linewidth}
        \raisebox{0.08em}{
            \includegraphics[width=\linewidth]{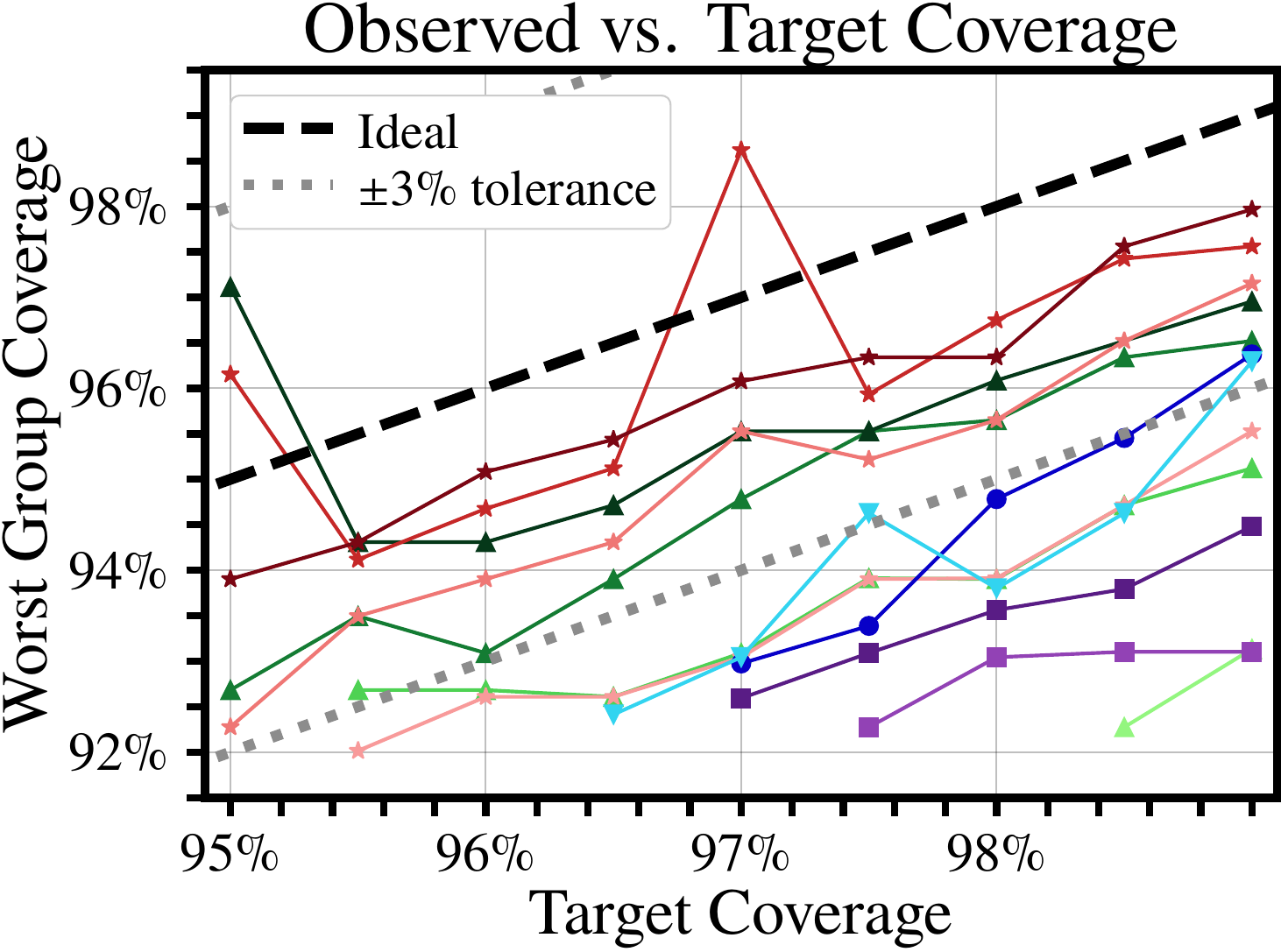}
            }
        \end{subfigure}
        \hspace{0.2em}
        \raisebox{0.5em}{
        \begin{subfigure}[t]{0.30\linewidth}
            \includegraphics[width=\linewidth]{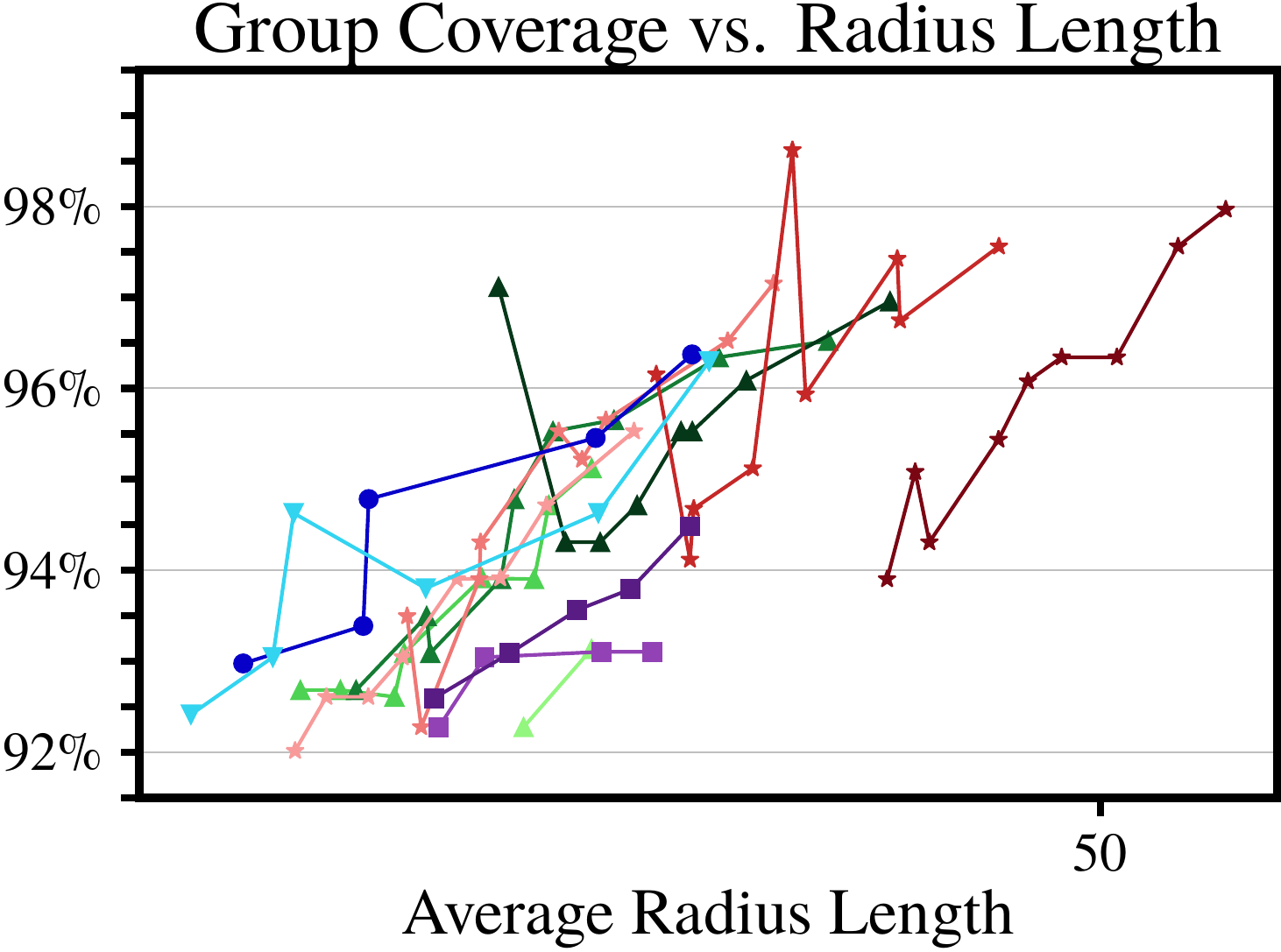}
        \end{subfigure}
        }
        \hspace{0.2em}
        \begin{subfigure}[t]{0.31\linewidth}
            \includegraphics[width=\linewidth]{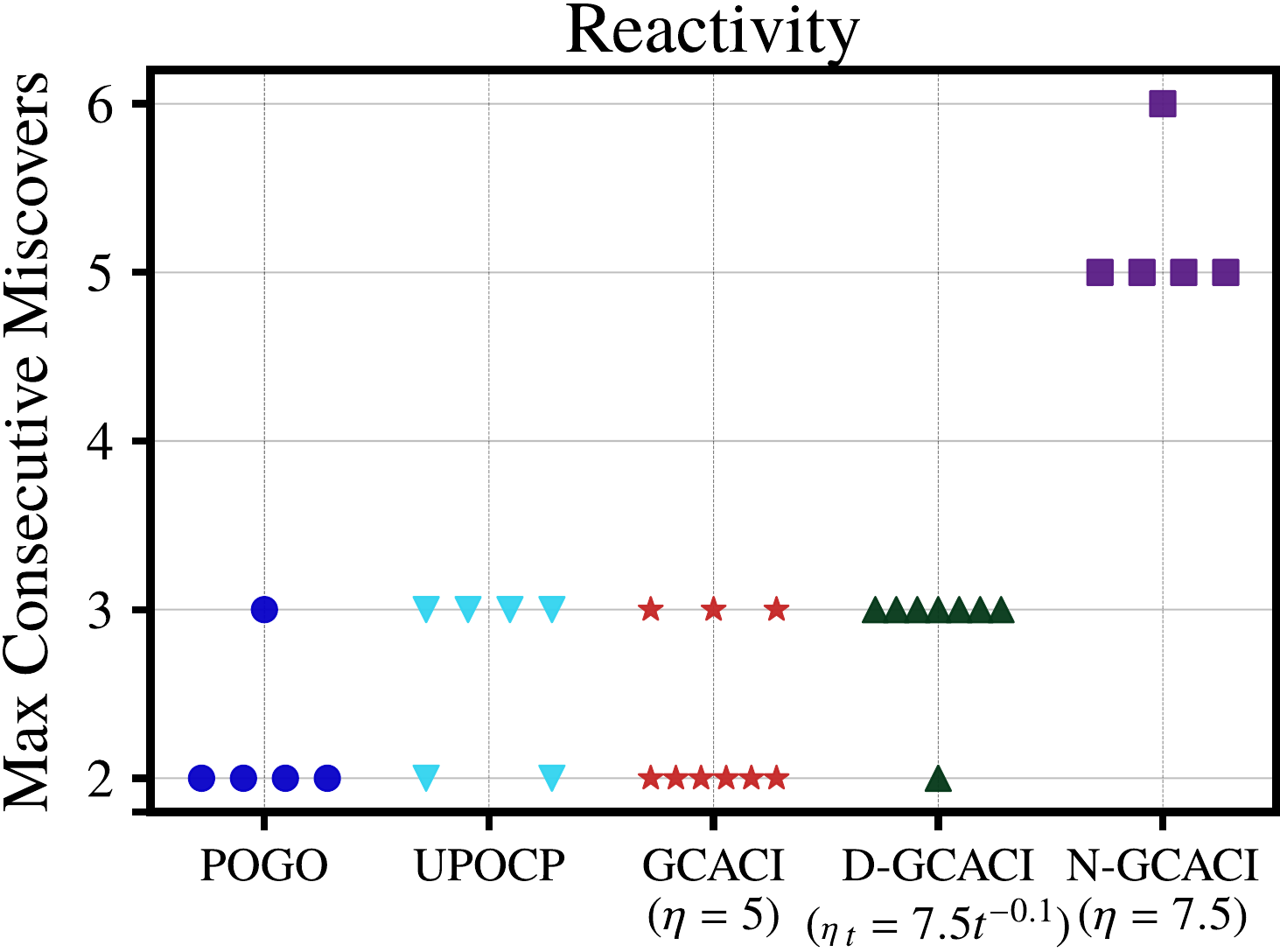}
        \end{subfigure}
        \caption{Results for Meta (\texttt{META}) stock.}
    \end{subfigure}
    \vspace{0.5em}
        
    \begin{subfigure}[t]{\linewidth}
        \begin{subfigure}[t]{0.31\linewidth}
        \raisebox{0.08em}{
            \includegraphics[width=\linewidth]{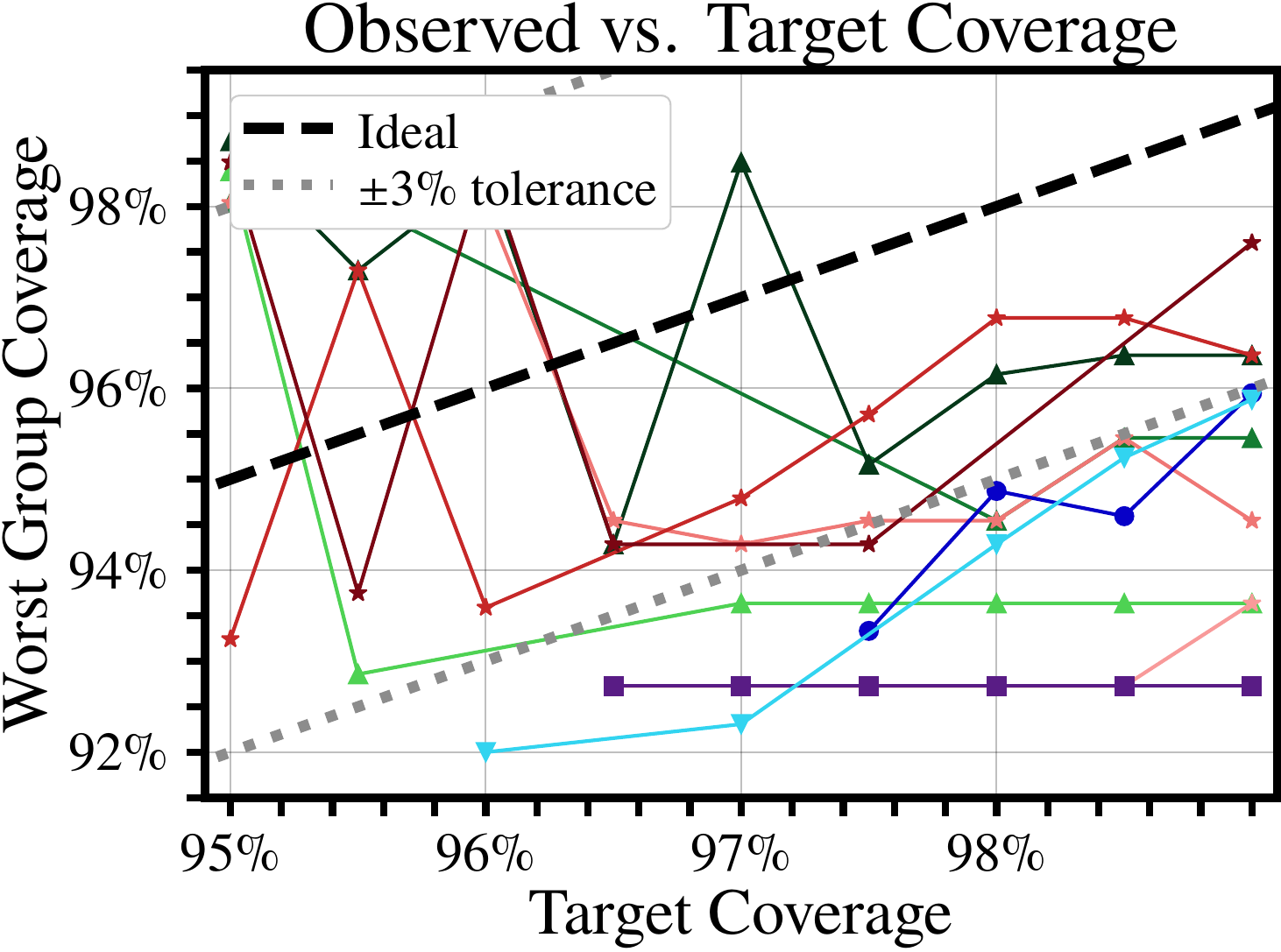}
            }
        \end{subfigure}
        \hspace{0.2em}
        \begin{subfigure}[t]{0.31\linewidth}
            \includegraphics[width=\linewidth]{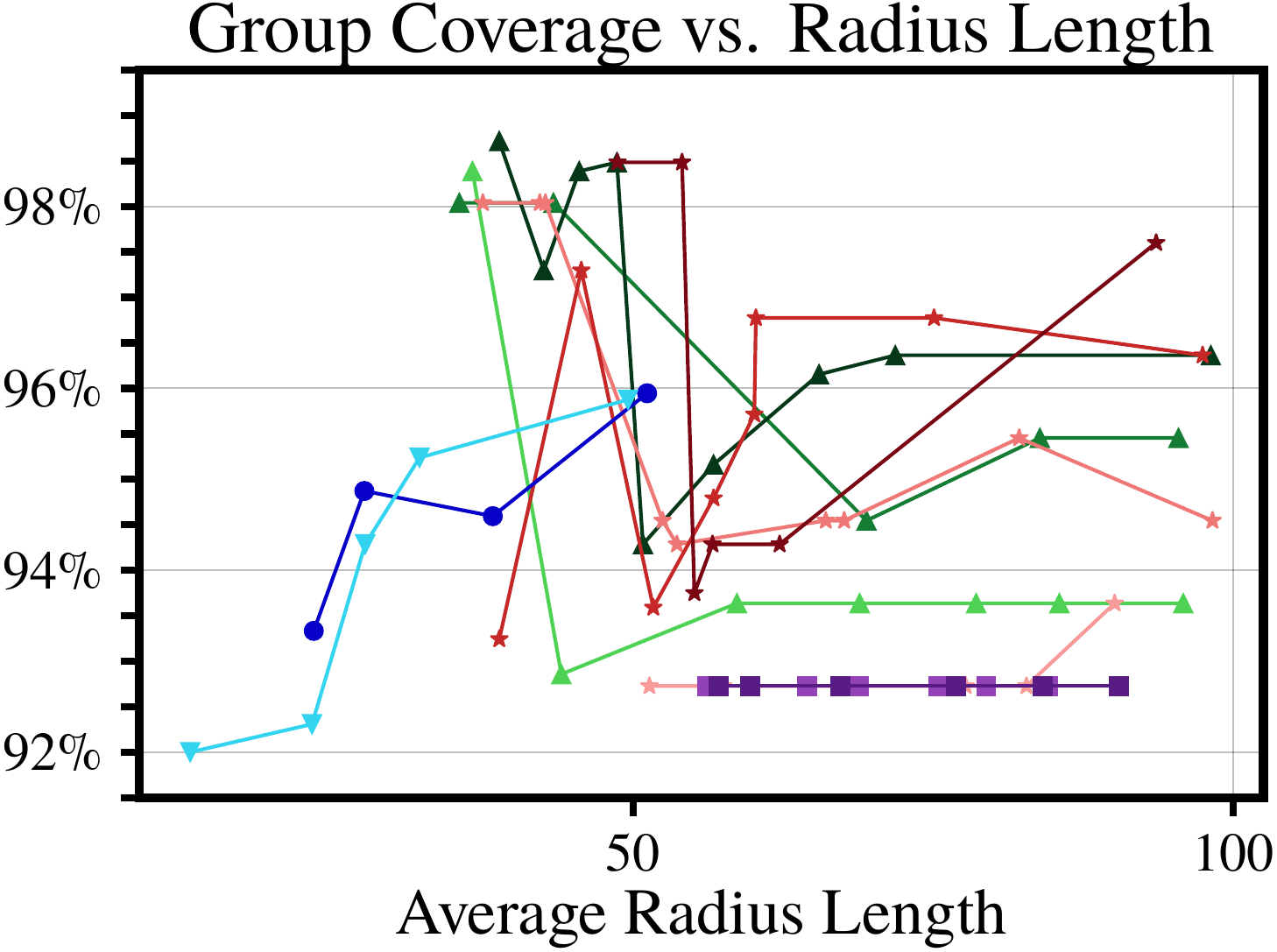}
        \end{subfigure}
        \hspace{0.2em}
        \begin{subfigure}[t]{0.31\linewidth}
            \includegraphics[width=\linewidth]{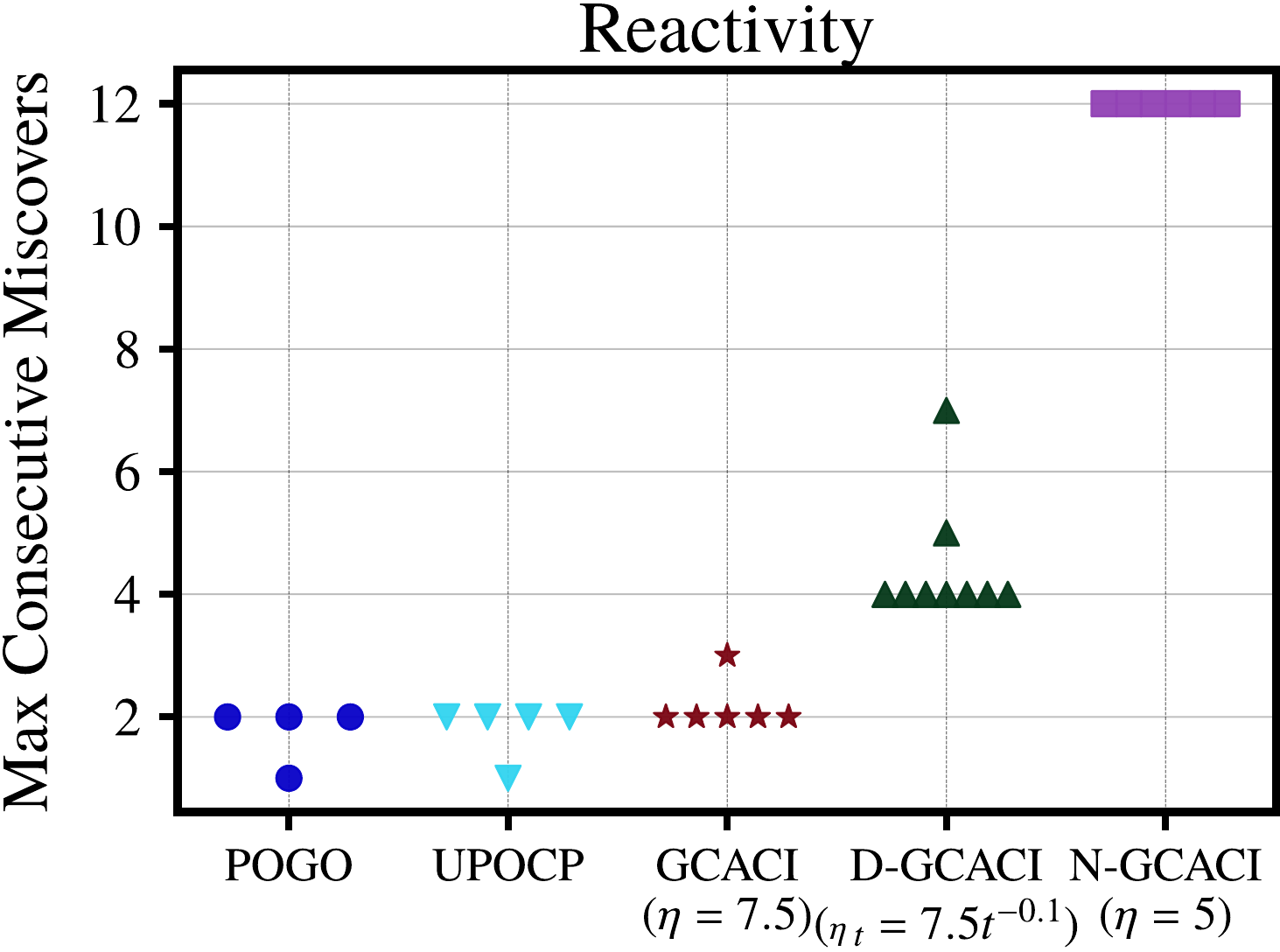}
        \end{subfigure}
        \caption{Results for Moderna Inc (\texttt{MRNA}) stock.}
    \end{subfigure}
    \vspace{0.5em}

    \begin{subfigure}[t]{\linewidth}
        \begin{subfigure}[t]{0.31\linewidth}
        \raisebox{0.08em}{
            \includegraphics[width=\linewidth]{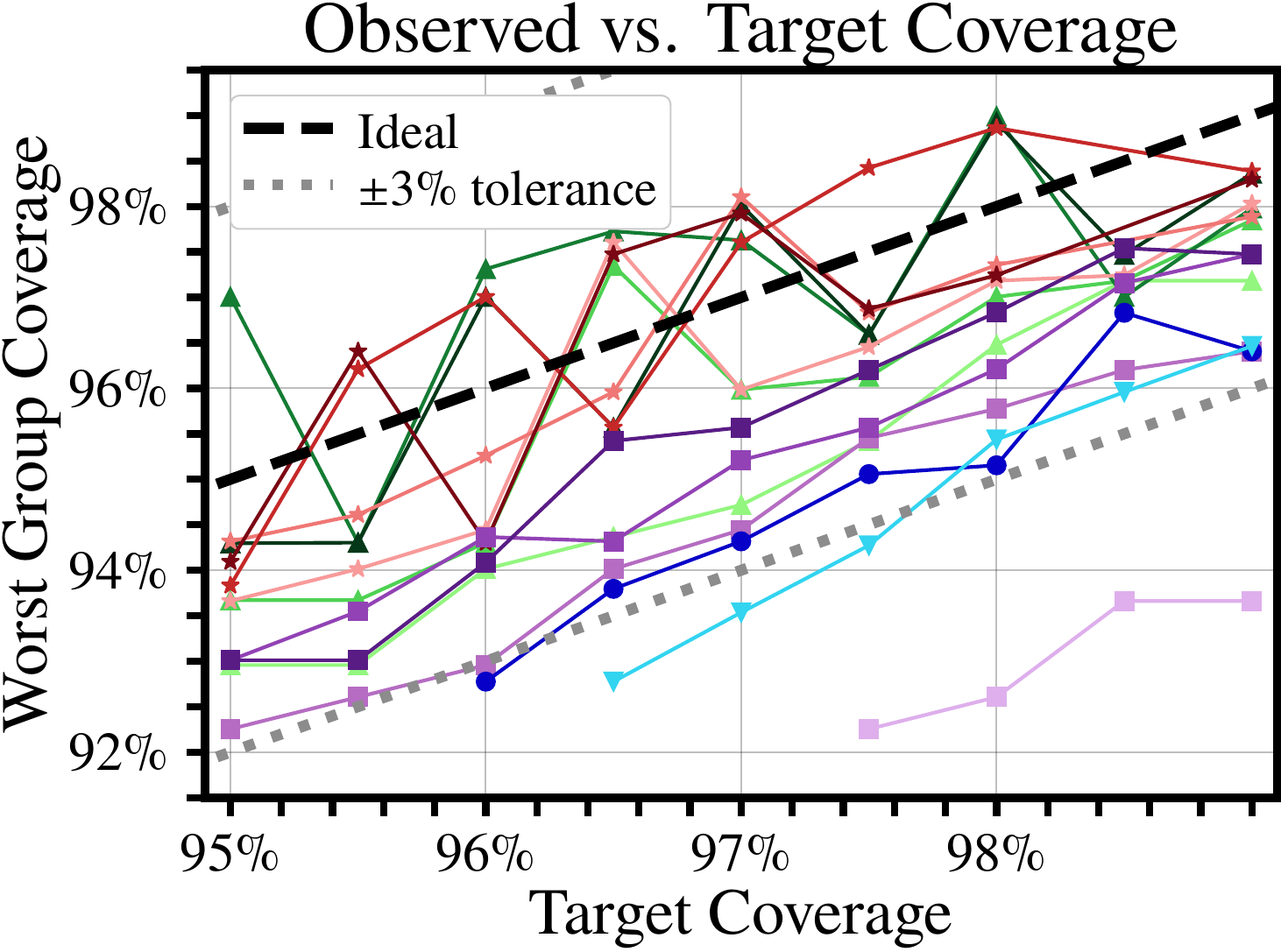}
            }
        \end{subfigure}
        \hspace{0.2em}
        \begin{subfigure}[t]{0.31\linewidth}
            \includegraphics[width=\linewidth]{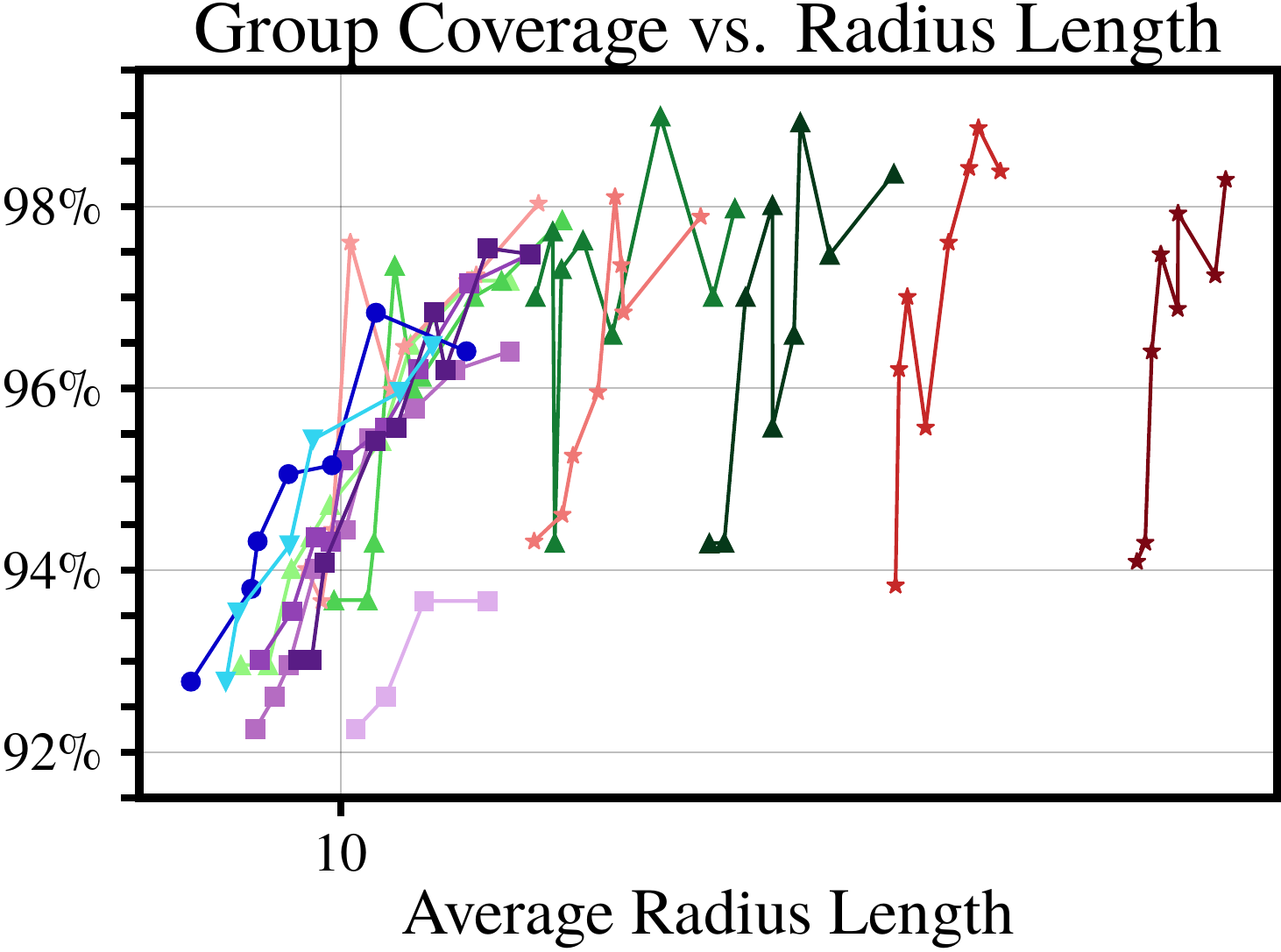}
        \end{subfigure}
        \hspace{0.2em}
        \begin{subfigure}[t]{0.31\linewidth}
            \includegraphics[width=\linewidth]{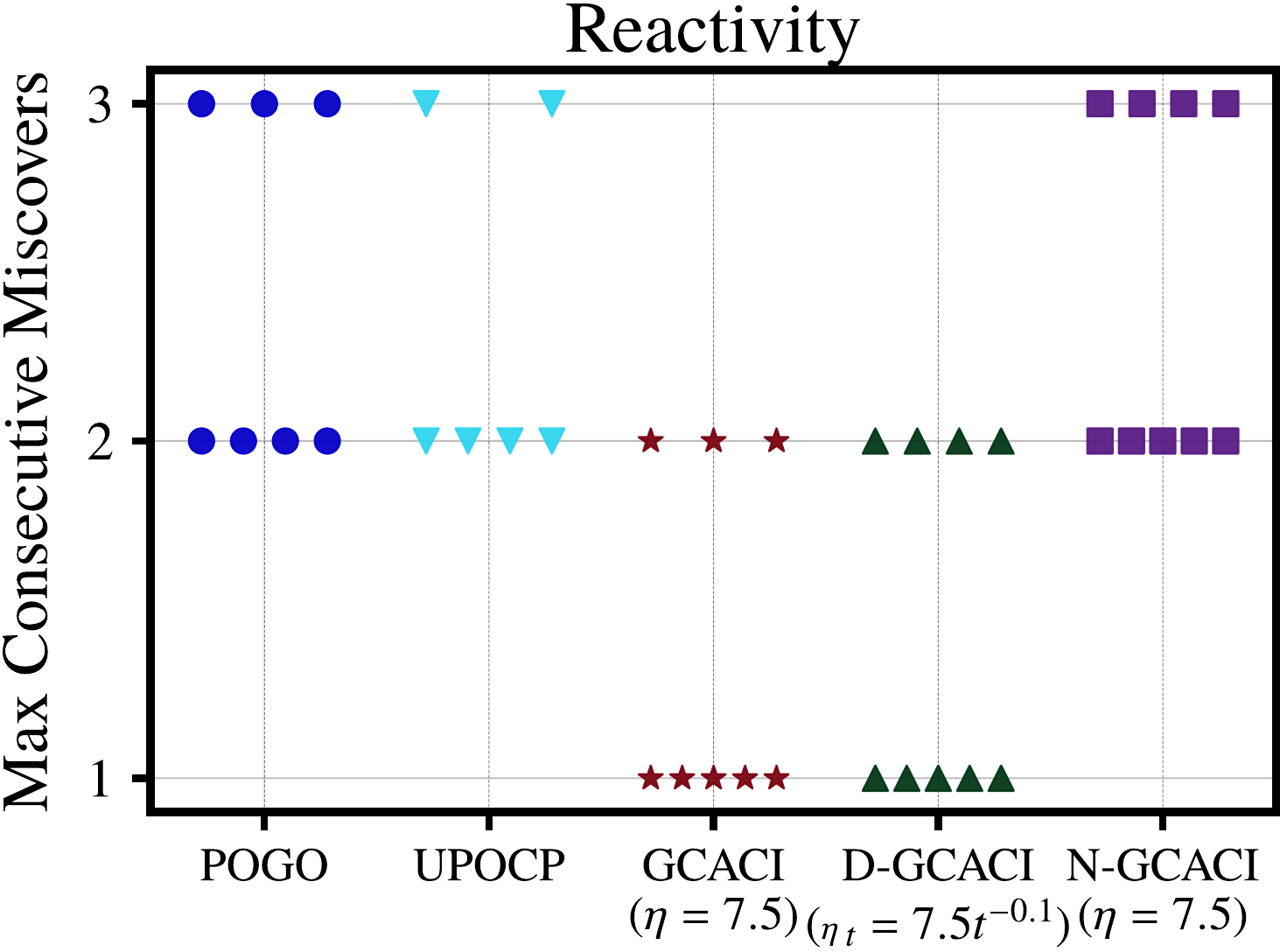}
        \end{subfigure}
        \caption{Results for United Airlines Holdings Inc (\texttt{UAL}) stock.}
    \end{subfigure}
    \vspace{0.5em}

    \centering
    \begin{subfigure}{1\textwidth}
        \includegraphics[width=\linewidth]{figures/iclr/legend_syn_2_horz.pdf}
    \end{subfigure}

    \caption{Performance across target coverage levels \(1-\alpha \in \{0.95, 0.955, \ldots, 0.99\}\). {\bf Left:} worst-group coverage vs. target coverage. {\bf Middle:} worst-group coverage vs. radius length (curves closer to the top-left indicate better performance). {\bf Right:} Max consecutive miscovers across target coverage levels. For clarity, among the GCACI, N-GCACI, and D-GCACI families, we show only the variant with the lowest median max consecutive miscovers. Data point shown if coverage \(\geq 0.92\).}
    \vspace{-1em}
\end{figure}

\newpage

\subsection{\label{supp:experiments_mimic}Additional experimental details: Length-of-Stay in the ICU.}

\paragraph{Data pre-processing.} We selected ICU stays that were at least one day long, and whose patient's date of death was unknown or after discharge ($\approx 66 \times 10^3$ total stays). As features, we used the precomputed MIMIC-IV first day concepts including labs, vital signs, the SOFA score \cite{vincent1996sofa}, and the GCS score \cite{teasdale1974assessment} (169 features total). We imputed missing values with the mean and standardized each feature using statistics from the training set.

\paragraph{Training details.} We trained a 4-layer MLP with hidden size of 128 and dropout rate of 0.3 on 50\% of the data for 500 epochs. We used the Huber loss \citep{huber1992robust} on the log-transformed ICU LOS (in days) and Adam optimizer \citep{kingma2014adam} (learning rate of $10^{-4}$, weight decay of $10^{-3}$). We used 10\% of the data for model selection, and left the remaining 40\% ($\approx 26 \times 10^{3}$ stays) for online calibration. The trained network achieves an MAE (in days) of 2.23 on the test set, which is compatible with the difficulty of the task using information collected during the first day only. We remark that the MIMIC-IV dataset does not include true admission dates to protect patient privacy. Instead, patient dates are shifted to fictional future dates. Consequently, we cannot guarantee true chronological ordering of patients, so we assign patients to the train, validation, and test splits at random. 

\paragraph{Group definitions.} We define groups in the following manner:
\begin{itemize}
    \item {\bf Race (8 groups).} Given the heterogeneity of self-reported race data in MIMIC-IV, we grouped entries as described in \cref{table:race_groups}.
    \item {\bf Insurance (5 groups).} We kept 4 original MIMIC-IV entries as groups: {\tt Medicare, Medicaid, Private, No charge}, and mapped {\tt Other} to an ``unknown'' group.
    \item {\bf Biological sex (2 groups).} We kept the original binary {\tt F, M} entries in MIMIC-IV.
\end{itemize}

\begin{table}[h]
\centering
\small
\caption{\label{table:race_groups}Mapping between race groups and raw self-reported race in the MIMIC-IV dataset.}
\vspace{1em}
\resizebox{\linewidth}{!}{%
\begin{tabular}{ll}
\toprule
Race group              & MIMIC-IV Entries\\
\midrule
Unknown                 & {\tt UNKNOWN, OTHER, UNABLE TO OBTAIN, PATIENT DECLINED TO ANSWER}\\
Hispanic                & {\tt HISPANIC OR LATINO, PUERTO RICAN, DOMINICAN, GUATEMALAN, SALVADORAN,}\\
                        & {\tt COLUMBIAN, CUBAN, MEXICAN, HONDURAN, CENTRAL AMERICAN, SOUTH AMERICAN}\\
White                   & {\tt WHITE, OTHER EUROPEAN, RUSSIAN, EASTERN EUROPEAN, BRAZILIAN, PORTUGUESE}\\
Asian                   & {\tt ASIAN, CHINESE, SOUTH EAST ASIAN, ASIAN INDIAN, KOREAN}\\
Black                   & {\tt AFRICAN AMERICAN, CAPE VERDEAN, CARIBBEAN ISLAND, AFRICAN}\\
Mult                    & {\tt MULTIPLE RACE/ETHNICITY}\\
AIAN                    & {\tt AMERICAN INDIAN/ALASKA NATIVE}\\
NHPI                    & {\tt NATIVE HAWAIIAN OR OTHER PACIFIC ISLANDER}\\
\bottomrule
\end{tabular}}
\end{table}

\begin{table}[h]
\centering
\small
\caption{\label{table:group_prevalence} Group proportions in the train ($\approx 39,000$) and test ($\approx 26,000$) splits of MIMIC-IV.}
\begin{minipage}[t]{0.3\linewidth}
\vspace{0pt}
\resizebox{\linewidth}{!}{%
\begin{tabular}{lcc}
\toprule
            & \multicolumn{2}{c}{Proportion}\\
            \cmidrule{2-3}
Race        & Train     & Test\\
\midrule
White       & 67.55\%   & 67.50\%   \\
Unknown     & 14.34\%   & 14.46\%\\
Black       & 10.73\%   & 10.68\%\\
Hispanic    & 4.01\%    & 3.87\%\\
Asian       & 2.94\%    & 3.07\%\\
AIAN        & 0.21\%    & 0.22\%\\
NHPI        & 0.13\%    & 0.11\%\\
Mult        & 0.09\%    & 0.09\%\\
\bottomrule
\end{tabular}}
\end{minipage}
\hfill
\begin{minipage}[t]{0.3\linewidth}
\vspace{0pt}
\resizebox{\linewidth}{!}{%
\begin{tabular}{lcc}
\toprule
            & \multicolumn{2}{c}{Proportion}\\
            \cmidrule{2-3}
Insurance   & Train     & Test\\
\midrule
Medicare    & 54.30\%   & 54.26\%\\
Private     & 27.08\%   & 26.79\%\\
Medicaid    & 14.95\%   & 15.09\%\\
Unknown     & 3.66\%    & 3.85\%\\
No charge   & 0.01\%    & 0.02\%\\
\bottomrule
\end{tabular}}
\end{minipage}
\hfill
\begin{minipage}[t]{0.33\linewidth}
\vspace{0pt}
\resizebox{\linewidth}{!}{%
\begin{tabular}{lcc}
\toprule
                & \multicolumn{2}{c}{Proportion}\\
                \cmidrule{2-3}
Biological sex  & Train     & Test\\
\midrule
Male            & 56.79\%   & 57.20\%\\
Female          & 43.21\%   & 42.80\%\\
\bottomrule
\end{tabular}}
\end{minipage}
\end{table}

\section{Algorithms\label{app:algorithms}}
\subsection{POGO Implementation Details for Soft Group Memberships \label{supp:soft_group_implementation}}
At iteration \(t\), POGO defines \(k\) functions
\begin{align}
    w_j(\bar{\lambda})
    =
    \prod_{i=1}^{t-1}
    \left[
        \bar{\lambda}
        \left(1-\frac{g_{i,j}}{\alpha}\right)
        +
        (1-\bar{\lambda})
        \left(1+\frac{g_{i,j}}{1-\alpha}\right)
    \right],
    \qquad j\in\{1,\ldots,k\},
\end{align}
and updates the \(k\) portfolios \(\lambda_{t,1},\ldots,\lambda_{t,k}\) via
\begin{align}
    \lambda_{t,j}
    \gets
    \frac{
        \int_0^1 \bar{\lambda}w_j(\bar{\lambda})
        \,d\mu(\bar{\lambda})
    }{
        \int_0^1 w_j(\bar{\lambda})
        \,d\mu(\bar{\lambda})
    },
    \qquad j\in\{1,\ldots,k\}.
\end{align}
When the grouping functions \(c_j\) take values in \([0,1]\), these
integrals generally do not admit closed-form solutions. They can be
approximated numerically using NumPy's \texttt{trapezoid} function on a uniform grid of \(20{,}001\) points over \((10^{-12},1-10^{-12})\), with spacing $h = \frac{1-2\times10^{-12}}{20{,}000}$
The grid density should be initialized using Jeffreys prior $p_0(\bar{\lambda})\propto [\bar{\lambda}(1-\bar{\lambda})]^{-1/2}$.
At each iteration, the trapezoidal rule can be used to normalize the updated grid density and compute its mean, yielding the next portfolio weight.

\begin{algorithm}[H]
    \caption{\label{algo:up_ocp} UP-OCP}
    \raggedright\textbf{Input:} Target miscoverage level $\alpha$.\\
    \raggedright\textbf{Initialize:} Wealth $W_{0} \gets 1$ and Jeffrey prior $\mu(\bar  \lambda) = 1/(\bar \lambda(1- \bar \lambda))^{1/2}$
    \begin{algorithmic}[1]
        \FOR{$t = 1, \dots$}
        \STATE Observe feature vector $X_t$ and compute prediction $\hat{Y}_t$
        \STATE Define $w \colon \mathbb{R} \to \mathbb{R}$ as $w(\bar \lambda) = \prod^{t-1}_{i=1} \bar \lambda \left(1 - \frac{g_{i}}{\alpha}\right) + (1-\bar \lambda)\left(1 + \frac{g_{i}}{1-\alpha}\right)$
        \STATE Update portfolio: $\lambda_{t} \gets \frac{\int^1_{0}\bar \lambda w(\bar \lambda)d\mu(\bar \lambda)}{\int^1_{0}w(\bar \lambda)d\mu(\bar \lambda)}$
        \STATE Update radius: $\tau_{t} \gets W_{t-1}\left(\frac{\lambda_t - \alpha}{\alpha(1-\alpha)}\right)$ 
        \STATE Observe label $Y_t$
        \STATE Compute non-conformity score, $S_t \gets |Y_t - \hat{Y}_t|$
        \STATE $g_t \gets [\mathbf{1}\{S_t \leq \tau_t\} - (1-\alpha)]$
        \STATE $W_{t} \gets W_{t-1}(1 - \frac{\lambda_t - \alpha}{\alpha(1-\alpha)} g_{t})$
        \ENDFOR
    \end{algorithmic}
\end{algorithm}

\end{document}